\documentclass[journal]{IEEEtran}
\newcommand{\figurescalefactor}{0.995}
\makeatletter
\def\subsubsection{\@startsection{subsubsection} {3} {\z@} {1.5ex plus 1.5ex minus 0.5ex} {0.7ex plus .5ex minus 0ex} {\normalfont\normalsize\itshape}}
\makeatother
\usepackage[absolute]{textpos}
\usepackage[cmex10]{amsmath}
\usepackage{amssymb}
\usepackage[thmmarks,amsmath]{ntheorem}
\usepackage{braket}
\usepackage{mathtools}
\usepackage{enumitem}
\setlist[enumerate]{label=(\alph*)}
\newcommand{\N}{\mathbb{N}}
\newcommand{\Z}{\mathbb{Z}}
\newcommand{\R}{\mathbb{R}}
\newcommand{\ROI}{B}
\newcommand{\conv}{\ast}
\newcommand{\discint}[2]{\{#1,\dotsc,#2\}}
\newcommand{\inint}[2]{\in\discint{#1}{#2}}
\newcommand{\nceil}[1]{\lceil #1 \rceil}
\newcommand{\bceil}[1]{\big\lceil #1 \big\rceil}
\newcommand{\transp}{^T}
\DeclareMathOperator{\vect}{vec}
\DeclareMathOperator{\id}{id}
\DeclareMathOperator{\convop}{conv}
\DeclareMathOperator{\avg}{avg}
\DeclareMathOperator{\rdim}{rdim}
\DeclareMathOperator{\cdim}{cdim}
\DeclareMathOperator{\Subsignal}{Subsignal}
\DeclareMathOperator{\Slide}{Slide}
\DeclareMathOperator{\Stride}{Stride}
\DeclareMathOperator{\Fragmentation}{Frag}
\DeclareMathOperator{\Defragmentation}{Defrag}
\DeclareMathOperator{\EvalStride}{EvalStride}
\DeclareMathOperator{\EvalSlide}{EvalSlide}
\DeclareMathOperator{\Stuffing}{Stuff}
\DeclareMathOperator{\Trimming}{Trim}
\DeclareMathOperator{\Patch}{Patch}
\renewcommand{\div}[2]{\operatorname{div}(#1,\ #2)}
\newcommand{\rem}[2]{\operatorname{rem}(#1,\ #2)}
\newcommand{\row}{\operatorname{row}}
\newcommand{\col}{\operatorname{col}}
\newcommand{\equsing}[1]{\overset{\mathclap{\text{#1}}}{=}}
\newcommand{\gequsing}[1]{\overset{\mathclap{\text{#1}}}{\geq}}
\usepackage{multibib}
\newcites{ms}{References}
\newcites{trconv}{References}
\newcites{dil}{References}
\newcites{rlx}{References}
\newcites{rlxsld}{References}
\DeclareMathOperator{\DilatedSubsignal}{DilatedSubsignal}
\DeclareMathOperator{\Dilate}{Dilate}
\DeclareMathOperator{\EvalDilate}{EvalDilate}
\DeclareMathOperator{\EvalRelax}{EvalRelax}
\DeclareMathOperator{\EvalRelaxSlide}{EvalRelaxSlide}
\DeclareMathOperator{\Downsampling}{Down}
\DeclareMathOperator{\Upsampling}{Up}
\DeclareMathOperator{\Padding}{Pad}
\DeclareMathOperator{\Cropping}{Crop}
\DeclareMathOperator{\Spreading}{Spread}
\DeclareMathOperator{\TransposedConvolution}{TransposedConvolution}
\DeclareMathOperator{\ZOH}{ZOH}
\DeclareMathOperator{\DUC}{DUC}
\newcommand{\PaddingParams}{\Padding_R^\vartheta}
\DeclareMathOperator{\SubsignalPad}{SubsignalPad}
\newcommand{\SubsignalPadParams}{\SubsignalPad_{(d,\; R)}^{\vartheta}}
\newcommand{\SubsignalPadROIParams}{\SubsignalPad_{(\ROI,\; R)}^{\vartheta}}
\DeclareMathOperator{\MultiScaleSubsignal}{MultiScaleSubsignal}
\newcommand{\MultiScaleSubsignalParams}{\MultiScaleSubsignal_{(d,\; R,\; k)}^{(\vartheta,\; H)}}
\newcommand{\MultiScaleSubsignalROIParams}{\MultiScaleSubsignal_{(\ROI,\; R,\; k)}^{(\vartheta,\; H)}}
\DeclareMathOperator{\MultiScaleIndex}{MultiScaleIndex}
\DeclareMathOperator{\Original}{Org}
\newcommand{\Dirichlet}{1\text{st}}
\newcommand{\Neumann}{2\text{nd}}
\usepackage{lscape}
\makeatletter
\renewcommand*{\LS@rot}{\setbox\@outputbox\vbox{\hbox{\rotatebox{-90}{\box\@outputbox}}}}
\makeatother

\begin{document}
\begin{textblock*}{170mm}(17.5mm,9.3375mm)
\parindent0mm
\normalfont\scriptsize
This is an extended preprint of:
M. Thom and F. Gritschneder,
"Rapid Exact Signal Scanning with Deep Convolutional Neural Networks,"
IEEE Transactions on Signal Processing,
vol.~65, no.~5, pp.~1235--1250, 2017.
Digital Object Identifier 10.1109/TSP.2016.2631454.

Pages 1--16 only:
Copyright $\copyright$ 2016 IEEE.
Personal use of this material is permitted.
Permission from IEEE must be obtained for all other uses, in any current or future media, including reprinting/republishing this material for advertising or promotional purposes, creating new collective works, for resale or redistribution to servers or lists, or reuse of any copyrighted component of this work in other works.
\end{textblock*}

\title{\vspace{9mm}Rapid Exact Signal Scanning with\\Deep Convolutional Neural Networks}%
\author{Markus Thom and Franz Gritschneder%
\thanks{M. Thom was with driveU / Institute of Measurement, Control and Microtechnology, Ulm University, 89081 Ulm, Germany.
He is now with Daimler AG, 89081 Ulm, Germany (e-mail: markus.thom@daimler.com).}
\thanks{F. Gritschneder is with driveU / Institute of Measurement, Control and Microtechnology, Ulm University, 89081 Ulm, Germany (e-mail: franz.gritschneder@uni-ulm.de).}
\thanks{
This work was supported by Daimler AG, Germany.}
}
\maketitle

\begin{abstract}
A rigorous formulation of the dynamics of a signal processing scheme aimed at dense signal scanning without any loss in accuracy is introduced and analyzed.
Related methods proposed in the recent past lack a satisfactory analysis of whether they actually fulfill any exactness constraints.
This is improved through an exact characterization of the requirements for a sound sliding window approach.
The tools developed in this paper are especially beneficial if Convolutional Neural Networks are employed, but can also be used as a more general framework to validate related approaches to signal scanning.
The proposed theory helps to eliminate redundant computations and renders special case treatment unnecessary, resulting in a dramatic boost in efficiency particularly on massively parallel processors.
This is demonstrated both theoretically in a computational complexity analysis and empirically on modern parallel processors.
\end{abstract}


\section{Introduction}
\IEEEPARstart{E}{ven} though today's signal processing systems have achieved an unprecedented complexity, a multitude of them have a very basic commonality:
The application of a translation-invariant function to a large signal in a sliding fashion facilitates the dense computation of interesting output values for each possible spatial location.
Consider filter-based signal denoising as an example:
Here, each entry of the denoised output signal always depends on a fixed computation rule applied to only a limited number of samples within the input signal, or in other words, on a \emph{subsignal} of the input signal.
The computation rule is completely agnostic with regard to the actual position, it is merely important that the input samples are drawn accordingly from the input signal.

Of course, modern systems apply more sophisticated techniques than mere filtering.
However, recently an architecture essentially made up of simple filtering building blocks has displayed advantages over any other approach in a wide variety of practical applications.
Due to significant advances in the design of massively parallel processors and the availability of huge annotated data sets, deep artificial neural networks which learn desired behavior by adapting their degrees of freedom to concrete sample data rather than being programmed explicitly have become the de facto state-of-the-art in the domains of signal restoration and signal classification.

The most important architecture for analyzing signals that possess a spatial structure, such as images where pixels are arranged on a two-dimensional grid, was inspired by findings on the dynamics of mammalian visual cortex~\cite{Hubel1962}:
\emph{Convolutional Neural Networks (CNNs)}~\cite{Fukushima1980,LeCun1990a,LeCun1998} respect the weight-sharing principle, hence convolution with trainable filters becomes the actual workhorse for data processing.
This principle greatly reduces the network's degrees of freedom, making it less susceptible to overfitting, and it incorporates a strong prior with respect to the spatial layout of the input data.
In fact, this particular architecture has proven highly successful both for image restoration tasks~\cite{Jain2009,Xu2015,Dong2016} and pattern recognition problems~\cite{Ciresan2012a,Krizhevsky2013,Szegedy2015}.

If a CNN trained for object categorization is evaluated at each feasible image position, it is possible to assign class membership estimations to all the pixels in an image, yielding a semantic segmentation of a scene~\cite{Ning2005,Grangier2009,Farabet2013}.
This representation is much more powerful than what can be gained from a strict conventional object detection approach which solely outputs bounding boxes of found object instances.
Instead, it facilitates applications such as automated biological or medical image analysis~\cite{Ning2005,Giusti2013,Thong2016} and dense vehicle environment perception~\cite{Nuss2014}.
While the computational complexity of a sophisticated classification system used in conjunction with a sliding window approach may seem excessive at first glance, the weight-sharing principle of a CNN can be exploited so that intermediate computation results can be shared among adjacent image patches, resulting in a speedup of several orders of magnitude.
Although this was already realized for CNNs without pooling layers more than two decades ago~\cite{Vaillant1993}, approaches that also account for pooling layers emerged only recently~\cite{Giusti2013,Li2014,Sermanet2014}.

The approach of Giusti \emph{et al.}~\cite{Giusti2013} achieves fast scanning of entire images through the introduction of a fragmentation data structure.
Here, the internal representations of a CNN are decomposed using a spatial reordering operation after each pooling layer, allowing the evaluation of convolutions on contiguous signals at all times.
The intermediate signals are however inhomogeneous with respect to their dimensionality, leaving the possibility for the use of efficient tensor convolution routines unclear.
Li \emph{et al.}~\cite{Li2014}, on the other hand, propose enlarging the filter banks of convolutional layers by inserting vanishing entries at regular locations.
These sparse filter banks require a cumbersome re-engineering of efficient convolution implementations, which may not be able to achieve maximum throughput on modern massively parallel processors.
Sermanet \emph{et al.}~\cite{Sermanet2014} use the same processing pipeline for patches and entire images, which incurs relaxations with accuracy loss effects where the actual impact is hard to predict.

All these approaches have in common that it is not inherently clear what they actually compute or if the result is even the desired one.
Instead of a rigorous mathematical proof of correctness, only toy examples are available, illustrating the implementation of these approaches.
This situation is especially unsatisfactory if, instead of pure convenience functions, systems subject to safety considerations should be realized where precise statements rather than only an empirical evaluation are required.

The key contributions of this paper are
(i) the development of an original theory on \emph{subsignal compatible transformations} as exact characterization of functions that fulfill the invariants required for a sound sliding window approach,
(ii) the proposition of a method for dense signal scanning \emph{provably without any accuracy loss} that yields significant speedups due to homogeneous data structures and elimination of redundant computations and special case treatment,
and (iii) the demonstration how CNNs interconnect with the theory and how they can be \emph{exactly} transformed from subsignal-based application to signal-based application \emph{without} any necessary adjustments to the computationally most demanding tensor convolution.
To the authors' best knowledge, they are the first to actually have mathematically rigorous statements to support their claims on the correctness of dense signal scanning with CNNs.
Due to the generality of the results, the herein developed theoretical framework can also serve as a basis for analyzing related and emerging methods for signal processing based on translation-invariant functions applied in a sliding fashion.

The remainder of this paper is structured as follows.
Section~\ref{sect:prerequisites} presents an introduction to the CNN structure, fixes the notation and introduces what is meant by subsignals.
Section~\ref{sect:subsignal-compatible-transformations} establishes the basics of the theory on subsignal compatible transformations and shows how the building blocks of CNNs fit into the theory.
In the following Sect.~\ref{sect:strided_functions}, the theory is extended to functions applied in a strided fashion, which is particularly important for pooling operators evaluated on non-overlapping blocks.
Section~\ref{sect:computational_complexity} provides a theoretical computational complexity analysis.
Practical considerations for image processing and the results of experiments on real parallel processors are discussed in Sect.~\ref{sect:experimental_evaluation}.
The paper is concluded with a discussion of the results in Sect.~\ref{sect:conclusions}.

\section{Prerequisites}
\label{sect:prerequisites}
This section begins with an introduction to the building blocks of a CNN.
Then the notation used throughout the paper is established.
The section concludes with the definition of the subsignal extraction operator and statements on its properties.

\subsection{Convolutional Neural Networks}
CNNs are organized in a number of specialized layers~\cite{LeCun1998}.
Each layer receives input data from its predecessor, processes it, and sends the result to the next layer.
The network's output is then the output of the final layer.
The training process consists of tuning the network's degrees of freedom until the network produces the desired output given concrete input sample data~\cite{Bishop1995}.
After a network has been trained, it can be used as a predictor on previously unseen data in regression or classification tasks.

The different specialized layer types are given as follows.
\emph{Convolutional layers} respect the weight-sharing principle: They convolve their input with a trainable filter bank and add a trainable scalar bias to form the layer output.
These layers fall into the class of subsignal compatible transformations detailed in Sect.~\ref{sect:subsignal-compatible-transformations}, a mathematical analysis of the involved computations is given in Sect.~\ref{sect:CNNs-wo-pooling}.

\emph{Fully-connected layers} are a special case of convolutional layers in that they carry out a convolution with unit spatial filter size.
Mathematical treatment of these layers is hence superseded by the analysis of convolutional layers.

\emph{Non-linearity layers} independently pass each sample of a signal through a scalar transfer function.
This prevents the entire network from forming a purely linear system and hence enhances the network's representational capacity.
Since these operations are agnostic with respect to any spatial structure, an analysis is straightforward and handled in Sect.~\ref{sect:CNNs-wo-pooling}.

Eventually, \emph{pooling layers} strengthen a network's invariance to small translations of the input data by evaluation of a fixed pooling kernel followed by a downsampling operation.
For brevity of the presentation, only functions applied to non-overlapping blocks are considered here.
Pooling requires an extension of the plain theory of subsignal compatible transformations, provided in Sect.~\ref{sect:strided_functions}.

This paper proves that CNNs can be transformed from subsignal-based application to signal-based application by transforming strided function evaluation into sliding function evaluation and inserting special helper layers, namely fragmentation, defragmentation, stuffing and trimming.
This transformation is completely lossless, both subsignal-based application and signal-based application lead to the same results.
Even after the transformation, CNNs can be further fine-tuned with standard optimization methods.
An example for this process is given in Sect.~\ref{sect:experimental_evaluation}.

\subsection{Notation}
For the sake of simplicity, the mathematical analysis is restricted to vector-shaped signals.
The generalization to more complex signals such as images is straightforward through application of the theory to the two independent spatial dimensions of images.
This is briefly discussed in Sect.~\ref{sect:experimental_evaluation}.

$\N_1 := \N\setminus\set{0}$ represents the positive natural numbers.
If $M$ is a set and $q\in\N_1$, then $M^q$ denotes the set of all $q$-tuples with entries from $M$.
The elements of $M^q$ are called \emph{signals}, their $q$ entries are called \emph{samples}.
If $\xi = (\xi_1,\dotsc,\xi_q)\in M^q$ is a signal and $I\in\discint{1}{q}^r$ is an index list with $r$ entries, the \emph{formal sum} $\omega := \sum_{\nu = 1}^r \xi_{I_\nu}\cdot e_\nu^r$ is used for the element $\omega\in M^r$ with $\omega_\nu = \xi_{I_\nu}$ for all $\nu\inint{1}{r}$.
For example, when $M$ equals the set of real numbers $\R$ and hence $M^r$ is the $r$-dimensional Euclidean space, then the formal sum $\omega$ corresponds to the linear combination of canonical basis vectors $e_\nu^r$ weighted with selected coordinates of the signal $\xi$.

For $\xi\in M^q$, $\dim_M(\xi) = q$ represents the \emph{dimensionality} of $\xi$.
This does not need to correspond exactly with the concept of dimensionality in the sense of linear algebra.
If for example $M = \N^c$ for categorical data with $c\in\N_1$ features, then $M^q$ is not a vector space over $M$.
The theory presented in this paper requires algebraic structures such as vector spaces or analytic structures such as the real numbers only for certain examples.
The bulk of the results hold for signals with samples from arbitrary sets.

If $M$ is a set and $c\in\N_1$ is a positive natural number, then $\cup_c(M) := \cup_{q = c}^\infty M^q$ is written for the set that contains all the signals of dimensionality greater than or equal to $c$ with samples from $M$.
For example, if $\xi\in\cup_c(M)$ then there is a natural number $q\geq c$ so that $\xi = (\xi_1,\dotsc,\xi_q)$ with $\xi_\nu\in M$ for all $\nu\inint{1}{q}$.
Note that $\cup_1(M)$ contains all non-empty signals with samples from $M$.

\subsection{Division of a Signal into Subsignals}
A \emph{subsignal} is a contiguous list of samples contained in a larger signal.
First, the concept of extracting subsignals with a fixed number of samples from a given signal is formalized:
\begin{definition}
\label{def:subsignal}
Let $M$ be an arbitrary set and let $d\in\N_1$ denote a fixed subsignal dimensionality.
Then the function $\Subsignal_d\colon\bigcup_{D = d}^\infty\big(M^D\times\discint{1}{D - d + 1}\big)\to M^d$,
\begin{displaymath}
  (\xi,\; i)\mapsto\sum\nolimits_{\nu = 1}^d \xi_{i + \nu - 1}\cdot e_\nu^d\text{,}
\end{displaymath}
is called the \emph{subsignal extraction operator}.
Here, $\xi$ is the input signal and $i$ denotes the \emph{subsignal index}.
\end{definition}

\begin{figure}[t]
  \centering
  \scalebox{\figurescalefactor}{\includegraphics[page=1]{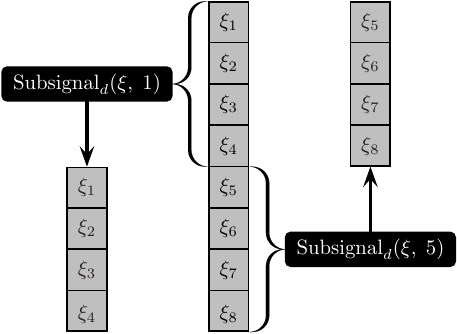}}
  \caption{Illustration of the subsignal extraction operator applied to a signal $\xi$ with $D = 8$ samples for extraction of subsignals with $d = 4$ samples.
    The left-hand side shows the first subsignal and the right-hand side shows the final subsignal with the maximum subsignal index of $D - d + 1 = 5$.}
  \label{fig:subsignal}
\end{figure}

It is straightforward to verify that $\Subsignal_d$ is well-defined and actually returns all possible $D - d + 1$ contiguous subsignals of length $d$ from a given signal with $D$ samples (see Fig.~\ref{fig:subsignal}).
Note that for application of this operator, it must always be ensured that the requested subsignal index $i$ is within bounds, that is $i\inint{1}{D - d + 1}$ must hold to address a valid subsignal.

Iterated extraction of subsignals of different length can be collapsed into one operator evaluation:
\begin{lemma}
\label{lem:subsignal-composition}
Let $M$ be a set and further let $c,d\in\N_1$, $c\leq d$, be two subsignal dimensionalities.
Then for all $\xi\in\cup_c(M)$, $i\inint{1}{\dim_M(\xi) - d + 1}$ and $j\inint{1}{d - c + 1}$ it is $\Subsignal_c(\Subsignal_d(\xi,\; i),\; j) = \Subsignal_c(\xi,\; i + j - 1)$.
\end{lemma}
\begin{proof}
The subsignal indices of the left-hand side are well within bounds.
Since $i + j - 1 \inint{1}{\dim_M(\xi) - c + 1}$ this also holds for the right-hand side.
Now
\begin{align*}
     & \Subsignal_c(\Subsignal_d(\xi,\; i),\; j)\\
  \equsing{D.~\ref{def:subsignal}}\ \ & \sum\nolimits_{\lambda = 1}^c\Subsignal_d(\xi,\;i)_{j + \lambda - 1}\cdot e_\lambda^c\\
  \equsing{D.~\ref{def:subsignal}}\ \ & \sum\nolimits_{\lambda = 1}^c\left(\sum\nolimits_{\nu = 1}^d \xi_{i + \nu - 1}\cdot e_\nu^d\right)_{j + \lambda - 1}\cdot e_\lambda^c\\
  \equsing{($\lozenge$)}\ \ & \sum\nolimits_{\lambda = 1}^c \xi_{(i + j - 1) + \lambda - 1}\cdot e_\lambda^c\\
  \equsing{D.~\ref{def:subsignal}}\ \ & \Subsignal_c(\xi,\; i + j - 1)\text{,}
\end{align*}
where in the ($\lozenge$) step $\nu = j + \lambda - 1$ has been substituted.
\end{proof}

\section{Subsignal Compatible Transformations}
\label{sect:subsignal-compatible-transformations}
This section introduces the concept of subsignal compatible transformations.
These are functions that can be applied to an entire signal at once and then yield the same result as if they had been applied to each subsignal independently.
It is shown that functions applied in a sliding fashion can be characterized as subsignal compatible transformations, and that the composition of subsignal compatible transformations is again a subsignal compatible transformation.

At the end of this section, CNNs without pooling layers are considered and it is demonstrated that these satisfy the requirements of subsignal compatible transformations.
As a consequence, such networks can be applied to the whole input signal at once without having to handle individual subsignals.
CNNs that \emph{do} contain pooling layers require more theoretical preparations and are discussed verbosely in Sect.~\ref{sect:strided_functions}.

Now the primary definition of this section:
\begin{definition}
\label{def:subsignal_compatible}
Let $M$ and $N$ be sets, let $c\in\N_1$ be a positive natural number, and let $T\colon \cup_c(M) \to \cup_1(N)$ be a function.
$T$ is then called a \emph{subsignal compatible transformation with dimensionality reduction constant $c$} if and only if these two properties hold:
\begin{enumerate}[label=(\roman*)]
  \item \emph{Dimensionality reduction property (DRP):}\\ $\dim_N(T(\xi)) = \dim_M(\xi) - c + 1$ for all $\xi\in\cup_c(M)$.
  \item \emph{Exchange property (XP):}\\ For all subsignal dimensionalities $d\in\N_1$, $d\geq c$, it holds that $T(\Subsignal_d(\xi,\;i)) = \Subsignal_{d - c + 1}(T(\xi),\;i)$ for all $\xi\in\cup_d(M)$ and all $i\inint{1}{\dim_M(\xi) - d + 1}$.
\end{enumerate}
\end{definition}

The first property guarantees that $T$ reduces the dimensionality of its argument always by the same amount regardless of the concrete input.
The second property states that if $T$ is applied to an individual subsignal, then this is the same as applying $T$ to the entire signal and afterwards extracting the appropriate samples from the resulting signal.
Therefore, if with subsignal-based application of $T$ the outcome for \emph{all} feasible subsignals should be determined, it suffices to carry out signal-based application of $T$ on the entire input signal once, preventing redundant computations.
These concepts are illustrated in Fig.~\ref{fig:exchange-property}.

\begin{figure}[t]
  \centering
  \scalebox{\figurescalefactor}{\includegraphics[page=2]{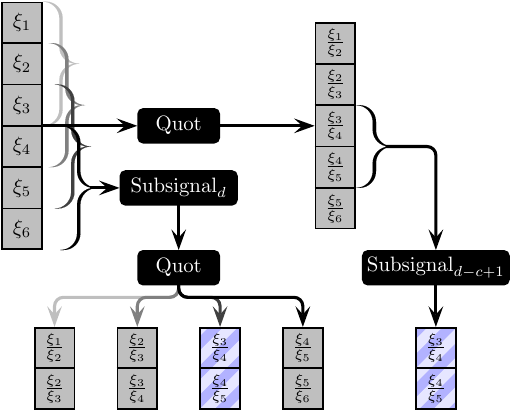}}
  \caption{Example of a subsignal compatible transformation.
    Here, $\operatorname{Quot}$ is a non-linear operator that computes the quotient of $c = 2$ adjacent samples and always reduces the dimensionality of its input by one sample, satisfying the dimensionality reduction property.
    The lower part shows the result of first extracting subsignals with $d = 3$ samples from the input signal $\xi$ and then evaluating $\operatorname{Quot}$.
    This yields processed subsignals $\operatorname{Quot}\left(\Subsignal_d(\xi,\;i)\right)$ with $d - c + 1 = 2$ samples each.
    The exchange property guarantees that these processed subsignals can also be found in $\operatorname{Quot}(\xi)$, exemplarily shown at the right-hand side of the graphics for subsignal index $i = 3$.}
  \label{fig:exchange-property}
\end{figure}

Note that the exchange property is well-defined:
The dimensionality reduction property guarantees that the dimensionalities on both sides of the equation match.
Further, the subsignal index $i$ is within bounds on both sides.
This is trivial for the left-hand side, and can be seen for the right-hand side since $\dim_M(\xi) - d + 1 = (\dim_M(\xi) - c + 1) - (d - c + 1) + 1$.

An identity theorem for subsignal compatible transformations immediately follows:
\begin{theorem}
\label{thm:subsignal-identity}
Let $M,N$ be sets and $T_1,T_2\colon\cup_c(M)\to\cup_1(N)$ two subsignal compatible transformations with dimensionality reduction constant $c\in\N_1$.
If $T_1(\rho) = T_2(\rho)$ holds for all $\rho\in M^c$, then already $T_1 = T_2$.
\end{theorem}
\begin{proof}
Let $\xi\in\cup_c(M)$.
For $\mu\inint{1}{\dim_M(\xi) - c + 1}$, applying the precondition (PC) and the exchange property where the subsignal dimensionality $d$ is set to $c$ yields:
$T_1(\xi)_\mu \equsing{XP} T_1(\Subsignal_c(\xi,\;\mu)) \equsing{PC} T_2(\Subsignal_c(\xi,\;\mu)) \equsing{XP} T_2(\xi)_\mu$.
Hence all samples of the transformed signals match, thus $T_1(\xi) = T_2(\xi)$ for all $\xi$ in the domain of $T_1$ and $T_2$.
\end{proof}

\subsection{Relationship between Functions Applied in a Sliding Fashion and Subsignal Compatible Transformations}
Turning now to functions applied to a signal in a \emph{sliding} fashion, first a definition what is meant hereby:
\begin{definition}
\label{def:sliding-function}
Let $M$ and $N$ be sets, let $c\in\N_1$ be a positive natural number and let $f\colon M^c\to N$ be a function.
Then $\Slide_f\colon\cup_c(M)\to\cup_1(N)$,
\begin{displaymath}
  \xi\mapsto\sum\nolimits_{i = 1}^{\dim_M(\xi) - c + 1} f(\Subsignal_c(\xi,\;i))\cdot e_i^{\dim_M(\xi) - c + 1}\text{,}
\end{displaymath}
is the operator that applies $f$ in a \emph{sliding fashion} to all the subsignals of length $c$ of the input signal and stores the result in a contiguous signal.
The sliding window is always advanced by exactly one entry after each evaluation of $f$.
\end{definition}

The next result states that functions applied in a sliding fashion are essentially the same as subsignal compatible transformations, and that the exchange property could be weakened to hold only for the case where the dimensionality reduction constant equals the subsignal dimensionality:
\begin{theorem}
\label{thm:sliding-subsignal}
Let $M$ and $N$ be sets, let $c\in\N_1$ and let $T\colon \cup_c(M) \to \cup_1(N)$ be a function.
Then the following are equivalent:
\begin{enumerate}
  \item \label{thm:sliding-subsignal-a} $T$ is a subsignal compatible transformation with dimensionality reduction constant $c$.
  \item \label{thm:sliding-subsignal-b} $T$ fulfills the dimensionality reduction property, and for all $\xi\in\cup_c(M)$ and all $i\inint{1}{\dim_M(\xi) - c + 1}$ it holds that $T(\Subsignal_c(\xi,\;i)) = T(\xi)_i$.
  \item \label{thm:sliding-subsignal-c} There is a unique function $f\colon M^c\to N$ with $T = \Slide_f$.
\end{enumerate}
\end{theorem}
\begin{proof}
\ref{thm:sliding-subsignal-a} $\Rightarrow$ \ref{thm:sliding-subsignal-b}:
Trivial, since the dimensionality reduction property is fulfilled by definition, and the claimed condition is only the special case of the exchange property where $d = c$.

\ref{thm:sliding-subsignal-b} $\Rightarrow$ \ref{thm:sliding-subsignal-c}:
For showing existence, define $f\colon M^c\to N$, $\xi\mapsto T(\xi)$.
For $\xi\in M^c$ it is $\dim_N(T(\xi)) = 1$ due to the dimensionality reduction property, therefore $f$ is well-defined.
Now let $\xi\in\cup_c(M)$ and define $D := \dim_M(\xi)$.
It is clear that $\dim_N(T(\xi)) = \dim_N(\Slide_f(\xi)) = D - c + 1$.
Now let $i\inint{1}{D - c + 1}$, then the precondition (PC) implies
$\Slide_f(\xi)_i \equsing{D.~\ref{def:sliding-function}} f(\Subsignal_c(\xi,\;i)) = T(\Subsignal_c(\xi,\;i)) \equsing{PC} T(\xi)_i$,
hence $T = \Slide_f$.

Considering uniqueness, suppose that there exist functions $f_1,f_2\colon M^c\to N$ with $T = \Slide_{f_1} = \Slide_{f_2}$.
Let $\rho\in M^c$ be arbitrary, then Definition~\ref{def:sliding-function} gives $f_1(\rho) = \Slide_{f_1}(\rho) = \Slide_{f_2}(\rho) = f_2(\rho)$, therefore $f_1 = f_2$ on $M^c$.

\ref{thm:sliding-subsignal-c} $\Rightarrow$ \ref{thm:sliding-subsignal-a}:
Suppose there is a function $f\colon M^c\to N$ with $T = \Slide_f$.
$\Slide_f$ inherently fulfills the dimensionality reduction property.
Let $d\in\N_1$, $d\geq c$, be an arbitrary subsignal dimensionality and let $\xi\in\cup_d(M)$ be a signal.
Further, let $i\inint{1}{\dim_M(\xi) - d + 1}$ be an arbitrary subsignal index.
Remembering that $\dim_M(\Subsignal_d(\xi,\;i)) = d$ and using Lemma~\ref{lem:subsignal-composition} gives
\begin{align*}
       & \Slide_f(\Subsignal_d(\xi,\;i))\\
  \equsing{D.~\ref{def:sliding-function}}\ \ & \sum\nolimits_{j = 1}^{d - c + 1} f(\Subsignal_c(\Subsignal_d(\xi,\; i),\; j)) \cdot e_j^{d - c + 1}\\
  \equsing{L.~\ref{lem:subsignal-composition}}\ \ & \sum\nolimits_{j = 1}^{d - c + 1} f(\Subsignal_c(\xi,\; i + j - 1)) \cdot e_j^{d - c + 1}\\
  \equsing{D.~\ref{def:sliding-function}}\ \ & \sum\nolimits_{j = 1}^{d - c + 1} \Slide_f(\xi)_{i + j - 1} \cdot e_j^{d - c + 1}\\
  \equsing{D.~\ref{def:subsignal}}\ \ & \Subsignal_{d - c + 1}(\Slide_f(\xi),\;i)\text{,}
\end{align*}
thus the exchange property is satisfied as well.
\end{proof}

Therefore, for each subsignal compatible transformation there is a unique function that \emph{generates} the transformation.
This yields a succinct characterization which helps in deciding whether a given transformation fulfills the dimensionality reduction property and the exchange property.
It is further clear that subsignal compatible transformation evaluations themselves can be parallelized since there is no data dependency between individual samples of the outcome.

Reconsidering Fig.~\ref{fig:exchange-property} it is now obvious that the $\operatorname{Quot}$ operator introduced there is no more than the quotient of two samples evaluated in a sliding fashion.
It seems plausible from this example that convolution is also a subsignal compatible transformation.
This is proven rigorously in Sect.~\ref{sect:CNNs-wo-pooling}.

Before discussing more theoretical properties, first an example of a transformation that is \emph{not} subsignal compatible:
\begin{example}
Let $\Z$ denote the integers and consider the function $T\colon\cup_1(\Z) \to \cup_1(\Z)$, $\xi\mapsto (-1)^{\dim_{\Z}(\xi)}\cdot\xi$, which fulfills the dimensionality reduction property with dimensionality reduction constant $c := 1$.
The exchange property is, however, not satisfied:
Let $d := 1$ and $\xi\in\Z^2$, then $i := 1$ yields $T(\Subsignal_d(\xi,\;i)) = T(\xi_1) = -\xi_1$, but it is $\Subsignal_{d - c + 1}(T(\xi),\;i) = \Subsignal_{d - c + 1}(\xi) = \xi_1$.
Since $\xi_1 \neq -\xi_1$ unless $\xi_1$ vanishes, $T$ cannot be a subsignal compatible transformation.
\end{example}

\subsection{Composition of Subsignal Compatible Transformations}
The composition of subsignal compatible transformations is again a subsignal compatible transformation, where the dimensionality reduction constant has to be adjusted:
\begin{theorem}
\label{thm:sscomp-composition}
Let $M$, $N$ and $P$ be sets and let $c_1,c_2\in\N_1$.
Suppose $T_1\colon \cup_{c_1}(M) \to \cup_1(N)$ is a subsignal compatible transformation with dimensionality reduction constant $c_1$, and $T_2\colon \cup_{c_2}(N) \to \cup_1(P)$ is a subsignal compatible transformation with dimensionality reduction constant $c_2$.

Define $c := c_1 + c_2 - 1\in\N_1$.
Then $T\colon\cup_c(M)\to\cup_1(P)$, $\xi\mapsto T_2(T_1(\xi))$, is a subsignal compatible transformation with dimensionality reduction constant $c$.
\end{theorem}
\begin{proof}
Note first that $c\geq 1$ since $c_1\geq 1$ and $c_2\geq 1$, hence indeed $c\in\N_1$.
Let $\xi\in\cup_c(M)$ be arbitrary for demonstrating that $T$ is well-defined.
As $c\geq c_1$ because of $c_2\geq 1$, this yields $\cup_c(M) \subseteq\cup_{c_1}(M)$ and hence $T_1(\xi)$ is well-defined.
Further, $\dim_N(T_1(\xi)) = \dim_M(\xi) - c_1 + 1 \geq c - c_1 + 1 = c_2$ using the dimensionality reduction property of $T_1$, therefore $T_1(\xi)\in\cup_{c_2}(N)$.
Thus $T_2(T_1(\xi))$ is well-defined, and so is $T$.

For all $\xi\in\cup_c(M)$, the dimensionality reduction property of $T_1$ and $T_2$ now implies
$\dim_P(T(\xi)) = \dim_P(T_2(T_1(\xi))) \equsing{DRP} \dim_N(T_1(\xi)) - c_2 + 1 \equsing{DRP} \dim_M(\xi) - c_1 + 1 - c_2 + 1 = \dim_M(\xi) - c + 1$,
therefore $T$ fulfills the dimensionality reduction property.

Let $d\in\N_1$, $d\geq c$, be arbitrary, and let $\xi\in\cup_d(M)$ and $i\inint{1}{\dim_M(\xi) - d + 1}$.
Since both $T_1$ and $T_2$ satisfy the exchange property, it follows that
$T(\Subsignal_d(\xi,\;i)) \equsing{XP} T_2(\Subsignal_{d - c_1 + 1}(T_1(\xi,\;i))) \equsing{XP} \Subsignal_{d - c + 1}(T(\xi),\;i)$,
where $d\geq c_1$ and $d - c_1 + 1 \geq c_2$ hold during the two respective applications of the exchange property.
Therefore, $T$ also fulfills the exchange property.
\end{proof}

This result can be generalized immediately to compositions of more than two subsignal compatible transformations:
\begin{corollary}
\label{cor:sscomp-mult-composition}
Let $n\in\N$, $n\geq 2$, and let $M_1,\dotsc,M_{n+1}$ be sets.
For each $\lambda\inint{1}{n}$ let $T_\lambda\colon\cup_{c_\lambda}(M_\lambda)\to\cup_1(M_{\lambda + 1})$ be a subsignal compatible transformation with dimensionality reduction constant $c_\lambda\in\N_1$.
Then the composed function $T\colon\cup_c(M_1)\to\cup_1(M_{n + 1})$, $\xi\mapsto\left(\circ_1^{\lambda = n}\ T_\lambda\right)(\xi)$, is a subsignal compatible transformation with dimensionality reduction constant $c := \sum_{\mu = 1}^n c_\mu - n + 1\in\N_1$.
\end{corollary}
\begin{proof}
Define $S_1 := T_1$, and for each $\lambda\inint{2}{n}$ let $S_\lambda\colon\cup_{\sum_{\mu = 1}^\lambda c_\mu - \lambda + 1}(M_1)\to\cup_1(M_{\lambda + 1})$, $\xi\mapsto T_\lambda(S_{\lambda - 1}(\xi))$, be a function.
Since $T = S_n$, the claim follows when it is shown with induction for $\lambda$ that $S_\lambda$ is a subsignal compatible transformation with dimensionality reduction constant $\sum_{\mu = 1}^\lambda c_\mu - \lambda + 1$.
While the situation $\lambda = 1$ is trivial, the induction step follows with Theorem~\ref{thm:sscomp-composition}.
\end{proof}

\subsection{CNNs without Pooling Layers}
\label{sect:CNNs-wo-pooling}
To conclude this section, a demonstration is provided of how CNNs without any pooling layers fit in the theory developed so far.
Since pooling layers require a non-trivial extension of the theory, they are detailed in Sect.~\ref{sect:strided_functions}.

Convolutional layers are the most substantial ingredient of CNNs, the trainable degrees of freedom which facilitate adaptation of the network to a specific task are located here.
In these layers, multi-channel input feature maps are convolved channel-wise with adjustable filter banks, the result is accumulated and an adjustable bias is added to yield the output feature map.

First, the introduction of the indexing rules for iterated structures to account for the multi-channel nature of the occurring signals.
Let $M$ be a set, $a,b\in\N_1$ positive natural numbers and $\xi\in(M^a)^b$ a multi-channel signal.
It is then $\xi_j\in M^a$ for indices $j\inint{1}{b}$, and moreover $(\xi_j)_i\in M$ for indices $j\inint{1}{b}$ and $i\inint{1}{a}$.
This rule is extended naturally to sets written explicitly as products with more than two factors.
Therefore, if $\xi\in((M^a)^b)^c$ for another number $c\in\N_1$, then for example $(\xi_k)_j\in M^a$ for indices $k\inint{1}{c}$ and $j\inint{1}{b}$.

These rules become clearer if the multi-channel convolution operation $\conv$ is considered.
Suppose the samples are members of a ring $R$, $m\in\N_1$ denotes the number of input channels, $n\in\N_1$ is the number of output channels, and $c\in\N_1$ equals the number of samples considered at any one time during convolution with the filter bank, or in other words the receptive field size of the convolutional layer.
Then input signals or feature maps with $D\in\N_1$ samples have form $\xi\in(R^m)^D$, and filter banks can be represented by a tensor $w\in((R^n)^m)^c$.
Here $D\geq c$ must hold, that is the filter kernel should be smaller than the input signal.

The output feature map $(\xi\conv w)\in(R^n)^{D - c + 1}$ is then
\begin{displaymath}
  (\xi\conv w)_i := \sum\nolimits_{\lambda = 1}^m\sum\nolimits_{\mu = 1}^c (w_\mu)_\lambda \cdot (\xi_{c + i - \mu})_\lambda \in R^n
\end{displaymath}
for indices $i\inint{1}{D - c + 1}$.
Note that $(w_\mu)_\lambda\in R^n$ and $(\xi_{c + i - \mu})_\lambda\in R$, so that the result of their product is understood here as scalar product.
The operation is well-defined since $c + i - \mu\inint{1}{D}$, which follows immediately through substitution of the extreme values of $i$ and $\mu$.

This multi-channel convolution operation is indeed a subsignal compatible transformation as shown explicitly here:
\begin{example}
Define $M := R^m$ and $N := R^n$ and consider $f_{\convop}\colon M^c\to N$,
\begin{displaymath}
  \xi\mapsto\sum\nolimits_{\lambda = 1}^m\sum\nolimits_{\mu = 1}^c (w_\mu)_\lambda \cdot (\xi_{c - \mu + 1})_\lambda\text{.}
\end{displaymath}
Since $\mu\inint{1}{c}$ it is $c - \mu + 1\inint{1}{c}$, hence $f_{\convop}$ is well-defined.
For all $\xi\in M^D$ and any $i\inint{1}{D - c + 1}$ follows
\begin{align*}
     & \Slide_{f_{\convop}}(\xi)_i\\
  \equsing{D.~\ref{def:sliding-function}}\ \ &f_{\convop}(\Subsignal_c(\xi,\;i))\\
  \equsing{D.~\ref{def:subsignal}}\ \ &f_{\convop}\left(\sum\nolimits_{\nu = 1}^c \xi_{i + \nu - 1}\cdot e_\nu^c\right)\\
  =\ \ &\sum\nolimits_{\lambda = 1}^m\sum\nolimits_{\mu = 1}^c (w_\mu)_\lambda \cdot \left(\left(\sum\nolimits_{\nu = 1}^c \xi_{i + \nu - 1}\cdot e_\nu^c\right)_{c - \mu + 1}\right)_\lambda\\
  \equsing{($\lozenge$)}\ \ &\sum\nolimits_{\lambda = 1}^m\sum\nolimits_{\mu = 1}^c (w_\mu)_\lambda \cdot (\xi_{i + c - \mu + 1 - 1})_\lambda\\
  =\ \ &(\xi\conv w)_i\text{,}
\end{align*}
where $\nu = c - \mu + 1$ was substituted in the ($\lozenge$) step.
The multi-channel convolution operation as defined above is hence in fact the application of $f_{\convop}$ in a sliding fashion.
Therefore, Theorem~\ref{thm:sliding-subsignal} guarantees that $\conv$ is a subsignal compatible transformation with dimensionality reduction constant $c$.
\end{example}

Since fully-connected layers are merely a special case of convolutional layers, these do not need any special treatment here.
Addition of biases does not require any knowledge on the spatial structure of the convolution's result and is therefore a trivial subsignal compatible transformation with dimensionality reduction constant $1$.
Non-linearity layers are nothing but the application of a scalar-valued function to all the samples of an input signal.
Hence these layers also form subsignal compatible transformations with dimensionality reduction constant $1$ due to Theorem~\ref{thm:sliding-subsignal}.

Furthermore, compositions of these operations can also be understood as subsignal compatible transformations with Corollary~\ref{cor:sscomp-mult-composition}.
As a consequence, the exchange property facilitates application of CNNs without pooling layers to an entire signal at once instead of each subsignal independently, all without incurring any accuracy loss.
The next section will extend this result to CNNs that may also feature pooling layers.

\section{Pooling Layers and Functions Applied in a Strided Fashion}
\label{sect:strided_functions}
So far it has been shown how convolutional layers and non-linearity layers of a CNN fit in the theoretical framework of subsignal compatible transformations.
This section analyzes pooling layers which apply a pooling kernel to non-overlapping blocks of the input signal.
This is equivalent to a function applied in a sliding fashion followed by a downsampling operation, which will here be referred to as the application of a function in a \emph{strided} fashion.

The theory developed herein can of course also be applied to other functions than the pooling kernels encountered in ordinary CNNs.
For example, multi-channel convolution in which the filter bank is advanced by the receptive field size is essentially $f_{\convop}$ from Sect.~\ref{sect:CNNs-wo-pooling} applied in a strided fashion.
Application of convolution where the filter banks are advanced by more than one sample has however no benefit in terms of execution speed for signal-based application.
This is discussed at the end of Sect.~\ref{sect:fragmentation_relationship} after having developed sufficient theory to analyze this notion.

This section demonstrates how these functions can be turned into efficiently computable subsignal compatible transformations using a data structure recently introduced as fragmentation by Giusti \emph{et al.}~\cite{Giusti2013}.
Here, that proposed method is generalized and rigorously proven correct.
As an added benefit of these results, the dynamics of the entire signal processing chain can also be accurately described, including the possibility of tracking down the position of each processed subsignal in the fragmentation data structure.

Moreover, the circumstances under which the fragment dimensionalities are guaranteed to always be homogeneous are analyzed.
This is a desirable property as it facilitates the application of subsequent operations to signals which all have the same number of samples, rendering cumbersome handling of special cases obsolete and thus resulting in accelerated execution on massively parallel processors.
For CNNs this means that conventional tensor convolutions can be used without any modifications whatsoever, which is especially beneficial if a highly-optimized implementation is readily available.

\begin{figure}[t]
  \centering
  \scalebox{\figurescalefactor}{\includegraphics[page=3]{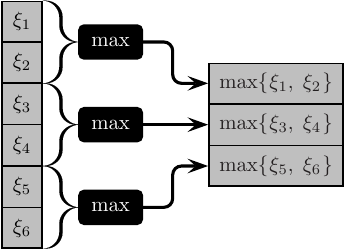}}
  \caption{Illustration of the pooling kernel $\max$ that determines the maximum of $k = 2$ adjacent samples.
    It is here applied in a strided fashion to non-overlapping subsignals of an input signal $\xi$ with $D = 6$ samples, yielding an output signal $\Stride_{\max}(\xi)$ with $\frac{D}{k} = 3$ samples.
    Since here the dimensionality halves and therefore the dimensionality reduction property is violated, this is \emph{not} a subsignal compatible transformation.}
  \label{fig:strided-function}
\end{figure}

First, a more precise statement on what the application of a function in a strided fashion means (see Fig.~\ref{fig:strided-function} for orientation):
\begin{definition}
\label{def:strided-function}
Let $M$ and $N$ be sets, let $k\in\N_1$ be a positive natural number and let $g\colon M^k\to N$ be a function.
Then $\Stride_g\colon\cup_{q = 1}^\infty M^{kq}\to\cup_1(N)$,
\begin{displaymath}
  \xi\mapsto\sum\nolimits_{i = 1}^{\dim_M(\xi) / k} g(\Subsignal_k(\xi,\;k(i - 1) + 1))\cdot e_i^{\dim_M(\xi) / k}\text{,}
\end{displaymath}
is the operator that applies $g$ in a \emph{strided fashion} to signals where the number of samples is a multiple of $k$.
The subsignal indices are chosen here so that all non-overlapping subsignals are fed through $g$, starting with the first valid subsignal.
\end{definition}

Since it is $k(i - 1) + 1\inint{1}{\dim_M(\xi) - k + 1}$ for all $i\inint{1}{\dim_M(\xi) / k}$, $\Stride_g$ is well-defined.
Further, $\dim_M(\xi) / \dim_N(\Stride_g(\xi)) = k$ for all $\xi$ in the domain of $\Stride_g$.
Since the input dimensionality is reduced here through division with a natural number rather than a subtraction, the dimensionality reduction property cannot be fulfilled unless $k = 1$.
The situation in which $k = 1$ is, however, not particularly interesting since then $\Stride_g = \Slide_g$ which was already handled in Sect.~\ref{sect:subsignal-compatible-transformations}.

Before continuing with fragmentation, first consider multi-channel pooling kernels commonly encountered in CNNs:
\begin{example}
Assume the goal is to process real-valued signals with $m\in\N_1$ channels, that is $M = N = \R^m$, where each channel should be processed independently of the others, and $k\in\N_1$ adjacent samples should be compressed into one output sample.
\emph{Average pooling} is then realized by the pooling kernel $g_{\avg}(\xi) := \frac{1}{k}\sum_{\nu = 1}^k \xi_\nu$, which determines the channel-wise empirical mean value of the samples.
Another example is \emph{max-pooling}, where the maximum entry in each channel should be determined.
This can be achieved with the pooling kernel $g_{\max}(\xi) := \sum_{\lambda = 1}^m \left(\max_{\nu = 1}^k (\xi_\nu)_\lambda \right)\cdot e_\lambda^m$.
\end{example}

\subsection{Fragmentation}
The fragmentation operator~\cite{Giusti2013} performs a spatial reordering operation.
Its precise analysis requires a recap of some elementary number theory.
For all numbers $a\in\N$ and $b\in\N_1$, \emph{Euclidean division} guarantees that there are unique numbers $\div{a}{b}\in\N$ and $\rem{a}{b}\inint{0}{b - 1}$ so that $a = \div{a}{b}\cdot b + \rem{a}{b}$.
Here is a small collection of results on these operators for further reference:
\begin{proposition}
\label{prop:number-theory}
It is $\div{a}{1} = a$ and $\rem{a}{1} = 0$ for all $a\in\N$.
Moreover, $\div{a + bc}{c} = \div{a}{c} + b$ and $\rem{a + bc}{c} = \rem{a}{c}$ for all $a,b\in\N$ and $c\in\N_1$.
\end{proposition}

If the fragmentation operator is applied to a signal, it puts certain samples into individual fragments which can be grasped as signals themselves.
If a collection of fragments is fragmented further, a larger collection of fragments results.
The total number of samples is, however, left unchanged after these operations.
For the sake of convenience, matrices are used here as concrete data structure for fragmented signals, where columns correspond to fragments and rows correspond to signal samples.

First, some notation needs to be defined.
If $M$ is a set and $a,b\in\N_1$, then $M^{a\times b}$ denotes the set of all matrices with $a$ rows and $b$ columns with entries from $M$.
In the present context, this represents a collection of $b$ fragments where each signal has $a$ samples.
For $\xi\in M^{a\times b}$, $\rdim_M(\xi) = a$ and $\cdim_M(\xi) = b$ denote the number of rows and columns, respectively.
Furthermore, $\xi_{i,\; j}$ is the entry in the $i$-th row and $j$-th column of $\xi$ where $i\inint{1}{a}$ and $j\inint{1}{b}$.
The transpose of $\xi$ is written as $\xi\transp$.

The vectorization operator~\cite{Neudecker1969} stacks all the columns of a matrix on top of another:
\begin{definition}
\label{def:vectorization}
Let $M$ be a set and $a,b\in\N_1$.
The \emph{vectorization operator} $\vect_{a\times b}\colon M^{a\times b}\to M^{ab}$ is characterized by $\vect_{a\times b}(\xi)_j = \xi_{\rem{j - 1}{a} + 1,\;\div{j - 1}{a} + 1}$ for all indices $j\inint{1}{ab}$ and all matrices $\xi\in M^{a\times b}$.
The \emph{inverse vectorization operator} $\vect_{a\times b}^{-1}\colon M^{ab}\to M^{a\times b}$ is given by $\vect_{a\times b}^{-1}(\xi)_{i,\;j} = \xi_{(j - 1)a + i}$ for all indices $i\inint{1}{a}$, $j\inint{1}{b}$ and all vectors $\xi\in M^{ab}$.
\end{definition}

It can be verified directly that these two operators are well-defined permutations and inversely related to one another.
With their help the fragmentation operator may now be defined:
\begin{definition}
\label{def:fragmentation}
Let $M$ be a set and $k\in\N_1$.
For arbitrary vector dimensionalities $q\in\N_1$ and numbers of input fragments $s\in\N_1$ the function $\Fragmentation_k\colon M^{kq\times s}\to M^{q\times ks}$,
\begin{displaymath}
  \xi\mapsto\left(\vect_{ks\times q}^{-1}\left( \vect_{s\times kq}\left(\xi\transp\right) \right)\right)\transp\!\text{,}
\end{displaymath}
is called the \emph{fragmentation operator}.
\end{definition}

Here, $k$ equals the corresponding parameter from the application of a function in a strided fashion.
$\Fragmentation_k$ is clearly well-defined, and the number of output fragments is $ks$.
Next consider this operator that undoes the ordering of the fragmentation operator:
\begin{definition}
Let $M$ be a set, let $k\in\N_1$, and let $q\in\N_1$ denote a vector dimensionality and $s\in\N_1$ a number of output fragments.
Then $\Defragmentation_k\colon M^{q\times ks}\to M^{kq\times s}$,
\begin{displaymath}
  \xi\mapsto\left(\vect_{s\times kq}^{-1}\left( \vect_{ks\times q}\left(\xi\transp\right) \right)\right)\transp\!\text{,}
\end{displaymath}
is called the \emph{defragmentation operator}.
\end{definition}

Note that $\Defragmentation_k$ is well-defined and the number of input fragments must equal $ks$.
Fragmentation and defragmentation are inversely related, that is $\Defragmentation_k\circ \Fragmentation_k = \id_{M^{kq\times s}}$ and $\Fragmentation_k\circ \Defragmentation_k = \id_{M^{q\times ks}}$.
An illustration of the operations performed during fragmentation and defragmentation is depicted in Fig.~\ref{fig:fragmentation-operator}.

\begin{figure}[t]
  \centering
  \scalebox{\figurescalefactor}{\includegraphics[page=4]{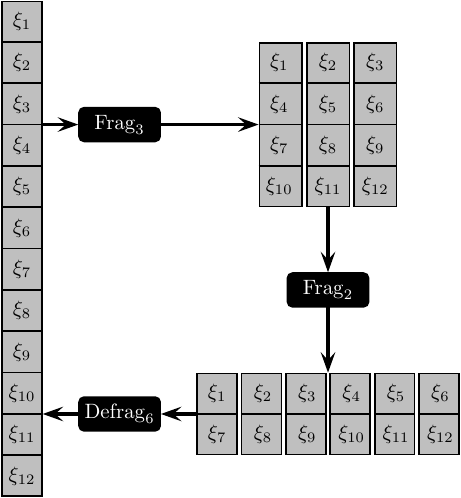}}
  \caption{Illustration of the fragmentation and defragmentation operators.
    The left-hand side shows an input signal $\xi$ with $q_0 = 12$ samples in a single fragment, that is $s_0 = 1$.
    Fragmentation with the parameter $k_1 = 3$ yields a signal with $q_1 = \frac{q_0}{k_1} = 4$ samples in each of $s_1 = k_1 s_0 = 3$ fragments, shown at the upper right in the graphics.
    A second application of fragmentation with $k_2 = 2$ results in $s_2 = k_2 s_1 = 6$ fragments with $q_2 = \frac{q_1}{k_2} = 2$ samples each, depicted at the lower right.
    A complete defragmentation with parameter $k^* = k_1 k_2 = 6$ yields the original input signal $\xi$ again, with $q_3 = k^* q_2 = 12$ samples in a single fragment, that is $s_3 = \frac{s_2}{k^*} = 1$.}
  \label{fig:fragmentation-operator}
\end{figure}

Fragmentation is merely a certain reordering operation:
\begin{lemma}
\label{lem:frag-ops}
Suppose that $M$ is a set, $k,q,s\in\N_1$ and ${\xi\in M^{kq\times s}}$.
Then $\rdim_M(\Fragmentation_k(\xi)) = \frac{1}{k}\cdot\rdim_M(\xi)$ and $\cdim_M(\Fragmentation_k(\xi)) = k\cdot\cdim_M(\xi)$.
Further, $\Fragmentation_k(\xi)_{\mu,\;\nu} = \xi_{\div{ (\mu - 1)ks + \nu - 1}{s} + 1,\;\rem{ (\mu - 1)ks + \nu - 1}{s} + 1}$
for all indices $\mu\inint{1}{q}$ and $\nu\inint{1}{ks}$.
\end{lemma}
\begin{proof}
The dimensionality statements are obvious by the definition of $\Fragmentation_k$.
To prove the identity, let $\mu\inint{1}{q}$ and $\nu\inint{1}{ks}$.
One yields
\begin{align*}
     & \Fragmentation_k(\xi)_{\mu,\;\nu}\\
  \equsing{D.~\ref{def:fragmentation}}\ \ & \vect_{ks\times q}^{-1}\left( \vect_{s\times kq}\left(\xi\transp\right) \right)_{\nu,\;\mu}
  \ \equsing{D.~\ref{def:vectorization}}\ \vect_{s\times kq}\left(\xi\transp\right)_{(\mu - 1)ks + \nu}\\
  \equsing{D.~\ref{def:vectorization}}\ \ & \left(\xi\transp\right)_{\rem{ (\mu - 1)ks + \nu - 1}{s} + 1,\;\div{ (\mu - 1)ks + \nu - 1}{s} + 1}\text{,}
\end{align*}
and the claim follows.
\end{proof}

Similar properties are fulfilled by defragmentation:
\begin{lemma}
\label{lem:defrag-ops}
Let $M$ be a set.
Let $k,q,s\in\N_1$ be positive natural numbers and let $\xi\in M^{q\times ks}$ be an arbitrary fragmented signal.
Then $\rdim_M(\Defragmentation_k(\xi)) = k\cdot\rdim_M(\xi)$, $\cdim_M(\Defragmentation_k(\xi)) = \frac{1}{k}\cdot\cdim_M(\xi)$, and $\Defragmentation_k(\xi)_{\mu,\;\nu} = \xi_{\div{ (\mu - 1)s + \nu - 1}{ks} + 1,\;\rem{ (\mu - 1)s + \nu - 1}{ks} + 1}$
for all indices $\mu\inint{1}{kq}$, $\nu\inint{1}{s}$.
\end{lemma}
\begin{proof}
Completely analogous to Lemma~\ref{lem:frag-ops}.
\end{proof}

As already outlined in Fig.~\ref{fig:fragmentation-operator}, compositions of the fragmentation operator are equivalent to a single fragmentation with an adjusted parameterization:

\begin{remark}
\label{rem:comp-frag}
Let $M$ be a set and $k_1,k_2,q,s\in\N_1$.
Then $\Fragmentation_{k_2}(\Fragmentation_{k_1}(\xi)) = \Fragmentation_{k_1k_2}(\xi)$ for all $\xi\in M^{k_1k_2q\times s}$.
\end{remark}
\begin{proof}
The claim follows through entry-wise comparison between $\Fragmentation_{k_2}(\Fragmentation_{k_1}(\xi))$ and $\Fragmentation_{k_1k_2}(\xi)$ using Lemma~\ref{lem:frag-ops}.
\end{proof}

It follows immediately that fragmentation is a commutative operation:
\begin{remark}
If $M$ denotes a set, $k_1,k_2,q,s\in\N_1$ are natural numbers and $\xi\in M^{k_1k_2q\times s}$ is a fragmented signal, then $\Fragmentation_{k_2}(\Fragmentation_{k_1}(\xi)) = \Fragmentation_{k_1}(\Fragmentation_{k_2}(\xi))$.
\end{remark}
\begin{proof}
Obvious with Remark~\ref{rem:comp-frag} as multiplication in $\N_1$ is commutative.
\end{proof}

\subsection{Relationship between Fragmentation, Functions Applied in a Strided Fashion and Subsignal Compatible Transformations}
\label{sect:fragmentation_relationship}
A bit more background is necessary before analyzing how functions applied in a strided fashion fit into the theory of subsignal compatible transformations.
The outcome of a subsignal compatible transformation applied to a fragmented signal is defined naturally:
\begin{definition}
\label{def:fragmentwise-evaluation}
Let $M,N$ be sets and $T\colon \cup_c(M) \to \cup_1(N)$ a subsignal compatible transformation with dimensionality reduction constant $c\in\N_1$.
Let $\xi\in M^{D\times s}$ be a fragmented signal with $D\in\N_1$ samples in each of the $s\in\N_1$ fragments, where $D\geq c$ holds.
For $\gamma\inint{1}{s}$ let $\xi^{_{(\gamma)}} := \sum_{\nu = 1}^D \xi_{\nu,\;\gamma}\cdot e_\nu^D\in M^D$ denote the individual fragments.
The output of $T$ applied to $\xi$ is then defined as $T(\xi) := [T(\xi^{_{(1)}}),\dotsc,T(\xi^{_{(s)}})]\in N^{(D - c + 1)\times s}$,
that is $T$ is applied to all the fragments independently.
\end{definition}

Since there is no data dependency between fragments, parallelization of subsignal compatible transformation evaluation over all output samples is straightforward.
What follows is the formal introduction of the processing chain concept which captures and generalizes all the dynamics of a CNN, and two notions of its application to signal processing:
\begin{definition}
\label{def:processing-chain}
The collection of the following objects is called a \emph{processing chain}:
A fixed subsignal dimensionality $\ROI\in\N_1$,
a number of layers $L\in\N_1$,
a sequence of sets $M_0,\dotsc,M_{L},N_1,\dotsc,N_L$,
and for each $j\inint{1}{L}$ subsignal compatible transformations $T_j\colon\cup_{c_j}(M_{j - 1})\to\cup_1(N_j)$ with dimensionality reduction constant $c_j\in\N_1$
and functions $g_j\colon N_j^{k_j}\to M_j$ where $k_j\in\N_1$.
The numbers $k_j^* := \prod_{\nu = 1}^j k_\nu$ for $j\inint{0}{L}$ are called the \emph{stride products} of the processing chain.
This implies that $k_0^* = 1$.
For $j\inint{0}{L}$, the operator $\EvalStride_j\colon M_0^\ROI\to\cup_1(M_j)$, 
\begin{displaymath}
  \rho\mapsto
  \begin{cases}
    \rho\text{,} & \text{if } j = 0\text{,}\\
    \Stride_{g_j}(T_j(\EvalStride_{j - 1}(\rho)))\text{, } & \text{if } j > 0\text{,}
  \end{cases}
\end{displaymath}
applies the processing chain in a \emph{strided} fashion, and further $\EvalSlide_j\colon\cup_\ROI(M_0)\to\cup_{q = 1}^\infty\cup_{s = 1}^\infty M_j^{q\times s}$, 
\begin{displaymath}
  \xi\mapsto
  \begin{cases}
    \xi\text{,} & \text{if } j = 0\text{,}\\
    \Fragmentation_{k_j}(\Slide_{g_j}(T_j(\EvalSlide_{j - 1}(\xi))))\text{, } & \text{if } j > 0\text{,}
  \end{cases}
\end{displaymath}
is the operator that applies the processing chain in a \emph{sliding} fashion.
Note that these two functions are \emph{not} well-defined unless additional conditions are fulfilled, detailed below.
\end{definition}

The number $B$ here represents the extent of the region that is fed into a CNN, or in other words the entire network's receptive field size.
This size is a design parameter of the network and depends on the concrete definitions of all of its layers.
The functions $T_j$ in a processing chain can be substituted with the appropriate layer types discussed earlier, such as convolutions or non-linearities, or compositions thereof.
Pooling kernels and other functions applied in a strided fashion to non-overlapping blocks can be plugged into a processing chain via the $g_j$ functions.
The recursive definitions of $\EvalStride_j$ and $\EvalSlide_j$ represent the alternating evaluation of a subsignal compatible transformation and a function applied in a strided fashion up to the specified layer index $j$.

The rationale for the $\EvalStride$ operator is the naive \emph{subsignal-based application} of a CNN:
Here, the CNN is applied in the ordinary way to signals of length equal to the network's receptive field size $B$.
According application of the network using a sliding window approach involves extraction of all feasible overlapping subsignals of length $B$ and feeding them through the network independently of each other.

The $\EvalSlide$ operator differs from $\EvalStride$ in that it corresponds to the \emph{signal-based application} of a CNN:
No overlapping subsignals need to be processed separately here, preventing redundant computations.
Instead, the complete input signal is processed in its entirety, sharing intermediate computation results among adjacent subsignals.
Using $\EvalSlide$, the $g_j$ functions are applied in a sliding rather than a strided fashion, followed by a fragmentation operation.

Definition~\ref{def:processing-chain} thus describes a recipe for how a CNN can be transformed from a subsignal-based application to signal-based application.
A concrete example will be discussed in Sect.~\ref{sect:experimental_evaluation}.
First, however, a theoretical justification that this method indeed produces the correct outcome under all circumstances will be presented.
The next result states when the application of a processing chain is well-defined, and it proves that the result of the $\EvalStride$ operator applied to a subsignal of a larger signal can be found within the result of $\EvalSlide$ applied to the entire signal.
This then implies that both approaches deliver the very same values and hence verify $\EvalSlide$ involves no accuracy loss whatsoever.
\begin{lemma}
\label{lem:processing-chain}
Given a processing chain with the same notation as in Definition~\ref{def:processing-chain},
first assume that $k_j$ divides $\dim_{N_j}(T_j(\EvalStride_{j - 1}(\rho)))$ and that $\EvalStride_{j}(\rho)$ is non-empty for all $j\inint{1}{L}$ and all $\rho\in M_0^\ROI$.
In other words, the application of the processing chain in a strided fashion should be well-defined.

Let $D\in\N_1$, $D\geq\ROI$, be a signal dimensionality so that the number of subsignals $D - \ROI + 1$ of length $\ROI$ is divisible by the final stride product $k_L^*$, and let $\xi\in M_0^D$ be the considered signal.
Then the application of the processing chain in a sliding fashion to $\xi$ is well-defined, and additional statements hold:

Let $u_j := \dim_{M_j}(\EvalStride_j(\Subsignal_\ROI(\xi,\; i)))\in\N_1$ for all $j\inint{0}{L}$ be an abbreviation for the dimensionality of the intermediate representations in each layer of the $\EvalStride$ cascade.
Note that these numbers are actually independent of any subsignal index $i$.
Further, for all $j\inint{0}{L}$, let $U_j^{\col} := \cdim_{M_j}(\EvalSlide_j(\xi))\in\N_1$ and $U_j^{\row} := \rdim_{M_j}(\EvalSlide_j(\xi))\in\N_1$ be defined as abbreviations for the number of fragments and the fragmented signal dimensionality, respectively, after each layer using the $\EvalSlide$ operator.
Then for all $j\inint{0}{L}$ the following holds:
\begin{enumerate}\setlength{\itemsep}{.5ex}
  \item \label{lem:processing-chain-a} $u_j = \frac{1}{k_j^*}\left(\ROI - \sum_{\mu = 1}^j k_{\mu - 1}^*(c_\mu - 1)\right)$.
  \item \label{lem:processing-chain-b} $U_j^{\row} = \frac{1}{k_j^*}\left(D - k_j^* + 1 - \sum_{\mu = 1}^j k_{\mu - 1}^*(c_\mu - 1)\right)$ and $U_j^{\col} = k_j^*$.
  \item \label{lem:processing-chain-c} $U_j^{\row} - u_j + 1 = \frac{1}{k_j^*}(D - \ROI + 1)$. In other words, the number of distinct subsignals with $u_j$ samples in each fragment of the fragmented signals equals the original number of distinct subsignals divided by the corresponding number of fragments.
  \item \label{lem:processing-chain-d} For $i\inint{1}{D - \ROI + 1}$ and $\mu\inint{1}{u_j}$ it is
        \begin{align*}
             & \EvalStride_j(\Subsignal_\ROI(\xi,\; i))_\mu\\
          =\ & \EvalSlide_j(\xi)_{\div{i - 1}{k_j^*} + \mu,\;\rem{i - 1}{k_j^*} + 1}\text{.}
        \end{align*}
        Here, the latter can also be understood as one sample of the $\Subsignal_{u_j}$ operator applied to a certain fragment of $\EvalSlide_j(\xi)$.
\end{enumerate}
\end{lemma}
\begin{proof}
\ref{lem:processing-chain-a}
Let $i\inint{1}{D - \ROI + 1}$ be arbitrary and define $\rho := \Subsignal_\ROI(\xi,\;i)\in M_0^\ROI$ as an abbreviation.
It is $u_0 = \dim_{M_0}(\EvalStride_0(\rho)) = \dim_{M_0}(\rho) = \ROI$,
and the right-hand side of the claim trivially equals $\ROI$ for $j = 0$.
Carrying out induction for $j - 1 \to j$ yields:
\begin{align*}
  u_j\ &\equsing{D.~\ref{def:processing-chain}}\ \dim_{M_j}(\Stride_{g_j}(T_j(\EvalStride_{j - 1}(\rho))))\\
       &\equsing{D.~\ref{def:strided-function}}\ \tfrac{1}{k_j}\dim_{N_j}(T_j(\EvalStride_{j - 1}(\rho)))\\
       &\equsing{DRP}\ \tfrac{1}{k_j}\left(\dim_{M_{j - 1}}(\EvalStride_{j - 1}(\rho)) - c_j + 1\right)\\
       &\equsing{IH}\ \tfrac{1}{k_j}\left(\tfrac{1}{k_{j - 1}^*}\left(\ROI - \sum\nolimits_{\mu = 1}^{j - 1} k_{\mu - 1}^*(c_\mu - 1)\right) - (c_j - 1)\right)\\
       &=\ \tfrac{1}{k_j k_{j - 1}^*} \left(\ROI - \sum\nolimits_{\mu = 1}^{j - 1} k_{\mu - 1}^*(c_\mu - 1) - k_{j - 1}^*(c_j - 1)\right)\text{,}
\end{align*}
where IH denotes substitution of the induction hypothesis.
Hence, the claimed expression follows since $k_j^* = k_j k_{j - 1}^*$.
Note that $u_j$ is indeed a positive natural number because $k_j$ divides $\dim_{N_j}(T_j(\EvalStride_{j - 1}(\rho)))$ and $\EvalStride_{j}(\rho)$ is non-empty by requirement.

\ref{lem:processing-chain-b}
Besides the statements on $U_j^{\row}$ and $U_j^{\col}$ it is shown here that the application of the processing chain in a sliding fashion is well-defined using induction for $j$.
For $j = 0$ follows $\EvalSlide_0(\xi) = \xi$, which is trivially well-defined and by definition it is $\xi\in M_0^D = M_0^{D\times 1}$.
Therefore, $U_0^{\row} = D$ and $U_0^{\col} = 1$ which equals the claimed expressions since $k_0^* = 1$.

For $j - 1 \to j$, it is first demonstrated that $k_j$ divides ${\chi := \rdim_{M_j}(\Slide_{g_j}(T_j(\EvalSlide_{j - 1}(\xi))))}$ which implies well-definedness since the fragmentation operator can then indeed be applied.
It follows that
\begin{align*}
  \chi\ &\equsing{D.~\ref{def:sliding-function}}\ \rdim_{N_j}(T_j(\EvalSlide_{j - 1}(\xi))) - k_j + 1\\
        &\equsing{DRP}\ \rdim_{M_{j - 1}}(\EvalSlide_{j - 1}(\xi)) - c_j + 1 - k_j + 1\\
        &=\ U_{j - 1}^{\row} - c_j + 1 - k_j + 1\\
        &\equsing{IH}\ \tfrac{1}{k_{j - 1}^*}\left(D - k_{j - 1}^* + 1 - \sum\nolimits_{\mu = 1}^{j - 1} k_{\mu - 1}^*(c_\mu - 1)\right)\\
        &\phantom{\equsing{IH}}\ + \tfrac{1}{k_{j - 1}^*}\left(-k_{j - 1}^*(c_j - 1) - k_j^* + k_{j - 1}^*\right)\\
        &=\ \tfrac{1}{k_{j - 1}^*}\left(D - k_j^* + 1 - \sum\nolimits_{\mu = 1}^j k_{\mu - 1}^*(c_\mu - 1)\right)\text{.}
\end{align*}
By requirement on the signal length $D$ there exists a number $t\in\N_1$ so that $D - \ROI + 1 = k_L^* t$.
Substitution yields
\begin{align*}
  \chi\ &=\ \tfrac{1}{k_{j - 1}^*}\left(\ROI - \sum\nolimits_{\mu = 1}^j k_{\mu - 1}^*(c_\mu - 1) + k_L^* t - k_j^*\right)\\
       &\equsing{\ref{lem:processing-chain-a}}\ k_j u_j + k_j\cdots k_L\cdot t - k_j\text{.}
\end{align*}
Proposition~\ref{prop:number-theory} implies that $k_j$ divides $\chi$ since $u_j\in\N_1$ as shown in~\ref{lem:processing-chain-a}, hence the processing chain can be applied until the $j$-th layer.
With Lemma~\ref{lem:frag-ops} follows $U_j^{\row} = \frac{1}{k_j}\chi$, which immediately yields the claimed expression.
Since only fragmentation influences the number of columns in the processing chain application, it follows that $U_j^{\col} = k_j U_{j - 1}^{\col}$ with Lemma~\ref{lem:frag-ops}, proving the claimed identity.

\ref{lem:processing-chain-c}
Using~\ref{lem:processing-chain-a} and~\ref{lem:processing-chain-b} one obtains $U_j^{\row} - u_j + 1 = \tfrac{1}{k_j^*}\left(D - k_j^* + 1 - \ROI\right) + 1 = \tfrac{D - \ROI + 1}{k_j^*}$,
which is a natural number as the number of subsignals was required to be divisible by $k_L^*$, implying divisibility by $k_j^*$.

\ref{lem:processing-chain-d}
This is proved by induction for $j$.
For $j = 0$, the left-hand side equals $\Subsignal_\ROI(\xi,\; i)_\mu = \xi_{i + \mu - 1}$ using Definition~\ref{def:subsignal} for all $i\inint{1}{D - \ROI + 1}$ and all $\mu\inint{1}{\ROI}$.
Since $k_0^* = 1$, Proposition~\ref{prop:number-theory} shows that the right-hand side equals $\xi_{\div{i - 1}{1} + \mu,\;\rem{i - 1}{1} + 1} = \xi_{i - 1 + \mu,\; 1}$, hence both sides are equal.

Turning now to $j - 1\to j$,
let $\mu\inint{1}{u_j}$ be arbitrary, let $i\inint{1}{D - \ROI + 1}$ be a fixed subsignal index and write $\tau := \EvalStride_{j - 1}(\Subsignal_\ROI(\xi,\;i))\in M_{j - 1}^{u_{j - 1}}$ as an abbreviation.
The left-hand side of the claim leads to
\begin{align*}
  & \EvalStride_j(\Subsignal_\ROI(\xi,\;i))_\mu\\
  \equsing{D.~\ref{def:processing-chain}}\ \ & \Stride_{g_j}(T_j(\tau))_\mu\\
  \equsing{D.~\ref{def:strided-function}}\ \ & g_j\!\left( \sum\nolimits_{\nu = 1}^{k_j} T_j(\tau)_{k_j(\mu - 1) + \nu}\cdot e_\nu^{k_j} \right)\\
  \equsing{T.~\ref{thm:sliding-subsignal}}\ \ & g_j\!\left( \sum\nolimits_{\nu = 1}^{k_j} T_j\!\left( \sum\nolimits_{\lambda = 1}^{c_j} \tau_{k_j(\mu - 1) + \nu + \lambda - 1}\cdot e_\lambda^{c_j} \right)\cdot e_\nu^{k_j} \right)\\
\equsing{IH}\ \ & g_j\!\left( \sum\nolimits_{\nu = 1}^{k_j} T_j\!\left( \sum\nolimits_{\lambda = 1}^{c_j} \EvalSlide_{j - 1}(\xi)_{\psi,\;\omega}\cdot e_\lambda^{c_j} \right)\cdot e_\nu^{k_j} \right)\text{,}
\end{align*}
where ${\psi := \div{i - 1}{k_{j - 1}^*} + k_j(\mu - 1) + \nu + \lambda - 1\in\N_1}$ and ${\omega := \rem{i - 1}{k_{j - 1}^*} + 1\in\N_1}$ are abbreviations.

Let $\pi := \EvalSlide_{j - 1}(\xi)\in M_{j - 1}^{U_{j - 1}^{\row} \times U_{j - 1}^{\col}}$ be an abbreviation for the analysis of the right-hand side of the claim.
Now
\begin{align*}
  & \EvalSlide_j(\xi)_{\div{i - 1}{k_j^*} + \mu,\;\rem{i - 1}{k_j^*} + 1}\\
  \equsing{D.~\ref{def:processing-chain}}\ \ &\Fragmentation_{k_j}(\Slide_{g_j}(T_j(\pi))_{\div{i - 1}{k_j^*} + \mu,\;\rem{i - 1}{k_j^*} + 1}\\
  \equsing{L.~\ref{lem:frag-ops}}\ \ & \Slide_{g_j}(T_j(\pi))_{\div{\phi}{k_{j - 1}^*} + 1,\;\rem{\phi}{k_{j - 1}^*} + 1}\text{,}
\end{align*}
where the number of input fragments to $\Fragmentation_{k_j}$ was $k_{j - 1}^*$, as already shown in~\ref{lem:processing-chain-b}, and where it has been defined that
$\phi := \left(\div{i - 1}{k_j^*} + \mu - 1\right)k_jk_{j - 1}^* + \rem{i - 1}{k_j^*}\in\N$.
By the definition of the operators from Euclidean division follows that $\phi = i - 1 + (\mu - 1)k_jk_{j - 1}^*$.
Proposition~\ref{prop:number-theory} implies
$\div{\phi}{k_{j - 1}^*} = \div{i - 1}{k_{j - 1}^*} + k_j(\mu - 1)$ and $\rem{\phi}{k_{j - 1}^*} + 1 = \rem{i - 1}{k_{j - 1}^*} + 1 = \omega$.
Hence
\begin{align*}
  & \EvalSlide_j(\xi)_{\div{i - 1}{k_j^*} + \mu,\;\rem{i - 1}{k_j^*} + 1}\\
  \equsing{D.~\ref{def:sliding-function}}\ \ & g_j\left( \sum\nolimits_{\nu = 1}^{k_j} T_j(\pi)_{\div{i - 1}{k_{j - 1}^*} + k_j(\mu - 1) + \nu,\;\omega}\cdot e_\nu^{k_j} \right)\\
  \equsing{T.~\ref{thm:sliding-subsignal}}\ \ & g_j\left( \sum\nolimits_{\nu = 1}^{k_j} T_j\left( \sum\nolimits_{\lambda = 1}^{c_j} \pi_{\psi,\;\omega}\cdot e_\lambda^{c_j} \right) \cdot e_\nu^{k_j} \right)\text{,}
\end{align*}
which equals the left-hand side of the claim as shown earlier and thus the proof is finished.
\end{proof}

\begin{figure*}[t]
  \centering
  \scalebox{\figurescalefactor}{\includegraphics[page=5]{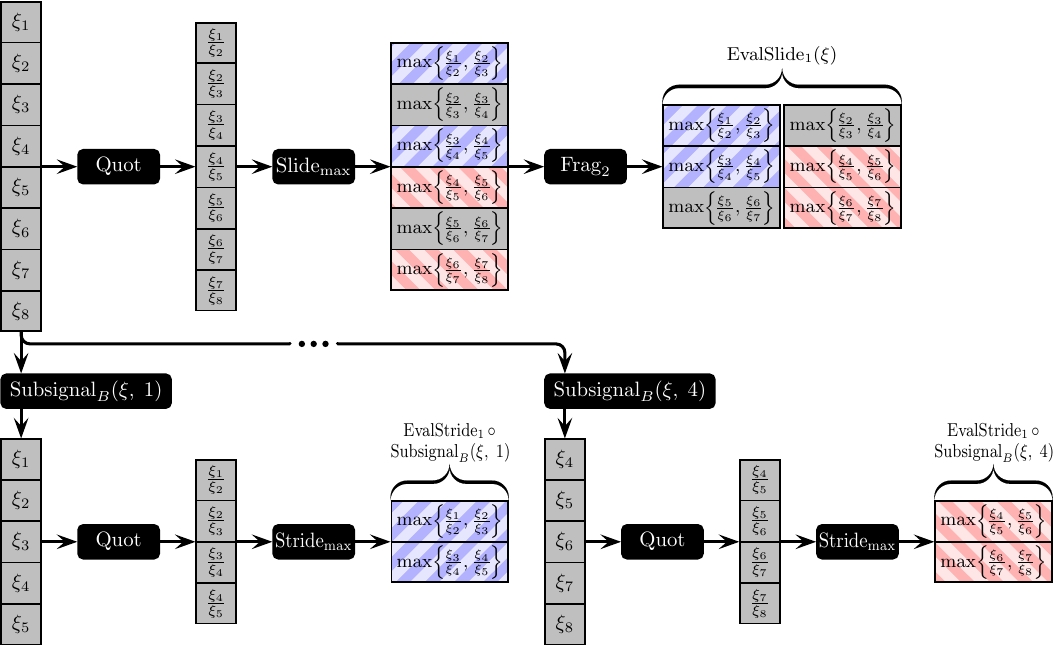}}
  \caption{Example of the two notions of a one-layered ($L = 1$) processing chain with receptive field size $B = 5$ applied to an input signal $\xi$ with $D = 8$ samples.
    The processing chain consists of the $\operatorname{Quot}$ operator as subsignal compatible transformation followed by max-pooling with a stride of $k_1 = 2$.
    In this example, the processing chain is well-defined and all divisibility requirements are met.    
    The upper part shows the result of $\EvalSlide$ evaluated on $\xi$:
    Here, max-pooling is applied to $\operatorname{Quot}(\xi)$ in a sliding fashion, followed by fragmentation producing two fragments.
    The lower part shows the outcome of $\EvalStride$ applied to two different subsignals of length $B$.
    Max-pooling is used here in a strided fashion which halves dimensionality.
    It is evident that the outcome of $\EvalStride$ can be found within $\Slide_{\max}\left(\operatorname{Quot}(\xi)\right)$, although the individual output signals are interleaved (highlighted with a colored pattern in the graphics).
    The reordering through fragmentation corrects interleaving so that the outcome is always available contiguously.
    A second layer in the processing chain has then the ability to further process independent contiguous signals, which is particularly effective on real computing machines.}
  \label{fig:evalchain-example}
\end{figure*}

In conclusion, the result of the $\EvalStride$ operator applied to an arbitrary subsignal of an input signal emerges contiguously in the result of the $\EvalSlide$ operator which processes the signal in its entirety.
The concrete position in the fragmentation data structure can be determined with Lemma~\ref{lem:processing-chain}\ref{lem:processing-chain-d}.
An example for this is depicted in Fig.~\ref{fig:evalchain-example}.
It will be shown later how defragmentation can be used eventually to restore the expected order of the resulting samples.
It is further intuitively clear that $\EvalSlide$ is much more efficient than $\EvalStride$ since redundant computations are avoided.
This is analyzed rigorously in Sect.~\ref{sect:computational_complexity}.

Before continuing with the theory, a short discussion of two notable special cases.
First, a processing chain is of course not required to have a pooling layer following every subsignal compatible transformation, or to even have pooling layers at all.
Consider a fixed layer index $j\inint{1}{L}$.
By setting $k_j := 1$ and $g_j\colon N_j\to M_j$, $\tau\mapsto\tau$, where $N_j = M_j$, one sees that $\Stride_{g_j}$, $\Slide_{g_j}$ and $\Fragmentation_{k_j}$ are the identity functions on their respective domains.
In other words, this parameterization of pooling layer $j$ causes it to act like a neutral bypass operation.
If, on the other hand, one would want to have two pooling layers one directly after the other, it is completely analogous to achieve a neutral subsignal compatible transformation.

The other special case is to have a convolution which is evaluated in a non-overlapping manner, or equivalently in a strided fashion.
This can be achieved by plugging $f_{\convop}$ from Sect.~\ref{sect:CNNs-wo-pooling} into a pooling kernel within a processing chain.
Non-overlapping convolution has, however, no advantage in computational complexity if entire input signals should be processed using a sliding window approach without accuracy loss:
Lemma~\ref{lem:processing-chain} states that strided fashion has then to be turned into sliding fashion, which is essentially the same as carrying out convolution conventionally by advancing the filter banks exactly one sample after each evaluation.

\subsection{Defragmentation and Arbitrary Input Signal Length}
Lemma~\ref{lem:processing-chain} requires the length of the input signal to satisfy certain divisibility constraints.
Extension of its statements to signals of arbitrary length requires two additional operators:
\begin{definition}
\label{def:stuff-trim}
Let $r\in\N$ be a natural number, let $M$ be a set and $\zeta\in M$ be an arbitrary dummy element from $M$.
Then $\Stuffing_r\colon\cup_1(M) \to \cup_{r + 1}(M)$,
\begin{displaymath}
  (\xi_1,\dotsc,\xi_q)\mapsto \sum\nolimits_{\nu = 1}^q \xi_\nu\cdot e_\nu^{q + r} + \sum\nolimits_{\nu = 1}^r \zeta\cdot e_{q + \nu}^{q + r}\text{,}
\end{displaymath}
is called the \emph{stuffing operator}, which appends $r$ copies of $\zeta$ to its argument.
Further, $\Trimming_r\colon\cup_{r + 1}(M) \to \cup_1(M)$,
\begin{displaymath}
  (\xi_1,\dotsc,\xi_q,\xi_{q + 1},\dotsc,\xi_{q + r}) \mapsto \sum\nolimits_{\nu = 1}^q \xi_\nu\cdot e_\nu^q\text{,}
\end{displaymath}
is called the \emph{trimming operator}, which removes the final $r$ entries from its argument.
\end{definition}

The concrete choice of the dummy element $\zeta$ does not matter in the following considerations since all output entries which are affected by its choice are trimmed away in the end.
It is now possible to state the main theoretical result of this section:
\begin{theorem}
\label{thm:processing-chain}
Consider a processing chain with the same notation as in Definition~\ref{def:processing-chain}, where the application in a strided fashion is well-defined as in Lemma~\ref{lem:processing-chain}.
Further assume that $\dim_{M_L}(\EvalStride_L(\rho)) = 1$ for all $\rho\in M_0^\ROI$, that is the output of the entire processing chain applied in a strided fashion consists of exactly one sample.

Let $\tilde{r}\colon\N_1\to\discint{0}{k_L^* - 1}$, 
\begin{displaymath}
  \delta\mapsto
  \begin{cases}
    0\text{,} & \!\!\!\text{if }k_L^*\text{ divides }\delta - \ROI + 1\text{,}\\
    k_L^* - \rem{\delta - \ROI + 1}{k_L^*}\text{, } & \!\!\!\text{otherwise,}
  \end{cases}
\end{displaymath}
denote the number of dummy samples that have to be padded to an original signal with $\delta$ samples to satisfy divisibility requirements.
Further define $r\colon\cup_1(M_0)\to\discint{0}{k_L^* - 1}$, $\xi\mapsto\tilde{r}(\dim_{M_0}(\xi))$, as an abbreviation that computes the required number of dummy samples in dependence on a concrete original signal $\xi$.
Consider $T\colon\cup_\ROI(M_0)\to\cup_1(M_L)$,
\begin{displaymath}
  \xi\mapsto\Trimming_{r(\xi)}( \Defragmentation_{k_L^*}( \EvalSlide_L( \Stuffing_{r(\xi)}(\xi) ) ) )\text{.}
\end{displaymath}
This function first stuffs the input signal with as many dummy samples such that each fragmentation operation during application of the processing chain in a sliding fashion comes out even, applies the processing chain in a sliding fashion, defragments the outcome and eventually removes all superfluous entries that emerged from the initial stuffing.
Then $T$ is a subsignal compatible transformation with dimensionality reduction constant $\ROI$.
Furthermore, $T(\rho) = \EvalStride_L(\rho)$ for all $\rho\in M_0^\ROI$ and $T = \Slide_{\EvalStride_L}$.
\end{theorem}
\begin{proof}
Note that $\tilde{r}(\delta)$ is well-defined if $k_L^*$ does not divide $\delta - \ROI + 1$, since then $\rem{\delta - \ROI + 1}{k_L^*}\inint{1}{k_L^* - 1}$ and thus $k_L^* - \rem{\delta - \ROI + 1}{k_L^*}\inint{1}{k_L^* - 1}$.

Let $\xi\in\cup_\ROI(M_0)$ and write $\tilde{D} := \dim_{M_0}(\xi)$.
Then $D := \dim_{M_0}(\Stuffing_{r(\xi)}(\xi)) = \tilde{D} + r(\xi)$.
If $k_L^*$ divides $\tilde{D} - \ROI + 1$ it is $r(\xi) = 0$, so $\rem{ \dim_{M_0}(\Stuffing_{r(\xi)}(\xi)) - \ROI + 1 }{k_L^*} = 0$.
If on the other hand $k_L^*$ does not divide $\tilde{D} - \ROI + 1$, then $r(\xi) > 0$ and
$\dim_{M_0}(\Stuffing_{r(\xi)}(\xi)) - \ROI + 1 = \tilde{D} + k_L^* - \rem{\tilde{D} - \ROI + 1}{k_L^*} - \ROI + 1$.
Hence $\rem{ \dim_{M_0}(\Stuffing_{r(\xi)}(\xi)) - \ROI + 1 }{k_L^*} = 0$ due to the idempotence of the $\operatorname{rem}$ operator.
Therefore, the number of subsignals of $\Stuffing_{r(\xi)}(\xi)$ with $\ROI$ samples is always divisible by $k_L^*$, as required for application of Lemma~\ref{lem:processing-chain}.

Lemma~\ref{lem:processing-chain} guarantees that $\pi := \EvalSlide_L( \Stuffing_{r(\xi)}(\xi) )$ is well-defined.
With Lemma~\ref{lem:processing-chain}\ref{lem:processing-chain-b} follows $\cdim_{M_L}(\pi) = k_L^*$.
Since $\dim_{M_L}(\EvalStride_L(\rho)) = 1$ was required for all $\rho\in M_0^\ROI$, $\rdim_{M_L}(\pi) = \frac{1}{k_L^*}(D - \ROI + 1)$ with Lemma~\ref{lem:processing-chain}\ref{lem:processing-chain-c}.
Therefore, $\Defragmentation_{k_L^*}(\pi)$ is well-defined with exactly one output fragment, which has the same number of samples as there are subsignals of length $B$ in the stuffed input signal: $\dim_{M_L}(\Defragmentation_{k_L^*}(\pi)) = D - \ROI + 1$.

Now $\dim_{M_L}(\Defragmentation_{k_L^*}(\pi)) \geq r(\xi) + 1$ holds because $D = \tilde{D} + r(\xi)$ where $\tilde{D} \geq \ROI$, thus $\Trimming_{r(\xi)}(\Defragmentation_{k_L^*}(\pi))$ is well-defined.
Since trimming reduces dimensionality by $r(\xi)$ follows
$\dim_{M_0}(\xi) - \dim_{M_L}(T(\xi)) = \tilde{D} - (D - \ROI + 1 - r(\xi)) = \ROI - 1$.
Therefore, $T$ fulfills the dimensionality reduction property with dimensionality reduction constant $\ROI$.

To prove $T$ is subsignal compatible it is hence sufficient to use Theorem~\ref{thm:sliding-subsignal} to show $T(\Subsignal_\ROI(\xi,\;i)) = T(\xi)_i\in M_L$ for all $i\inint{1}{\tilde{D} - \ROI + 1}$.
It is first demonstrated that $T(\rho) = \EvalStride_L(\rho)$ for all $\rho\in M_0^\ROI$.
This can then be used to prove the weakened exchange property.

Let $\rho\in M_0^\ROI$, then $\Subsignal_\ROI(\Stuffing_{r(\rho)}(\rho),\;1) = \rho$ by the definition of the stuffing operator.
$T(\rho)$ consists of a single sample since the dimensionality reduction constant of $T$ is $\ROI$, hence $T(\rho) = T(\rho)_1$.
Extraction of the very first sample of the result of the trimming operator is equal here to the extraction of the very first sample of the trimming operator's argument.
Therefore,
\begin{align*}
  & T(\rho)\\
  =\ \ \ &\Defragmentation_{k_L^*}( \EvalSlide_L( \Stuffing_{r(\rho)}(\rho) ) )_{1,\;1}\\
  \equsing{L.~\ref{lem:defrag-ops}}\ \ \ &\EvalSlide_L( \Stuffing_{r(\rho)}(\rho) )_{\div{0}{k_L^*} + 1,\;\rem{0}{k_L^*} + 1}\\
  \equsing{L.~\ref{lem:processing-chain}\ref{lem:processing-chain-d}}\ \ \ &\EvalStride_L(\Subsignal_\ROI(\Stuffing_{r(\rho)}(\rho),\; 1))_1\\
  =\ \ \ &\EvalStride_L(\rho)\text{.}
\end{align*}

Returning to the exchange property, let $\xi\in M_0^{\tilde{D}}$ as before, and let $i\inint{1}{\tilde{D} - \ROI + 1}$ be an arbitrary subsignal index.
Omitting the trimming operator as before yields
\begin{align*}
  & T(\xi)_i\\
  =\ \ \ & \Defragmentation_{k_L^*}( \EvalSlide_L( \Stuffing_{r(\xi)}(\xi) ) )_{i,\;1}\\
  \equsing{L.~\ref{lem:defrag-ops}}\ \ \ & \EvalSlide_L( \Stuffing_{r(\xi)}(\xi) )_{\div{i - 1}{k_L^*} + 1,\;\rem{i - 1}{k_L^*} + 1}\\
  \equsing{L.~\ref{lem:processing-chain}\ref{lem:processing-chain-d}}\ \ \ & \EvalStride_L(\Subsignal_\ROI(\Stuffing_{r(\xi)}(\xi),\; i))_1\\
  \equsing{D.~\ref{def:stuff-trim}}\ \ \ & \EvalStride_L(\Subsignal_\ROI(\xi,\; i))\\
  =\ \ \ & T(\Subsignal_\ROI(\xi,\;i))\text{.}
\end{align*}
Hence $T$ is a subsignal compatible transformation due to Theorem~\ref{thm:sliding-subsignal}.
Theorem~\ref{thm:subsignal-identity} finally implies $T = \Slide_{\EvalStride_L}$.
\end{proof}

It can be concluded that CNNs can be turned into efficiently computable subsignal compatible transformations using the $\EvalSlide$ operator regardless of the input signal's dimensionality.
One could suspect that stuffing the input signal with dummy samples might have a negative effect on the efficiency.
However, the number of stuffed samples is always less than the stride product $k_L^*$ of the final layer and hence very small for reasonably sized CNNs.

Moreover, stuffing guarantees that all fragments encountered during evaluation are homogeneous.
This enables tensors to be used as the sole data structure for input data, intermediate representations and computation results.
This is far more efficient than storing each fragment individually, especially on massively parallel processors where then simple parallelized implementations can achieve maximum throughput.

\section{Computational Complexity Analysis}
\label{sect:computational_complexity}
In this section, a detailed theoretical analysis of the computational complexity of processing chain evaluation is carried out.
As introduced in Definition~\ref{def:processing-chain}, this corresponds to the alternating application of subsignal compatible transformations and functions applied in a strided fashion.
For measuring the computational complexity of the $\EvalStride$ and $\EvalSlide$ operators, the function evaluations required for computing the output of each layer are counted.
Regarding the subsignal compatible transformations, the unique functions that generate the transformations on account of Theorem~\ref{thm:sliding-subsignal} are considered.

It is shown that $\EvalSlide$ requires at most the same number of function evaluations as $\EvalStride$ in each layer.
This then implies that $\EvalSlide$ is more efficient on the global scale of an entire processing chain where all the layers are evaluated subsequently.
Further, it is shown that the theoretical speedup can be factorized into simple expressions, facilitating statements on the effect of individual parameters.

\subsection{Identification of the Number of Function Evaluations}
Assume a situation of Lemma~\ref{lem:processing-chain} in which an input signal $\xi\in M_0^D$ and a well-defined processing chain are given.
An arbitrary layer index $j\inint{1}{L}$ is fixed and $\EvalStride_j$ is analyzed for an arbitrary subsignal index $i\inint{1}{D - B + 1}$.
As in the proof of Lemma~\ref{lem:processing-chain}\ref{lem:processing-chain-d}, write $\tau := \EvalStride_{j - 1}(\Subsignal_\ROI(\xi,\;i))\in M_{j - 1}^{u_{j - 1}}$ which results in $\EvalStride_j(\Subsignal_\ROI(\xi,\;i)) = \Stride_{g_j}(T_j(\tau))$.
Now Theorem~\ref{thm:sliding-subsignal} guarantees that there is exactly one function $f_j\colon M_{j - 1}^{c_j}\to N_j$ satisfying $T_j = \Slide_{f_j}$.
Thus
\begin{displaymath}
  T_j(\tau)\ \ \equsing{D.~\ref{def:sliding-function}}\ \ \sum\nolimits_{\mu = 1}^{u_{j - 1} - c_j + 1} f_j(\Subsignal_{c_j}(\tau,\;\mu))\cdot e_\mu^{u_{j - 1} - c_j + 1}\text{,}
\end{displaymath}
hence $f_j$ has to be evaluated $u_{j - 1} - c_j + 1$ times.
Considering the strided application of $g_j$, one notes $u_{j - 1} - c_j + 1 = k_j u_j$ with Lemma~\ref{lem:processing-chain}\ref{lem:processing-chain-a} and furthermore
\begin{align*}
     & \Stride_{g_j}(T_j(\tau))\\
  \equsing{D.~\ref{def:strided-function}}\ \ &\sum\nolimits_{\nu = 1}^{u_j} g_j(\Subsignal_{k_j}(T_j(\tau),\;k_j(\nu - 1) + 1))\cdot e_\nu^{u_j}\text{.}
\end{align*}
Here, $u_j$ evaluations of $g_j$ are necessary.
Since all function evaluations have to be carried out for each of the $D - B + 1$ possible subsignals, the total number of function evaluations increases proportionately with this factor.

Redundant computations are avoided if $\EvalSlide_j$ is used instead.
Here the complexity of processing an individual fragment in layer $j$ is analyzed.
The overall complexity then results from multiplication with the number of input fragments $U_{j - 1}^{\col}$.
Analogous to the proof of Lemma~\ref{lem:processing-chain}\ref{lem:processing-chain-d}, let $\gamma\inint{1}{U_{j - 1}^{\col}}$ be a fragment index and define
\begin{displaymath}
  \pi := \sum\nolimits_{\kappa = 1}^{U_{j - 1}^{\row}} \EvalSlide_{j - 1}(\xi)_{\kappa,\;\gamma}\cdot e_{\kappa}^{U_{j - 1}^{\row}}\in M_{j - 1}^{U_{j - 1}^{\row}}
\end{displaymath}
as an abbreviation for the input fragment with index $\gamma$.
Now, the output of the $j$-th layer of the processing chain is
\begin{displaymath}
  \EvalSlide_j(\xi) = \Fragmentation_{k_j}(\Slide_{g_j}(T_j(\EvalSlide_{j - 1}(\xi))))\text{.}
\end{displaymath}
The complexity of the fragmentation operator is neglected here because it is merely a structured permutation with very little overhead.
Considering a single fragment now yields
\begin{displaymath}
  T_j(\pi)\ \equsing{D.~\ref{def:sliding-function}}\ \sum\nolimits_{\mu = 1}^{U_{j - 1}^{\row} - c_j + 1} f_j(\Subsignal_{c_j}(\pi,\;\mu))\cdot e_\mu^{U_{j - 1}^{\row} - c_j + 1}\text{,}
\end{displaymath}
accounting for $U_{j - 1}^{\row} - c_j + 1$ evaluations of $f_j$.
Application of $g_j$ in a sliding fashion yields
\begin{displaymath}
  \Slide_{g_j}(T_j(\pi))\ \ \equsing{D.~\ref{def:sliding-function}}\ \ \sum\nolimits_{\nu = 1}^{\chi} g_j(\Subsignal_{k_j}(T_j(\pi),\;\nu))\cdot e_\nu^{\chi}\text{,}
\end{displaymath}
which requires $\chi := U_{j - 1}^{\row} - c_j - k_j + 2$ evaluations of $g_j$.
Lemma~\ref{lem:processing-chain}\ref{lem:processing-chain-b} finally implies $\chi = k_j U_j^{\row}$.

\subsection{Analysis for the Subsignal Compatible Transformation Evaluation Component}
To determine the resulting speedup when redundant computations are avoided, the ratio of the number of function evaluations required for the naive approach $\EvalStride_j$ to the number needed when the input signal is processed in its entirety using $\EvalSlide_j$ is evaluated.
Considering $f_j$ this yields
\begin{displaymath}
  S_{f_j} := \frac{(D - B + 1) (u_{j - 1} - c_j + 1)}{U_{j - 1}^{\col} (U_{j - 1}^{\row} - c_j + 1)}\text{,}
\end{displaymath}
where the number of subsignals and the number of fragments was included in the numerator and denominator, respectively.
With Lemma~\ref{lem:processing-chain}\ref{lem:processing-chain-b} and Lemma~\ref{lem:processing-chain}\ref{lem:processing-chain-c} one obtains
$\frac{D - B + 1}{U_{j - 1}^{\col}} = U_{j - 1}^{\row} - u_{j - 1} + 1$,
and by substituting this and after minor algebraic manipulation one sees that
\begin{displaymath}
  S_{f_j} = 1 + \frac{(U_{j - 1}^{\row} - u_{j - 1}) (u_{j - 1} - c_j)}{U_{j - 1}^{\row} - c_j + 1}\text{.}
\end{displaymath}
Since $1\leq c_j\leq u_{j - 1}\leq U_{j - 1}^{\row}$ using Lemma~\ref{lem:processing-chain}\ref{lem:processing-chain-c} it follows that $S_{f_j} \geq 1$, which means $\EvalSlide_j$ requires at most the same number of applications of $f_j$ as $\EvalStride_j$.
Merely in the special cases where the extent $u_{j - 1}$ of the region fed into layer $j$ equals the length of the signal fragments ($u_{j - 1} = U_{j - 1}^{\row}$) or the dimensionality reduction constant of the subsignal compatible transformation ($u_{j - 1} = c_j$) the speedup attains unity, indicating that both approaches require the same number of function evaluations.

If $S_{f_j}$ is understood as a function dependent upon the signal dimensionality $D$, then the only quantity that depends on $D$ in the derived expression is $U_{j - 1}^{\row}$ since $c_j$ is constant and $u_{j - 1}$ is independent of $D$ as can be seen from Lemma~\ref{lem:processing-chain}\ref{lem:processing-chain-a}.
As Lemma~\ref{lem:processing-chain} requires $D - B + 1$ to be divisible by $k_L^*$, the next larger feasible signal dimensionality is $D + k_L^* =: D_+$.
Subtracting $S_{f_j}$ evaluated for signal dimensionality $D$ from its value for an extended signal dimensionality $D_+$ yields
\begin{align*}
     &S_{f_j}(D_+) - S_{f_j}(D)\\[.5ex]
  =\ &\frac{(u_{j - 1} - c_j) (u_{j - 1} - c_j + 1) (U_{j - 1}^{\row}(D_+) - U_{j - 1}^{\row}(D))}{(U_{j - 1}^{\row}(D_+) - c_j + 1) (U_{j - 1}^{\row}(D) - c_j + 1)}\text{.}
\end{align*}
Now Lemma~\ref{lem:processing-chain}\ref{lem:processing-chain-b} implies that $U_{j - 1}^{\row}(D_+) - U_{j - 1}^{\row}(D) = \frac{k_L^*}{k_{j - 1}^*}$, therefore $S_{f_j}(D_+) \geq S_{f_j}(D)$ and thus the speedup increases if the signal dimensionality is increased.
In the limit case of arbitrarily large input signals one obtains $\lim_{D\to\infty} S_{f_j}(D) = u_{j - 1} - c_j + 1$, that is the speedup asymptotically attains a finite value.
From this it is evident that greatest speedups can be achieved for large regions of interest $u_{j - 1}$ and small dimensionality reduction constants $c_j$.

\subsection{Analysis for the Strided Function Evaluation Component}
Turning now to the function $g_j$ applied in a strided fashion, incorporating the number of subsignals and fragments for $\EvalStride_j$ and $\EvalSlide_j$, respectively, yields the ratio
\begin{displaymath}
  S_{g_j} := \frac{(D - B + 1) u_j}{U_{j - 1}^{\col}k_j U_j^{\row}} = 1 + \frac{(U_j^{\row} - u_j) (u_j - 1)}{U_j^{\row}}\text{.}
\end{displaymath}
With Lemma~\ref{lem:processing-chain}\ref{lem:processing-chain-c} follows $S_{g_j} \geq 1$, hence here the complexity of $\EvalSlide_j$ is also less than that of $\EvalStride_j$.
The speedup is unity only if $u_j = U_j^{\row}$ or $u_j = 1$.

Grasping $S_{g_j}$ as a function of $D$ and denoting the next larger signal dimensionality with $D_+ := D + k_L^*$, one obtains
\begin{displaymath}
  S_{g_j}(D_+) - S_{g_j}(D) = \frac{u_j (u_j \! - \! 1) (U_j^{\row}(D_+) \! - \! U_j^{\row}(D))}{U_j^{\row}(D_+) U_j^{\row}(D)} \geq 0\text{.}
\end{displaymath}
Thus speedups increase with larger signal dimensionality.
In the limit case of arbitrarily large input signals it is $\lim_{D\to\infty} S_{g_j}(D) = u_j$, hence here the speedup is bounded from above as well.

\subsection{Discussion}
Simple expressions have been derived that imply the number of function evaluations required by $\EvalSlide$ is always less or equal than that for $\EvalStride$ in each layer.
Therefore this also holds for the subsequent evaluation of all the layers.
Further, the special cases have been identified in which the computational complexity of $\EvalSlide$ matches that of $\EvalStride$.
Moreover, it was demonstrated that the speedup becomes more significant for larger input signals, although its growth is not unbounded.

The analysis was restricted to only the amount of necessary function evaluations and neglected the parallelization potential of the individual approaches.
On a massively parallel processor, the throughput of the $\EvalSlide$ approach might be substantially lower than that of $\EvalStride$ since here coarse-grained parallelism on the subsignal level can be exploited facilitating load balancing on thousands of parallel computing cores.
However, experiments in the next section demonstrate that, as predicted by the theoretical analysis, in practice $\EvalSlide$ is orders of magnitude faster than $\EvalStride$.

\section{Practical Considerations for Image Processing and Experimental Evaluation}
\label{sect:experimental_evaluation}
This section discusses practical considerations when using the theory proposed in this paper for image processing tasks.
Further, results of experiments on semantic segmentation and runtime measurements on real processors are reported.

\subsection{Generalization to 2D Image Data and Transformation for Image-Based Application}
\label{sect:generalization_image_data}
In the context of image processing, two-dimensional signals are referred to as \emph{images} and two-dimensional subsignals as \emph{patches}.
The generalization of subsignal compatible transformation theory to two spatial dimensions is straightforward and shown here exemplarily for Definition~\ref{def:subsignal_compatible}.
Images are represented by matrices, patch indices are two-dimensional and patch extraction is the same as forming a submatrix of adjacent entries.
Now suppose $T$ is a transformation from the space of sufficiently large images $\xi$ with pixels in the set $M$ to images with pixels from the set $N$.
Then $T$ fulfills the dimensionality reduction property with dimensionality reduction constants $r,c\in\N_1$ if both $\rdim_N(T(\xi)) = \rdim_M(\xi) - r + 1$ and $\cdim_N(T(\xi)) = \cdim_M(\xi) - c + 1$.
The exchange property generalizes to $T(\Patch_{d_r\times d_c}(\xi,\;i,\;j)) = \Patch_{(d_r - r + 1)\times(d_c - c + 1)}(T(\xi),\;i,\;j)$, which is a condition in two dimensions.
Here, $\Patch$ is the patch extraction operator with subscripts specifying the dimensionalities of the extracted patches, $d_r,d_c\in\N_1$, $d_r\geq r$, $d_c\geq c$, are patch dimensionalities and $i\inint{1}{\rdim_M(\xi) - d_r + 1}$ and $j\inint{1}{\cdim_M(\xi) - d_c + 1}$ are patch indices.
If both properties are fulfilled, $T$ is called a \emph{patch compatible transformation}.

The remaining theory can be generalized analogously.
In practice it is sufficient to use plain 4D tensors as the only data structure for CNN input data, intermediate representations and computation results.
Here, one dimension accounts for the feature map index and two dimensions account for the spatial position within each feature map.
The fourth dimension represents an image index.
Although image processing fragmentation requires two dimensions, these can be collapsed into one dimension by linearizing the two-dimensional fragment indices.
Since all computations on fragments are carried out independently by Definition~\ref{def:fragmentwise-evaluation}, this essentially corresponds to the meaning of the image index dimension in common 4D tensor processing~\cite{Chetlur2014}.
Therefore, fragmentation does not require any modifications whatsoever to the computationally demanding routines such as tensor convolution.

\begin{figure}[t]
  \centering
  \includegraphics[page=6]{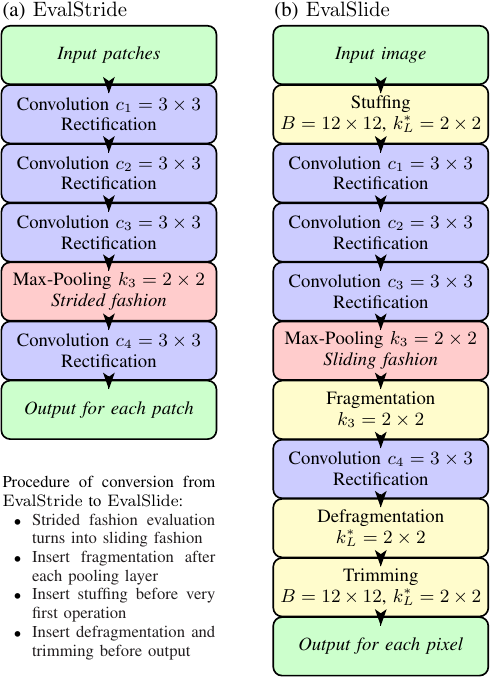}
  \caption{Example network architecture for $L = 4$ targeted at the processing of two-dimensional images.
    For $\EvalStride$ (left-hand side) suitable for patch-based application, this architecture consists of four $3\times 3$ convolutional layers, where after the third convolutional layer $2\times 2$ max-pooling is carried out in a strided fashion.
    The right-hand side shows how this architecture has been transformed for the $\EvalSlide$ approach suitable for image-based application:
    Pooling is now applied in a sliding fashion and followed by fragmentation.
    The input image is stuffed to guarantee fragmentation always comes out even and hence homogeneous 4D tensors can be used as the only data structure.
    The effects of stuffing and fragmentation are undone in the end, yielding an output for each pixel in the input image.
    Here, $B$ denotes the entire CNN's receptive field size and $k_L^*$ is the stride product of the entire network.}
  \label{fig:evalslide-net}
\end{figure}

For an experimental evaluation, CNNs with a varying number of convolutional layers $L$ were created using the following scheme.
Similar to~\cite{Simonyan2015}, each convolutional layer was parameterized for a filter size of $3\times 3$~pixels.
The output of each convolutional layer was fed through a rectification non-linearity~\cite{Sanger1989a}.
A $2\times 2$ max-pooling layer was inserted after each three pairs of convolution and rectification layers, unless the pooling layer would be the final layer of the network~\cite{Simonyan2015}.
Fig.~\ref{fig:evalslide-net} depicts such a network architecture as it was used for the $\EvalStride$ operator and how it was transformed to account for image-based application using the $\EvalSlide$ operator (see Definition~\ref{def:processing-chain} for the formal recipe).

\subsection{Semantic Image Segmentation}
The practicality of the approach proposed in this paper has been verified by realizing a semantic image segmentation through evaluation of a classifier on all feasible patches in an image~\cite{Ning2005,Grangier2009,Farabet2013}.
In doing so, images recorded with a wide angle camera attached to the windshield of an experimental vehicle were manually labeled to yield regions which only contain pixels of the four object categories road, vehicle, person, and background.
As a classifier, a CNN with $L = 12$ convolutional layers as described above was used.
The number of output feature maps was set to $16$ for the first three convolutional layers and subsequently doubled after each pooling layer.
A fully-connected layer (spatial dimension $1\times 1$) with four output feature maps was appended to project from the high-dimensional feature space into the label space, followed by a final softmax non-linearity~\cite{Bishop1995}.

The CNN was first trained with patches extracted from random positions of the images from the learning set ($\EvalStride$ notion), transformed to image-based application, and thereafter fine-tuned using entire images ($\EvalSlide$ notion).
Therefore, a huge number of weight updates on unbiased learning examples was carried out in the early phase of training, facilitating fast learning progress~\cite{Bottou2004}.
After the transformation, learning on entire images ensured all feasible patches were considered, improving homogeneity and reducing remaining misclassification artifacts.
Note that backpropagating gradients through a transformed CNN is straightforward:
The gradients of stuffing and trimming are trivial, and fragmentation and defragmentation are merely inversely related permutations.
Elementary calculus is sufficient for determination of the gradient of max-pooling in a sliding fashion.

\begin{figure}[t]
  \centering
  \includegraphics[page=7]{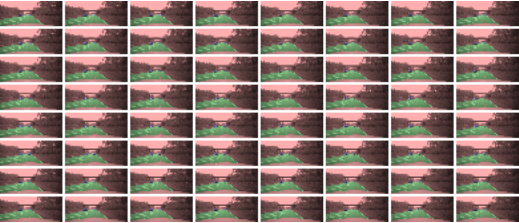}
  \caption{Visualization of the fragmented classification decisions. Each fragment carries only low-resolution information, which is afterwards combined with a defragmentation operation (see Fig.~\ref{fig:tinted} for the result).}
  \label{fig:fragmented-tinted}
\end{figure}

\begin{figure}[t]
  \centering
  \includegraphics[page=8]{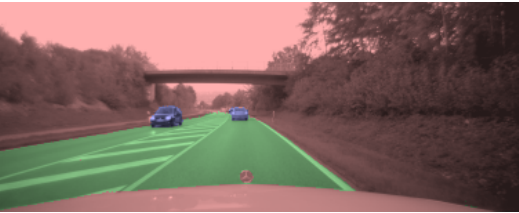}
  \caption{Visualization of the final output of the classifier (best viewed in color). The grayscale input image has been tinted with the defragmented classification decisions. The information obtained facilitates an accurate representation of the scene in front of the vehicle.}
  \label{fig:tinted}
\end{figure}

The classification decisions on an image from the test set just before defragmentation is carried out are shown in Fig.~\ref{fig:fragmented-tinted}.
Since for $L = 12$ three $2\times 2$ pooling layers are involved, accounting for a stride product of $k_L^* = 8\times 8$, there are $64$~fragments in total.
Each of the fragments represents a low-resolution version of the final output shifted by the corresponding number of pixels in either spatial dimension.
Defragmentation yields a single output image, see Fig.~\ref{fig:tinted}, where the classification decisions are available at high resolution.
This output can subsequently be employed for vehicle environment perception~\cite{Nuss2014}.
It is however more efficient to additionally employ multi-scale analysis which incorporates context information and hence provides more discriminant features~\cite{Farabet2013}.
This technique was not considered in this paper due to space constraints.
An elaborate discussion of its theoretical background is available in the technical report~\cite{Thom2016}.

\subsection{Runtime Measurements and Speedup Verification}
Section~\ref{sect:computational_complexity} analyzed the theoretical speedup that can be expected if redundant computations are eliminated by means of the $\EvalSlide$ operator.
To confirm whether speed can be significantly increased when real processors are employed, CNNs were applied to entire two-dimensional images using patch-based and image-based application.
The CNNs were parameterized as described above, where the number of layers $L\inint{1}{15}$ was varied and where the number of output feature maps was always set to $128$.
This renders the computational complexity of each layer about equal, allowing an undistorted evaluation of the overall speedup.
All degrees of freedom of the CNNs were initialized with random numbers and random image data was used as input.
This is no restriction to the generality or practicality of the results since the focus was on an analysis of runtime measurements and an assessment of the achieved speedups.

The experiments were run on a massively parallel GPU implementation and on a parallel CPU implementation.
For the GPU variant, an NVIDIA GeForce GTX 980 Ti graphics card was used.
Computations where sped up through the cuDNN~\cite{Chetlur2014} software library.
The CPU implementation employed an Intel Core i7-5930K processor.
Here, the Intel Integrated Performance Primitives~\cite{Taylor2007} and OpenMP~\cite{Dagum1998} libraries were used for increasing efficiency.

The images had a height of $240$~pixels, a width of $320$~pixels and a single feature map was used as input for the networks.
The time required for carrying out both operators on GPU and on CPU was measured and the ratio taken to determine the speedup.
All measurements were repeated twenty times on twenty distinct input images, and the average of the resulting four hundred runs was used for further evaluation.

For $\EvalStride$, neither the time required for extracting all feasible patches and storing them in a dedicated tensor nor the time required for assembling the final output tensor were included in the measurements.
Due to memory constraints, the tensor with all the patches had to be broken down into batches before processing on the GPU.
The batch size was maximized with respect to the available GPU memory to ensure maximum throughput of the graphics card.
For $\EvalSlide$, the time required for stuffing, fragmentation, defragmentation and trimming was included in the measurements.
Here, splitting into batches was not necessary since redundancies in overlapping patches are avoided and the memory demands were therefore very low.
Any measured speedups are hence biased in favor of $\EvalStride$ as here only computations but no overhead in organizing data structures were considered.

\begin{figure}[t]
  \centering
  \includegraphics[page=9]{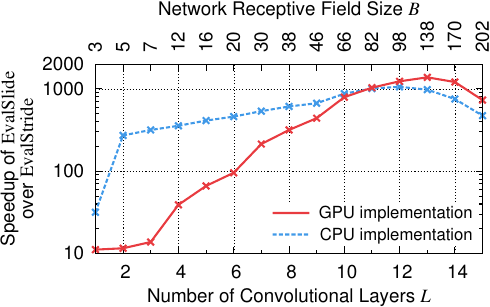}
  \caption{Speedup factors of $\EvalSlide$ over $\EvalStride$ as measured on GPU and on CPU in dependence on the number of convolutional layers $L$ in the networks.
    The second axis on top depicts the resulting receptive field size $B$ of the entire network in either spatial dimension.
    Note that these values are non-linear in $L$ since pooling layers were inserted at regular intervals in the networks.
    The measurements indicate that especially for deep networks, huge speedups can be obtained if images are processed in their entirety.}
  \label{fig:eval-withpool}
\end{figure}

The achieved speedups both for GPU and CPU of $\EvalSlide$ over $\EvalStride$ in dependence on the number of convolutional layers $L$ is depicted in Fig.~\ref{fig:eval-withpool}.
Although the input images were rather small, a notable speedup could be determined even for very shallow networks.
Even for $L = 1$, significant speedups could be measured although here the theoretical number of function evaluations was equal for both operators.
However, there was still a huge memory redundancy in the tensor storing all the patches necessary for $\EvalStride$, which was disadvantageous in terms of memory throughput.
While the CPU implementation achieved speedup factors beyond one hundred for $L\geq 2$, the GPU implementation required deeper networks with $L\geq 7$ for a similar speedup.
This is because of the GPU's superior parallelization capabilities which facilitated an overproportional throughput of the $\EvalStride$ operator where the patches could be processed in parallel.

A peak in the speedup for $L = 13$ and $L = 12$ for GPU and CPU, respectively, was noted.
If deeper networks were used, the relative speedup decreased but remained at a high level.
The reason for this was the large receptive field size $B$ of such deep networks:
Since the patch size almost matched the image dimensions, there were only relatively few patches compared to the situation of a smaller receptive field size.
As predicted by the theoretical results from Sect.~\ref{sect:computational_complexity}, the degree of redundancy of $\EvalStride$ decreases in this case resulting in a decreased relative speedup.

Finally, the execution times of $\EvalSlide$ using the GPU implementation versus the CPU implementation were compared.
Averaging over $L\inint{1}{15}$ yielded a relative speedup of the GPU implementation of factor~$89$ over the CPU implementation, demonstrating the massive parallelization potential.

\section{Conclusions}
\label{sect:conclusions}
This paper introduced and analyzed the concept of subsignal compatible transformations, functions that allow exchanging subsignal extraction with function evaluation without any effect on the outcome.
In doing so, it was demonstrated how CNNs can be applied efficiently without accuracy loss to large signals using a sliding window approach and homogeneous data structures while eliminating redundant computations and special case treatment.
A theoretical analysis has proven the computational complexity of processing an input signal in its entirety is inferior to subsignal-based application, which was subsequently verified through numerical experiments.
All theoretical results have been proven rigorously, mathematically demonstrating the exactness of the proposed approach.
The theoretical framework developed in this paper facilitates further research for gaining deeper insight into related methods for dense signal scanning.

\bibliographystyle{IEEEtran}
\bibliography{IEEEabrv,the}

\begin{thebibliography}{1}
\providecommand{\url}[1]{#1}
\csname url@samestyle\endcsname
\providecommand{\newblock}{\relax}
\providecommand{\bibinfo}[2]{#2}
\providecommand{\BIBentrySTDinterwordspacing}{\spaceskip=0pt\relax}
\providecommand{\BIBentryALTinterwordstretchfactor}{4}
\providecommand{\BIBentryALTinterwordspacing}{\spaceskip=\fontdimen2\font plus
\BIBentryALTinterwordstretchfactor\fontdimen3\font minus
  \fontdimen4\font\relax}
\providecommand{\BIBforeignlanguage}[2]{{%
\expandafter\ifx\csname l@#1\endcsname\relax
\typeout{** WARNING: IEEEtran.bst: No hyphenation pattern has been}%
\typeout{** loaded for the language `#1'. Using the pattern for}%
\typeout{** the default language instead.}%
\else
\language=\csname l@#1\endcsname
\fi
#2}}
\providecommand{\BIBdecl}{\relax}
\BIBdecl

\bibitem{Li2014dil}
H.~Li, R.~Zhao, and X.~Wang, ``{H}ighly {E}fficient {F}orward and {B}ackward
  {P}ropagation of {C}onvolutional {N}eural {N}etworks for {P}ixelwise
  {C}lassification,'' Tech. Rep. arXiv:1412.4526, 2014.

\bibitem{Long2015dil}
J.~Long, E.~Shelhamer, and T.~Darrell, ``{F}ully {C}onvolutional {N}etworks for
  {S}emantic {S}egmentation,'' in \emph{Proceedings of the IEEE Conference on
  Computer Vision and Pattern Recognition}, 2015, pp. 3431--3440.

\bibitem{Tschopp2015dil}
F.~Tschopp, ``{E}fficient {C}onvolutional {N}eural {N}etworks for {P}ixelwise
  {C}lassification on {H}eterogeneous {H}ardware {S}ystems,'' ETH Z\"urich,
  Tech. Rep. arXiv:1509.03371, 2015.

\bibitem{Yu2016dil}
F.~Yu and V.~Koltun, ``{M}ulti-{S}cale {C}ontext {A}ggregation by {D}ilated
  {C}onvolutions,'' in \emph{Proceedings of the International Conference on
  Learning Representations}.\hskip 1em plus 0.5em minus 0.4em\relax
  arXiv:1511.07122, 2016.

\bibitem{Chen2016dil}
L.-C. Chen, G.~Papandreou, I.~Kokkinos, K.~Murphy, and A.~L. Yuille,
  ``{D}eep{L}ab: {S}emantic {I}mage {S}egmentation with {D}eep {C}onvolutional
  {N}ets, {A}trous {C}onvolution, and {F}ully {C}onnected {CRF}s,'' Tech. Rep.
  arXiv:1606.00915, 2016.

\bibitem{Zlateski2016dil}
A.~Zlateski, K.~Lee, and H.~S. Seung, ``{ZNN}$i$ -- {M}aximizing the
  {I}nference {T}hroughput of 3{D} {C}onvolutional {N}etworks on {M}ulti-{C}ore
  {CPU}s and {GPU},'' Tech. Rep. arXiv:1606.05688, 2016.

\end{thebibliography}


\begin{thebibliography}{1}
\providecommand{\url}[1]{#1}
\csname url@samestyle\endcsname
\providecommand{\newblock}{\relax}
\providecommand{\bibinfo}[2]{#2}
\providecommand{\BIBentrySTDinterwordspacing}{\spaceskip=0pt\relax}
\providecommand{\BIBentryALTinterwordstretchfactor}{4}
\providecommand{\BIBentryALTinterwordspacing}{\spaceskip=\fontdimen2\font plus
\BIBentryALTinterwordstretchfactor\fontdimen3\font minus
  \fontdimen4\font\relax}
\providecommand{\BIBforeignlanguage}[2]{{%
\expandafter\ifx\csname l@#1\endcsname\relax
\typeout{** WARNING: IEEEtran.bst: No hyphenation pattern has been}%
\typeout{** loaded for the language `#1'. Using the pattern for}%
\typeout{** the default language instead.}%
\else
\language=\csname l@#1\endcsname
\fi
#2}}
\providecommand{\BIBdecl}{\relax}
\BIBdecl

\bibitem{Farabet2013ms}
C.~Farabet, C.~Couprie, L.~Najman, and Y.~LeCun, ``{L}earning {H}ierarchical
  {F}eatures for {S}cene {L}abeling,'' \emph{IEEE Transactions on Pattern
  Analysis and Machine Intelligence}, vol.~35, no.~8, pp. 1915--1929, 2013.

\end{thebibliography}


\begin{thebibliography}{10}
\providecommand{\url}[1]{#1}
\csname url@samestyle\endcsname
\providecommand{\newblock}{\relax}
\providecommand{\bibinfo}[2]{#2}
\providecommand{\BIBentrySTDinterwordspacing}{\spaceskip=0pt\relax}
\providecommand{\BIBentryALTinterwordstretchfactor}{4}
\providecommand{\BIBentryALTinterwordspacing}{\spaceskip=\fontdimen2\font plus
\BIBentryALTinterwordstretchfactor\fontdimen3\font minus
  \fontdimen4\font\relax}
\providecommand{\BIBforeignlanguage}[2]{{%
\expandafter\ifx\csname l@#1\endcsname\relax
\typeout{** WARNING: IEEEtran.bst: No hyphenation pattern has been}%
\typeout{** loaded for the language `#1'. Using the pattern for}%
\typeout{** the default language instead.}%
\else
\language=\csname l@#1\endcsname
\fi
#2}}
\providecommand{\BIBdecl}{\relax}
\BIBdecl

\bibitem{Hubel1962}
D.~H. Hubel and T.~N. Wiesel, ``{R}eceptive {F}ields, {B}inocular {I}nteraction
  and {F}unctional {A}rchitecture in the {C}at's {V}isual {C}ortex,''
  \emph{Journal of Physiology}, vol. 160, no.~1, pp. 106--154, 1962.

\bibitem{Fukushima1980}
K.~Fukushima, ``{N}eocognitron: {A} {S}elf-{O}rganizing {N}eural {N}etwork
  {M}odel for a {M}echanism of {P}attern {R}ecognition {U}naffected by {S}hift
  in {P}osition,'' \emph{Biological Cybernetics}, vol.~36, no.~4, pp. 193--202,
  1980.

\bibitem{LeCun1990a}
Y.~LeCun, B.~Boser, J.~S. Denker, D.~Henderson, R.~E. Howard, W.~Hubbard, and
  L.~D. Jackel, ``{H}andwritten {D}igit {R}ecognition with a
  {B}ack-{P}ropagation {N}etwork,'' in \emph{Advances in Neural Information
  Processing Systems}, vol.~2, 1990, pp. 396--404.

\bibitem{LeCun1998}
Y.~LeCun, L.~Bottou, Y.~Bengio, and P.~Haffner, ``{G}radient-{B}ased {L}earning
  {A}pplied to {D}ocument {R}ecognition,'' \emph{Proceedings of the IEEE},
  vol.~86, no.~11, pp. 2278--2324, 1998.

\bibitem{Jain2009}
V.~Jain and H.~S. Seung, ``{N}atural {I}mage {D}enoising with {C}onvolutional
  {N}etworks,'' in \emph{Advances in Neural Information Processing Systems},
  vol.~21, 2009, pp. 769--776.

\bibitem{Xu2015}
L.~Xu, J.~S. Ren, C.~Liu, and J.~Jia, ``{D}eep {C}onvolutional {N}eural
  {N}etwork for {I}mage {D}econvolution,'' in \emph{Advances in Neural
  Information Processing Systems}, vol.~27, 2015, pp. 1790--1798.

\bibitem{Dong2016}
C.~Dong, C.~C. Loy, K.~He, and X.~Tang, ``{I}mage {S}uper-{R}esolution {U}sing
  {D}eep {C}onvolutional {N}etworks,'' \emph{IEEE Transactions on Pattern
  Analysis and Machine Intelligence}, vol.~38, no.~2, pp. 295--307, 2016.

\bibitem{Ciresan2012a}
D.~C. Cire\c{s}an, U.~Meier, and J.~Schmidhuber, ``{M}ulti-{C}olumn {D}eep
  {N}eural {N}etworks for {I}mage {C}lassification,'' in \emph{Proceedings of
  the IEEE Conference on Computer Vision and Pattern Recognition}, 2012, pp.
  3642--3649.

\bibitem{Krizhevsky2013}
A.~Krizhevsky, I.~Sutskever, and G.~E. Hinton, ``{I}mage{N}et {C}lassification
  with {D}eep {C}onvolutional {N}eural {N}etworks,'' in \emph{Advances in
  Neural Information Processing Systems}, vol.~25, 2013, pp. 1097--1105.

\bibitem{Szegedy2015}
C.~Szegedy, W.~Liu, Y.~Jia, P.~Sermanet, S.~Reed, D.~Anguelov, D.~Erhan,
  V.~Vanhoucke, and A.~Rabinovich, ``{G}oing {D}eeper with {C}onvolutions,'' in
  \emph{Proceedings of the IEEE Conference on Computer Vision and Pattern
  Recognition}, 2015, pp. 1--9.

\bibitem{Ning2005}
F.~Ning, D.~Delhomme, Y.~LeCun, F.~Piano, L.~Bottou, and P.~E. Barbano,
  ``{T}oward {A}utomatic {P}henotyping of {D}eveloping {E}mbryos from
  {V}ideos,'' \emph{IEEE Transactions on Image Processing}, vol.~14, no.~9, pp.
  1360--1371, 2005.

\bibitem{Grangier2009}
D.~Grangier, L.~Bottou, and R.~Collobert, ``{D}eep {C}onvolutional {N}etworks
  for {S}cene {P}arsing,'' in \emph{International Conference on Machine
  Learning, Workshop on Learning Feature Hierarchies}, 2009.

\bibitem{Farabet2013}
C.~Farabet, C.~Couprie, L.~Najman, and Y.~LeCun, ``{L}earning {H}ierarchical
  {F}eatures for {S}cene {L}abeling,'' \emph{IEEE Transactions on Pattern
  Analysis and Machine Intelligence}, vol.~35, no.~8, pp. 1915--1929, 2013.

\bibitem{Giusti2013}
A.~Giusti, D.~C. Cire\c{s}an, J.~Masci, L.~M. Gambardella, and J.~Schmidhuber,
  ``{F}ast {I}mage {S}canning with {D}eep {M}ax-{P}ooling {C}onvolutional
  {N}eural {N}etworks,'' in \emph{IEEE International Conference on Image
  Processing}, 2013, pp. 4034--4038.

\bibitem{Thong2016}
W.~Thong, S.~Kadoury, N.~Pich\'e, and C.~J. Pal, ``{C}onvolutional {N}etworks
  for {K}idney {S}egmentation in {C}ontrast-{E}nhanced {CT} {S}cans,''
  \emph{Computer Methods in Biomechanics and Biomedical Engineering: Imaging \&
  Visualization}, 2016.

\bibitem{Nuss2014}
D.~Nuss, M.~Thom, A.~Danzer, and K.~Dietmayer, ``{F}usion of {L}aser and
  {M}onocular {C}amera {D}ata in {O}bject {G}rid {M}aps for {V}ehicle
  {E}nvironment {P}erception,'' in \emph{Proceedings of the International
  Conference on Information Fusion}, 2014.

\bibitem{Vaillant1993}
R.~Vaillant, C.~Monrocq, and Y.~LeCun, ``{A}n {O}riginal {A}pproach for the
  {L}ocalization of {O}bjects in {I}mages,'' in \emph{Proceedings of the
  International Conference on Artificial Neural Networks}, 1993, pp. 26--30.

\bibitem{Li2014}
H.~Li, R.~Zhao, and X.~Wang, ``{H}ighly {E}fficient {F}orward and {B}ackward
  {P}ropagation of {C}onvolutional {N}eural {N}etworks for {P}ixelwise
  {C}lassification,'' Tech. Rep. arXiv:1412.4526, 2014.

\bibitem{Sermanet2014}
P.~Sermanet, D.~Eigen, X.~Zhang, M.~Mathieu, R.~Fergus, and Y.~LeCun,
  ``{O}ver{F}eat: {I}ntegrated {R}ecognition, {L}ocalization and {D}etection
  using {C}onvolutional {N}etworks,'' in \emph{Proceedings of the International
  Conference on Learning Representations}.\hskip 1em plus 0.5em minus
  0.4em\relax arXiv:1312.6229, 2014.

\bibitem{Bishop1995}
C.~M. Bishop, \emph{{N}eural {N}etworks for {P}attern {R}ecognition}.\hskip 1em
  plus 0.5em minus 0.4em\relax Clarendon Press, 1995.

\bibitem{Neudecker1969}
H.~Neudecker, ``{S}ome {T}heorems on {M}atrix {D}ifferentiation with {S}pecial
  {R}eference to {K}ronecker {M}atrix {P}roducts,'' \emph{Journal of the
  American Statistical Association}, vol.~64, no. 327, pp. 953--963, 1969.

\bibitem{Chetlur2014}
S.~Chetlur, C.~Woolley, P.~Vandermersch, J.~Cohen, J.~Tran, B.~Catanzaro, and
  E.~Shelhamer, ``cu{DNN}: {E}fficient {P}rimitives for {D}eep {L}earning,'' in
  \emph{Advances in Neural Information Processing Systems, Deep Learning and
  Representation Learning Workshop}.\hskip 1em plus 0.5em minus 0.4em\relax
  arXiv:1410.0759, 2014.

\bibitem{Simonyan2015}
K.~Simonyan and A.~Zisserman, ``{V}ery {D}eep {C}onvolutional {N}etworks for
  {L}arge-{S}cale {I}mage {R}ecognition,'' in \emph{Proceedings of the
  International Conference on Learning Representations}.\hskip 1em plus 0.5em
  minus 0.4em\relax arXiv:1409.1556, 2015.

\bibitem{Sanger1989a}
T.~D. Sanger, ``{O}ptimal {U}nsupervised {L}earning in {F}eedforward {N}eural
  {N}etworks,'' Master's thesis, Massachusetts Institute of Technology, 1989.

\bibitem{Bottou2004}
L.~Bottou and Y.~LeCun, ``{L}arge {S}cale {O}nline {L}earning,'' in
  \emph{Advances in Neural Information Processing Systems}, vol.~16, 2004, pp.
  217--224.

\bibitem{Thom2016}
M.~Thom and F.~Gritschneder, ``{R}apid {E}xact {S}ignal {S}canning with {D}eep
  {C}onvolutional {N}eural {N}etworks,'' Tech. Rep. arXiv:1508.06904, 2016.

\bibitem{Taylor2007}
S.~Taylor, \emph{{O}ptimizing {A}pplications for {M}ulti-{C}ore {P}rocessors,
  {U}sing the {I}ntel {I}ntegrated {P}erformance {P}rimitives}, 2nd~ed.\hskip
  1em plus 0.5em minus 0.4em\relax Intel Press, 2007.

\bibitem{Dagum1998}
L.~Dagum and R.~Menon, ``{O}pen{MP}: {A}n {I}ndustry-{S}tandard {API} for
  {S}hared-{M}emory {P}rogramming,'' \emph{IEEE Computational Science \&
  Engineering}, vol.~5, no.~1, pp. 46--55, 1998.

\end{thebibliography}


\begin{thebibliography}{1}
\providecommand{\url}[1]{#1}
\csname url@samestyle\endcsname
\providecommand{\newblock}{\relax}
\providecommand{\bibinfo}[2]{#2}
\providecommand{\BIBentrySTDinterwordspacing}{\spaceskip=0pt\relax}
\providecommand{\BIBentryALTinterwordstretchfactor}{4}
\providecommand{\BIBentryALTinterwordspacing}{\spaceskip=\fontdimen2\font plus
\BIBentryALTinterwordstretchfactor\fontdimen3\font minus
  \fontdimen4\font\relax}
\providecommand{\BIBforeignlanguage}[2]{{%
\expandafter\ifx\csname l@#1\endcsname\relax
\typeout{** WARNING: IEEEtran.bst: No hyphenation pattern has been}%
\typeout{** loaded for the language `#1'. Using the pattern for}%
\typeout{** the default language instead.}%
\else
\language=\csname l@#1\endcsname
\fi
#2}}
\providecommand{\BIBdecl}{\relax}
\BIBdecl

\bibitem{Sermanet2014rlx}
P.~Sermanet, D.~Eigen, X.~Zhang, M.~Mathieu, R.~Fergus, and Y.~LeCun,
  ``{O}ver{F}eat: {I}ntegrated {R}ecognition, {L}ocalization and {D}etection
  using {C}onvolutional {N}etworks,'' in \emph{Proceedings of the International
  Conference on Learning Representations}.\hskip 1em plus 0.5em minus
  0.4em\relax arXiv:1312.6229, 2014.

\bibitem{Long2015rlx}
J.~Long, E.~Shelhamer, and T.~Darrell, ``{F}ully {C}onvolutional {N}etworks for
  {S}emantic {S}egmentation,'' in \emph{Proceedings of the IEEE Conference on
  Computer Vision and Pattern Recognition}, 2015, pp. 3431--3440.

\end{thebibliography}


\begin{thebibliography}{1}
\providecommand{\url}[1]{#1}
\csname url@samestyle\endcsname
\providecommand{\newblock}{\relax}
\providecommand{\bibinfo}[2]{#2}
\providecommand{\BIBentrySTDinterwordspacing}{\spaceskip=0pt\relax}
\providecommand{\BIBentryALTinterwordstretchfactor}{4}
\providecommand{\BIBentryALTinterwordspacing}{\spaceskip=\fontdimen2\font plus
\BIBentryALTinterwordstretchfactor\fontdimen3\font minus
  \fontdimen4\font\relax}
\providecommand{\BIBforeignlanguage}[2]{{%
\expandafter\ifx\csname l@#1\endcsname\relax
\typeout{** WARNING: IEEEtran.bst: No hyphenation pattern has been}%
\typeout{** loaded for the language `#1'. Using the pattern for}%
\typeout{** the default language instead.}%
\else
\language=\csname l@#1\endcsname
\fi
#2}}
\providecommand{\BIBdecl}{\relax}
\BIBdecl

\bibitem{Chen2016rlxsld}
L.-C. Chen, G.~Papandreou, I.~Kokkinos, K.~Murphy, and A.~L. Yuille,
  ``{D}eep{L}ab: {S}emantic {I}mage {S}egmentation with {D}eep {C}onvolutional
  {N}ets, {A}trous {C}onvolution, and {F}ully {C}onnected {CRF}s,'' Tech. Rep.
  arXiv:1606.00915, 2016.

\end{thebibliography}


\begin{thebibliography}{1}
\providecommand{\url}[1]{#1}
\csname url@samestyle\endcsname
\providecommand{\newblock}{\relax}
\providecommand{\bibinfo}[2]{#2}
\providecommand{\BIBentrySTDinterwordspacing}{\spaceskip=0pt\relax}
\providecommand{\BIBentryALTinterwordstretchfactor}{4}
\providecommand{\BIBentryALTinterwordspacing}{\spaceskip=\fontdimen2\font plus
\BIBentryALTinterwordstretchfactor\fontdimen3\font minus
  \fontdimen4\font\relax}
\providecommand{\BIBforeignlanguage}[2]{{%
\expandafter\ifx\csname l@#1\endcsname\relax
\typeout{** WARNING: IEEEtran.bst: No hyphenation pattern has been}%
\typeout{** loaded for the language `#1'. Using the pattern for}%
\typeout{** the default language instead.}%
\else
\language=\csname l@#1\endcsname
\fi
#2}}
\providecommand{\BIBdecl}{\relax}
\BIBdecl

\bibitem{Long2015trconv}
J.~Long, E.~Shelhamer, and T.~Darrell, ``{F}ully {C}onvolutional {N}etworks for
  {S}emantic {S}egmentation,'' in \emph{Proceedings of the IEEE Conference on
  Computer Vision and Pattern Recognition}, 2015, pp. 3431--3440.

\bibitem{Wang2017trconv}
P.~Wang, P.~Chen, Y.~Yuan, D.~Liu, Z.~Huang, X.~Hou, and G.~Cottrell,
  ``{U}nderstanding {C}onvolution for {S}emantic {S}egmentation,'' Tech. Rep.
  arXiv:1702.08502, 2017.

\end{thebibliography}

\clearpage%
\onecolumn%
\makeatletter
\def\normalsize{\@setfontsize{\normalsize}{12}{13.92pt}}
\setlength{\@IEEEnormalsizeunitybaselineskip}{13.92pt}
\normalsize
\abovedisplayskip 1.5ex plus 6pt minus 4pt
\belowdisplayskip \abovedisplayskip
\abovedisplayshortskip 0pt plus 6pt
\belowdisplayshortskip 1.5ex plus 6pt minus 4pt
\def\small{\@setfontsize{\small}{10}{12pt}}
\def\footnotesize{\@setfontsize{\footnotesize}{9}{10.5pt}}
\def\scriptsize{\@setfontsize{\scriptsize}{8}{9pt}}
\def\tiny{\@setfontsize{\tiny}{6}{7pt}}
\def\sublargesize{\@setfontsize{\sublargesize}{14}{17pt}}
\def\large{\@setfontsize{\large}{14}{17pt}}
\def\Large{\@setfontsize{\Large}{17}{20pt}}
\def\LARGE{\@setfontsize{\LARGE}{20}{24pt}}
\def\huge{\@setfontsize{\huge}{22}{26pt}}
\def\Huge{\@setfontsize{\Huge}{24}{28pt}}
\makeatother
\appendix[Addendum to Subsignal Compatible Transformation Theory]
The expression from Lemma~\ref{lem:defrag-ops} regarding the source of the samples of the defragmentation operator can be simplified:
\begin{lemma}
\label{lem:defrag-ops-simplified}
Let $M$ be a set, $k,q,s\in\N_1$ and $\xi\in M^{q\times ks}$.
Then for all $\mu\inint{1}{kq}$, $\nu\inint{1}{s}$ it is $\Defragmentation_k(\xi)_{\mu,\;\nu} = \xi_{\div{\mu - 1}{k} + 1,\;\rem{\mu - 1}{k}\cdot s + \nu}$.
\end{lemma}
\begin{proof}
Let $\mu\inint{1}{kq}$ and $\nu\inint{1}{s}$ be arbitrary indices.
With Lemma~\ref{lem:defrag-ops} it is sufficient to show that $\div{ (\mu - 1)s + \nu - 1}{ks} = \div{\mu - 1}{k}$ and $\rem{ (\mu - 1)s + \nu - 1}{ks} = \rem{\mu - 1}{k}\cdot s + \nu - 1$.
By definition it holds that
\begin{displaymath}
  \mu - 1 = \div{\mu - 1}{k}\cdot k + \rem{\mu - 1}{k}
\end{displaymath}
and
\begin{displaymath}
  (\mu - 1)s + \nu - 1 = \div{ (\mu - 1)s + \nu - 1}{ks}\cdot ks + \rem{ (\mu - 1)s + \nu - 1}{ks}\text{.}
\end{displaymath}
Therefore,
\begin{align*}
  &\big( \div{ (\mu - 1)s + \nu - 1}{ks} - \div{\mu - 1}{k} \big)\cdot ks\\
  =\ & \big( (\mu - 1)s + \nu - 1 - \rem{ (\mu - 1)s + \nu - 1}{ks} \big) - \big( \mu - 1 - \rem{\mu - 1}{k} \big)\cdot s\\
  =\ & \rem{\mu - 1}{k}\cdot s + \nu - 1 - \rem{ (\mu - 1)s + \nu - 1}{ks}\text{.}
\end{align*}
The left-hand side of this equation is of the form $z\cdot ks$ for an integer $z\in\Z$.
The right-hand side is from the discrete interval $\discint{-(ks - 1)}{ks - 1}$ as can be seen from substituting the extreme values of $\nu - 1$ and the $\operatorname{rem}$ operator.
Hence this equation can only be satisfied if both sides vanish.
This yields the claimed identities.
\end{proof}

Here the proof to Remark~\ref{rem:comp-frag}, which was omitted in the main part of this paper due to space constraints:
\begin{proof}[Remark~\ref{rem:comp-frag}]
Let $\xi\in M^{k_1k_2q\times s}$ be a fragmented signal.
Define $A := \Fragmentation_{k_1}(\xi)\in M^{k_2q\times k_1 s}$, $B := \Fragmentation_{k_2}(A)\in M^{q\times k_1k_2s}$, and $C := \Fragmentation_{k_1k_2}(\xi)\in M^{q\times k_1k_2s}$.
Since $B$ and $C$ are of equal size, it is enough to show entry-wise equivalence.
Let $\mu\inint{1}{q}$ and $\nu\inint{1}{k_1k_2s}$.
With Lemma~\ref{lem:frag-ops} follows that $C_{\mu,\;\nu} = \xi_{\mu_C,\;\nu_C}$ and $B_{\mu,\;\nu} = A_{\mu_B,\;\nu_B} = \xi_{\mu_A,\;\nu_A}$ using the indices
\begin{align*}
  \mu_C &:= \div{ (\mu - 1)k_1k_2s + \nu - 1}{s} + 1\text{,}\\
  \nu_C &:= \rem{ (\mu - 1)k_1k_2s + \nu - 1}{s} + 1\text{,}\\
  \mu_B &:= \div{ (\mu - 1)k_1k_2s + \nu - 1}{k_1s} + 1\text{,}\\
  \nu_B &:= \rem{ (\mu - 1)k_1k_2s + \nu - 1}{k_1s} + 1\text{,}\\
  \mu_A &:= \div{ (\mu_B - 1)k_1s + \nu_B - 1}{s} + 1\text{,}\\
  \nu_A &:= \rem{ (\mu_B - 1)k_1s + \nu_B - 1}{s} + 1\text{.}
\end{align*}
Therefore, it only remains to be shown that $\mu_C = \mu_A$ and $\nu_C = \nu_A$.
It is
\begin{align*}
  \mu_A = \div{\phantom{+\ }   \operatorname{div}&((\mu - 1)k_1k_2s + \nu - 1,\ k_1s)\cdot k_1s\\
          + \operatorname{rem}&((\mu - 1)k_1k_2s + \nu - 1,\ k_1s)}{\ s} + 1\text{,}
\end{align*}
which equals $\mu_C$ since $\div{a}{b}\cdot b + \rem{a}{b} = a$ for all $a,b$.
Completely analogous follows $\nu_A = \nu_C$.
\end{proof}

\clearpage
\appendix[Multi-Scale Transformations]
\label{sect:multi_scale_transformations}
The main part of this paper has shown how CNNs can be efficiently evaluated on entire images through the theory of subsignal compatible transformations.
Now, functions that take multiple spatial resolutions of a single signal as input are considered.
Since the context of local regions is incorporated in addition here, this approach has proven highly effective in classification tasks~\citems{Farabet2013ms}.
Here it is assumed that the number of samples considered at any one time is fixed for all scale levels.
This facilitates the design of scale-invariant representations, for example by using the same classifier for all scales of the input~\citems{Farabet2013ms}.
The analysis here, however, is not restricted to the situation where the same classifier should be used for all scales.
Instead, the analysis is conducted directly for different functions which are applied to each scale.

\subsection{Multi-Scale Subsignals and Their Emergence in Downscaled Signals}
\label{sect:multi_scale_fund}
A signal is downscaled by application of a lowpass filter to reduce aliasing artifacts, followed by a downsampling operator which returns a subset of equidistant samples.
When a subsignal is extracted from a downscaled input signal, it should contain a downscaled copy of the corresponding subsignal from the original input signal.
This requires boundary-handling of the input signal, since for example the very first subsignal cannot be extended to allow for a larger context by means of only the original samples.

In the following, let $\Z$ denote the integers and let $\nceil{\cdot}$ denote the ceiling function that rounds up its argument to the next larger natural number.
First, the concepts of boundary handling and subsignal extraction subject to boundary handling are formalized:

\begin{definition}
\label{def:subsignalpad}
\label{def:padding}
Let $M$ be a set, let $d\in\N_1$ denote a subsignal dimensionality and let $R\in\N$ be a boundary size.
\begin{enumerate}
  \item A function $\vartheta\colon\cup_1(M)\times\Z\to M$ is called \emph{boundary-handling function} if and only if $\vartheta(\xi,\;\nu) = \xi_\nu$ for all $\xi\in\cup_1(M)$ and all $\nu\inint{1}{\dim_M(\xi)}$.
  \item The function $\PaddingParams\colon\cup_1(M)\to \cup_{1 + 2R}(M)$,
    \begin{displaymath}
      \xi\mapsto\sum\nolimits_{\nu = 1}^{\dim_M(\xi) + 2R} \vartheta(\xi,\;\nu - R)\cdot e_\nu^{\dim_M(\xi) + 2R}\text{,}
    \end{displaymath}
    which extends signals at both ends with $R$ samples subject to the boundary-handling function $\vartheta$ is called the \emph{padding operator}.
  \item The function $\SubsignalPadParams\colon\bigcup\nolimits_{D = d}^\infty\big(M^D\times\discint{1}{D - d + 1}\big)\to M^{d + 2R}$,
    \begin{displaymath}
      (\xi,\; i)\mapsto\sum\nolimits_{\nu = 1}^{d + 2R}\vartheta(\xi,\;i + \nu - R - 1)\cdot e_\nu^{d + 2R}\text{,}
    \end{displaymath}
    which extracts subsignals subject to the boundary-handling function $\vartheta$ and implicitly pads $R$ samples at both ends is called the \emph{padded subsignal extraction operator}.
\end{enumerate}
\end{definition}

The definition of a boundary-handling function here leaves open which concrete values should be returned for access outside of the original signal, allowing for great flexibility.
For example, the functions
\begin{displaymath}
  \vartheta_{\Dirichlet}(\xi,\;\nu) := \begin{cases}
    \xi_\nu\text{,} & \text{if }\nu\inint{1}{\dim_M(\xi)}\text{,}\\
    0\text{,} & \text{otherwise,}
  \end{cases}
  \text{ and }
  \vartheta_{\Neumann}(\xi,\;\nu) := \begin{cases}
    \xi_1\text{,} & \text{if }\nu < 1\text{,}\\
    \xi_\nu\text{,} & \text{if }\nu\inint{1}{\dim_M(\xi)}\text{,}\\
    \xi_{\dim_M(\xi)}\text{,} & \text{if }\nu > \dim_M(\xi)\text{,}
  \end{cases}
\end{displaymath}
realize Dirichlet (first-type) and Neumann (second-type) boundary conditions, respectively.

The padded subsignal extraction operator is a strict generalization of the subsignal extraction operator:
For $R = 0$, the second argument to the boundary-handling function fulfills $i + \nu - R - 1\inint{1}{D}$ for all $i\inint{1}{D - d + 1}$ and all $\nu\inint{1}{d + 2R}$ where $D := \dim_M(\xi)$ denotes the input signal dimensionality.
Therefore, in the case of $R = 0$ it holds that $\vartheta(\xi,\;i + \nu - R - 1) = \xi_{i + \nu - 1}$ by definition of a boundary-handling function, hence $\SubsignalPad_{(d,\; 0)}^{\vartheta} = \Subsignal_d$.

Before the presentation of theoretical results, extraction of downscaled subsignals using the concepts just introduced is defined:
\begin{definition}
\label{def:downsampling}
Let $M$ be a set and let $k\in\N_1$ denote a downsampling step size.
\begin{enumerate}
  \item Then the function $\Downsampling_k\colon\cup_1(M)\to \cup_1(M)$,
    \begin{displaymath}
      \xi\mapsto\sum\nolimits_{\nu = 1}^{\nceil{\dim_M(\xi) / k}} \xi_{k (\nu - 1) + 1}\cdot e_\nu^{\nceil{\dim_M(\xi) / k}}\text{,}
    \end{displaymath}
    which extracts samples from equidistant locations is called the \emph{downsampling operator}.
  \item The function $\MultiScaleIndex_k\colon\N_1\to\N_1$,
    \begin{displaymath}
      i\mapsto k\cdot\div{i - 1}{k} + 1\text{,}
    \end{displaymath}
    is called the \emph{multi-scale subsignal index transformation}.
  \item Suppose $H\colon M^h\to M$ is a lowpass filter kernel of size $h\in\N_1$.
    Here, $h\geq k$ should hold to avoid aliasing artifacts.
    Further, let $d\in\N_1$ be a subsignal dimensionality, let $R\in\N$ be a boundary size, and let $\vartheta\colon\cup_1(M)\times\Z\to M$ be a boundary-handling function.
    Then the function $\MultiScaleSubsignalParams\colon\bigcup\nolimits_{D = d}^\infty\big(M^D\times\discint{1}{D - d + 1}\big)\to\cup_1(M)$,
    \begin{displaymath}
      (\xi,\; i)\mapsto\Downsampling_k(\Slide_H(\SubsignalPadParams(\xi,\;i)))\text{,}
    \end{displaymath}
    is called the \emph{multi-scale subsignal extraction operator}.
\end{enumerate}
\end{definition}

The downsampling operator is well-defined because $k (\nu - 1) + 1\inint{1}{\dim_M(\xi)}$ holds for all $\nu\inint{1}{\nceil{\dim_M(\xi) / k}}$ as can be seen from a case-by-case analysis, depending on whether $k$ divides $\dim_M(\xi)$ or not.
The other defined functions are clearly well-defined.

There are a few requirements so that extraction of downscaled subsignals makes sense.
Most important is here the correct determination of the boundary size $R$ in the definition of the $\MultiScaleSubsignal$ operator.
It should be chosen so that the extracted subsignals from each scale level are always centered exactly around the corresponding subsignals from the original scale level.
It is moreover beneficial if the entire input signal can be downscaled in its entirety using only one operation, so that the output of the $\MultiScaleSubsignal$ operator equals simple extraction of subsignals from that downscaled signal.

However, if this approach is pursued there are subsignals in the original signal which do not possess a downscaled counterpart in this representation.
The $\MultiScaleIndex$ function alleviates this problem through computation of an appropriate subsignal index which is always guaranteed to possess a downscaled counterpart.
Although this is merely an approximation, it is assured that the correct subsignal index in the downscaled signal is always less than one sample off.
The next result formalizes these thoughts, an illustration of its statements is presented in Fig.~\ref{fig:multiscale-example} and Fig.~\ref{fig:multiscale-xp}.

\begin{figure*}[t]
  \centering
  \includegraphics[page=10]{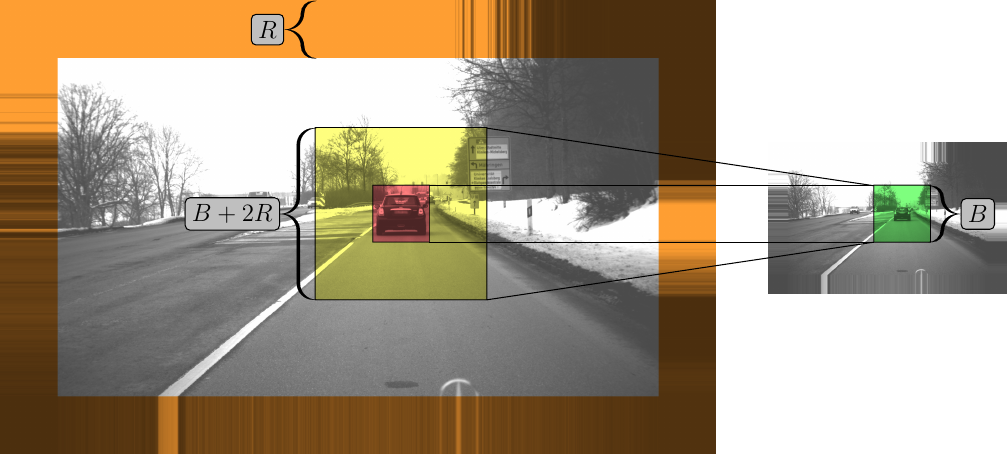}
  \caption{Illustration of the multi-scale analysis approach detailed in Lemma~\ref{lem:multiscale-props} for image processing as an example.
    The original image (gray area on the left-hand side) is padded with $R$ pixels (orange area) using Neumann boundary conditions.
    Downscaling of the padded image with a scale factor of $k = 3$ yields the small image on the right-hand side.
    The red area in the original image is a region of interest with $B$ pixels in either dimension, it can be extracted conventionally using the $\Subsignal$ operator.
    The $\SubsignalPad$ operator extracts an extended region of interest with $B + 2R$ pixels in either dimension (yellow area), which provides more context than the original region of interest.
    If that extended region is downscaled, it is guaranteed by the choice of $R$ that the resulting signal possesses $B$ samples in either dimension.
    This downscaled, padded region is equivalent to application of the $\MultiScaleSubsignal$ operator.
    Its outcome can also be found in the downscaled, padded image (green area) if the indices are adjusted through the $\MultiScaleIndex$ function.
    Figure best viewed in color.}
  \label{fig:multiscale-example}
\end{figure*}

\begin{figure}[p]
  \centering
  \scalebox{1.09}{\includegraphics[page=11]{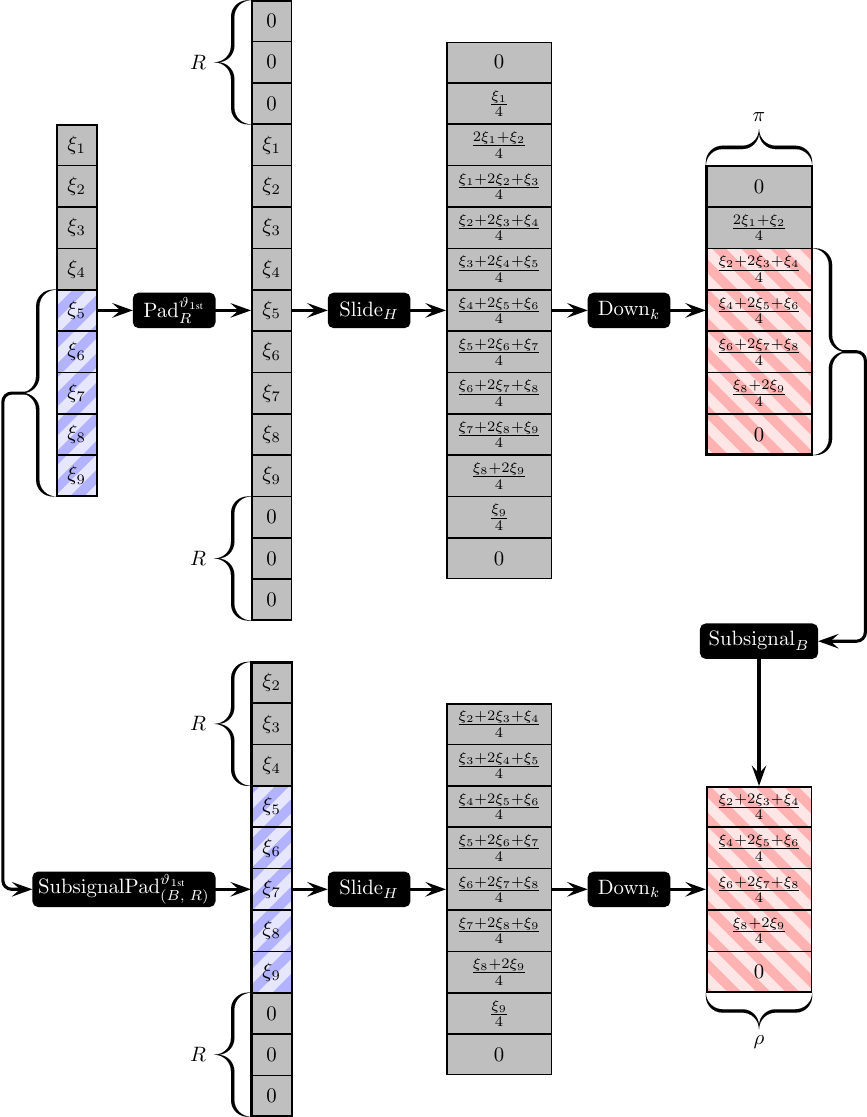}}
  \caption{Example for the properties of the multi-scale operators stated in Lemma~\ref{lem:multiscale-props}.
    Here, the downsampling factor $k = 2$, a binomial lowpass filter kernel $H$ with $h = 3$, that is $H(\sigma) = \frac{\sigma_1 + 2\sigma_2 + \sigma_3}{4}$, and a subsignal dimensionality of $B = 5$ are used.
    Therefore, $\tilde{R} = 6$ and $R = 3$ due to Lemma~\ref{lem:multiscale-props}.
    The upper part shows how an input signal $\xi$ with $D = 9$ samples is padded with $R$ samples at either end using Dirichlet boundary conditions $\vartheta_{\Dirichlet}$, lowpass-filtered with $H$ and then downsampled using the step size $k$.
    This yields the downscaled signal $\pi$.
    The lower part illustrates the process of extracting the padded subsignal with subsignal index $j = 5$, and then downscaling it yielding $\rho = \MultiScaleSubsignal_{(\ROI,\; R,\; k)}^{(\vartheta_{\Dirichlet},\; H)}(\xi,\;j)$.
    As stated in Lemma~\ref{lem:multiscale-props}, the padded subsignal is centered around the original subsignal, illustrated by a colored pattern in the graphics.
    Further, $\rho$ can be located as a subsignal in $\pi$, illustrated by a different colored pattern, as guaranteed by Lemma~\ref{lem:multiscale-props} since $j$ is a fixed point of $\MultiScaleIndex_k$.
    Note that for example the fourth subsignal from $\xi$ does not possess an exact downscaled correspondence in $\pi$, as predicted theoretically.}
  \label{fig:multiscale-xp}
\end{figure}

\begin{lemma}
\label{lem:multiscale-props}
Let $M$ be a set, and let $k\in\N$ be a downsampling step size where it is required that $k \geq 2$.
Moreover, let $H\colon M^h\to M$ be a lowpass filter kernel of size $h\in\N_1$, $h \geq k$, and let $\ROI\in\N_1$ be a subsignal dimensionality and suppose $\vartheta\colon\cup_1(M)\times\Z\to M$ is a boundary-handling function.
Define $\tilde{R} := (k - 1)\ROI + h - k\in\N$ and $R := \bceil{\tilde{R} / 2}\in\N$ as the boundary size.
Assume a signal $\xi\in\cup_\ROI(M)$ is given and write $D := \dim_M(\xi)$ as an abbreviation.
Finally, let $\pi := \Downsampling_k(\Slide_H(\PaddingParams(\xi)))\in\cup_1(M)$ denote the downscaled signal.
Then the following holds:
\begin{enumerate}\setlength{\itemsep}{.5ex}
  \item \label{lem:multiscale-props-a} $\Subsignal_\ROI(\xi,\;i)_\mu = \SubsignalPadROIParams(\xi,\;i)_{\mu + R}$ for all samples with index $\mu\inint{1}{\ROI}$ and all subsignals with index $i\inint{1}{D - \ROI + 1}$.
    In other words, the padded subsignals are centered around the original subsignals.
  \item \label{lem:multiscale-props-b} $\dim_M(\pi) = \bceil{\frac{D - \ROI + 1 + \rem{\tilde{R}}{2}}{k}} + \ROI - 1$, hence there are at least $\bceil{\frac{D - \ROI + 1}{k}}$ subsignals with $\ROI$ samples in $\pi$, and at most one additional subsignal.
  \item \label{lem:multiscale-props-c} Let $i\inint{1}{D - \ROI + 1}$ be a subsignal index and write $j := \MultiScaleIndex_k(i)$ as the result of the index transformation.
        Then $j\inint{1}{D - \ROI + 1}$, $j \leq i$ and $i - j < k$.
        The index adjustment hence decreases subsignal indices by at most $k - 1$ samples with respect to the original scale level.
  \item \label{lem:multiscale-props-d} It is
\begin{displaymath}
     \Subsignal_\ROI(\pi,\;\div{i - 1}{k} + 1)
  \ =\ \MultiScaleSubsignalROIParams(\xi,\;\MultiScaleIndex_k(i))
\end{displaymath} for all $i\inint{1}{D - \ROI + 1}$.
        In other words, the subsignals from $\pi$ equal downscaled subsignals from the original signal $\xi$ where the subsignal index was adjusted through $\MultiScaleIndex_k$.
\end{enumerate}
\end{lemma}
\begin{proof}
\ref{lem:multiscale-props-a}
If $\mu\inint{1}{\ROI}$ and $i\inint{1}{D - \ROI + 1}$, then
\begin{displaymath}
  \SubsignalPadROIParams(\xi,\;i)_{\mu + R}
  \ \equsing{D.~\ref{def:subsignalpad}}\ \ \vartheta(\xi,\;i + \mu - 1)
  \ \equsing{($\lozenge$)}\ \xi_{i + \mu - 1}
  \ \equsing{D.~\ref{def:subsignal}}\ \Subsignal_\ROI(\xi,\;i)_\mu\text{,}
\end{displaymath}
where in the ($\lozenge$) step $i + \mu - 1\inint{1}{D}$ has been used.
Here, the boundary handling function evaluates to an original sample of the input signal.
Hence all the samples in the middle of $\SubsignalPadROIParams(\xi,\;i)$ stem from the input signal $\xi$ and are \emph{not} subject to boundary conditions.

\ref{lem:multiscale-props-b}
First note that
$2R = \tilde{R} + \rem{\tilde{R}}{2} = (k - 1)\ROI + h - k + \rem{\tilde{R}}{2}$
by the definition of the ceiling function, which is marked with ($\lozenge$) in the following.
Therefore
\begin{align*}
  \dim_M(\pi)
  \ \ &=\ \ \dim_M(\Downsampling_k(\Slide_H(\PaddingParams(\xi))))\\
  &\equsing{D.~\ref{def:downsampling}}\ \ \left\lceil \tfrac{1}{k}\cdot\dim_M(\Slide_H(\PaddingParams(\xi))) \right\rceil\\
  &\equsing{D.~\ref{def:sliding-function}}\ \ \left\lceil \tfrac{1}{k}\left( \dim_M(\PaddingParams(\xi)) - h + 1 \right) \right\rceil\\
  &\equsing{D.~\ref{def:padding}}\ \ \left\lceil \tfrac{1}{k}\left( \dim_M(\xi) + 2R - h + 1 \right) \right\rceil\\
  &\equsing{($\lozenge$)}\ \ \left\lceil \tfrac{D - \ROI + 1 + \rem{\tilde{R}}{2}}{k} + \ROI - 1 \right\rceil\\
  &=\ \ \left\lceil \tfrac{D - \ROI + 1 + \rem{\tilde{R}}{2}}{k} \right\rceil + \ROI - 1\text{,}
\end{align*}
where $\ROI - 1\in\N$ was used in the final step so that this term could be moved outside of the ceiling function.
Since $\rem{\tilde{R}}{2}\in\set{0,\;1}$, there is at most one superfluous subsignal of length $B$ in $\pi$, which is irrelevant in the following discussion.

\ref{lem:multiscale-props-c}
Let $i$ and $j := k\cdot\div{i - 1}{k} + 1$ be given as in the claim.
Clearly, $j\in\N$.
Since $\div{i - 1}{k} \geq 0$ follows $j \geq 1$.
On the other hand, Euclidean division yields $j = k\cdot\div{i - 1}{k} + 1 = i - \rem{i - 1}{k} \leq i \leq D - \ROI + 1$ because $\rem{i - 1}{k} \geq 0$.
Analogously, $i - j = \rem{i - 1}{k}\inint{0}{k - 1}$ follows, which proves the claimed inequalities.

\ref{lem:multiscale-props-d}
Let $i\inint{1}{D - \ROI + 1}$ be an arbitrary subsignal index.
It is first shown that the right-hand side of the claimed identity is indeed of dimensionality $\ROI$.
For this, define $j := \MultiScaleIndex_k(i)$ and $\rho := \MultiScaleSubsignalROIParams(\xi,\;j)$ as abbreviations.
Analogously to~\ref{lem:multiscale-props-b} where an expression for $2R$ has been deduced, marked with ($\lozenge$) in the following, one obtains
\begin{align*}
  &\dim_M(\rho)\\
  \equsing{D.~\ref{def:downsampling}}\ \ &\left\lceil \tfrac{1}{k}\cdot\dim_M(\Slide_H(\SubsignalPadROIParams(\xi,\;j))) \right\rceil\\
  \equsing{D.~\ref{def:sliding-function}}\ \ &\left\lceil \tfrac{1}{k}\left( \dim_M(\SubsignalPadROIParams(\xi,\;j)) - h + 1 \right) \right\rceil\\
  \equsing{D.~\ref{def:subsignalpad}}\ \ &\left\lceil \tfrac{1}{k}\left( \ROI + 2R - h + 1 \right) \right\rceil\\
  \equsing{($\lozenge$)}\ \ &\ROI - 1 + \left\lceil \tfrac{1 + \rem{\tilde{R}}{2}}{k} \right\rceil\text{.}
\end{align*}
As $1 + \rem{\tilde{R}}{2} \in\set{1,\;2}$ and $k\geq 2$ by requirement, it follows that $0 < \tfrac{1}{k}\big(1 + \rem{\tilde{R}}{2}\big) \leq 1$ and hence $\dim_M(\rho) = \ROI$.

Now let $\mu\inint{1}{\ROI}$ for comparing both sides of the claim sample-wise.
It is
\begin{align*}
  \rho_\mu\ \;\equsing{D.~\ref{def:downsampling}}\ \ &\Slide_H(\SubsignalPadROIParams(\xi,\;j))_{k(\mu - 1) + 1}\\
  \equsing{D.~\ref{def:sliding-function}}\ \ &H\left(\sum\nolimits_{\nu = 1}^h \SubsignalPadROIParams(\xi,\;j)_{k(\mu - 1) + \nu}\cdot e_\nu^h \right)\\
  \equsing{D.~\ref{def:subsignalpad}}\ \ &H\left(\sum\nolimits_{\nu = 1}^h \vartheta(\xi,\;j + k(\mu - 1) + \nu - R - 1)\cdot e_\nu^h \right)\text{.}
\end{align*}
The corresponding sample of the left-hand side equals
\begin{align*}
  &\Subsignal_\ROI(\pi,\;\div{i - 1}{k} + 1)_\mu\\
  \equsing{D.~\ref{def:subsignal}}\ \ &\Downsampling_k(\Slide_H(\PaddingParams(\xi)))_{\div{i - 1}{k} + \mu}\\
  \equsing{D.~\ref{def:downsampling}}\ \ &\Slide_H(\PaddingParams(\xi))_{k\cdot\div{i - 1}{k} + k(\mu - 1) + 1}\\
  \equsing{D.~\ref{def:sliding-function}}\ \ &H\left( \sum\nolimits_{\nu = 1}^h \PaddingParams(\xi)_{j + k(\mu - 1) + \nu - 1}\cdot e_\nu^h \right)\\
  \equsing{D.~\ref{def:padding}}\ \ &H\left( \sum\nolimits_{\nu = 1}^h \vartheta(\xi,\;j + k(\mu - 1) + \nu - R - 1)\cdot e_\nu^h \right)\text{,}
\end{align*}
which is the same as $\rho_\mu$, which proves the claimed identity.
\end{proof}

\subsection{Multi-Scale Evaluation of Subsignal Compatible Transformations}
The ultimate goal here is to analyze functions applied to different scale levels of a signal and propose an efficient evaluation scheme.
The first step has already been taken by analyzing the connection between downscaled subsignal extraction and subsignal extraction from a downscaled signal in Lemma~\ref{lem:multiscale-props}.
The complement of downscaling in this course of action is to repeat samples as many times as samples were omitted during downsampling.
This leads to the following definition:

\begin{definition}
\label{def:upsampling}
Let $M$ be a set and $k\in\N_1$.
Then the function $\Upsampling_k\colon\cup_1(M)\to\cup_k(M)$,
\begin{displaymath}
  \xi\mapsto\sum\nolimits_{\mu = 1}^{\dim_M(\xi)}\left(\xi_\mu\cdot \sum\nolimits_{\lambda = 1}^k e_{k (\mu - 1) + \lambda}^{k\cdot\dim_M(\xi)} \right)\text{,}
\end{displaymath}
is called the \emph{upsampling operator with zero-order hold}.
\end{definition}

In the definition it holds that $k (\mu - 1) + \lambda\inint{1}{k\cdot\dim_M(\xi)}$, which is indeed a bijection between index sets, therefore this operator is well-defined and each sample of its output is a copy of a certain sample from the input signal.
A statement on which samples go where during upsampling directly follows:
\begin{lemma}
\label{lem:upsampling-props}
Let $M$ be a set, $q\in\N_1$ and $\xi\in M^q$.
Then $\Upsampling_k(\xi)_\nu = \xi_{\div{\nu - 1}{k} + 1}$ for all $\nu\inint{1}{kq}$.
\end{lemma}
\begin{proof}
With Definition~\ref{def:upsampling} there exists $\mu\inint{1}{q}$ with $\Upsampling_k(\xi)_\nu = \xi_\mu$, where $\nu = k (\mu - 1) + \lambda$ and $\lambda\inint{1}{k}$.
One obtains
$k\cdot (\mu - 1) + (\lambda - 1) + 1 = (\nu - 1) + 1 = k\cdot\div{\nu - 1}{k} + \rem{\nu - 1}{k} + 1$.
Here, $\lambda - 1\inint{0}{k - 1}$, hence uniqueness of Euclidean division implies $\mu - 1 = \div{\nu - 1}{k}$, and the claim follows.
\end{proof}

Now the main result of this appendix, which states under which circumstances a function that accepts inputs in both the original scale and in a downscaled version can be evaluated efficiently.
Indexing rules are here as follows.
Suppose $P$ and $Q$ are sets and $\chi\in\cup_1(P\times Q)$ is a signal with paired samples from $P\times Q$.
Then there exists a dimensionality $d\in\N_1$ so that $\chi\in(P\times Q)^d$.
Since $(P\times Q)^d\cong P^d\times Q^d$ one can also express $\chi$ as a pair of signals, say $\chi = (\psi,\;\omega)$ where $\psi\in P^d$ and $\omega\in Q^d$.
If $\nu\inint{1}{d}$ is an index, then $\chi_i := (\psi_i,\;\omega_i)\in P\times Q$ is an individual sample.

\begin{theorem}
\label{tmh:multiscale}
Let $M,N$ be sets and let $\ROI\in\N_1$ be a constant subsignal dimensionality.
Further, let $f\colon M^\ROI\times M^\ROI\to N$ be a function that accepts signals in the original scale and in a downscaled version.
Assume $f$ can be factorized into functions $g_{\Original}\colon M^\ROI\to P$, $g_{\Downsampling}\colon M^\ROI\to Q$ and $g\colon P\times Q\to N$ with
\begin{displaymath}
  f(\rho,\;\tau) = g(g_{\Original}(\rho),\;g_{\Downsampling}(\tau))\text{ for all }\rho,\tau\in M^\ROI\text{,}
\end{displaymath}
where $P$ and $Q$ are sets.

As in Lemma~\ref{lem:multiscale-props}, let $k\in\N$, $k \geq 2$, be a downsampling step size.
Further, let $h\in\N_1$, $h \geq k$, and let $H\colon M^h\to M$ be a lowpass filter kernel and $\vartheta\colon\cup_1(M)\times\Z\to M$ a boundary-handling function.
Let $\tilde{R} := (k - 1)\ROI + h - k\in\N$ and let $R := \bceil{\tilde{R} / 2}\in\N$ denote the required boundary size.
Considering a signal $\xi\in\cup_\ROI(M)$, write $D := \dim_M(\xi)$, $\pi := \Downsampling_k(\Slide_H(\PaddingParams(\xi)))\in\cup_1(M)$, and moreover
\begin{displaymath}
  r := \rem{\tilde{R}}{2} - \rem{D - \ROI + 1 + \rem{\tilde{R}}{2}}{k}
       + \begin{cases}
           0\text{,} & \text{if }k\text{ divides }D - \ROI + 1 + \rem{\tilde{R}}{2}\text{,}\\
           k\text{, } & \text{otherwise.}
         \end{cases}
\end{displaymath}
Then $r\in\N$ and
\begin{align*}
     & f\!\left(\Subsignal_\ROI(\xi,\;i),\; \MultiScaleSubsignalROIParams(\xi,\;\MultiScaleIndex_k(i))\right)\\
  =\ & \Slide_g\!\left(\Slide_{g_{\Original}}(\xi),\; \Trimming_{r}(\Upsampling_k(\Slide_{g_{\Downsampling}}(\pi)))\right)_i
\end{align*}
for all $i\inint{1}{D - \ROI + 1}$.
In other words, $f$ applied to the subsignals of $\xi$ and certain multi-scale subsignals of $\xi$ equals the output samples of $g$, $g_{\Original}$ and $g_{\Downsampling}$ applied in a sliding fashion to signals derived from $\xi$.
After $g_{\Downsampling}$ has been applied to the downscaled signal $\pi$, the result has to be upsampled and the superfluous $r$ trailing entries have to be trimmed away.
\end{theorem}
\begin{proof}
Application of Lemma~\ref{lem:multiscale-props}\ref{lem:multiscale-props-b} yields $\dim_M(\pi) = \bceil{\frac{D - \ROI + 1 + \rem{\tilde{R}}{2}}{k}} + \ROI - 1$.
As $\dim_M(\pi) \geq \ROI$, $\Slide_{g_{\Downsampling}}(\pi)$ is well-defined and one obtains
\begin{displaymath}
  \dim_Q(\Slide_{g_{\Downsampling}}(\pi))
  \ \equsing{D.~\ref{def:sliding-function}}\ \dim_M(\pi) - \ROI + 1
  \ =\ \left\lceil\tfrac{D - \ROI + 1 + \rem{\tilde{R}}{2}}{k}\right\rceil\text{.}
\end{displaymath}
Therefore,
\begin{displaymath}
  \dim_Q(\Upsampling_k(\Slide_{g_{\Downsampling}}(\pi)))\ \ 
  \equsing{D.~\ref{def:upsampling}}\ \ k\cdot \left\lceil\tfrac{D - \ROI + 1 + \rem{\tilde{R}}{2}}{k}\right\rceil\text{.}
\end{displaymath}

In the case of $k$ dividing $D - \ROI + 1 + \rem{\tilde{R}}{2}$ it is $\dim_Q(\Upsampling_k(\Slide_{g_{\Downsampling}}(\pi))) = D - \ROI + 1 + \rem{\tilde{R}}{2}$ and by definition holds $r = \rem{\tilde{R}}{2}\in\N$, hence
\begin{displaymath}
  \dim_Q(\Trimming_{r}(\Upsampling_k(\Slide_{g_{\Downsampling}}(\pi))))
  \ \ \equsing{D.~\ref{def:stuff-trim}}\ \ D - \ROI + 1\text{.}
\end{displaymath}

In the other case of $k$ not dividing $D - \ROI + 1 + \rem{\tilde{R}}{2}$ follows
\begin{displaymath}
  \left\lceil\tfrac{D - \ROI + 1 + \rem{\tilde{R}}{2}}{k}\right\rceil
  = \div{D - \ROI + 1 + \rem{\tilde{R}}{2}}{k} + 1\text{.}
\end{displaymath}
The definition of Euclidean division now implies that
\begin{displaymath}
     \dim_Q(\Upsampling_k(\Slide_{g_{\Downsampling}}(\pi)))
  =  D - \ROI + 1 + \rem{\tilde{R}}{2} - \rem{D - \ROI + 1 + \rem{\tilde{R}}{2}}{k} + k\text{.}
\end{displaymath}
Further, by definition it is
\begin{displaymath}
  r = \rem{\tilde{R}}{2} - \rem{D - \ROI + 1 + \rem{\tilde{R}}{2}}{k} + k\text{.}
\end{displaymath}
Since $\rem{D - \ROI + 1 + \rem{\tilde{R}}{2}}{k}\inint{0}{k - 1}$ it holds that $r \geq 1$.
Therefore $r\in\N$ and $\Trimming_{r}(\Upsampling_k(\Slide_{g_{\Downsampling}}(\pi)))$ is well-defined and $\dim_Q(\Trimming_{r}(\Upsampling_k(\Slide_{g_{\Downsampling}}(\pi)))) = D - \ROI + 1$ in this case as well.

Hence the number of samples in $\Slide_{g_{\Original}}(\xi)$ equals the number of samples in $\Trimming_{r}(\Upsampling_k(\Slide_{g_{\Downsampling}}(\pi)))$ in both cases.
Since $g$ works on single samples, it follows that
\begin{displaymath}
  \Slide_g\!\big(\Slide_{g_{\Original}}(\xi),\; \Trimming_{r}(\Upsampling_k(\Slide_{g_{\Downsampling}}(\pi)))\big)
\end{displaymath}
consists of $D - \ROI + 1$ samples from $N$.

Let $i\inint{1}{D - \ROI + 1}$ be arbitrary and define
\begin{displaymath}
  \tau := \MultiScaleSubsignalROIParams(\xi,\;\MultiScaleIndex_k(i))
\end{displaymath}
as an abbreviation.
Then Lemma~\ref{lem:multiscale-props}\ref{lem:multiscale-props-d} implies $\tau\in M^\ROI$.
From Definition~\ref{def:sliding-function} it immediately follows that $\Slide_{g_{\Original}}(\xi)_i = g_{\Original}(\Subsignal_\ROI(\xi,\;i))$.
Considering the second argument to $\Slide_g$ one obtains
\begin{align*}
  &\Trimming_{r}(\Upsampling_k(\Slide_{g_{\Downsampling}}(\pi)))_i\\
  =\ \ \ &\Upsampling_k(\Slide_{g_{\Downsampling}}(\pi))_i\\
  \equsing{L.~\ref{lem:upsampling-props}}\ \ \ &\Slide_{g_{\Downsampling}}(\pi)_{\div{i - 1}{k} + 1}\\
  \equsing{D.~\ref{def:sliding-function}}\ \ \ &g_{\Downsampling}(\Subsignal_\ROI(\pi,\;\div{i - 1}{k} + 1))\\
  \equsing{L.~\ref{lem:multiscale-props}\ref{lem:multiscale-props-d}}\ \ \ &g_{\Downsampling}(\tau)\text{,}
\end{align*}
where $i\leq D - \ROI + 1$ has been used in the first step so the trimming operator could be eliminated.

Combining these results and using the precondition (PC) that $f$ can be factorized leads to
\begin{align*}
  &\Slide_g\!\big(\Slide_{g_{\Original}}(\xi),\;\Trimming_{r}(\Upsampling_k(\Slide_{g_{\Downsampling}}(\pi)))\big)_i\\
  \equsing{D.~\ref{def:sliding-function}}\ \ &g\big(\Slide_{g_{\Original}}(\xi)_i,\;\Trimming_{r}(\Upsampling_k(\Slide_{g_{\Downsampling}}(\pi)))_i\big)\\
  =\ \ &g\big(\,g_{\Original}(\Subsignal_\ROI(\xi,\;i)),\;g_{\Downsampling}(\tau)\big)\\
  \equsing{PC}\ \ &f\big(\Subsignal_\ROI(\xi,\;i),\;\tau\big)\text{,}
\end{align*}
which equals the claimed expression.
\end{proof}

Theorem~\ref{tmh:multiscale} directly provides an algorithm for efficient multi-scale analysis.
The functional part $g_{\Original}$ operating on the original input signal should be applied in a sliding fashion.
If this function can be cast as a processing chain, as discussed in the main part of this paper, the theory proposed there can be used for efficient evaluation.
The multi-scale subsignal index approximation proved in Lemma~\ref{lem:multiscale-props} facilitates application of the functional part $g_{\Downsampling}$ operating on downscaled subsignals in a sliding fashion as well.
Therefore, subsignal compatible transformation theory can be applied here also.
It is finally noted that the generalization of the statements of Theorem~\ref{tmh:multiscale} to functions that process an arbitrary number of different downscaled signals is straightforward if a proper factorization is provided.

\bibliographystylems{IEEEtran}
\bibliographyms{IEEEabrv,the}

\clearpage
\appendix[Transposed Convolution]
Convolution with trainable filter banks enables artificial neural networks to solve important tasks in the domains of signal restoration and signal classification.
Transposed convolution is the operation that emerges during computation of the gradient of convolution:
Since conventional convolution can be expressed as a matrix-vector product, its gradient involves the transpose of that matrix.
This operator is sometimes also called backwards convolution, fractionally strided convolution, or deconvolution~\citetrconv{Long2015trconv}.
Note that the term deconvolution should not be mixed up in this context with the process of reversing the effects of convolution in the sense of solving an inverse problem.
In artificial neural networks, transposed convolution facilitates an upsampling of internal representations or the network output, which can for example be used to obtain output values for each feasible subsignal of a larger input signal~\citetrconv{Long2015trconv}.

This appendix analyzes transposed convolution in greater detail.
In doing so, an exact definition is given and proved well-defined.
Then, it is analyzed how upsampling with zero-order hold can be realized through transposed convolution.
Afterwards, an operator called Dense Upsampling Convolution, originally proposed by~\citetrconv{Wang2017trconv}, is introduced.
It is eventually shown that this operator is equivalent to a certain parameterization of transposed convolution.

\subsection{Transposed Convolution}
Before transposed convolution can be defined, a few prerequisites are necessary:
\begin{definition}
\label{def:cropping}
Let $P\in\N$ be a natural number, let $M$ be a set.
Then $\Cropping_P\colon\cup_{1 + 2P}(M) \to \cup_1(M)$,
\begin{displaymath}
  \xi\mapsto \sum\nolimits_{\nu = 1}^{\dim_M(\xi) - 2P} \xi_{\nu + P}\cdot e_{\nu}^{\dim_M(\xi) - 2P}\text{,}
\end{displaymath}
is called the \emph{cropping operator}, which removes $P$ samples from both the beginning and the end of the input signal.
\end{definition}

Cropping undoes the effect of padding:
\begin{remark}
For all $P\in\N$ and all boundary handling functions $\vartheta\colon\cup_1(M)\times\Z\to M$, it holds that $\Cropping_P \circ \Padding_P^\vartheta = \id_{\cup_1(M)}$.
\end{remark}
\begin{proof}
Let $P\in\N$ and let $\vartheta\colon\cup_1(M)\times\Z\to M$ be an arbitrary boundary handling function.
Further, let $\xi\in\cup_1(M)$ be an arbitrary signal and let $D := \dim_M(\xi)\in\N_1$ denote its length.
Then $\Padding_P^\vartheta(\xi)\in M^{D + 2P}$ with Definition~\ref{def:padding}.
As $D + 2P \geq 1 + 2P$, $\Cropping_P$ can be applied to the padded signal, producing an output signal from $M^D$ due to Definition~\ref{def:cropping}.
For $i\inint{1}{D}$ follows
\begin{displaymath}
  \Cropping_P(\Padding_P^\vartheta(\xi))_i
  \ \ \equsing{D.\ref{def:cropping}}\ \ \Padding_P^\vartheta(\xi)_{i + P}
  \ \ \equsing{D.\ref{def:padding}}\ \ \vartheta(\xi,\;i + P - P)
  \ \ \equsing{D.\ref{def:padding}}\ \ \xi_i\text{,}
\end{displaymath}
proving the claimed identity.
\end{proof}

Next, an operator which spreads a signal by filling zeros between its samples is defined:
\begin{definition}
\label{def:spreading}
Let $M$ be a set and let $k\in\N_1$ be a stride.
Then $\Spreading_k\colon\cup_1(M) \to \cup_1(M)$,
\begin{displaymath}
  \xi\mapsto \sum\nolimits_{\mu = 1}^{\dim_M(\xi)} \xi_{\mu}\cdot e_{k(\mu - 1) + 1}^{k(\dim_M(\xi) - 1) + 1}\text{,}
\end{displaymath}
is called the \emph{spreading operator}.
It inserts $k - 1$ zeros between the individual samples of the input signal, so that the distance between the original samples in the output signals equals $k$.
\end{definition}

\begin{figure}[t]
  \centering
  \scalebox{1.20}{\includegraphics[page=12]{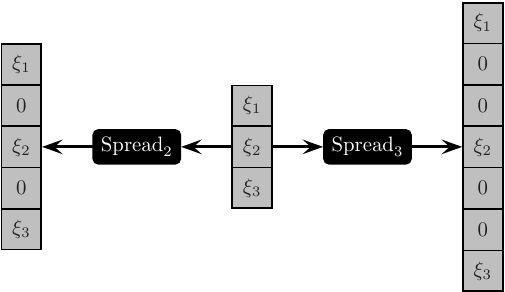}}
  \caption{The spreading operator inserts zeros between the samples of the input signal so that the distance between the original samples equals the stride parameter.
    The number of zeros between each pair of samples is therefore the stride minus one.
    In the depicted example, an input signal $\xi$ with $D = 3$ samples (at the center of the graphics) is spreaded using a stride of $k_1 = 2$ (left-hand side) and a stride of $k_2 = 3$ (right-hand side), yielding output signals with $k_1 (D - 1) + 1 = 5$ and $k_2 (D - 1) + 1 = 7$ samples, respectively.}
  \label{fig:spreading}
\end{figure}

An example of the spreading operator is depicted in Fig.~\ref{fig:spreading}.
A result on the mapping from output sample index to input sample index immediately follows:
\begin{lemma}
\label{lem:spreading}
Let $M$ be a set, let $k\in\N_1$ be a stride and let $\xi\in\cup_1(M)$ be an input signal.
Then for all $i\inint{1}{k(\dim_M(\xi) - 1) + 1}$ it holds that
\begin{displaymath}
  \Spreading_k(\xi)_i =
  \begin{cases}
    \xi_{\div{i - 1}{k} + 1}\text{,} & \text{if }k\text{ divides }i - 1\text{,}\\
    0\text{,} & \text{otherwise.}
  \end{cases}
\end{displaymath}
\end{lemma}
\begin{proof}
First note that for $\mu_1,\mu_2\inint{1}{\dim_M(\xi)}$ from $k(\mu_1 - 1) + 1 = k(\mu_2 - 1) + 1$ always follows that $\mu_1 = \mu_2$ as $k\neq 0$.
Therefore, the output coordinates in the definition of $\Spreading_k$ are distinct.

Let $i\inint{1}{k(\dim_M(\xi) - 1) + 1}$ and consider the case when $k$ divides $i - 1$.
Then clearly $i - 1 = k\cdot\div{i - 1}{k}$.
Define $\mu^* := \div{i - 1}{k} + 1$, then $\mu^*\inint{1}{\dim_M(\xi)}$ as $\mu^*$ is clearly a positive integer and $\mu^* \leq \div{k(\dim_M(\xi) - 1) + 1 - 1}{k} + 1 = \dim_M(\xi)$.
As $k(\mu^* - 1) + 1 = i$ it follows that $\Spreading_k(\xi)_i = \xi_{\mu^*} = \xi_{\div{i - 1}{k} + 1}$.

Now suppose $k$ does not divide $i - 1$.
There cannot exist an index $\mu^*\inint{1}{\dim_M(\xi)}$ with $k(\mu^* - 1) + 1 = i$ since this would imply that $k$ divides $i - 1$.
Therefore, $\Spreading_k(\xi)_i$ vanishes.
\end{proof}

After these preparations, the transposed convolution operator can be defined:
\begin{definition}
\label{def:trconv}
Let $R$ be a ring and let $w\in((R^n)^m)^c$ be a filter bank with spatial extent $c\in\N_1$ for a mapping from $m\in\N_1$ input channels to $n\in\N_1$ output channels.
Let $k\in\N_1$ be a stride and let $P\in\N$ be a padding size.
Let $q := \max\set{1,\ \nceil{\frac{1}{k}(2P - c + 1) + 1}}\in\N_1$ denote a minimum input signal length.
Then $\TransposedConvolution_{(w,\; k,\; P)}\colon\cup_q(R^m) \to \cup_1(R^n)$,
\begin{displaymath}
  \xi\mapsto \Cropping_P( \Padding_{c - 1}^{\vartheta_{\Dirichlet}}( \Spreading_k(\xi)) \conv w)\text{,}
\end{displaymath}
is called the \emph{transposed convolution operator}.
It first spreads the input signal $\xi$ using a stride of $k$.
Then, a ``full'' convolution with the filter bank $w$ is carried out by first padding the spreaded input signal respecting Dirichlet (first-type) boundary conditions and then performing conventional ``valid'' convolution.
Eventually, $P$ samples are removed from both ends of the result of convolution.
\end{definition}

Note that the names of the stride and padding parameters in Definition~\ref{def:trconv} actually refer to convolution, of which transposed convolution computes the gradient.
For transposed convolution, conventional unit-stride convolution is carried out on an input signal that is spreaded using the stride parameter.
Moreover, the padding parameter specifies the number of samples that are removed after the convolution operation.

The minimum input signal length depends on the spatial extent of the used filter bank, the stride, and the padding size.
The next result shows this requirement is correct and deduces the output signal length of transposed convolution:
\begin{lemma}
\label{lem:trconv}
In the situation of Definition~\ref{def:trconv}, the transposed convolution operator is well-defined.
The output signal's spatial extent is $\dim_{R^n}(\TransposedConvolution_{(w,\; k,\; P)}(\xi)) = k\cdot(\dim_{R^m}(\xi) - 1) + c - 2P$ for all $\xi$ in the domain of this operator.
\end{lemma}
\begin{proof}
Let $\xi$ be a signal from the domain of $\TransposedConvolution_{(w,\; k,\; P)}$.
Clearly, $\Spreading_k(\xi)$ and $\Padding_{c - 1}^{\vartheta_{\Dirichlet}}(\Spreading_k(\xi))$ are well-defined since they only require non-empty input signals.
One obtains
\begin{align*}
  & \dim_{R^m}( \Padding_{c - 1}^{\vartheta_{\Dirichlet}}( \Spreading_k(\xi)) )\\
  \equsing{D.~\ref{def:padding}}\ \ \ & \dim_{R^m}( \Spreading_k(\xi) ) + 2(c - 1)\\
  \equsing{D.~\ref{def:spreading}}\ \ \ & k\cdot(\dim_{R^m}(\xi) - 1) + 1 + 2(c - 1)\text{.}
\end{align*}
This length is at least $2c - 1$ since $k\geq 1$ and $\dim_{R^m}(\xi) \geq 1$.
As $2c - 1 \geq c$ because $c \geq 1$, $\Padding_{c - 1}^{\vartheta_{\Dirichlet}}(\Spreading_k(\xi))$ is large enough to be convolved by $w$ as stated in Sect.~\ref{sect:CNNs-wo-pooling}.
Furthermore,
\begin{align*}
  & \dim_{R^n}( \Padding_{c - 1}^{\vartheta_{\Dirichlet}}( \Spreading_k(\xi)) \conv w)\\
  \equsing{Sect.~\ref{sect:CNNs-wo-pooling}}\ \ \ \ & \dim_{R^m}( \Padding_{c - 1}^{\vartheta_{\Dirichlet}}( \Spreading_k(\xi)) ) - (c - 1)\\
  =\ \ \ \ & k\cdot(\dim_{R^m}(\xi) - 1) + c\text{.}
\end{align*}
For showing that transposed convolution is well-defined it has finally to be demonstrated that this length is at least $1 + 2P$ due to Definition~\ref{def:cropping}.
First consider the case when $\nceil{\frac{1}{k}(2P - c + 1) + 1} \leq 1$.
This implies that $2P - c + 1 \leq 0$.
As it is further guaranteed that $\dim_{R^m}(\xi) \geq 1$, it follows that $k\cdot(\dim_{R^m}(\xi) - 1) + c \geq c \geq 2P + 1$, hence well-definedness follows in this case.
If, on the other hand, $\nceil{\frac{1}{k}(2P - c + 1) + 1} > 1$, then
\begin{displaymath}
  k\cdot(\dim_{R^m}(\xi) - 1) + c
  > k\cdot(\nceil{\tfrac{1}{k}(2P - c + 1) + 1} - 1) + c
  \geq k\cdot\tfrac{1}{k}(2P - c + 1) + c
  = 2P + 1\text{.}
\end{displaymath}
Transposed convolution is therefore well-defined in this case as well.

The output length of transposed convolution is
\begin{align*}
  & \dim_{R^n}(\TransposedConvolution_{(w,\; k,\; P)}(\xi))\\
  \equsing{D.~\ref{def:trconv}}\ \ & \dim_{R^n}( \Cropping_P( \Padding_{c - 1}^{\vartheta_{\Dirichlet}}( \Spreading_k(\xi)) \conv w) )\\
  \equsing{D.~\ref{def:cropping}}\ \ & \dim_{R^n}( \Padding_{c - 1}^{\vartheta_{\Dirichlet}}( \Spreading_k(\xi)) \conv w) - 2P\\
  =\ \ & k\cdot(\dim_{R^m}(\xi) - 1) + c - 2P\text{,}
\end{align*}
as has been claimed.
\end{proof}

\subsection{Upsampling with Transposed Convolution}
From Lemma~\ref{lem:trconv} it is evident that transposed convolution is capable of producing signals with an output dimensionality strictly greater than the input dimensionality.
Therefore, this cannot be a subsignal compatible transformation.
Instead, it is possible to upsample signals with this operation.
It is shown here explicitly that a certain parameterization and filter bank lead to upsampling with zero-order hold as introduced in Definition~\ref{def:upsampling}.

The analysis is begun with the deduction of a parameterization that achieves the mandatory dimensionality for upsampling by an even integer factor:
\begin{remark}
\label{rem:upsmp-trconv}
Suppose that $u\in\N_1$ denotes an even upsampling factor.
Let $R$ be a ring and $w\in((R^n)^m)^c$ a filter bank for a number of input channels $m\in\N_1$ and output channels $n\in\N_1$, where the spatial extent is $c := 2u\in\N_1$.
Let $k := u\in\N_1$ and $P := u/2 \in\N_1$ denote the stride and the padding size, respectively.
Then $\dim_{R^n}(\TransposedConvolution_{(w,\; k,\; P)}(\xi)) = u\cdot\dim_{R^m}(\xi)$ for all $\xi\in\cup_1(R^m)$.
\end{remark}
\begin{proof}
The quantities $c$ and $k$ are clearly positive natural numbers.
This also applies to $P$ since $u$ has been required to be even and positive.
It is now $\nceil{\frac{1}{k}(2P - c + 1) + 1} = \nceil{\frac{1}{u}} = 1$ as $u \geq 2$, therefore the minimum input signal length for transposed convolution using this parameterization is $1$.
For $\xi\in\cup_1(R^m)$ it is then $\dim_{R^n}(\TransposedConvolution_{(w,\; k,\; P)}(\xi)) \ \equsing{L.~\ref{lem:trconv}}\ u\cdot(\dim_{R^m}(\xi) - 1) + 2u - u = u\cdot\dim_{R^m}(\xi)$.
\end{proof}

Now, a few statements regarding the properties of sample indices that emerge in the analysis of upsampling with transposed convolution are formulated.
The first result will be used later to show that in a special case convolution with a spreaded signal reduces to the multiplication of two scalars since all other relevant values of the spreaded signal vanish:
\begin{lemma}
\label{lem:upsmp-idx-unique}
Assume $u\in\N_1$, $i\in\N$ and $\mu\inint{1}{u}$.
Then $\rem{i - \mu}{u} = 0$ if and only if $\mu = \rem{i - 1}{u} + 1$.
\end{lemma}
\begin{proof}
Consider the function $h\colon\discint{1}{u}\to\discint{0}{u - 1}$, $\beta\mapsto\rem{i - \beta}{u}$.
Let $\gamma\inint{0}{u - 1}$ be arbitrary and define $\beta^* := \rem{i - \gamma - 1}{u} + 1 \inint{1}{u}$.
Then
\begin{displaymath}
  h(\beta^*) = \rem{i - \rem{i - \gamma - 1}{u} - 1}{u} = \rem{\gamma}{u} = \gamma
\end{displaymath}
due to the idempotence of the $\operatorname{rem}$ operator.
Therefore, $h$ is surjective.
As a mapping between finite sets with the same cardinality it is hence a bijection.
The claim then follows since $h( \rem{i - 1}{u} + 1 ) = 0$.
\end{proof}

The next statement will be used to prove that in a special case boundary handling is not necessary since all accessed samples are already within bounds:
\begin{lemma}
\label{lem:upsmp-no-boundhand}
Let $u,D\in\N_1$ and $i\inint{1}{uD}$. Then $i - \rem{i - 1}{u} \inint{1}{u(D - 1) + 1}$.
\end{lemma}
\begin{proof}
Clearly, $i - \rem{i - 1}{u}$ is an integer.
From $i - 1 = \div{i - 1}{u}\cdot u + \rem{i - 1}{u}$ follows $i - \rem{i - 1}{u} = \div{i - 1}{u}\cdot u + 1 \geq 1$ as $i \geq 1$.
On the other hand,
\begin{align*}
  & i - \rem{i - 1}{u}
  \ =\ \div{i - 1}{u}\cdot u + 1
  \ \leq\ \div{uD - 1}{u}\cdot u + 1\\
  \ =\ &uD - 1 - \rem{uD - 1}{u} + 1
  \ \equsing{P.~\ref{prop:number-theory}}\ uD - 1 - (u - 1) + 1
  \ =\ u(D - 1) + 1\text{.}
\end{align*}
The claim follows as both bounds are satisfied.
\end{proof}

Eventually, the combination of the index derived in Lemma~\ref{lem:upsmp-idx-unique} with the transformation due to spreading can be simplified:
\begin{lemma}
\label{lem:upsmp-idx-simplify}
Suppose that $u\in\N_1$ and $i\in\N$.
Then $\div{i - \rem{i - 1}{u} - 1}{u} = \div{i - 1}{u}$.
\end{lemma}
\begin{proof}
It holds that
\begin{align*}
  & i - \rem{i - 1}{u} - 1\\
  =\ & \div{i - \rem{i - 1}{u} - 1}{u}\cdot u + \rem{i - \rem{i - 1}{u} - 1}{u}\\
  =\ & \div{i - \rem{i - 1}{u} - 1}{u}\cdot u
\end{align*}
as the $\operatorname{rem}$ operator is idempotent.
As $i - 1 = \div{i - 1}{u}\cdot u + \rem{i - 1}{u}$, it follows that
\begin{align*}
  & u\cdot\big( \div{i - 1}{u} - \div{i - \rem{i - 1}{u} - 1}{u} \big)\\
  =\ & \big(i - 1 - \rem{i - 1}{u}\big) - \big( i - \rem{i - 1}{u} - 1 \big)
  \ =\ 0\text{.}
\end{align*}
The claim follows because $u\neq 0$.
\end{proof}

After these preparations, it can be shown how transposed convolution is able to implement upsampling with zero-order hold:
\begin{theorem}
\label{thm:trconvupzoh}
Let $u\in\N_1$ be an even upsampling factor and let $m\in\N_1$ denote a number of channels.
Moreover, let $R$ be a ring and suppose $w^{\ZOH}\in((R^m)^m)^{2u}$ is a filter bank, where the individual entries are given by
\begin{displaymath}
  ((w^{\ZOH}_\mu)_\lambda)_\kappa :=
  \begin{cases}
    1\text{,} & \text{if }\lambda = \kappa\text{ and }\mu\inint{1 + u/2}{u + u/2}\text{,}\\
    0\text{,} & \text{otherwise,}
  \end{cases}
\end{displaymath}
for all $\lambda,\kappa\inint{1}{m}$ and all $\mu\inint{1}{2u}$.
Then $\TransposedConvolution_{(w^{\ZOH},\; u,\; u/2)} = \Upsampling_u$ on $\cup_1(R^m)$.
In other words, transposed convolution with $w^{\ZOH}$ carries out upsampling with zero-order hold.
\end{theorem}
\begin{proof}
With Remark~\ref{rem:upsmp-trconv} follows that the parameterization in the claim performs an upsampling of factor $u$ for all non-empty input signals.
Therefore it only remains to be shown that the output samples match.
Let $\xi\in\cup_1(R^m)$ be an input signal and write $D := \dim_{R^m}(\xi)\in\N_1$.
Further, let $i\inint{1}{uD}$ be an arbitrary index with respect to the spatial output dimension and let $\kappa\inint{1}{m}$ denote the index of an arbitrary output channel.
Then
\begin{align*}
  & (\TransposedConvolution_{(w^{\ZOH},\; u,\; u/2)}(\xi)_i)_\kappa\\
  \equsing{D.~\ref{def:trconv}}\ \ \ \ & (\Cropping_{u/2}( \Padding_{2u - 1}^{\vartheta_{\Dirichlet}}( \Spreading_u(\xi)) \conv w^{\ZOH})_i)_\kappa\\
  \equsing{D.~\ref{def:cropping}}\ \ \ \ & ((\Padding_{2u - 1}^{\vartheta_{\Dirichlet}}( \Spreading_u(\xi)) \conv w^{\ZOH})_{i + u/2})_\kappa\\
  \equsing{Sect.~\ref{sect:CNNs-wo-pooling}}\ \ \ \ & \left(\sum\nolimits_{\lambda = 1}^m\sum\nolimits_{\mu = 1}^{2u} (w^{\ZOH}_\mu)_\lambda \cdot (\Padding_{2u - 1}^{\vartheta_{\Dirichlet}}( \Spreading_u(\xi))_{2u + i + u/2 - \mu})_\lambda\right)_\kappa\\
  \equsing{Sect.~\ref{sect:CNNs-wo-pooling}}\ \ \ \ & \sum\nolimits_{\lambda = 1}^m\sum\nolimits_{\mu = 1}^{2u} ((w^{\ZOH}_\mu)_\lambda)_\kappa \cdot (\Padding_{2u - 1}^{\vartheta_{\Dirichlet}}( \Spreading_u(\xi))_{2u + i + u/2 - \mu})_\lambda\\
  \equsing{($\lozenge$)}\ \ \ \ & \sum\nolimits_{\mu = 1}^{2u} ((w^{\ZOH}_\mu)_\kappa)_\kappa \cdot (\Padding_{2u - 1}^{\vartheta_{\Dirichlet}}( \Spreading_u(\xi))_{2u + i + u/2 - \mu})_\kappa\\
  \equsing{($\lozenge$)}\ \ \ \ & \sum\nolimits_{\mu = 1 + u/2}^{u + u/2} (\Padding_{2u - 1}^{\vartheta_{\Dirichlet}}( \Spreading_u(\xi))_{2u + i + u/2 - \mu})_\kappa\\
  \equsing{($\lozenge$)}\ \ \ \ & \sum\nolimits_{\mu = 1}^{u} (\Padding_{2u - 1}^{\vartheta_{\Dirichlet}}( \Spreading_u(\xi))_{2u + i + u/2 - \mu - u/2})_\kappa\\
  \equsing{D.~\ref{def:padding}}\ \ \ \ & \sum\nolimits_{\mu = 1}^{u} \vartheta_{\Dirichlet}( \Spreading_u(\xi),\;2u + i - \mu - (2u - 1))_\kappa\\
  \equsing{D.~\ref{def:spreading}}\ \ \ \ & \sum\nolimits_{\mu = 1}^{u}
      \begin{cases}
        (\Spreading_u(\xi)_{i - \mu + 1})_\kappa\text{,} & \text{if }i - \mu + 1\inint{1}{u(D - 1) + 1}\text{,}\\
        0\text{,} & \text{otherwise}
      \end{cases}\\
  \equsing{L.~\ref{lem:spreading}}\ \ \ \ & \sum\nolimits_{\mu = 1}^{u}
      \begin{cases}
        (\xi_{\div{i - \mu}{u} + 1})_\kappa\text{,} & \text{if }i - \mu + 1\inint{1}{u(D - 1) + 1}\text{ and }\rem{i - \mu}{u} = 0\text{,}\\
        0\text{,} & \text{otherwise}
      \end{cases}\\
  \equsing{L.~\ref{lem:upsmp-idx-unique}}\ \ \ \ &
      \begin{cases}
        (\xi_{\div{i - \rem{i - 1}{u} - 1}{u} + 1})_\kappa\text{,} & \text{if }i - \rem{i - 1}{u}\inint{1}{u(D - 1) + 1}\text{,}\\
        0\text{,} & \text{otherwise}
      \end{cases}\\
  \equsing{L.~\ref{lem:upsmp-no-boundhand}}\ \ \ \ & (\xi_{\div{i - \rem{i - 1}{u} - 1}{u} + 1})_\kappa\\
  \equsing{L.~\ref{lem:upsmp-idx-simplify}}\ \ \ \ & (\xi_{\div{i - 1}{u} + 1})_\kappa\\
  \equsing{L.~\ref{lem:upsampling-props}}\ \ \ \ & (\Upsampling_u(\xi)_i)_\kappa\text{.}
\end{align*}
In the first two ($\lozenge$) steps it has been used that $((w^{\ZOH}_\mu)_\lambda)_\kappa \neq 0$ holds only if $\lambda = \kappa$ and only if $\mu\inint{1 + u/2}{u + u/2}$, respectively.
Consequently, an index shift has been carried out in the third ($\lozenge$) step.
As both output signals are identical, the proof is completed.
\end{proof}

The filter bank $w^{\ZOH}$ from Theorem~\ref{thm:trconvupzoh} is sparse and binary.
First, there is no interconnection between distinct channels, that is each channel is processed independently of the others.
Second, non-vanishing entries are located only in the middle of the spatial dimension.
Therefore, each output sample depends on only a single input sample due to the spreading of the input signal.
More sophisticated upsampling methods such as linear interpolation can be realized through filter banks where the entire spatial dimension is populated with non-vanishing entries.

\subsection{Dense Upsampling Convolution}
Now, an operator known as Dense Upsampling Convolution (DUC)~\citetrconv{Wang2017trconv} is analyzed.
The DUC operator simultaneously carries out an upsampling operation and a mapping from a number of input channels to a number of output channels.
This can be used for example in classification tasks to achieve a projection of a downscaled internal representation into label space at an increased spatial resolution~\citetrconv{Wang2017trconv}.
Here, the employed filter bank is adapted on concrete sample data with the intention to recover detailed information that is otherwise lost if conventional interpolation was used.
This is, however, only a heuristic which does not lead to the same result as exact dense signal scanning.

First consider the definition of this operator:
\begin{definition}
\label{def:duc}
Let $u\in\N_1$ be an upsampling factor, let $R$ be a ring and let $m,n\in\N_1$ be positive natural numbers.
Furthermore, suppose $w\in((R^{un})^m)^1$ is a filter bank with unit spatial extent used for mapping from $m$ input channels to $un$ output channels.
For arbitrary positive natural numbers $a,b\in\N_1$ let $\Phi_{a,\; b}\colon (R^a)^b \to R^{b\times a}$, $\Phi_{a,\; b}(\xi)_{j,\; i} := (\xi_j)_i$ for all $\xi\in(R^a)^b$, all $j\inint{1}{b}$ and all $i\inint{1}{a}$, be a helper function for transforming from a multi-channel representation to a fragmented representation.
This function is clearly a bijection, its inverse is given by $\Phi_{a,\; b}^{-1}\colon R^{b\times a} \to (R^a)^b$, $(\Phi_{a,\; b}^{-1}(\xi)_j)_i = \xi_{j,\; i}$ for all $\xi\in R^{b\times a}$, all $j\inint{1}{b}$ and all $i\inint{1}{a}$.
Then $\DUC_w\colon\cup_1(R^m)\to\cup_u(R^n)$,
\begin{displaymath}
  \xi\mapsto \Phi_{n,\; u\cdot\dim_{R^m}(\xi)}^{-1}( \Defragmentation_u( \Phi_{un,\; \dim_{R^m}(\xi)}( \xi \conv w ) ) )\text{,}
\end{displaymath}
is called the \emph{Dense Upsampling Convolution operator}.
It convolves the input signal with a filter bank that provides an increased number of output channels followed by defragmentation to restore an intended number of output channels at an increased spatial resolution.
\end{definition}

Note that here it is not required that the upsampling factor is even.
The next statement formalizes well-definedness and output dimensionalities of $\DUC$:
\begin{lemma}
\label{lem:duc}
In the situation of Definition~\ref{def:duc}, the $\DUC$ operator is well-defined.
It increases the spatial extent of the output signal by factor $u$, that is $\dim_{R^n}(\DUC_w(\xi)) = u\cdot \dim_{R^m}(\xi)$ for all $\xi\in\cup_1(R^m)$.
\end{lemma}
\begin{proof}
Let $\xi\in\cup_1(R^m)$ be an input signal and $D := \dim_{R^m}(\xi)\in\N_1$ its length.
As $D \geq 1$, $\xi \conv w$ is well-defined and produces a signal from $(R^{un})^D$.
It is then $\Phi_{un,\; D}( \xi \conv w )\in R^{D\times un}$, hence defragmentation can be applied and leads to a signal $\Defragmentation_u( \Phi_{un,\; D}( \xi \conv w ) )\in R^{uD\times n}$.
Eventually, it holds that $\DUC_w(\xi) = \Phi_{n,\; uD}^{-1}( \Defragmentation_u( \Phi_{un,\; D}( \xi \conv w ) ) )\in (R^n)^{uD}$, which implies that $\dim_{R^n}(\DUC_w(\xi)) = uD$.
\end{proof}

Dense Upsampling Convolution is equivalent to a special case of transposed convolution:
\begin{theorem}
\label{thm:duc_trconv}
Let $R$ be a ring.
Suppose $m\in\N_1$ and $n\in\N_1$ denote a number of input and output channels, respectively.
Further, let $u\in\N_1$ denote an upsampling factor and let $w\in((R^n)^m)^u$ be a filter bank.
The function $\Psi\colon((R^n)^m)^u\to ((R^{un})^m)^1$, $((\Psi(w)_1)_\lambda)_\nu := ((w_{\div{\nu - 1}{n} + 1})_\lambda)_{\rem{\nu - 1}{n} + 1}$ for all $\nu\inint{1}{un}$ and all $\lambda\inint{1}{m}$, realizes a reordering of the filter bank to a shape compatible with the Dense Upsampling Convolution operator.
Then $\DUC_{\Psi(w)} = \TransposedConvolution_{(w,\; u,\; 0)}$ on $\cup_1(R^m)$.
Here, the transposed convolution operator is used with a spatial filter size that equals the stride, which is furthermore equivalent to the upsampling factor $u$.
Padding is not used since the padding size vanishes.
\end{theorem}
\begin{proof}
Suppose that $\xi\in(R^m)^D$ with $D\in\N_1$ is an input signal.
Then $\DUC_{\Psi(w)}(\xi)\in(R^n)^{uD}$ with Lemma~\ref{lem:duc}.
The minimum spatial input signal length for transposed convolution is here unity.
Further, $\TransposedConvolution_{(w,\; u,\; 0)}(\xi)\in(R^n)^{uD}$ due to Lemma~\ref{lem:trconv}.
The signals on both sides of the claimed identity are from the same set and can hence be compared on the sample level.

The analysis is begun with the Dense Upsampling Convolution operator.
Let $i\inint{1}{uD}$ and $\kappa\inint{1}{n}$ be arbitrary indices.
Then
\begin{align*}
  & (\DUC_{\Psi(w)}(\xi)_i)_\kappa\\
  \equsing{D.~\ref{def:duc}}\ \ \ \ & (\Phi_{n,\; uD}^{-1}( \Defragmentation_u( \Phi_{un,\; D}( \xi \conv \Psi(w) ) ) )_i)_\kappa\\
  \equsing{D.~\ref{def:duc}}\ \ \ \ & \Defragmentation_u( \Phi_{un,\; D}( \xi \conv \Psi(w) ) )_{i,\; \kappa}\\
  \equsing{L.~\ref{lem:defrag-ops-simplified}}\ \ \ \ & \Phi_{un,\; D}( \xi \conv \Psi(w) )_{\div{i - 1}{u} + 1,\; \rem{i - 1}{u}\cdot n + \kappa}\\
  \equsing{D.~\ref{def:duc}}\ \ \ \ & ( ( \xi \conv \Psi(w) )_{\div{i - 1}{u} + 1} )_{\rem{i - 1}{u}\cdot n + \kappa}\\
  \equsing{Sect.~\ref{sect:CNNs-wo-pooling}}\ \ \ \ & \left(\sum\nolimits_{\lambda = 1}^m (\Psi(w)_1)_\lambda \cdot (\xi_{1 + \div{i - 1}{u} + 1 - 1})_\lambda \right)_{\rem{i - 1}{u}\cdot n + \kappa}\\
  \equsing{Sect.~\ref{sect:CNNs-wo-pooling}}\ \ \ \ & \sum\nolimits_{\lambda = 1}^m ((\Psi(w)_1)_\lambda)_{\rem{i - 1}{u}\cdot n + \kappa} \cdot (\xi_{\div{i - 1}{u} + 1})_\lambda\\
  \equsing{($\lozenge$)}\ \ \ \ & \sum\nolimits_{\lambda = 1}^m ((w_{\rem{i - 1}{u} + 1})_\lambda)_\kappa \cdot (\xi_{\div{i - 1}{u} + 1})_\lambda\text{,}
\end{align*}
where in the ($\lozenge$) step the reordering of the filter bank has been made explicit:
\begin{align*}
  & ((\Psi(w)_1)_\lambda)_{\rem{i - 1}{u}\cdot n + \kappa}\\
  =\ \ & ((w_{\div{ \rem{i - 1}{u}\cdot n + \kappa - 1 }{n} + 1})_\lambda)_{\rem{ \rem{i - 1}{u}\cdot n + \kappa - 1 }{n} + 1}\\
  \equsing{P.~\ref{prop:number-theory}}\ \ & ((w_{\rem{i - 1}{u} + \div{\kappa - 1}{n} + 1})_\lambda)_{\rem{\kappa - 1}{n} + 1}\\
  =\ \ & ((w_{\rem{i - 1}{u} + 1})_\lambda)_\kappa\text{.}
\end{align*}
In the final step it has been used that $\kappa\inint{1}{n}$, implying $\div{\kappa - 1}{n} = 0$ and $\rem{\kappa - 1}{n} = \kappa - 1$.

Now, for considering transposed convolution suppose $i\inint{1}{uD}$ and $\kappa\inint{1}{n}$ are indices.
It is then
\begin{align*}
  & (\TransposedConvolution_{(w,\; u,\; 0)}(\xi)_i)_\kappa\\
  \equsing{D.~\ref{def:trconv}}\ \ \ \ & (\Cropping_0( \Padding_{u - 1}^{\vartheta_{\Dirichlet}}( \Spreading_u(\xi)) \conv w)_i)_\kappa\\
  \equsing{D.~\ref{def:cropping}}\ \ \ \ & ((\Padding_{u - 1}^{\vartheta_{\Dirichlet}}( \Spreading_u(\xi)) \conv w)_i)_\kappa\\
  \equsing{Sect.~\ref{sect:CNNs-wo-pooling}}\ \ \ \ & \left(\sum\nolimits_{\lambda = 1}^m\sum\nolimits_{\mu = 1}^u (w_\mu)_\lambda \cdot (\Padding_{u - 1}^{\vartheta_{\Dirichlet}}( \Spreading_u(\xi))_{u + i - \mu})_\lambda\right)_\kappa\\
  \equsing{Sect.~\ref{sect:CNNs-wo-pooling}}\ \ \ \ & \sum\nolimits_{\lambda = 1}^m\sum\nolimits_{\mu = 1}^u ((w_\mu)_\lambda)_\kappa \cdot (\Padding_{u - 1}^{\vartheta_{\Dirichlet}}( \Spreading_u(\xi))_{u + i - \mu})_\lambda\\
  \equsing{D.~\ref{def:padding}}\ \ \ \ & \sum\nolimits_{\lambda = 1}^m\sum\nolimits_{\mu = 1}^u ((w_\mu)_\lambda)_\kappa \cdot \vartheta_{\Dirichlet}( \Spreading_u(\xi),\;u + i - \mu - (u - 1))_\lambda\\
  \equsing{D.~\ref{def:spreading}}\ \ \ \ & \sum\nolimits_{\lambda = 1}^m\sum\nolimits_{\mu = 1}^u ((w_\mu)_\lambda)_\kappa \cdot
      \begin{cases}
          (\Spreading_u(\xi)_{i - \mu + 1})_\lambda\text{,} & \text{if }i - \mu + 1\inint{1}{u(D - 1) + 1}\text{,}\\
          0\text{,} & \text{otherwise}
        \end{cases}\\
  \equsing{L.~\ref{lem:spreading}}\ \ \ \ & \sum\nolimits_{\lambda = 1}^m\sum\nolimits_{\mu = 1}^u ((w_\mu)_\lambda)_\kappa \cdot
      \begin{cases}
          (\xi_{\div{i - \mu}{u} + 1})_\lambda\text{,} & \text{if }i - \mu + 1\inint{1}{u(D - 1) + 1} \\& \text{and }\rem{i - \mu}{u} = 0\text{,}\\
          0\text{,} & \text{otherwise}
        \end{cases}\\
  \equsing{L.~\ref{lem:upsmp-idx-unique}}\ \ \ \ & \sum\nolimits_{\lambda = 1}^m ((w_{\rem{i - 1}{u} + 1})_\lambda)_\kappa \cdot
      \begin{cases}
          (\xi_{\div{i - \rem{i - 1}{u} - 1}{u} + 1})_\lambda\text{,} & \text{if }\discint{1}{u(D - 1) + 1} \\& \text{contains }i - \rem{i - 1}{u}\text{,}\\
          0\text{,} & \text{otherwise}
        \end{cases}\\
  \equsing{L.~\ref{lem:upsmp-no-boundhand}}\ \ \ \ & \sum\nolimits_{\lambda = 1}^m ((w_{\rem{i - 1}{u} + 1})_\lambda)_\kappa \cdot (\xi_{\div{i - \rem{i - 1}{u} - 1}{u} + 1})_\lambda\\
  \equsing{L.~\ref{lem:upsmp-idx-simplify}}\ \ \ \ & \sum\nolimits_{\lambda = 1}^m ((w_{\rem{i - 1}{u} + 1})_\lambda)_\kappa \cdot (\xi_{\div{i - 1}{u} + 1})_\lambda\\
  =\ \ \ \ & (\DUC_{\Psi(w)}(\xi)_i)_\kappa\text{,}
\end{align*}
which proves the claim.
\end{proof}

Theorem~\ref{thm:duc_trconv} guarantees that Dense Upsampling Convolution can be expressed as a special case of transposed convolution without any accuracy loss.
In practice, either the implementation using conventional convolution and defragmentation or the alternative implementation using transposed convolution can be more efficient.
This depends on the concrete parameterization, employed processor, and implementation of the required routines.

\bibliographystyletrconv{IEEEtran}
\bibliographytrconv{IEEEabrv,the}

\clearpage
\appendix[Functions Applied in a Dilated Fashion]
This appendix studies an alternative to the $\EvalSlide$ operator proposed in the main part of this paper.
In their original work, Li \emph{et al.}~\citedil{Li2014dil} proposed modifying the way convolution and pooling are carried out to transform a CNN from a subsignal-based application to a signal-based application.
In doing so, their $d$-regularly sparse kernels enlarge the receptive field size of these two operators through insertion of vanishing entries at regular intervals.
Until now, it has not been proven rigorously that this method does not suffer from accuracy loss.
Moreover, similar if not equivalent concepts have been reported most recently under different names, namely 
filter rarefactions~\citedil{Long2015dil}, strided kernels~\citedil{Tschopp2015dil}, dilated convolutions~\citedil{Yu2016dil}, and atrous convolutions~\citedil{Chen2016dil}.
It has further been claimed by~\citedil{Zlateski2016dil} that these methods are equivalent to a fragmentation-based approach as advocated in this paper.

In this appendix, subsignal compatible transformation theory is used to analyze such modified operators.
For notational convenience, this is hereafter referred to as the application of functions in a dilated fashion.
It is proved eventually that there is a direct connection between the naive subsignal-based application of a CNN using the $\EvalStride$ operator and the dilated signal processing scheme.
As a consequence, it is shown here that the approach of \citedil{Li2014dil} does not involve any accuracy loss.
Further, a computational complexity analysis shows that dilated function application and the application of functions in a sliding fashion to fragmented signals share the same number of required function applications.
In practice, however, a fragmentation-based approach is more beneficial since it facilitates direct usage of readily available routines for the computationally most demanding tensor convolution.
This has also been noted by~\citedil{Chen2016dil}, who first implemented their atrous convolution using a dilated function application scheme before switching to fragmentation.

\subsection{Dilated Subsignals and Functions Applied in a Dilated Fashion}
First, the definition of how access to dilated subsignals is realized:
\begin{definition}
\label{def:dilatedsubsignal}
Let $M$ be a set.
Let $d\in\N_1$ denote a dimensionality and $k\in\N_1$ a stride.
The function $\DilatedSubsignal_{(d,\; k)}\colon\bigcup_{D = k(d - 1) + 1}^\infty\big(M^D\times\discint{1}{D - k(d - 1)}\big)\to M^d$,
\begin{displaymath}
  (\xi,\; i)\mapsto\sum\nolimits_{\nu = 1}^d \xi_{i + k(\nu - 1)}\cdot e_\nu^d\text{,}
\end{displaymath}
is called the \emph{dilated subsignal extraction operator}.
\end{definition}

\begin{figure}[t]
  \centering
  \scalebox{1.20}{\includegraphics[page=13]{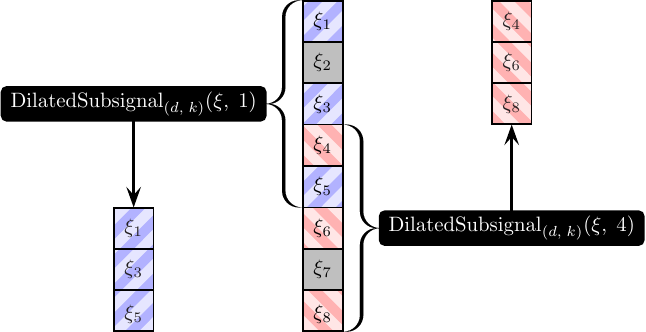}}
  \caption{Extraction of dilated subsignals from a signal $\xi$ with $D = 8$ samples, where each extracted dilated subsignal with $d = 3$ samples has a stride of $k = 2$ within the original signal. On the left-hand side, the first feasible dilated subsignal is shown. The right-hand side depicts the final dilated subsignal with the maximum index of $D - k(d - 1) = 4$.}
  \label{fig:dilatedsubsignal}
\end{figure}

An illustration of this operator is shown in Fig.~\ref{fig:dilatedsubsignal}.
It is clear by definition that the $i$-th dilated subsignal starts at the $i$-th sample of the original input signal, and that $k - 1$ samples from the input signal are always skipped during extraction of individual samples.
The minimum input signal length of $k(d - 1) + 1$ ensures the set of feasible dilated subsignal indices is always non-empty, and hence always contains unity.
The maximum dilated subsignal index of $D - k(d - 1)$ guarantees that the maximum sample index used for accessing the input signal corresponds to the signal length $D$.
Therefore, all accesses are within bounds and this operator is well-defined.
Eventually, this operator becomes a strict generalization of the subsignal extraction operator since for $k = 1$ follows $\DilatedSubsignal_{(d,\; 1)} = \Subsignal_d$.

Now to the usage of this operator for function evaluation:
\begin{definition}
\label{def:dilate}
Let $M$ and $N$ be sets, let $c\in\N_1$ and let $f\colon M^c\to N$ be a function.
Suppose $k\in\N_1$ is a stride.
Then $\Dilate_{(f,\; k)}\colon\cup_{k(c - 1) + 1}(M)\to\cup_1(N)$,
\begin{displaymath}
  \xi\mapsto\sum\nolimits_{i = 1}^{\dim_M(\xi) - k(c - 1)} f(\DilatedSubsignal_{(c,\;k)}(\xi,\;i))\cdot e_i^{\dim_M(\xi) - k(c - 1)}\text{,}
\end{displaymath}
is the operator that applies $f$ in a \emph{dilated fashion} with a stride of $k$.
\end{definition}

In other words, this operator extracts all feasible dilated subsignals from the input signal $\xi$ using the dimensionality $c$ of the function $f$ and the specified stride $k$, applies $f$ and stores the outcome contiguously in a new signal.
The minimum length of the input signal $k(c - 1) + 1$ is chosen here so that at least one dilated subsignal exists.
The dimensionality of the result, $\dim_M(\xi) - k(c - 1)$, matches exactly the number of dilated subsignals in $\xi$.
Therefore, $\Dilate_{(f,\; k)}$ is well-defined.

\begin{figure}[t]
  \centering
  \scalebox{1.20}{\includegraphics[page=14]{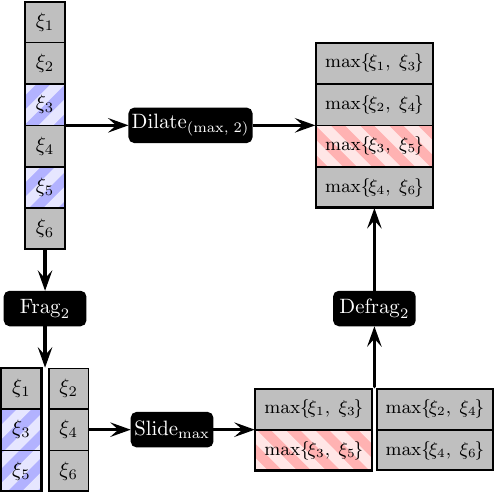}}
  \caption{The upper part depicts an illustration of the $\max$ operator applied in a dilated fashion to an input signal $\xi$ with $D = 6$ samples.
    Here, $\max$ accepts $c = 2$ input samples and is applied to dilated subsignals using a stride of $k = 2$.
    The lower portion illustrates the statement from Remark~\ref{rem:dilated-fragment}:
    If the input is first fragmented using the same parameter $k = 2$, followed by application of $\max$ in a sliding fashion and a defragmentation operation, this has the same outcome as the application of $\max$ in a dilated fashion.
    The combination of a corresponding pair of input samples and resulting output sample is highlighted with colored patterns in the graphics, illustrating that $\Dilate$ has to access strided data for each output sample whereas $\Slide$ is able to process contiguous samples due to the fragmentation carried out beforehand.
}
  \label{fig:dilate}
\end{figure}

Moreover, this is a generalization of the application of a function in a sliding fashion, since for $k = 1$ it is $\Dilate_{(f,\; 1)} = \Slide_f$.
It is straightforward to verify that the application of a function in a dilated fashion has a direct connection to the application in a sliding fashion, where for $k > 1$ non-trivial fragmentation and defragmentation are involved (see Fig.~\ref{fig:dilate} for orientation):
\begin{remark}
\label{rem:dilated-fragment}
Let $M,N$ be sets, $c\in\N_1$, $f\colon M^c\to N$ a function, and $k\in\N_1$.
Then
\begin{displaymath}
  \Dilate_{(f,\; k)}(\xi) = \Defragmentation_k(\Slide_f(\Fragmentation_k(\xi)))
\end{displaymath}
for all $\xi\in\cup_{k(c - 1) + 1}(M)$ where $k$ divides $\dim_M(\xi)$.
\end{remark}
\begin{proof}
$\Dilate_{(f,\; k)}$ is well-defined by the choice of $\xi$.
$\Fragmentation_k(\xi)$ is well-defined since $\dim_M(\xi)$ was required to be divisible by $k$.
Since $\Slide_f$ does not alter the number of fragments, $\Defragmentation_k$ can be applied to its result here.

Write $D := \dim_M(\xi)$, then the left-hand side has $D - k(c - 1)$ samples.
Since $\Fragmentation_k(\xi)\in M^{(D/k)\times k}$ follows $\Slide_f(\Fragmentation_k(\xi))\in N^{(D/k - c + 1)\times k}$, therefore the result of defragmentation has exactly one fragment with $k\cdot (D/k - c + 1) = D - k(c - 1)$ samples.
As dimensionalities match on both sides, it remains to be shown that all the samples are equivalent.

Let $\mu\inint{1}{D - k(c - 1)}$, then
\begin{align*}
  & \Defragmentation_k(\Slide_f(\Fragmentation_k(\xi)))_\mu\\
  \equsing{L.~\ref{lem:defrag-ops}}\ \ & \Slide_f(\Fragmentation_k(\xi))_{\div{\mu - 1}{k} + 1,\; \rem{\mu - 1}{k} + 1}\\
  \equsing{D.~\ref{def:sliding-function}}\ \ & f\left(\sum\nolimits_{\nu = 1}^c \Fragmentation_k(\xi)_{\div{\mu - 1}{k} + \nu,\; \rem{\mu - 1}{k} + 1}\cdot  e_\nu^c\right)\\
  \equsing{L.~\ref{lem:frag-ops}}\ \ & f\left(\sum\nolimits_{\nu = 1}^c \xi_{\div{(\div{\mu - 1}{k} + \nu - 1)\cdot k + \rem{\mu - 1}{k}}{1} + 1,\;\rem{(\div{\mu - 1}{k} + \nu - 1)\cdot k + \rem{\mu - 1}{k}}{1} + 1}\cdot  e_\nu^c\right)\\
  \equsing{P.~\ref{prop:number-theory}}\ \ & f\left(\sum\nolimits_{\nu = 1}^c \xi_{\div{\mu - 1}{k}\cdot k + \rem{\mu - 1}{k} + k(\nu - 1) + 1,\; 1}\cdot  e_\nu^c\right)\\
  \equsing{($\lozenge$)}\ \ & f\left(\sum\nolimits_{\nu = 1}^c \xi_{\mu + k(\nu - 1)}\cdot  e_\nu^c\right)\\
  \equsing{D.~\ref{def:dilatedsubsignal}}\ \ & f(\DilatedSubsignal_{(c,\;k)}(\xi,\;\mu))\\
  \equsing{D.~\ref{def:dilate}}\ \ & \Dilate_{(f,\; k)}(\xi)_\mu\text{,}
\end{align*}
where in the ($\lozenge$) step $\div{a}{b}\cdot b + \rem{a}{b} = a$ for all $a$ and all $b$ has been used.
Since all the samples are equal, both signals are equal.
\end{proof}

\subsection{Application of Processing Chains in a Dilated Fashion}
First consider the operator that uses the concepts just introduced for processing chain application:
\begin{definition}
\label{def:evaldilate}
Consider a processing chain with the same notation as in Definition~\ref{def:processing-chain}.
Additionally, with Theorem~\ref{thm:sliding-subsignal} let $f_j\colon M_{j - 1}^{c_j}\to N_j$ be the unique functions with $T_j = \Slide_{f_j}$ for all $j\inint{1}{L}$.
For $j\inint{0}{L}$, the operator $\EvalDilate_j\colon\cup_\ROI(M_0)\to\cup_1(M_j)$,
\begin{displaymath}
  \xi\mapsto
  \begin{cases}
    \xi\text{,} & \text{if } j = 0\text{,}\\
    \Dilate_{(g_j,\; k_{j - 1}^*)}( \Dilate_{(f_j,\; k_{j - 1}^*)}( \EvalDilate_{j - 1}(\xi) ) )\text{, } & \text{if } j > 0\text{,}
  \end{cases}
\end{displaymath}
is said to apply the processing chain in a \emph{dilated} fashion.
\end{definition}

In other words, each layer in the $\EvalDilate$ cascade applies the functions $f_j$ and $g_j$ in a dilated fashion to the output of the previous layer.
The stride for function application is chosen to equal the previous layer's stride product.
Analogous to the main part of this paper, this notion can be analyzed rigorously:
\begin{lemma}
\label{lem:evaldilate}
Consider a processing chain as in Definition~\ref{def:evaldilate}.
Suppose that for all $j\inint{1}{L}$ and all $\rho\in M_0^\ROI$ it holds that $k_j$ divides $\dim_{N_j}(\Slide_{f_j}(\EvalStride_{j - 1}(\rho)))$ and that $\EvalStride_{j}(\rho)$ is non-empty so that $\EvalStride$ is well-defined.

Let $\xi\in M_0^D$ be an input signal with $D\in\N_1$, $D\geq\ROI$, samples.
In contrast to Lemma~\ref{lem:processing-chain}, no divisibility constraints are imposed here on $D$.
This facilitates processing of input signals of arbitrary length greater than or equal to the processing chain's receptive field size $\ROI$, but prevents results of the $\EvalSlide$ operator from being used here.
As in Lemma~\ref{lem:processing-chain}, let $u_j := \dim_{M_j}(\EvalStride_j(\Subsignal_\ROI(\xi,\; i)))\in\N_1$ for all $j\inint{0}{L}$ denote the dimensionalities encountered during $\EvalStride$ evaluation, which are independent of any concrete subsignal index $i$.

Now let $V_j := \dim_{M_j}(\EvalDilate_j(\xi))\in\N_1$ for all $j\inint{0}{L}$ denote the dimensionalities of the intermediate representations of each layer during application of the processing chain in a dilated fashion.
Then for all $j\inint{0}{L}$ the following holds:
\begin{enumerate}
  \item \label{lem:evaldilate-a} $V_j = D - \sum_{\mu = 1}^j k_{\mu - 1}^*(c_\mu + k_\mu - 2)$.
  \item \label{lem:evaldilate-b} $V_j - k_j^*(u_j - 1) = D - \ROI + 1$. Therefore, there are $D - \ROI + 1$ dilated subsignals with $u_j$ samples and a stride of $k_j^*$ in $\EvalDilate_j(\xi)$.
  \item \label{lem:evaldilate-c} $V_j\in\N_1$ and $\EvalDilate_j(\xi)$ is well-defined.
  \item \label{lem:evaldilate-d} For all subsignal indices $i\inint{1}{D - \ROI + 1}$ it is 
    \begin{displaymath}
      \EvalStride_j(\Subsignal_\ROI(\xi,\; i)) = \DilatedSubsignal_{(u_j,\; k_j^*)}(\EvalDilate_j(\xi),\; i)\text{.}
    \end{displaymath}
    In other words, all subsignals of the input signal $\xi$ fed through the $\EvalStride$ operator emerge as dilated subsignals in the result of the $\EvalDilate$ operator applied to $\xi$.
\end{enumerate}
\end{lemma}
\begin{proof}
\ref{lem:evaldilate-a}
For $j = 0$ follows $V_0 = \dim_{M_0}(\xi) = D$, which equals the claimed expression.
For $j - 1 \to j$, one obtains
\begin{align*}
  V_j\ \ 
  &\equsing{D.~\ref{def:evaldilate}}\ \ \dim_{M_j}( \Dilate_{(g_j,\; k_{j - 1}^*)}( \Dilate_{(f_j,\; k_{j - 1}^*)}( \EvalDilate_{j - 1}(\xi) ) ) )\\
  &\equsing{D.~\ref{def:dilate}}\ \ \dim_{N_j}( \Dilate_{(f_j,\; k_{j - 1}^*)}( \EvalDilate_{j - 1}(\xi) ) ) - k_{j - 1}^*(k_j - 1)\\
  &\equsing{D.~\ref{def:dilate}}\ \ \dim_{M_{j - 1}}( \EvalDilate_{j - 1}(\xi) ) - k_{j - 1}^*(c_j - 1) - k_{j - 1}^*(k_j - 1)\\
  &\equsing{IH}\ \ D - \sum\nolimits_{\mu = 1}^{j - 1} k_{\mu - 1}^*(c_\mu + k_\mu - 2) - k_{j - 1}^*(c_j + k_j - 2)\\
  &=\ \ D - \sum\nolimits_{\mu = 1}^j k_{\mu - 1}^*(c_\mu + k_\mu - 2)\text{.}
\end{align*}
Hence the claim holds.

\ref{lem:evaldilate-b}
Using~\ref{lem:evaldilate-a} and Lemma~\ref{lem:processing-chain}\ref{lem:processing-chain-a} leads to
\begin{align*}
  V_j - k_j^*(u_j - 1)
  &= D - \sum\nolimits_{\mu = 1}^j k_{\mu - 1}^*(c_\mu + k_\mu - 2) - k_j^* \cdot \tfrac{1}{k_j^*}\left(\ROI - \sum\nolimits_{\mu = 1}^j k_{\mu - 1}^*(c_\mu - 1)\right) + k_j^*\\
  &= D - \ROI + \sum\nolimits_{\mu = 1}^j k_{\mu - 1}^*(-c_\mu - k_\mu + 2 + c_\mu - 1) + k_j^*\\
  &= D - \ROI + \sum\nolimits_{\mu = 1}^j k_{\mu - 1}^* - \sum\nolimits_{\mu = 1}^j k_\mu^* + k_j^*\\
  &= D - \ROI + k_0^* + \sum\nolimits_{\mu = 2}^j k_{\mu - 1}^* - \sum\nolimits_{\mu = 1}^{j - 1} k_\mu^*\\
  &= D - \ROI + 1\text{.}
\end{align*}

\ref{lem:evaldilate-c}
From~\ref{lem:evaldilate-a} it is clear that $V_j$ is an integer.
With~\ref{lem:evaldilate-b} follows $V_j = D - \ROI + 1 + k_j^*(u_j - 1)$.
Now $D \geq \ROI$ by definition.
Since the processing chain itself should be well-defined, it must hold $u_j\in\N_1$ and hence $u_j \geq 1$.
Eventually, $k_j^*\geq 1$ by definition and hence $V_j \geq 1$, therefore $V_j\in\N_1$.

To prove that $\EvalDilate_j(\xi)$ is well-defined, it must first be shown that $\Dilate_{(f_j,\; k_{j - 1}^*)}$ is applicable to $\EvalDilate_{j - 1}(\xi)$.
For this, the latter must contain at least $k_{j - 1}^*(c_j - 1) + 1$ samples, see Definition~\ref{def:dilate}.
It has been shown in~\ref{lem:evaldilate-a} that $V_j = V_{j - 1} - k_{j - 1}^*(c_j + k_j - 2)$ immediately before the induction hypothesis was substituted.
Since $V_j \geq 1$ follows $V_{j - 1} \geq k_{j - 1}^*(c_j + k_j - 2) + 1$, and since $k_j \geq 1$ by definition follows $V_{j - 1} \geq k_{j - 1}^*(c_j - 1) + 1$.
Therefore $\Dilate_{(f_j,\; k_{j - 1}^*)}$ is applicable.

Analogously, for $\Dilate_{(g_j,\; k_{j - 1}^*)}$ to be applicable to $\Dilate_{(f_j,\; k_{j - 1}^*)}( \EvalDilate_{j - 1}(\xi) )$ it must hold that $V_{j - 1} - k_{j - 1}^*(c_j - 1) \geq k_{j - 1}^*(k_j - 1) + 1$.
This condition is equivalent to $V_j \geq 1$, which was shown earlier.
In conclusion, $\EvalDilate_j(\xi)$ is well-defined.

\ref{lem:evaldilate-d}
In the case of $j = 0$, it holds that $\EvalStride_0(\Subsignal_\ROI(\xi,\; i)) = \Subsignal_\ROI(\xi,\; i)$ and further $\DilatedSubsignal_{(u_0,\; k_0^*)}(\EvalDilate_0(\xi),\; i) = \DilatedSubsignal_{(\ROI,\; 1)}(\xi,\; i) = \Subsignal_\ROI(\xi,\; i)$.

Now consider the induction step $j - 1 \to j$.
Completely analogous to the proof of Lemma~\ref{lem:processing-chain}\ref{lem:processing-chain-d}, let $\mu\inint{1}{u_j}$ be arbitrary, and let $\tau := \EvalStride_{j - 1}(\Subsignal_\ROI(\xi,\;i))\in M_{j - 1}^{u_{j - 1}}$ be an abbreviation.
The $\mu$-th sample of the left-hand side of the claim equals
\begin{align*}
  & \EvalStride_j(\Subsignal_\ROI(\xi,\;i))_\mu\\
  \equsing{D.~\ref{def:processing-chain}}\ \ & \Stride_{g_j}(\Slide_{f_j}(\tau))_\mu\\
  \equsing{D.~\ref{def:strided-function}}\ \ & g_j\!\left( \sum\nolimits_{\nu = 1}^{k_j} \Slide_{f_j}(\tau)_{k_j(\mu - 1) + \nu}\cdot e_\nu^{k_j} \right)\\
  \equsing{D.~\ref{def:sliding-function}}\ \ & g_j\!\left( \sum\nolimits_{\nu = 1}^{k_j} f_j\!\left(\sum\nolimits_{\lambda = 1}^{c_j}  \tau_{k_j(\mu - 1) + (\nu - 1) + (\lambda - 1) + 1}\cdot e_\lambda^{c_j}\right)\cdot e_\nu^{k_j} \right)\text{.}
\end{align*}
Consider the same sample of the right-hand side of the claim:
\begin{align*}
  & \DilatedSubsignal_{(u_j,\; k_j^*)}(\EvalDilate_j(\xi),\; i)_\mu\\
  \equsing{D.~\ref{def:dilatedsubsignal}}\ \ & \EvalDilate_j(\xi)_{i + k_j^*(\mu - 1)}\\
  \equsing{D.~\ref{def:evaldilate}}\ \ & \Dilate_{(g_j,\; k_{j - 1}^*)}( \Dilate_{(f_j,\; k_{j - 1}^*)}( \EvalDilate_{j - 1}(\xi) ) )_{i + k_j^*(\mu - 1)}\\
  \equsing{D.~\ref{def:dilate}}\ \ & g_j\!\left( \sum\nolimits_{\nu = 1}^{k_j} \Dilate_{(f_j,\; k_{j - 1}^*)}( \EvalDilate_{j - 1}(\xi) )_{i + k_j^*(\mu - 1) + k_{j - 1}^*(\nu - 1)} \cdot e_\nu^{k_j} \right)\\
  \equsing{D.~\ref{def:dilate}}\ \ & g_j\!\left( \sum\nolimits_{\nu = 1}^{k_j} f_j\!\left(\sum\nolimits_{\lambda = 1}^{c_j} \EvalDilate_{j - 1}(\xi)_{i + k_j^*(\mu - 1) + k_{j - 1}^*(\nu - 1) + k_{j - 1}^*(\lambda - 1)}\cdot e_\lambda^{c_j}\right) \cdot e_\nu^{k_j} \right)\text{.}
\end{align*}
Now, the innermost part of the final expression can be manipulated:
\begin{align*}
  & \EvalDilate_{j - 1}(\xi)_{i + k_j^*(\mu - 1) + k_{j - 1}^*(\nu - 1) + k_{j - 1}^*(\lambda - 1)}\\
  =\ \ & \EvalDilate_{j - 1}(\xi)_{i + k_{j - 1}^*\left(k_j(\mu - 1) + (\nu - 1) + (\lambda - 1) + 1 - 1\right)}\\
  \equsing{D.~\ref{def:dilatedsubsignal}}\ \ & \DilatedSubsignal_{(u_{j-1},\; k_{j - 1}^*)}(\EvalDilate_{j - 1}(\xi),\;i)_{k_j(\mu - 1) + (\nu - 1) + (\lambda - 1) + 1}\\
  \equsing{IH}\ \ & \EvalStride_{j - 1}(\Subsignal_\ROI(\xi,\; i))_{k_j(\mu - 1) + (\nu - 1) + (\lambda - 1) + 1}\text{.}
\end{align*}
For the application of the $\DilatedSubsignal_{(u_{j-1},\; k_{j - 1}^*)}$ operator it must be verified here that the selected coordinate $\upsilon := k_j(\mu - 1) + (\nu - 1) + (\lambda - 1) + 1$ lies within the set $\discint{1}{u_{j - 1}}$.
Substituting the lower bounds of $\mu$, $\nu$ and $\lambda$ yields $\upsilon \geq 1$.
Now Lemma~\ref{lem:processing-chain}\ref{lem:processing-chain-a} implies $u_j = \frac{1}{k_j}(u_{j - 1} - c_j + 1)$, and therefore $\upsilon \leq k_j(u_j - 1) + (k_j - 1) + (c_j - 1) + 1 = u_{j - 1}$.

Eventually, plugging the result into the $ \mu$-th sample of the right-hand side of the claim and remembering that $\tau = \EvalStride_{j - 1}(\Subsignal_\ROI(\xi,\;i))$ yields
\begin{align*}
  & \DilatedSubsignal_{(u_j,\; k_j^*)}(\EvalDilate_j(\xi),\; i)_\mu\\
  =\ \ & g_j\!\left( \sum\nolimits_{\nu = 1}^{k_j} f_j\!\left(\sum\nolimits_{\lambda = 1}^{c_j} \tau_{k_j(\mu - 1) + (\nu - 1) + (\lambda - 1) + 1}\cdot e_\lambda^{c_j}\right) \cdot e_\nu^{k_j} \right)\text{,}
\end{align*}
which equals the left-hand side of the claim as shown earlier.
\end{proof}

Therefore, $\EvalDilate$ is well-defined.
Similar to $\EvalSlide$, the intermediate representations of the $\EvalDilate$ operator contain the complete output of the EvalStride operator, which is applied to all feasible subsignals of the input signal through to the relevant layer.
Remark~\ref{rem:dilated-fragment} has shown a connection between dilated function application and fragmentation.
This statement can be used to establish a connection between $\EvalDilate$ and $\EvalSlide$ independent of Lemma~\ref{lem:evaldilate}:
\begin{remark}
\label{rem:evaldilate}
It is also possible to prove the identity from Lemma~\ref{lem:evaldilate}\ref{lem:evaldilate-d} using the results from the main part of this paper.
Since here the $\EvalSlide$ operator is required to be well-defined, the final stride product $k_L^*$ must divide $D - \ROI + 1$ as in Lemma~\ref{lem:processing-chain}.
Now the following holds for all $j\inint{0}{L}$ in the situation of Lemma~\ref{lem:evaldilate}:
\begin{enumerate}
  \item \label{rem:evaldilate-a} $\Defragmentation_{k_{j - 1}^*} = \Defragmentation_{k_j^*} \circ \Fragmentation_{k_j}$.
  \item \label{rem:evaldilate-b} $\EvalDilate_j(\xi) = \Defragmentation_{k_j^*}(\EvalSlide_j(\xi))$.
  \item \label{rem:evaldilate-c} $\EvalStride_j(\Subsignal_\ROI(\xi,\; i)) = \DilatedSubsignal_{(u_j,\; k_j^*)}(\Defragmentation_{k_j^*}(\EvalSlide_j(\xi)),\; i)$.
  \item \label{rem:evaldilate-d} $\EvalStride_j(\Subsignal_\ROI(\xi,\; i)) = \DilatedSubsignal_{(u_j,\; k_j^*)}(\EvalDilate_j(\xi),\; i)$.
\end{enumerate}
\end{remark}
\begin{proof}
\ref{rem:evaldilate-a}
From Remark~\ref{rem:comp-frag} it follows that
\begin{align*}
  & \Defragmentation_{k_j^*} \circ \Fragmentation_{k_j}\\
  =\ \ & \Defragmentation_{k_j^*} \circ \Fragmentation_{k_j} \circ \Fragmentation_{k_{j - 1}^*} \circ \Defragmentation_{k_{j - 1}^*}\\
  \equsing{R.~\ref{rem:comp-frag}}\ \ & \Defragmentation_{k_j^*} \circ \Fragmentation_{k_j^*} \circ \Defragmentation_{k_{j - 1}^*}\\
  =\ \ & \Defragmentation_{k_{j - 1}^*}
\end{align*}
since fragmentation and defragmentation are inversely related permutations.

\ref{rem:evaldilate-b}
This is shown through induction.
For $j = 0$, one obtains $\EvalDilate_0(\xi) = \xi$ and furthermore $\Defragmentation_{k_0^*}(\EvalSlide_0(\xi)) = \Defragmentation_{1}(\xi) = \xi$.
For $j - 1 \to j$ follows
\begin{align*}
  & \EvalDilate_j(\xi)\\
  \equsing{D.~\ref{def:evaldilate}}\ \ & \Dilate_{(g_j,\; k_{j - 1}^*)}( \Dilate_{(f_j,\; k_{j - 1}^*)}( \EvalDilate_{j - 1}(\xi) ) )\\
  \equsing{IH}\ \ & \Dilate_{(g_j,\; k_{j - 1}^*)}( \Dilate_{(f_j,\; k_{j - 1}^*)}(\Defragmentation_{k_{j - 1}^*}(\EvalSlide_{j - 1}(\xi))))\\
  \equsing{R.~\ref{rem:dilated-fragment}}\ \ & \Defragmentation_{k_{j - 1}^*}( \Slide_{g_j}( \Fragmentation_{k_{j - 1}^*}( \Defragmentation_{k_{j - 1}^*}( \Slide_{f_j}( \Fragmentation_{k_{j - 1}^*}( \Defragmentation_{k_{j - 1}^*}( \EvalSlide_{j - 1}(\xi) ) ) ) ) ) ) )\\
  \equsing{($\lozenge$)}\ \ & \Defragmentation_{k_{j - 1}^*}( \Slide_{g_j}( \Slide_{f_j}( \EvalSlide_{j - 1}(\xi) ) ) )\\
  \equsing{\ref{rem:evaldilate-a}}\ \ & \Defragmentation_{k_j^*}( \Fragmentation_{k_j}( \Slide_{g_j}( \Slide_{f_j}( \EvalSlide_{j - 1}(\xi) ) ) ) )\\
  \equsing{D.~\ref{def:processing-chain}}\ \ & \Defragmentation_{k_j^*}( \EvalSlide_j(\xi) )\text{,}
\end{align*}
where in the ($\lozenge$) step the composition $\Fragmentation_{k_{j - 1}^*} \circ \Defragmentation_{k_{j - 1}^*}$ equals identity.

\ref{rem:evaldilate-c}
Let $\mu\inint{1}{u_j}$.
Then the $\mu$-th sample of the left-hand side is
\begin{displaymath}
  \EvalStride_j(\Subsignal_\ROI(\xi,\; i))_\mu
  \ \ \ \equsing{L.~\ref{lem:processing-chain}\ref{lem:processing-chain-d}}\ \ \ \EvalSlide_j(\xi)_{\div{i - 1}{k_j^*} + \mu,\;\rem{i - 1}{k_j^*} + 1}\text{.}
\end{displaymath}
Considering the right-hand side, one obtains
\begin{align*}
  & \DilatedSubsignal_{(u_j,\; k_j^*)}(\Defragmentation_{k_j^*}(\EvalSlide_j(\xi)),\; i)_\mu\\
  \equsing{D.~\ref{def:dilatedsubsignal}}\ \ & \Defragmentation_{k_j^*}(\EvalSlide_j(\xi))_{i + k_j^*(\mu - 1),\; 1}\\
  \equsing{L.~\ref{lem:defrag-ops}}\ \ & \EvalSlide_j(\xi)_{\div{i - 1 + k_j^*(\mu - 1)}{k_j^*} + 1,\;\rem{i - 1 + k_j^*(\mu - 1)}{k_j^*} + 1}\\
  \equsing{P.~\ref{prop:number-theory}}\ \ & \EvalSlide_j(\xi)_{\div{i - 1}{k_j^*} + \mu,\;\rem{i - 1}{k_j^*} + 1}\text{.}
\end{align*}
It has been used here that $\Defragmentation_{k_j^*}(\EvalSlide_j(\xi))$ consists of exactly one fragment due to Lemma~\ref{lem:processing-chain}\ref{lem:processing-chain-b}.
Since all samples on both sides of the identity match, the claim follows.

\ref{rem:evaldilate-d}
Follows directly from~\ref{rem:evaldilate-c} and~\ref{rem:evaldilate-b}.
\end{proof}

Subsequent to this preparation, the main result of this appendix can be formulated.
It states that application of a processing chain in a dilated fashion leads to the same outcome as naive subsignal-based application:
\begin{theorem}
Consider a processing chain using the same notation as in Definition~\ref{def:evaldilate} for the $\EvalDilate$ operator and as in Definition~\ref{def:processing-chain} for the $\EvalStride$ operator.
Suppose that $k_j$ divides $\dim_{N_j}(\Slide_{f_j}(\EvalStride_{j - 1}(\rho)))$ and that $\EvalStride_{j}(\rho)$ is non-empty for all $j\inint{1}{L}$ and all $\rho\in M_0^\ROI$, so that $\EvalStride$ is well-defined.
Further, assume $\dim_{M_L}(\EvalStride_L(\rho)) = 1$ for all $\rho\in M_0^\ROI$, or in other words, the processing chain outputs signals with unit spatial size using subsignal-based application.
Then $\EvalDilate_L = \Slide_{\EvalStride_L}$, and therefore $\EvalDilate_L$ is a subsignal compatible transformation with dimensionality reduction constant $\ROI$.
\end{theorem}
\begin{proof}
Let $\xi\in M_0^D$ be an input signal where $D\in\N_1$, $D\geq\ROI$.
Let $i\inint{1}{D - \ROI + 1}$ be an arbitrary subsignal index.
Let $u_L := \dim_{M_L}(\EvalStride_L(\rho))$ for any $\rho\in M_0^\ROI$, then by requirement $u_L = 1$.
Hence
\begin{align*}
  & \Slide_{\EvalStride_L}(\xi)_i\\
  \equsing{D.~\ref{def:sliding-function}}\ \ \ & \EvalStride_L(\Subsignal_\ROI(\xi,\; i))\\
  \equsing{L.~\ref{lem:evaldilate}\ref{lem:evaldilate-d}}\ \ \ & \DilatedSubsignal_{(1,\; k_L^*)}(\EvalDilate_L(\xi),\; i)\\
  \equsing{D.~\ref{def:dilatedsubsignal}}\ \ \ & \EvalDilate_L(\xi)_i\text{.}
\end{align*}
Now Lemma~\ref{lem:evaldilate}\ref{lem:evaldilate-b} yields $V_L = D - \ROI + 1 + k_L^*(u_L - 1) = D - \ROI + 1$, hence all accesses are within bounds.
Theorem~\ref{thm:sliding-subsignal} finally implies subsignal compatibility of $\EvalDilate_L$.
\end{proof}

\subsection{Computational Complexity Analysis}
After it has been proved that the $\EvalDilate$ operator produces the correct outcome under all circumstances, its computational complexity is analyzed.
Analogous to the analysis of the $\EvalStride$ and $\EvalSlide$ operators, here the required function evaluations are counted and compared.
It will ultimately be shown that the computational complexity of $\EvalDilate$ measured by this means exactly equals that of $\EvalSlide$.
Therefore, the theoretical results obtained for the complexity of $\EvalSlide$ with respect to asymptotic behavior and the comparison with the $\EvalStride$ operator hold for $\EvalDilate$ as well.

The comparison with $\EvalSlide$ requires the input signal length to satisfy a certain divisibility constraint since only this guarantees $\EvalSlide$ is well-defined.
First consider a simple relationship which will turn out helpful later:
\begin{remark}
\label{rem:evaldilate-complexity}
In the situation of Lemma~\ref{lem:evaldilate} assume $k_L^*$ divides $D - \ROI + 1$, which renders $\EvalSlide$ well-defined by Lemma~\ref{lem:processing-chain}.
Then $V_j = U_j^{\col}U_j^{\row}$ for $j\inint{0}{L}$.
\end{remark}
\begin{proof}
One obtains
\begin{displaymath}
  V_j - U_j^{\col}U_j^{\row}
  \ \equsing{($\lozenge$)}\ \left(D - \ROI + 1 + k_j^*(u_j - 1)\right) - \left(D - \ROI + 1 + k_j^*u_j - k_j^*\right)
  \ =\ 0\text{,}
\end{displaymath}
where in the ($\lozenge$) step Lemma~\ref{lem:evaldilate}\ref{lem:evaldilate-b} and Lemma~\ref{lem:processing-chain}\ref{lem:processing-chain-c} have been used, where $U_j^{\col} = k_j^*$ was substituted for the latter due to Lemma~\ref{lem:processing-chain}\ref{lem:processing-chain-b}.
\end{proof}

Now consider the application of the processing chain in a dilated fashion.
Let $\xi\in M_0^D$ be an input signal where the number of samples $D$ has been chosen so that $\EvalSlide$ is well-defined.
Further, let $j\inint{1}{L}$ be a fixed layer index and let $\sigma := \EvalDilate_{j - 1}(\xi)\in M_{j - 1}^{V_{j - 1}}$ be an abbreviation.
The output of the $j$-th layer using the $\EvalDilate$ operator is now
\begin{displaymath}
  \EvalDilate_j(\xi) = \Dilate_{(g_j,\; k_{j - 1}^*)}( \Dilate_{(f_j,\; k_{j - 1}^*)}( \sigma ) )
\end{displaymath}
due to Definition~\ref{def:evaldilate}.
It is
\begin{displaymath}
  \Dilate_{(f_j,\; k_{j - 1}^*)}( \sigma )
  \ \ \equsing{D.~\ref{def:dilate}}\ \ \sum\nolimits_{\mu = 1}^{V_{j - 1} - k_{j - 1}^*(c_j - 1)} f_j(\DilatedSubsignal_{(c_j,\;k_{j - 1}^*)}(\sigma,\;\mu))\cdot e_\mu^{V_{j - 1} - k_{j - 1}^*(c_j - 1)}\text{,}
\end{displaymath}
therefore $f_j$ has to be evaluated $V_{j - 1} - k_{j - 1}^*(c_j - 1)$ times.
Further
\begin{displaymath}
  \Dilate_{(g_j,\; k_{j - 1}^*)}( \Dilate_{(f_j,\; k_{j - 1}^*)}( \sigma ) )
  \ \ \equsing{D.~\ref{def:dilate}}\ \ \sum\nolimits_{\nu = 1}^{V_j} g_j(\DilatedSubsignal_{(k_j,\;k_{j - 1}^*)}(\Dilate_{(f_j,\; k_{j - 1}^*)}( \sigma ),\;\nu))\cdot e_\nu^{V_j}\text{,}
\end{displaymath}
where it has been used that $V_j = V_{j - 1} - k_{j - 1}^*(c_j - 1) - k_{j - 1}^*(k_j - 1)$ is the dimensionality of $\EvalDilate_j(\xi)$ by definition and by Lemma~\ref{lem:evaldilate}\ref{lem:evaldilate-a}.
In conclusion, $V_j$ evaluations of $g_j$ are necessary.

Accounting for the number of input fragments, $\EvalSlide_j$ needs $U_{j - 1}^{\col} (U_{j - 1}^{\row} - c_j + 1)$ evaluations of $f_j$ as shown in the main part of this paper.
Subtracting this from the number of evaluations $\EvalDilate_j$ requires yields
\begin{displaymath}
  V_{j - 1} - k_{j - 1}^*(c_j - 1) - U_{j - 1}^{\col} (U_{j - 1}^{\row} - c_j + 1)
  = \left(V_{j - 1} - U_{j - 1}^{\col}U_{j - 1}^{\row}\right) - (c_j - 1)\left(k_{j - 1}^* - U_{j - 1}^{\col}\right)
  = 0\text{,}
\end{displaymath}
where Remark~\ref{rem:evaldilate-complexity} and Lemma~\ref{lem:processing-chain}\ref{lem:processing-chain-b} have been used in the final step.
Since the difference vanishes, both $\EvalSlide_j$ and $\EvalDilate_j$ require the exact same number of evaluations of $f_j$.

It has further been shown previously that $\EvalSlide_j$ requires $U_{j - 1}^{\col}k_j U_j^{\row}$ evaluations of $g_j$ in total.
Since $U_{j - 1}^{\col}k_j = U_j^{\col}$ with Lemma~\ref{lem:processing-chain}\ref{lem:processing-chain-b}, it follows that $\EvalDilate_j$ requires the exact same number of evaluations of $g_j$ with Remark~\ref{rem:evaldilate-complexity}.

In conclusion, the computational complexity in terms of function evaluations is the same for the dilated function application and a fragmentation-based approach.
While for $\EvalDilate$ all necessary functions such as convolution have to be generalized to support strided input data, the permutation realized through fragmentation in the $\EvalSlide$ approach facilitates usage of routines operating on contiguous data.
Therefore, for CNNs, readily available highly-optimized implementations of tensor convolutions can be employed without any modifications whatsoever.

\bibliographystyledil{IEEEtran}
\bibliographydil{IEEEabrv,the}

\clearpage
\appendix[Application of Processing Chains in a Relaxed Fashion]
The focus of this paper has been to study efficient signal processing schemes that carry out dense signal scanning without any accuracy loss.
This appendix analyzes a method of processing chain application that naively uses the same processing pipeline for signal-based application as for subsignal-based application.
The theory developed in this paper so far suggests that such an approach cannot produce the same result as exact dense signal scanning and hence must involve relaxations of some kind.
Here, the negative effects on the output signal quality of relaxed processing chain application are exactly characterized.
This analysis sheds light on approaches that employ this scheme as a central part of their computations such as~\citerlx{Sermanet2014rlx,Long2015rlx} and helps in identifying possible limitations and optimization potential.
Further, the relationship between relaxed application and a fragmentation-based approach is investigated and a computational complexity analysis is conducted.

\subsection{Application of Processing Chains in a Relaxed Fashion}
First, consider the definition of the operator that naively extends the application of a processing chain in a strided fashion to larger input signals:
\begin{definition}
\label{def:evalrelax}
Suppose a processing chain as in Definition~\ref{def:processing-chain} is given and let $f_j\colon M_{j - 1}^{c_j}\to N_j$ be the unique functions with $T_j = \Slide_{f_j}$ for all $j\inint{1}{L}$ due to Theorem~\ref{thm:sliding-subsignal}.
Then for $j\inint{0}{L}$, the operator $\EvalRelax_j\colon\cup_\ROI(M_0)\to\cup_1(M_j)$,
\begin{displaymath}
  \xi\mapsto
  \begin{cases}
    \xi\text{,} & \text{if } j = 0\text{,}\\
    \Stride_{g_j}(\Slide_{f_j}(\EvalRelax_{j - 1}(\xi)))\text{, } & \text{if } j > 0\text{,}
  \end{cases}
\end{displaymath}
applies the processing chain in a \emph{relaxed} fashion.
\end{definition}

Note that the definition of $\EvalRelax$ is essentially the same as that of $\EvalStride$ except for its domain, which allows for arbitrarily large input signals with more than $\ROI$ samples.
Further, $\EvalRelax$ is not well-defined unless constraints on the input signal length are fulfilled.
These constraints and the fundamental dynamics of this operator are detailed below:

\begin{lemma}
\label{lem:evalrelax}
Suppose a processing chain as in Definition~\ref{def:evalrelax} is given.
Assume that $k_j$ divides $\dim_{N_j}(\Slide_{f_j}(\EvalStride_{j - 1}(\rho)))$ and that $\EvalStride_{j}(\rho)$ is non-empty for all $j\inint{1}{L}$ and all $\rho\in M_0^\ROI$, or in other words, $\EvalStride$ should be well-defined.

Let $D\in\N_1$, $D\geq\ROI$, be an input signal dimensionality.
Assume that the stride product of the final layer $k_L^*$ divides the difference $D - \ROI$.
Let $\xi\in M_0^D$ be an input signal and let $W_j := \dim_{M_j}(\EvalRelax_j(\xi))\in\N_1$ for all $j\inint{0}{L}$ denote the length of the intermediate representations when the processing chain is applied in a relaxed fashion to $\xi$.
As in Lemma~\ref{lem:processing-chain}, let $u_j := \dim_{M_j}(\EvalStride_j(\Subsignal_\ROI(\xi,\; i)))\in\N_1$ for $j\inint{0}{L}$ denote the dimensionalities encountered during application of the processing chain in a strided fashion, which are actually independent of any concrete subsignal index $i$.
Note that for these numbers only it is not necessary that the divisibility requirements on the number of subsignals $D - \ROI + 1$ from Lemma~\ref{lem:processing-chain} are fulfilled.
Then the following holds for all $j\inint{0}{L}$:
\begin{enumerate}\setlength{\itemsep}{.5ex}
  \item \label{lem:evalrelax-a} $W_j = \frac{1}{k_j^*}\left(D - \sum_{\mu = 1}^j k_{\mu - 1}^*(c_\mu - 1)\right)$.
  \item \label{lem:evalrelax-b} $W_j - u_j = \frac{1}{k_j^*}(D - \ROI)$.
  \item \label{lem:evalrelax-c} The application of the processing chain in a relaxed fashion up to layer $j$ is well-defined.
  \item \label{lem:evalrelax-d} For all $i\inint{1}{W_j - u_j + 1}$ it holds that
        \begin{displaymath}
          \EvalStride_j(\Subsignal_\ROI(\xi,\; k_j^*(i - 1) + 1))
          = \Subsignal_{u_j}(\EvalRelax_j(\xi),\; i)\text{.}
        \end{displaymath}
        The index $i$ here represents all feasible subsignals of $\EvalRelax_j(\xi)$, therefore this identity provides a complete characterization of the information available in the intermediate representations of the application of a processing chain in a relaxed fashion.
        On the other hand, the left-hand side describes the result of the $\EvalStride_j$ operator applied to only each $k_j^*$-th subsignal of the input signal $\xi$.
        In other words, the application of a processing chain in a relaxed fashion results in the loss of spatial resolution for non-trivial stride products.
\end{enumerate}
\end{lemma}
\begin{proof}
\ref{lem:evalrelax-a}
Completely analogous to the proof of Lemma~\ref{lem:processing-chain}\ref{lem:processing-chain-a}.

\ref{lem:evalrelax-b}
Using~\ref{lem:evalrelax-a} and Lemma~\ref{lem:processing-chain}\ref{lem:processing-chain-a} leads to
\begin{displaymath}
  W_j - u_j
  = \tfrac{1}{k_j^*}\left(D - \sum\nolimits_{\mu = 1}^j k_{\mu - 1}^*(c_\mu - 1)\right) - \tfrac{1}{k_j^*}\left(\ROI - \sum\nolimits_{\mu = 1}^j k_{\mu - 1}^*(c_\mu - 1)\right)
  = \tfrac{1}{k_j^*}(D - \ROI)\text{.}
\end{displaymath}

\ref{lem:evalrelax-c}
For well-definedness of $\EvalRelax_j$ it must be shown that each component is applicable.
$\Slide_{f_j}$ is always applicable to non-empty signals.
Therefore, it is enough to show that $\Stride_{g_j}$ is applicable to the outcome of $\Slide_{f_j}$, and that the overall result of the $j$-th layer is a non-empty signal.
With Definition~\ref{def:strided-function} it is required that the dimensionality of the input to $\Stride_{g_j}$ is divisible by $k_j$.
By requirement on the input signal length $D$ there is a natural number $t\in\N$ with $D - \ROI = k_L^* t$.
Since $\EvalStride_j$ is well-defined by requirement, there is another natural number $s\in\N_1$ so that $\dim_{N_j}(\Slide_{f_j}(\EvalStride_{j - 1}(\rho))) = k_j s$ for any $\rho\in M_0^\ROI$.
Now the input to $\Stride_{g_j}$ is of dimensionality
\begin{align*}
  & \dim_{N_j}(\Slide_{f_j}(\EvalRelax_{j - 1}(\xi)))\\
  \equsing{D.~\ref{def:sliding-function}}\ \ \ & W_{j - 1} - c_j + 1\\
  \equsing{\ref{lem:evalrelax-a}}\ \ \ & \tfrac{1}{k_{j - 1}^*}\left(D - \ROI + \ROI - \sum\nolimits_{\mu = 1}^{j - 1} k_{\mu - 1}^*(c_\mu - 1)\right) - c_j + 1\\
  =\ \ \ & \tfrac{1}{k_{j - 1}^*}(D - \ROI) + \tfrac{1}{k_{j - 1}^*}\left(\ROI - \sum\nolimits_{\mu = 1}^{j - 1} k_{\mu - 1}^*(c_\mu - 1)\right) - c_j + 1\\
  \equsing{L.~\ref{lem:processing-chain}\ref{lem:processing-chain-a}}\ \ \ & \tfrac{1}{k_{j - 1}^*}(D - \ROI) + \dim_{N_j}(\Slide_{f_j}(\EvalStride_{j - 1}(\rho)))\\
  \equsing{($\lozenge$)}\ \ \ & \tfrac{1}{k_{j - 1}^*}k_L^* t + k_j s\\
  =\ \ \ & (k_{j + 1}\cdots k_L\cdot t + s)\cdot k_j\text{,}
\end{align*}
where in the ($\lozenge$) step the relations from the divisibility requirements have been substituted.
Since $k_{j + 1}\cdots k_L\cdot t + s\in\N_1$, this positive dimensionality is divisible by $k_j$, and hence $\Stride_{g_j}$ is applicable.
The resulting signal length is
\begin{align*}
  W_j\ \ &\equsing{D.~\ref{def:evalrelax}}\ \ \dim_{M_j}(\Stride_{g_j}(\Slide_{f_j}(\EvalRelax_{j - 1}(\xi))))\\
  &\equsing{D.~\ref{def:strided-function}}\ \ \tfrac{1}{k_j}\dim_{N_j}(\Slide_{f_j}(\EvalRelax_{j - 1}(\xi)))\\
  &=\ \ k_{j + 1}\cdots k_L\cdot t + s\text{,}
\end{align*}
which is a positive natural number.
Hence the output signal of the $j$-th layer is non-empty, and in conclusion, $\EvalRelax_j$ is well-defined.

\ref{lem:evalrelax-d}
This is shown through induction.
For $j = 0$ it is $W_0 = D$ and $u_0 = \ROI$.
Let $i\inint{1}{D - \ROI + 1}$ be arbitrary, then the left-hand side of the claim equals
\begin{displaymath}
  \EvalStride_0(\Subsignal_\ROI(\xi,\; k_0^*(i - 1) + 1)) = \EvalStride_0(\Subsignal_\ROI(\xi,\; i)) = \Subsignal_\ROI(\xi,\; i)\text{.}
\end{displaymath}
Consideration of the right-hand side leads to $\Subsignal_{u_0}(\EvalRelax_0(\xi),\; i) = \Subsignal_\ROI(\xi,\; i)$.
Since all index accesses are clearly within bounds, the claim follows for $j = 0$.

For $j - 1 \to j$, first note that the dimensionalities on both sides of the claim match by definition of $u_j$.
Let $i\inint{1}{W_j - u_j + 1}$ be a subsignal index and let $\mu\inint{1}{u_j}$ be an arbitrary sample index.
Then the $\mu$-th sample of the left-hand side of the claim is
\begin{align*}
  & \EvalStride_j(\Subsignal_\ROI(\xi,\; k_j^*(i - 1) + 1))_\mu\\
  \equsing{D.~\ref{def:processing-chain}}\ \ & \Stride_{g_j}(\Slide_{f_j}(\EvalStride_{j - 1}(\Subsignal_\ROI(\xi,\; k_j^*(i - 1) + 1))))_\mu\\
  \equsing{D.~\ref{def:strided-function}}\ \ & g_j\!\left( \sum\nolimits_{\nu = 1}^{k_j} \Slide_{f_j}(\EvalStride_{j - 1}(\Subsignal_\ROI(\xi,\; k_j^*(i - 1) + 1)))_{k_j(\mu - 1) + \nu}\cdot e_\nu^{k_j} \right)\\
  \equsing{D.~\ref{def:sliding-function}}\ \ & g_j\!\left( \sum\nolimits_{\nu = 1}^{k_j} f_j\!\left(\sum\nolimits_{\lambda = 1}^{c_j} \EvalStride_{j - 1}(\Subsignal_\ROI(\xi,\; k_j^*(i - 1) + 1))_{k_j(\mu - 1) + \nu + \lambda - 1}\cdot e_\lambda^{c_j}\right)\cdot e_\nu^{k_j} \right)\text{.}
\end{align*}

The same sample of the right-hand side of the claim equals
\begin{align*}
  & \Subsignal_{u_j}(\EvalRelax_j(\xi),\; i)_\mu\\
  \equsing{D.~\ref{def:evalrelax}}\ \ & \Stride_{g_j}(\Slide_{f_j}(\EvalRelax_{j - 1}(\xi)))_{i + \mu - 1}\\
  \equsing{D.~\ref{def:strided-function}}\ \ & g_j\!\left( \sum\nolimits_{\nu = 1}^{k_j} \Slide_{f_j}(\EvalRelax_{j - 1}(\xi))_{k_j(i - 1) + k_j(\mu - 1) + \nu}\cdot e_\nu^{k_j} \right)\\
  \equsing{D.~\ref{def:sliding-function}}\ \ & g_j\!\left( \sum\nolimits_{\nu = 1}^{k_j} f_j\!\left(\sum\nolimits_{\lambda = 1}^{c_j} \EvalRelax_{j - 1}(\xi)_{k_j(i - 1) + k_j(\mu - 1) + \nu + \lambda - 1}\cdot e_\lambda^{c_j}\right)\cdot e_\nu^{k_j} \right)\text{.}
\end{align*}
Considering the innermost expression $\EvalRelax_{j - 1}(\xi)_{k_j(i - 1) + k_j(\mu - 1) + \nu + \lambda - 1}$, define $\tilde{i} := k_j(i - 1) + 1$ and $\tilde{\mu} := k_j(\mu - 1) + \nu + \lambda - 1$.
Then
\begin{align*}
  & \Subsignal_{u_{j - 1}}(\EvalRelax_{j - 1}(\xi),\; \tilde{i})_{\tilde{\mu}}\\
  \equsing{D.~\ref{def:subsignal}}\ \ & \EvalRelax_{j - 1}(\xi)_{\tilde{i} + \tilde{\mu} - 1}\\
  =\ \ & \EvalRelax_{j - 1}(\xi)_{k_j(i - 1) + k_j(\mu - 1) + \nu + \lambda - 1}\text{,}
\end{align*}
where it must be shown that $\tilde{i}\inint{1}{W_{j - 1} - u_{j - 1} + 1}$ and $\tilde{\mu}\inint{1}{u_{j - 1}}$.
Since both indices are clearly integers it suffices to verify they are within bounds.
Obviously, $\tilde{i}\geq 1$.
Now, from~\ref{lem:evalrelax-a} follows that $W_j = \frac{1}{k_j}(W_{j - 1} - c_j + 1)$ and Lemma~\ref{lem:processing-chain}\ref{lem:processing-chain-a} implies that $u_j = \frac{1}{k_j}(u_{j - 1} - c_j + 1)$.
Therefore substitution of the maximum value of $i$ yields
\begin{displaymath}
  \tilde{i}
  \leq k_j(W_j - u_j + 1 - 1) + 1
  = (W_{j - 1} - c_j + 1) - (u_{j - 1} - c_j + 1) + 1
  = W_{j - 1} - u_{j - 1} + 1\text{.}
\end{displaymath}
Since $\mu,\nu,\lambda\geq 1$ follows $\tilde{\mu}\geq 1$.
Substitution of the maximum values leads to
\begin{displaymath}
  \tilde{\mu}
  \leq k_j(u_j - 1) + k_j + c_j - 1
  =  k_j u_j + c_j - 1
  = u_{j - 1}\text{.}
\end{displaymath}
Hence the indices $\tilde{i}$ and $\tilde{\mu}$ are both within bounds.

Substitution of the induction hypothesis now yields
\begin{align*}
  & \EvalRelax_{j - 1}(\xi)_{k_j(i - 1) + k_j(\mu - 1) + \nu + \lambda - 1}\\
  \equsing{IH}\ \ & \EvalStride_{j - 1}(\Subsignal_\ROI(\xi,\; k_{j - 1}^*(\tilde{i} - 1) + 1))_{\tilde{\mu}}\\
  =\ \ & \EvalStride_{j - 1}(\Subsignal_\ROI(\xi,\; k_j^*(i - 1) + 1))_{k_j(\mu - 1) + \nu + \lambda - 1}\text{.}
\end{align*}
In conclusion, the $\mu$-th sample of the right-hand side of the claim is
\begin{align*}
  & \Subsignal_{u_j}(\EvalRelax_j(\xi),\; i)_\mu\\
  =\ \ & g_j\!\left( \sum\nolimits_{\nu = 1}^{k_j} f_j\!\left(\sum\nolimits_{\lambda = 1}^{c_j} \EvalStride_{j - 1}(\Subsignal_\ROI(\xi,\; k_j^*(i - 1) + 1))_{k_j(\mu - 1) + \nu + \lambda - 1}\cdot e_\lambda^{c_j}\right)\cdot e_\nu^{k_j} \right)\text{,}
\end{align*}
which is the same as the $\mu$-th sample of the left-hand side of the claim.
Therefore the identity holds.
\end{proof}

\begin{figure}[p]
  \centering
  \scalebox{1.04}{\includegraphics[page=15]{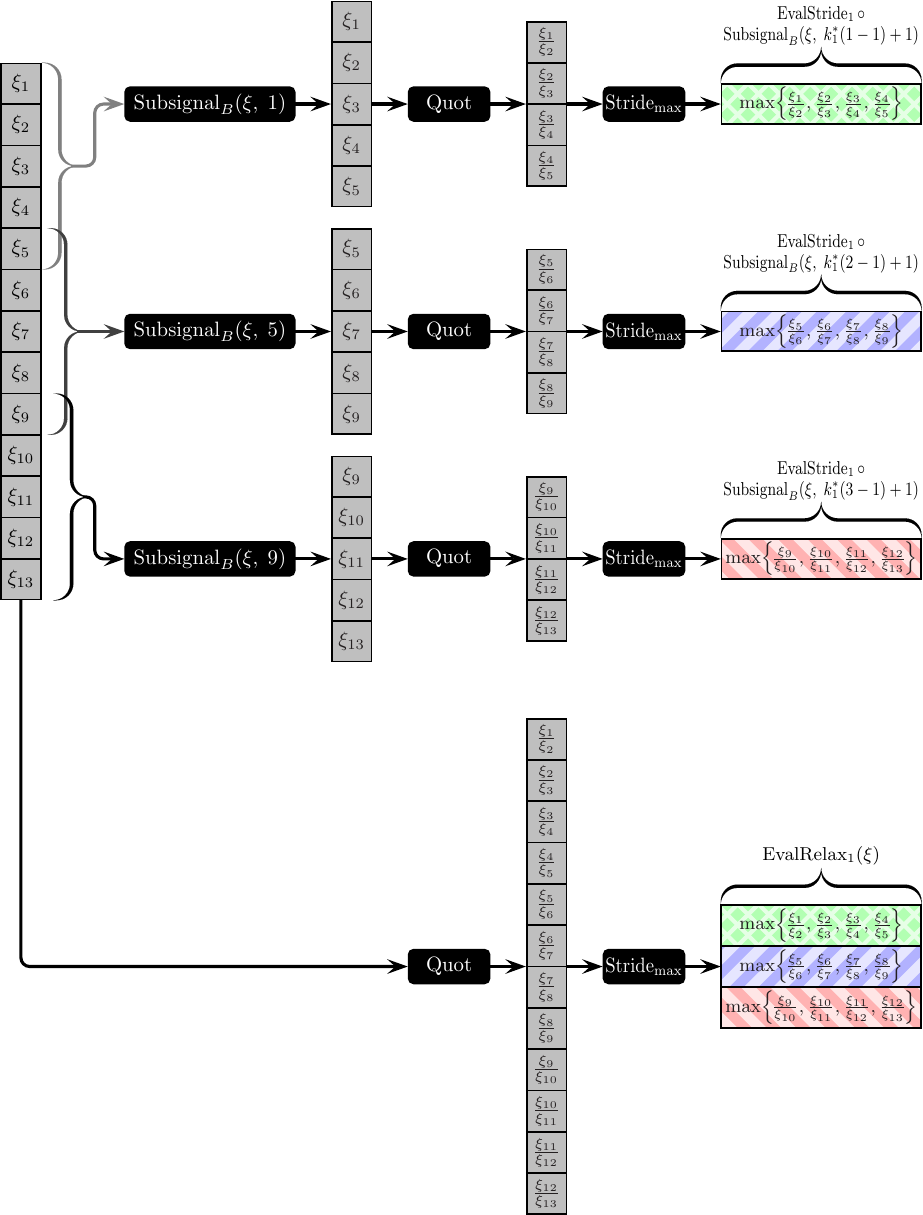}}
  \caption{Example of the application of a processing chain in a strided fashion to a subset of the feasible subsignals of an input signal (upper part of the graphics) and in a relaxed fashion to the entire input signal (lower part).
    The processing chain here has $L = 1$ layer, which consists of the $\operatorname{Quot}$ operator followed by max-pooling with a stride of $k_1 = 4$, resulting in a receptive field size of $\ROI = 5$.
    $\EvalStride$ is here well-defined and computes output signals with $u_1 = 1$ sample.
    The input signal $\xi$ has $D = 13$ samples in total.
    Since this implies that $k_1^* = 4$ divides $D - \ROI = 8$, the statements of Lemma~\ref{lem:evalrelax} hold.
    In particular, the result of $\EvalRelax$ evaluated on $\xi$ has $W_1 = 3$ samples.
    It contains the result of $\EvalStride$ evaluated on the subsignals of $\xi$ with the indices $k_1^*(i - 1) + 1$ for $i\inint{1}{W_1 - u_1 + 1} = \{1,\;2,\;3\}$ (highlighted with a colored pattern in the graphics).}
  \label{fig:relaxed}
\end{figure}

Lemma~\ref{lem:evalrelax} has shown that the output of the $\EvalRelax$ operator consists of the result of the $\EvalStride$ operator applied to only some of the subsignals of the input signal.
An example for this is illustrated in Fig.~\ref{fig:relaxed}.
The next result analyzes this in greater detail and further shows how the divisibility constraints of Lemma~\ref{lem:evalrelax} can be satisfied for arbitrary input signals without incurring any more output signal quality loss than already involved with the application of processing chains in a relaxed fashion.
\begin{theorem}
\label{tmh:evalrelax}
Suppose a processing chain as in Definition~\ref{def:evalrelax} is given so that $\EvalStride$ is well-defined.
Assume the processing chain applied in a strided fashion outputs signals with exactly one sample, that is $\dim_{M_L}(\EvalStride_L(\rho)) = 1$ for all $\rho\in M_0^\ROI$.
Further, assume the stride product of the $L$-th layer is non-trivial, that is $k_L^*\neq 1$ should hold.

Let $r_{\EvalRelax_L}\colon\cup_\ROI(M_0)\to\discint{0}{k_L^* - 1}$, $\xi\mapsto\rem{\dim_{M_0}(\xi) - \ROI}{k_L^*}$, denote the number of samples that have to be trimmed away so that $\EvalRelax_L$ can be applied due to divisibility constraints.
Consider the function $T\colon\cup_\ROI(M_0)\to\cup_1(M_L)$,
\begin{displaymath}
  \xi\mapsto\EvalRelax_L(\Trimming_{r_{\EvalRelax_L}(\xi)}(\xi))\text{,}
\end{displaymath}
which applies the processing chain in a relaxed fashion to an input signal where the final samples have been trimmed away.
Then $T = \Downsampling_{k_L^*}\circ \Slide_{\EvalStride_L}$, that is the result of $T$ equals exactly the result of conventional dense signal scanning downsampled by the final stride product $k_L^*$ (see Definition~\ref{def:downsampling} for the downsampling operator).
\end{theorem}
\begin{proof}
Consider an input signal $\xi\in\cup_\ROI(M_0)$ and let $\tilde{D} := \dim_{M_0}(\xi)$ denote its dimensionality.
It must be verified first that the trimming actually permits $\EvalRelax_L$ to be applied.
For showing this, define $D := \dim_{M_0}(\Trimming_{r_{\EvalRelax_L}(\xi)}(\xi))$, then $D = \tilde{D} - \rem{\tilde{D} - \ROI}{k_L^*}$ by Definition~\ref{def:stuff-trim} and by the definition of $r_{\EvalRelax_L}$.
Now $\rem{D - \ROI}{k_L^*} = \rem{\tilde{D} - \ROI - \rem{\tilde{D} - \ROI}{k_L^*}}{k_L^*} = 0$ due to the idempotence of the $\operatorname{rem}$ operator.
Therefore, Lemma~\ref{lem:evalrelax} guarantees $\EvalRelax_L$ can be applied to $\Trimming_{r_{\EvalRelax_L}(\xi)}(\xi)$.
Hence $T$ is well-defined.

Next, it is verified that the dimensionalities on both sides of the claimed identity match.
Since $u_L = 1$, Lemma~\ref{lem:evalrelax}\ref{lem:evalrelax-b} yields $\dim_{M_L}(T(\xi)) = \frac{1}{k_L^*}(D - \ROI) + 1$.
For the other side follows
\begin{displaymath}
  \dim_{M_L}(\Downsampling_{k_L^*}(\Slide_{\EvalStride_L}(\xi)))
  \ \ \equsing{D.~\ref{def:downsampling}}\ \ \nceil{\tfrac{1}{k_L^*} \dim_{M_L}(\Slide_{\EvalStride_L}(\xi))}
  \ \ \equsing{D.~\ref{def:sliding-function}}\ \ \nceil{\tfrac{1}{k_L^*}(\tilde{D} - \ROI + 1)}\text{.}
\end{displaymath}

First consider the case in which $k_L^*$ divides $\tilde{D} - \ROI + 1$.
This implies that $\rem{\tilde{D} - \ROI}{k_L^*} = k_L^* - 1$, and hence $\dim_{M_L}(T(\xi)) = \frac{1}{k_L^*}(\tilde{D} - \ROI + 1 - k_L^*) + 1 = \frac{1}{k_L^*}(\tilde{D} - \ROI + 1)$.
Since this is a natural number by requirement, it is a fixed point of the ceiling function, and hence dimensionalities match on both sides.

Now to the case where $k_L^*$ does not divide $\tilde{D} - \ROI + 1$.
Here, $\rem{\tilde{D} - \ROI}{k_L^*} \neq k_L^* - 1$ since $k_L^* > 1$.
This leads to
\begin{align*}
  & \dim_{M_L}(T(\xi))\\
  =\ \ & \tfrac{1}{k_L^*}(D - \ROI) + 1\\
  =\ \ & \div{D - \ROI}{k_L^*} + 1\\
  \equsing{($\lozenge$)}\ \ & \div{D - \ROI + 1 + \rem{\tilde{D} - \ROI}{k_L^*}}{k_L^*} + 1\\
  =\ \ & \div{\tilde{D} - \ROI + 1}{k_L^*} + 1\text{,}
\end{align*}
where in the ($\lozenge$) step $\rem{D - \ROI}{k_L^*} = 0$ and $1 + \rem{\tilde{D} - \ROI}{k_L^*}\inint{1}{k_L^* - 1}$ have been used.
On the other hand, $\dim_{M_L}(\Downsampling_{k_L^*}(\Slide_{\EvalStride_L}(\xi))) = \div{\tilde{D} - \ROI + 1}{k_L^*} + 1$ as $\rem{\tilde{D} - \ROI + 1}{k_L^*}\neq 0$.
In conclusion, dimensionalities match on both sides of the claimed identity.

Now the output signals can be compared sample-wise.
Let $\mu\inint{1}{\frac{1}{k_L^*}(D - \ROI) + 1}$ be an arbitrary sample index.
Then
\begin{align*}
  & T(\xi)_\mu\\
  =\ \ \ & \EvalRelax_L(\Trimming_{r_{\EvalRelax_L}(\xi)}(\xi))_\mu\\
  \equsing{D.~\ref{def:subsignal}}\ \ \ & \Subsignal_{u_L}(\EvalRelax_L(\Trimming_{r_{\EvalRelax_L}(\xi)}(\xi)),\; \mu)\\
  \equsing{L.~\ref{lem:evalrelax}\ref{lem:evalrelax-d}}\ \ \ & \EvalStride_L(\Subsignal_\ROI(\Trimming_{r_{\EvalRelax_L}(\xi)}(\xi),\; k_L^*(\mu - 1) + 1))\text{.}
\end{align*}
Here, it is
\begin{displaymath}
  \Subsignal_\ROI(\Trimming_{r_{\EvalRelax_L}(\xi)}(\xi),\; k_L^*(\mu - 1) + 1) \ \ \equsing{D.~\ref{def:subsignal}}\ \  \sum\nolimits_{\nu = 1}^\ROI \Trimming_{r_{\EvalRelax_L}(\xi)}(\xi)_{k_L^*(\mu - 1) + \nu}\cdot e_\nu^{\ROI}\text{.}
\end{displaymath}
The maximum sample index is here $k_L^*(\mu - 1) + \nu \leq k_L^*\cdot\frac{1}{k_L^*}(D - \ROI) + \ROI = D$, which exactly equals the dimensionality of $\Trimming_{r_{\EvalRelax_L}(\xi)}(\xi)$.
Thus the trimming operator can be omitted with Definition~\ref{def:stuff-trim}:
\begin{displaymath}
  T(\xi)_\mu
  = \EvalStride_L(\Subsignal_\ROI(\xi,\; k_L^*(\mu - 1) + 1))\text{.}
\end{displaymath}
On the other hand,
\begin{align*}
  & \Downsampling_{k_L^*}(\Slide_{\EvalStride_L}(\xi))_\mu\\
  \equsing{D.~\ref{def:downsampling}}\ \ & \Slide_{\EvalStride_L}(\xi)_{k_L^*(\mu - 1) + 1}\\
  \equsing{D.~\ref{def:sliding-function}}\ \ & \EvalStride_L(\Subsignal_\ROI(\xi,\; k_L^*(\mu - 1) + 1))\text{,}
\end{align*}
which proves the claimed identity.
\end{proof}

In conclusion, by processing appropriately trimmed signals through $\EvalRelax$ one obtains the same result as with dense signal scanning followed by downsampling.
In other words, $\EvalRelax$ involves no precision loss but a loss in spatial resolution since output samples are available only for each $k_L^*$-th input subsignal.
This is illustrated in Fig.~\ref{fig:relaxed-down} for an exemplary processing chain.

Hence the amount of pooling layers has a significant influence on the output signal quality if a processing chain is applied in a relaxed fashion rather than using any of the exact notions discussed earlier.
This degradation can be ameliorated for example using a skip architecture~\citerlx{Long2015rlx}, which includes forward connections from intermediate pooling results to the final processing chain output to restore a limited amount of spatial information.
Nevertheless, this method does not yield the same result as exact dense signal scanning.

\begin{figure}[p]
  \centering
  \scalebox{0.92}{\includegraphics[page=16]{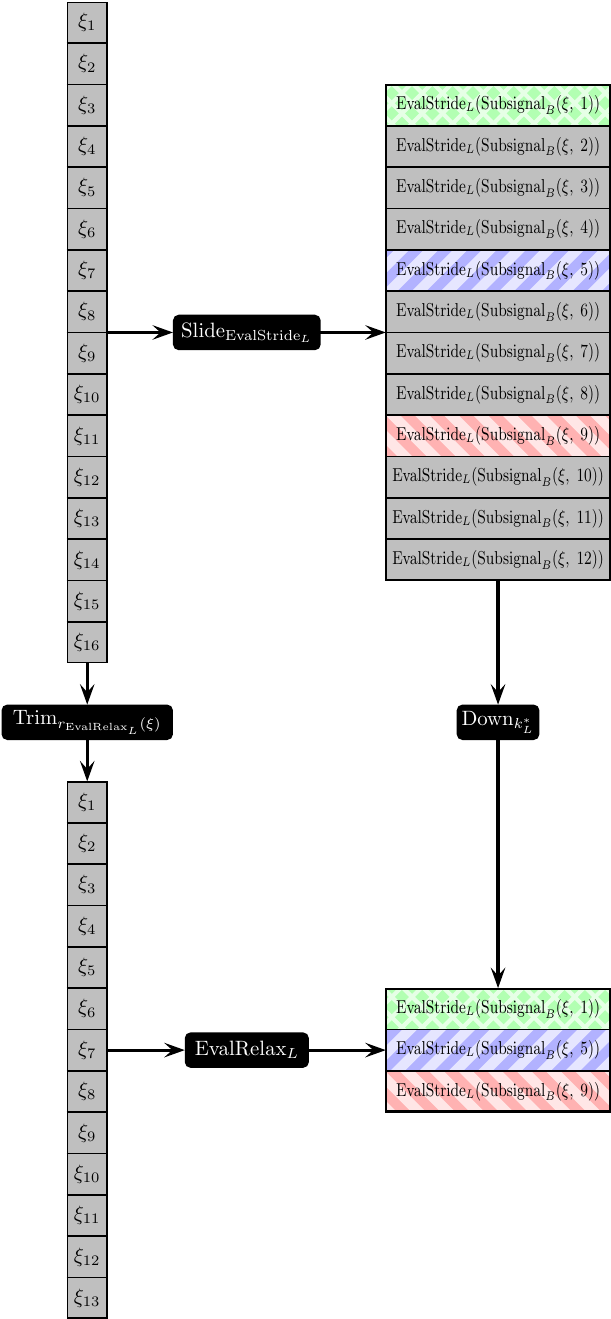}}
  \caption{Illustration of the statements of Theorem~\ref{tmh:evalrelax}.
    If the $\EvalStride$ operator is applied in a sliding fashion, an output value for each subsignal of the input signal with length equal to the processing chain's receptive field size $\ROI$ is gained.
    This is shown in the upper part of the graphics, where an input signal $\xi$ with $\tilde{D} = 16$ samples is processed using a receptive field size of $\ROI = 5$, yielding $\tilde{D} - \ROI + 1 = 12$ output samples.
    Application of the processing chain with a stride product of $k_L^* = 4$ in a relaxed fashion first requires that the input signal is trimmed to satisfy divisibility requirements.
    After its final $r_{\EvalRelax_L}(\xi) = 3$ entries have been trimmed away, the $\EvalRelax$ operator can be applied since $k_L^*$ divides $\tilde{D} - r_{\EvalRelax_L}(\xi) - \ROI = 8$.
    See Fig.~\ref{fig:relaxed} for a detailed illustration of processing chain application in a relaxed fashion.
    It here yields an output signal with three samples (lower part of the graphics).
    As shown in Theorem~\ref{tmh:evalrelax}, this corresponds exactly with the result of conventional dense signal scanning downsampled by the final stride product $k_L^*$ (highlighted with a colored pattern in the graphics).}
  \label{fig:relaxed-down}
\end{figure}

\subsection{Recovery of Full Spatial Resolution}
Although $\EvalRelax$ reduces the spatial output signal resolution by a factor equal to the final layer's stride product, it is straightforward to perform multiple passes through this operator using shifted versions of the input signal and afterwards combine the result to obtain full resolution output.
This has also been noted by~\citerlx{Long2015rlx}, and~\citerlx{Sermanet2014rlx} exploited a similar method to increase spatial resolution.
An exact \emph{shift-and-stitch} approach is here analyzed in greater detail and put in context with a fragmentation-based approach (see Fig.~\ref{fig:relaxed-stitch} for an illustration):

\begin{figure}[t]
  \centering
  \scalebox{1.035}{\includegraphics[page=17]{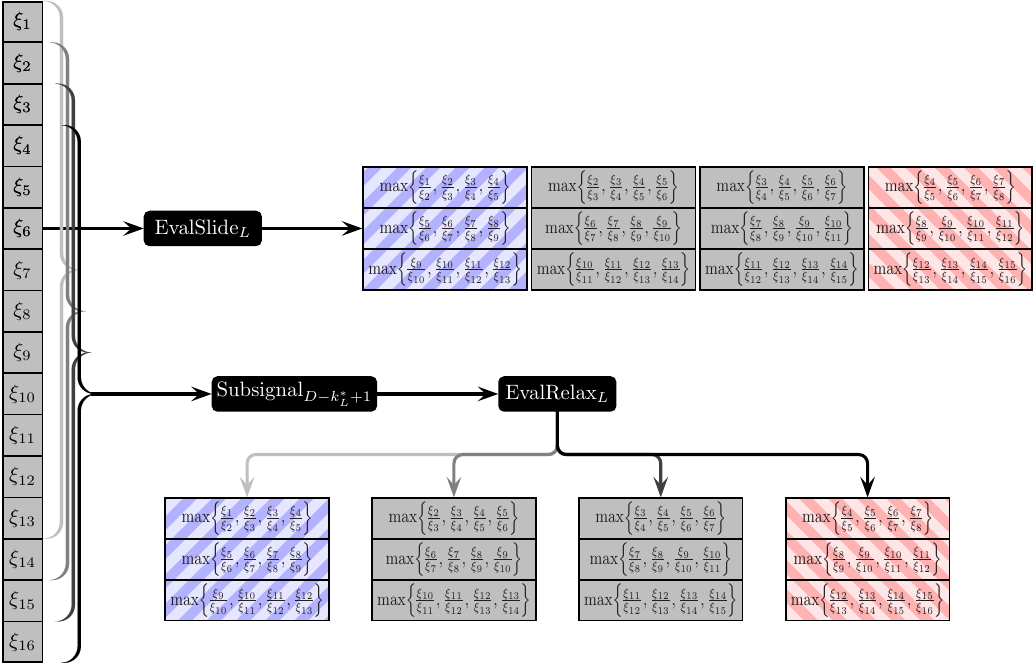}}
  \caption{Although the $\EvalRelax$ operator only computes output values for a regular subset of the feasible subsignals of an input signal $\xi$, repeated evaluation on shifted versions of $\xi$ leads to dense output values for all the feasible subsignals.
    This is illustrated in the lower part of the graphics for an input signal with $D = 16$ samples and for a processing chain with $L = 1$ layer consisting of the $\operatorname{Quot}$ operator and max-pooling with a stride of $k_1 = 4$ (as in Fig.~\ref{fig:relaxed}), resulting in a final stride product of $k_L^* = 4$.
    As shown in Theorem~\ref{thm:evalrelax_rec}, $\EvalRelax$ has to be applied to all $k_L^*$ subsignals of length $D - k_L^* + 1 = 13$ of the input signal $\xi$.
    The outcome of each shifted evaluation corresponds exactly to an output fragment of the $\EvalSlide$ operator, which processes $\xi$ in its entirety (upper part of the graphics, exemplary correspondences are highlighted with a colored pattern).
    In particular, the first output fragment matches the result of $\EvalRelax$ applied to a trimmed version of the input signal, as implied by Theorem~\ref{tmh:evalrelax}.}
  \label{fig:relaxed-stitch}
\end{figure}

\begin{theorem}
\label{thm:evalrelax_rec}
Consider a processing chain as in Definition~\ref{def:evalrelax} where $\EvalStride$ is well-defined.
Let $D\in\N_1$, $D\geq\ROI$, be an input signal dimensionality so that $k_L^*$ divides $D - \ROI + 1$ and let $\xi\in M_0^D$ be an input signal.
Then for all $\gamma\inint{1}{k_L^*}$ it holds that
\begin{displaymath}
  \sum\nolimits_{\nu = 1}^{U_L^{\row}} \EvalSlide_L(\xi)_{\nu,\;\gamma}\cdot e_\nu^{U_L^{\row}}
  = \EvalRelax_L(\Subsignal_{D - k_L^* + 1}(\xi,\;\gamma))\text{,}
\end{displaymath}
where $U_L^{\row}$ denotes the output signal dimensionality of $\EvalSlide_L$ applied to a signal of length $D$ as in Lemma~\ref{lem:processing-chain}.
Here, the left-hand side equals the $\gamma$-th fragment of the result of the processing chain applied in a sliding fashion to $\xi$.
The right-hand side is the outcome of the processing chain applied in a relaxed fashion to the subsignal starting at the $\gamma$-th sample of $\xi$, where the length has been chosen so that $\EvalRelax_L$ is applicable.
In other words, since $\EvalRelax$ is known to produce only a downsampled version of the result of exact dense signal scanning due to Theorem~\ref{tmh:evalrelax}, considering shifted versions of the input signal can be used to restore the original resolution.
\end{theorem}
\begin{proof}
The left-hand side of the claim is well-defined since the prerequisites of Lemma~\ref{lem:processing-chain} are satisfied and all accesses are within bounds due to Lemma~\ref{lem:processing-chain}\ref{lem:processing-chain-b}.
Considering the right-hand side, first note that there are actually $D - (D - k_L^* + 1) + 1 = k_L^*$ distinct subsignals with $D - k_L^* + 1$ samples in $\xi$.
It has further to be verified that the signal input to the $\EvalRelax_L$ operator has at least $\ROI$ samples.
By requirement there is a positive natural number $t\in\N_1$ with $D - \ROI + 1 = k_L^* t$.
Therefore, $\ROI = D - k_L^* t + 1 \leq D - k_L^* + 1$.
Moreover, Proposition~\ref{prop:number-theory} implies that $\rem{(D - k_L^* + 1) - \ROI}{k_L^*} = \rem{D - \ROI + 1}{k_L^*}$.
Since this number vanishes by requirement, the input signal length divisibility constraint of Lemma~\ref{lem:evalrelax} is fulfilled, and the right-hand side is well-defined.
Furthermore, the dimensionality of the right-hand side is
\begin{displaymath}
  W_L
  \ \ \ \equsing{L.~\ref{lem:evalrelax}\ref{lem:evalrelax-a}}\ \ \ \tfrac{1}{k_L^*}\left( (D - k_L^* + 1) - \sum\nolimits_{\mu = 1}^L k_{\mu - 1}^*(c_\mu - 1)\right)
  \ \ \ \equsing{L.~\ref{lem:processing-chain}\ref{lem:processing-chain-b}}\ \ \ U_L^{\row}\text{.}
\end{displaymath}
Here, $W_L$ is the output signal length of $\EvalRelax_L$ applied to a signal of length $D - k_L^* + 1$ from Lemma~\ref{lem:evalrelax}.
Hence the number of samples on both sides of the claim are equal, and the contents of the output signals can be compared.

The comparison is here carried out on the subsignal level rather than on the sample level, which allows for a more direct application of Lemma~\ref{lem:processing-chain}.
For this, let $\gamma\inint{1}{k_L^*}$ be a fixed fragment index or shifting offset.
The number of subsignals of length $u_L$, see Lemma~\ref{lem:processing-chain} for its concrete definition, in the claimed identity is $U_L^{\row} - u_L + 1$.
With Lemma~\ref{lem:processing-chain}\ref{lem:processing-chain-c} this number equals $\frac{1}{k_L^*}(D - \ROI + 1)$, which is a natural number by requirement.
Let $\tilde{i}\inint{1}{\frac{1}{k_L^*}(D - \ROI + 1)}$ be an arbitrary subsignal index and define $i := \gamma + k_L^*(\tilde{i} - 1)$.
Then clearly $i\inint{1}{D - \ROI + 1}$.
Let $\mu\inint{1}{u_L}$ be arbitrary, then
\begin{displaymath}
  \div{i - 1}{k_L^*} + \mu
  \ =\ \div{\gamma + k_L^*(\tilde{i} - 1) - 1}{k_L^*} + \mu
  \ \equsing{P.~\ref{prop:number-theory}}\ \div{\gamma - 1}{k_L^*} + \tilde{i} + \mu - 1
  \ =\ \tilde{i} + \mu - 1
\end{displaymath}
since $\div{\gamma - 1}{k_L^*} = 0$ as $\gamma - 1\inint{0}{k_L^* - 1}$.
For the same reason it follows that
\begin{displaymath}
  \rem{i - 1}{k_L^*} + 1
  \ =\ \rem{\gamma + k_L^*(\tilde{i} - 1) - 1}{k_L^*} + 1
  \ \equsing{P.~\ref{prop:number-theory}}\ \rem{\gamma - 1}{k_L^*} + 1
  \ =\ \gamma\text{.}
\end{displaymath}
Therefore, extraction of the subsignal with index $\tilde{i}$ from the $\gamma$-th fragment of the left-hand side of the claim leads to
\begin{align*}
  & \Subsignal_{u_L}\left(\sum\nolimits_{\nu = 1}^{U_L^{\row}} \EvalSlide_L(\xi)_{\nu,\;\gamma}\cdot e_\nu^{U_L^{\row}},\;\tilde{i}\right)\\
  \equsing{D.~\ref{def:subsignal}}\ \ \ & \sum\nolimits_{\mu = 1}^{u_L} \EvalSlide_L(\xi)_{\tilde{i} + \mu - 1,\;\gamma}\cdot e_\mu^{u_L}\\
  =\ \ \ & \sum\nolimits_{\mu = 1}^{u_L} \EvalSlide_L(\xi)_{\div{i - 1}{k_L^*} + \mu,\;\rem{i - 1}{k_L^*} + 1}\cdot e_\mu^{u_L}\\
  \equsing{L.~\ref{lem:processing-chain}\ref{lem:processing-chain-d}}\ \ \ & \sum\nolimits_{\mu = 1}^{u_L} \EvalStride_L(\Subsignal_\ROI(\xi,\; i))_\mu\cdot e_\mu^{u_L}\\
  =\ \ \ & \EvalStride_L(\Subsignal_\ROI(\xi,\; \gamma + k_L^*(\tilde{i} - 1)))\text{.}
\end{align*}

The subsignal of the right-hand side of the claim with the same subsignal index equals
\begin{align*}
  & \Subsignal_{u_L}(\EvalRelax_L(\Subsignal_{D - k_L^* + 1}(\xi,\;\gamma)),\;\tilde{i})\\
  \equsing{L.~\ref{lem:evalrelax}\ref{lem:evalrelax-d}}\ \ \ & \EvalStride_L(\Subsignal_\ROI(\Subsignal_{D - k_L^* + 1}(\xi,\;\gamma),\; k_L^*(\tilde{i} - 1) + 1))\\
  \equsing{L.~\ref{lem:subsignal-composition}}\ \ \ & \EvalStride_L(\Subsignal_\ROI(\xi,\; \gamma + k_L^*(\tilde{i} - 1)))\text{,}
\end{align*}
where it is straightforward to verify all requirements of Lemma~\ref{lem:subsignal-composition} are met.
Since all feasible subsignals of both sides of the claimed identity match, the proof is finished.
\end{proof}

Contiguous output values from the multiple passes of the $\EvalRelax$ operator can finally be obtained with conventional defragmentation.
It is noteworthy that Theorem~\ref{thm:evalrelax_rec} implies that ordinary $\EvalRelax$ application is the same as using the first fragment in all the intermediate steps of $\EvalSlide$ evaluation and discarding the other fragments.
Since $\EvalSlide$ processes the input signal once in its entirety and the just investigated shift-and-stitch approach analyzes several shifted versions of the input signal one after the other, one might suspect the latter has an elevated computational complexity.
It is shown next that indeed a significant amount of redundant computations is involved with shift-and-stitch.

\subsection{Computational Complexity Analysis}
Theorem~\ref{thm:evalrelax_rec} has demonstrated that full spatial resolution can be recovered by applying a processing chain in a relaxed fashion to multiple shifted versions of the original input signal.
Now, the computational complexity of this method is analyzed and compared with the complexity of a fragmentation-based approach.

\subsubsection{Prerequisites}
The analysis is started with a few properties that are required later:
\begin{remark}
\label{rem:evalrelax-comp}
In the situation of Theorem~\ref{thm:evalrelax_rec}, the following holds for all $j\inint{1}{L}$:
\begin{enumerate}
  \item \label{rem:evalrelax-comp-a} $W_j = U_j^{\row} - \frac{k_L^*}{k_j^*} + 1$.
  \item \label{rem:evalrelax-comp-b} $u_j \geq \frac{k_L^*}{k_j^*}$.
\end{enumerate}
\end{remark}
\begin{proof}
\ref{rem:evalrelax-comp-a}
As the dimensionality of the input to $\EvalRelax$ is here $D - k_L^* + 1$, one obtains
\begin{align*}
  W_j\ \ \ &\equsing{L.~\ref{lem:evalrelax}}\ \ \ \tfrac{1}{k_j^*}\left(D - k_L^* + 1 - \sum\nolimits_{\mu = 1}^j k_{\mu - 1}^*(c_\mu - 1)\right)\\
  &=\ \ \ \tfrac{1}{k_j^*}\left(D - k_j^* + 1 - \sum\nolimits_{\mu = 1}^j k_{\mu - 1}^*(c_\mu - 1)\right) - \tfrac{1}{k_j^*}\left(k_L^* - k_j^*\right)\\
  &\equsing{L.~\ref{lem:processing-chain}\ref{lem:processing-chain-b}}\ \ \ U_j^{\row} - \tfrac{k_L^*}{k_j^*} + 1\text{.}
\end{align*}

\ref{rem:evalrelax-comp-b}
This is shown through reverse induction.
For $j = L$, the claim is $u_L \geq \frac{k_L^*}{k_L^*} = 1$, which is fulfilled since the processing chain outputs non-empty signals.
Now consider the induction step $j \to j - 1$.
From Lemma~\ref{lem:processing-chain}\ref{lem:processing-chain-a} follows that $u_j = \frac{1}{k_j}(u_{j - 1} - c_j + 1)$, hence
\begin{displaymath}
  u_{j - 1}
  \ =\  k_j u_j + c_j - 1
  \ \gequsing{IH}\ k_j \tfrac{k_L^*}{k_j^*} + c_j - 1
  \ =\ \tfrac{k_L^*}{k_{j - 1}^*} + c_j - 1
  \ \geq\ \tfrac{k_L^*}{k_{j - 1}^*}\text{,}
\end{displaymath}
where in the final step $c_j \geq 1$ has been used.
Therefore, the claim holds for $j - 1$.
\end{proof}

For demonstrating that redundancies are involved in the shift-and-stitch approach detailed in Theorem~\ref{thm:evalrelax_rec}, a simple inequality will be used twice:
\begin{remark}
\label{rem:evalrelax-comp-ineq}
Let $k,h\in\R$ be real numbers with $h \geq k \geq 1$.
Then $k\cdot\left(1 - \frac{k - 1}{h}\right) \geq 1$.
\end{remark}
\begin{proof}
The claim clearly holds for $k = 1$.
Now suppose that $k > 1$, then $h \geq k = \frac{k - 1}{k - 1}\cdot k = \frac{k - 1}{(k - 1) / k} = \frac{k - 1}{1 - \frac{1}{k}}$.
As $h > 0$ and $1 - \frac{1}{k} > 0$ follows $1 - \frac{1}{k} \geq \frac{k - 1}{h}$.
Therefore $1 - \frac{k - 1}{h} \geq \frac{1}{k}$, and the claim follows as $k > 0$.
\end{proof}

\subsubsection{Identification of the Number of Function Evaluations}
Let $j\inint{1}{L}$ be a fixed layer index and let $\xi\in M_0^D$ be an input signal.
Assume the requirements of Theorem~\ref{thm:evalrelax_rec} are satisfied, that is, $\EvalSlide$ should be applicable to $\xi$.
Further, let $\gamma\inint{1}{k_L^*}$ be a fixed offset and consider the application of $\EvalRelax_j$ to $\Subsignal_{D - k_L^* + 1}(\xi,\;\gamma)$.
Let $\pi_{\EvalRelax_j} := \EvalRelax_{j - 1}(\Subsignal_{D - k_L^* + 1}(\xi,\;\gamma))\in M_{j - 1}^{W_{j - 1}}$ be an abbreviation for the input to the $j$-th layer, which then has to evaluate $\EvalRelax_j(\Subsignal_{D - k_L^* + 1}(\xi,\;\gamma)) = \Stride_{g_j}(\Slide_{f_j}(\pi_{\EvalRelax_j}))$ due to Definition~\ref{def:evalrelax}.
Since
\begin{displaymath}
  \Slide_{f_j}(\pi_{\EvalRelax_j})
  \ \ \equsing{D.~\ref{def:sliding-function}}\ \ \sum\nolimits_{\mu = 1}^{W_{j - 1} - c_j + 1} f_j(\Subsignal_{c_j}(\pi_{\EvalRelax_j},\;\mu))\cdot e_\mu^{W_{j - 1} - c_j + 1}\text{,}
\end{displaymath}
$W_{j - 1} - c_j + 1$ evaluations of $f_j$ are necessary.
Further,
\begin{align*}
  & \Stride_{g_j}(\Slide_{f_j}(\pi_{\EvalRelax_j}))\\
  \equsing{D.~\ref{def:strided-function}}\ \ & \sum\nolimits_{\nu = 1}^{(W_{j - 1} - c_j + 1) / k_j} g_j(\Subsignal_{k_j}(\Slide_{f_j}(\pi_{\EvalRelax_j}),\;k_j(\nu - 1) + 1))\cdot e_\nu^{(W_{j - 1} - c_j + 1) / k_j}\text{.}
\end{align*}
Since $(W_{j - 1} - c_j + 1) / k_j = W_j$ due to Lemma~\ref{lem:evalrelax}\ref{lem:evalrelax-a}, $g_j$ has to be evaluated $W_j$ times.

\subsubsection{Analysis for the Subsignal Compatible Transformation Evaluation Component}
Since $\EvalRelax_j$ has to be carried out $k_L^*$ times in Theorem~\ref{thm:evalrelax_rec}, the overall number of function evaluations divided by the number of function evaluations of a fragmentation-based approach equals
\begin{displaymath}
  S_{f_j}^{\EvalRelax} := \frac{k_L^*(W_{j - 1} - c_j + 1)}{U_{j - 1}^{\col} (U_{j - 1}^{\row} - c_j + 1)}\text{.}
\end{displaymath}
Here, the denominator equals exactly that of $S_{f_j}$ from the main part of this paper.
It follows that
\begin{displaymath}
  S_{f_j}^{\EvalRelax}
  \ \ \ \equsing{R.~\ref{rem:evalrelax-comp}\ref{rem:evalrelax-comp-a}}\ \ \ \frac{k_L^*(U_{j - 1}^{\row} - \frac{k_L^*}{k_{j - 1}^*} + 1 - c_j + 1)}{U_{j - 1}^{\col} (U_{j - 1}^{\row} - c_j + 1)}
  \ \ \ \equsing{L.~\ref{lem:processing-chain}\ref{lem:processing-chain-b}}\ \ \ \frac{k_L^*}{k_{j - 1}^*}\left(1 - \frac{\tfrac{k_L^*}{k_{j - 1}^*} - 1}{U_{j - 1}^{\row} - c_j + 1}\right)\text{.}
\end{displaymath}

It is first shown that $S_{f_j}^{\EvalRelax} \geq 1$.
For this, let $\bar{k}_{j - 1} := \frac{k_L^*}{k_{j - 1}^*} = k_j\cdots k_L \geq 1$ and $h_{j - 1} := U_{j - 1}^{\row} - c_j + 1$ be abbreviations.
From $U_{j - 1}^{\row} \geq u_{j - 1}$ follows that
\begin{displaymath}
  h_{j - 1}
  \ \geq\ u_{j - 1} - c_j + 1
  \ \ \ \equsing{L.~\ref{lem:processing-chain}\ref{lem:processing-chain-a}}\ \ \ k_j u_j
  \ \ \ \gequsing{R.~\ref{rem:evalrelax-comp}\ref{rem:evalrelax-comp-b}}\ \ \ k_j \tfrac{k_L^*}{k_j^*}
  \ =\ \tfrac{k_L^*}{k_{j - 1}^*}
  \ =\ \bar{k}_{j - 1}\text{,}
\end{displaymath}
therefore Remark~\ref{rem:evalrelax-comp-ineq} finally implies
\begin{displaymath}
  1 \leq \bar{k}_{j - 1}\cdot\left(1 - \frac{\bar{k}_{j - 1} - 1}{h_{j - 1}}\right) = S_{f_j}^{\EvalRelax}\text{.}
\end{displaymath}
In other words, a fragmentation-based approach requires at most the same number of function evaluations as $\EvalRelax$ being applied $k_L^*$ times to differently shifted versions of the original input signal.

For an analysis of the properties of $S_{f_j}^{\EvalRelax}$ when the input signal length is increased, let $D_+ := D + k_L^*$ denote the next larger feasible input signal dimensionality.
The only quantity in $S_{f_j}^{\EvalRelax}$ that depends on the input signal length is $U_{j - 1}^{\row}$.
Therefore,
\begin{align*}
  S_{f_j}^{\EvalRelax}(D_+) - S_{f_j}^{\EvalRelax}(D)
  &= \bar{k}_{j - 1}\cdot(\bar{k}_{j - 1} - 1)\cdot\left(\frac{1}{U_{j - 1}^{\row}(D) - c_j + 1} - \frac{1}{U_{j - 1}^{\row}(D_+) - c_j + 1}\right)\\
  &= \bar{k}_{j - 1}\cdot(\bar{k}_{j - 1} - 1)\cdot\frac{U_{j - 1}^{\row}(D_+) - U_{j - 1}^{\row}(D)}{(U_{j - 1}^{\row}(D_+) - c_j + 1) (U_{j - 1}^{\row}(D) - c_j + 1)}\text{.}
\end{align*}
This difference is non-negative as $\bar{k}_{j - 1}\geq 1$ and $U_{j - 1}^{\row}(D_+) \geq U_{j - 1}^{\row}(D)$.
Hence the speedup increases for a larger input signal dimensionality.
Further, one obtains $\lim_{D\to\infty} S_{f_j}^{\EvalRelax}(D) = \frac{k_L^*}{k_{j - 1}^*}$.
In other words, the maximally achievable speedup is greatest in early layers.

\subsubsection{Analysis for the Strided Function Evaluation Component}
Analogously to the analysis for the evaluations of subsignal compatible transformations, the quotient of the number of strided function evaluations is
\begin{displaymath}
  S_{g_j}^{\EvalRelax} := \frac{k_L^* W_j}{U_{j - 1}^{\col}k_j U_j^{\row}}\text{.}
\end{displaymath}
Here, the numerator corresponds to $k_L^*$ evaluations of $\EvalRelax_j$.
The denominator corresponds to a fragmentation-based approach, see $S_{g_j}$ from the main part of this paper.
Using Remark~\ref{rem:evalrelax-comp}\ref{rem:evalrelax-comp-a} and Lemma~\ref{lem:processing-chain}\ref{lem:processing-chain-b} yields
\begin{displaymath}
  S_{g_j}^{\EvalRelax} = \frac{k_L^*}{k_j^*}\left(1 - \frac{\tfrac{k_L^*}{k_j^*} - 1}{U_j^{\row}}\right)\text{.}
\end{displaymath}
As $U_j^{\row} \geq u_j$, which is in turn greater than $\frac{k_L^*}{k_j^*}$ with Remark~\ref{rem:evalrelax-comp}\ref{rem:evalrelax-comp-b}, Remark~\ref{rem:evalrelax-comp-ineq} implies that $S_{g_j}^{\EvalRelax} \geq 1$, proving that here redundant computations are involved as well.

Suppose that $D_+ := D + k_L^*$ denotes the next larger feasible input signal dimensionality and $\bar{k}_j := \tfrac{k_L^*}{k_j^*}$ is an abbreviation, then
\begin{displaymath}
  S_{g_j}^{\EvalRelax}(D_+) - S_{g_j}^{\EvalRelax}(D)
  = \bar{k}_j\cdot(\bar{k}_j - 1)\cdot\frac{U_j^{\row}(D_+) - U_j^{\row}(D)}{U_j^{\row}(D_+) U_j^{\row}(D)}
  \geq 0\text{.}
\end{displaymath}
Moreover, $\lim_{D\to\infty} S_{g_j}^{\EvalRelax}(D) = \frac{k_L^*}{k_j^*}$.
Hence here the speedup converges increasingly to the ratio of two stride products as well.

\subsubsection{Conclusion and Discussion of Plain Application in a Relaxed Fashion}
The computational complexity analysis has proved that a significant amount of redundant computations is involved should $\EvalRelax$ be used to obtain full resolution output using a shift-and-stitch approach.
For this use case, a method that avoids redundant computations as discussed earlier in this paper should be preferred.
If, however, in a certain use case downsampled output signals are sufficient, the $\EvalRelax$ operator can be used to achieve this with reasonable computational complexity:
If only one instead of $k_L^*$ evaluations of $\EvalRelax$ are carried out, the computational complexity analysis above can be repeated while simultaneously omitting the factor $k_L^*$ in the numerators of the ratios of the number of function evaluations for $f_j$ and $g_j$.
These ratios are hence
\begin{displaymath}
  \tilde{S}_{f_j}^{\EvalRelax}
  := \frac{1}{k_{j - 1}^*}\left(1 - \frac{\tfrac{k_L^*}{k_{j - 1}^*} - 1}{U_{j - 1}^{\row} - c_j + 1}\right)
  = \tfrac{1}{k_L^*}\cdot S_{f_j}^{\EvalRelax}
\end{displaymath}
and
\begin{displaymath}
  \tilde{S}_{g_j}^{\EvalRelax}
  := \frac{1}{k_j^*}\left(1 - \frac{\tfrac{k_L^*}{k_j^*} - 1}{U_j^{\row}}\right)
  = \tfrac{1}{k_L^*}\cdot S_{g_j}^{\EvalRelax}\text{.}
\end{displaymath}

Furthermore, in the limit case of arbitrarily large input signals one yields $\lim_{D\to\infty} \tilde{S}_{f_j}^{\EvalRelax}(D) = \frac{1}{k_{j - 1}^*} \leq 1$ and $\lim_{D\to\infty} \tilde{S}_{g_j}^{\EvalRelax}(D) = \frac{1}{k_j^*} \leq 1$.
This demonstrates that here plain $\EvalRelax$ requires only a fraction of the operations and hence has a lower computational complexity, of course at the expense of only yielding low-resolution output signals compared with the exact approaches discussed earlier.
It should finally be noted that the maximum speedup that can be achieved in each layer is only that of intermediate stride products.
These are not necessarily equal to the final stride product $k_L^*$, which corresponds to the actual downsampling factor.
In conclusion, the overall speedup is inferior to $k_L^*$.

\bibliographystylerlx{IEEEtran}
\bibliographyrlx{IEEEabrv,the}

\clearpage
\appendix[Application of Processing Chains in a Mixed Fashion]
The application of a processing chain in a relaxed fashion as discussed in the previous appendix results in a downscaled output signal.
More precisely, output values are obtained only for each $k_L^*$-th input subsignal, where $k_L^*$ denotes the final stride product of the processing chain.
This appendix investigates the application of a processing chain in a mixed fashion:
The first layers are applied in a relaxed fashion, followed by the application of the remaining layers in a sliding fashion as in a fragmentation-based approach which has been studied in the main part of this paper.
It is shown here that this approach facilitates direct control over the spatial output resolution through choosing how many of the processing chain's layers should be applied in a relaxed fashion.
An example is illustrated in Fig.~\ref{fig:relaxslide-dec}.
Besides the degradation of the output resolution, there is no further precision loss involved with this method.

\begin{figure}[t]
  \centering
  \scalebox{1.20}{\includegraphics[page=18]{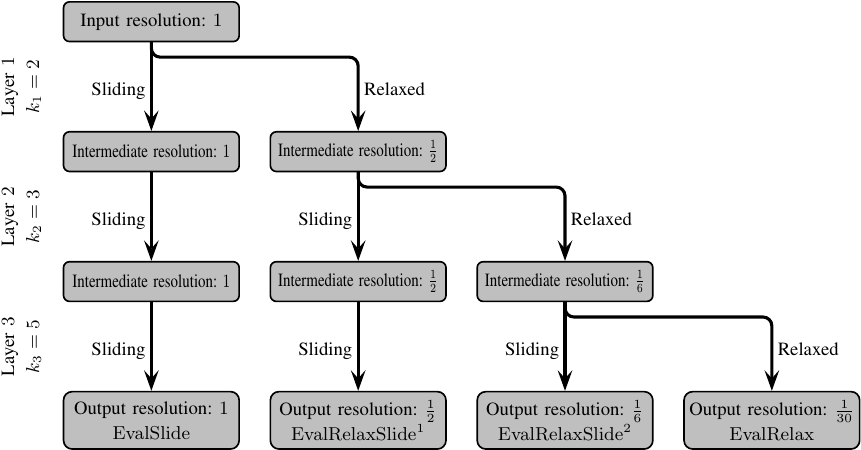}}
  \caption{Illustration of the effect of mixed processing chain application on the spatial resolution of the output signal using a three-layered ($L = 3$) processing chain with strides $k_1 = 2$, $k_2 = 3$, and $k_3 = 5$ as an example.
    A purely sliding evaluation using the $\EvalSlide$ operator, analyzed in the main part of this paper, involves no loss of spatial resolution whatsoever (left-hand side of the graphics).
    The other extreme of applying each layer in a relaxed fashion through the $\EvalRelax$ operator results in a decrease of spatial output resolution by a factor equal to the final stride product $k_L^* = 30$ (jagged transition from the input at the top left to the output at the bottom right of the graphics).
    The application in a mixed fashion using the $\EvalRelaxSlide^\ell$ operator introduced in this appendix offers a compromise between the two extremes.
    Here, the first $\ell$ layers are applied in a relaxed fashion, accounting for a resolution loss by only a factor of $k_\ell^*$.
    Then, the remaining $L - \ell$ layers are applied in a sliding fashion, preserving spatial resolution.
    For example, if $\ell = 2$ only the first two layers are applied in a relaxed fashion, and the third layer is applied in a sliding fashion.
    This results in a moderate resolution loss of factor $k_\ell^* = k_1 k_2 = 6$ (transition leading to the second output box from the right in the graphics).}    
  \label{fig:relaxslide-dec}
\end{figure}

A somewhat related approach is that of~\citerlxsld{Chen2016rlxsld}, where the stride of the application of dilated convolution and pooling has been systematically manipulated to obtain output signals with a reduced spatial resolution, which is however greater than the output resolution of a purely relaxed approach.
While it is unclear at first that such an approach would lead to the desired results and satisfies correctness statements in any sense, here it is proven rigorously that the combination of relaxed application with a fragmentation-based approach provides an efficient trade-off between full output resolution and completely reduced resolution.
The here developed results further show how divisibility requirements can be fulfilled prior to the actual processing chain application, enabling the use of homogeneous data structures at all times and therefore optimizing throughput on massively parallel processors.

\subsection{Application of Processing Chains in a Mixed Fashion}
The analysis starts with an exact definition of what is meant by the application of a processing chain in a mixed fashion:
\begin{definition}
\label{def:evalrelaxslide}
Consider a processing chain as in Definition~\ref{def:processing-chain} and let $\ell\inint{1}{L - 1}$ be a fixed layer index.
With Theorem~\ref{thm:sliding-subsignal} let $f_j\colon M_{j - 1}^{c_j}\to N_j$ be the unique functions which fulfill $T_j = \Slide_{f_j}$ for all $j\inint{1}{L}$.
For $j\inint{0}{L}$, the operator $\EvalRelaxSlide_j^\ell\colon\cup_{\ROI + k_L^* - k_\ell^*}(M_0)\to\cup_1(M_j)$, 
\begin{displaymath}
  \xi\mapsto
  \begin{cases}
    \xi\text{,} & \text{if } j = 0\text{,}\\
    \Stride_{g_j}(\Slide_{f_j}(\EvalRelaxSlide_{j - 1}^\ell(\xi)))\text{, } & \text{if } j\inint{1}{\ell}\text{,}\\
    \Fragmentation_{k_j}(\Slide_{g_j}(\Slide_{f_j}(\EvalRelaxSlide_{j - 1}^\ell(\xi))))\text{, } & \text{if } j\inint{\ell + 1}{L}\text{,}
  \end{cases}
\end{displaymath}
applies the processing chain in a fashion which is \emph{mixed of a relaxed and a sliding fashion}.
Divisibility requirements for well-definedness of this operator are detailed below.
\end{definition}

The just defined operator is parameterized with a layer index $\ell$.
The first $\ell$ layers of the processing chain are applied in a relaxed fashion, analogous to Definition~\ref{def:evalrelax}.
The remaining $L - \ell$ layers are applied in a sliding fashion using a fragmentation-based approach, as in Definition~\ref{def:processing-chain}.
Therefore, the extreme cases $\ell = 0$ and $\ell = L$ would correspond to a purely sliding fashion and a purely relaxed fashion, respectively, and were excluded in the definition as these have already been analyzed.

The following result elaborates on divisibility requirements for well-definedness.
Moreover, it is shown that the application of a processing chain in a mixed fashion results in the loss of spatial resolution, but not in any loss of precision.
The amount of the output signal degradation can be explicitly controlled with the $\ell$ parameter.

\begin{lemma}
\label{lem:evalrelaxslide}
Consider a processing chain as in Definition~\ref{def:evalrelaxslide}.
Suppose $\EvalStride$ is well-defined and let $u_j := \dim_{M_j}(\EvalStride_j(\rho))\in\N_1$ for $j\inint{0}{L}$ denote the lengths of the intermediate representations when this operator is applied to an arbitrary signal $\rho\in M_0^\ROI$.

Let $D\in\N_1$ be a signal length where there is a positive natural number $t\in\N_1$ with $D = \ROI + k_L^*t - k_\ell^*$.
As $t \geq 1$ this is at least the minimum input signal length of $\EvalRelaxSlide^\ell$ as required by Definition~\ref{def:evalrelaxslide}.
Suppose $\xi\in M_0^D$ denotes the input signal about to be processed.

For $j\inint{1}{\ell}$ let $W_j := \dim_{M_j}(\EvalRelaxSlide_j^\ell(\xi))\in\N_1$ denote the intermediate dimensionalities from the relaxed fashion part of $\EvalRelaxSlide^\ell$.
These numbers equal exactly those from Lemma~\ref{lem:evalrelax} with the same symbol.
Further, for $j\inint{\ell + 1}{L}$ let $\tilde{U}_j^{\col} := \cdim_{M_j}(\EvalRelaxSlide_j^\ell(\xi))\in\N_1$ and $\tilde{U}_j^{\row} := \rdim_{M_j}(\EvalRelaxSlide_j^\ell(\xi))\in\N_1$ denote the dimensionalities encountered during the application of the fragmentation-based part of $\EvalRelaxSlide^\ell$.
Then the following holds:
\begin{enumerate}
  \item \label{lem:evalrelaxslide-a} The intermediate stride product $k_\ell^*$ divides $D - \ROI$.
        This implies that $\EvalRelaxSlide_j^\ell(\xi)$ is well-defined for $j\inint{0}{\ell}$.
        The application up to the $\ell$-th layer yields an intermediate signal of dimensionality $W_\ell = \frac{1}{k_\ell^*}(D - \ROI) + u_\ell$.
  \item \label{lem:evalrelaxslide-b} $\EvalRelaxSlide_j^\ell(\xi)$ is well-defined with $\tilde{U}_j^{\col} = \frac{k_j^*}{k_\ell^*}$ and $\tilde{U}_j^{\row} - u_j + 1 = \frac{k_L^*}{k_j^*}t$ for all $j\inint{\ell + 1}{L}$.
  \item \label{lem:evalrelaxslide-c} For $j\inint{1}{\ell}$ and all $i\inint{1}{W_j - u_j + 1}$ it is
        \begin{displaymath}
          \EvalStride_j(\Subsignal_\ROI(\xi,\; k_j^*(i - 1) + 1))
          = \Subsignal_{u_j}(\EvalRelaxSlide_j^\ell(\xi),\; i)\text{.}
        \end{displaymath}
  \item \label{lem:evalrelaxslide-d} For all $j\inint{\ell + 1}{L}$, all $i\inint{1}{W_\ell - u_\ell + 1}$ and all $\mu\inint{1}{u_j}$ it holds that
        \begin{displaymath}
          \EvalStride_j(\Subsignal_\ROI(\xi,\; k_\ell^*(i - 1) + 1))_\mu
          = \EvalRelaxSlide_j^\ell(\xi)_{\div{i - 1}{k_j^* / k_\ell^*} + \mu,\;\rem{i - 1}{k_j^* / k_\ell^*} + 1}\text{.}
        \end{displaymath}
\end{enumerate}
In other words, $\EvalRelaxSlide_j^\ell(\xi)$ is well-defined for all $j\inint{0}{L}$ and the output of the $L$-th layer equals the result of $\EvalStride_L$ applied to each $k_\ell^*$-th subsignal of $\xi$ of length $\ROI$.
\end{lemma}
\begin{proof}
The proof presented here uses Lemma~\ref{lem:evalrelax} for the relaxed fashion parts and explicit arguments for the sliding fashion parts.
An alternative, but somewhat obscure proof could have proceeded with the application of Lemma~\ref{lem:processing-chain} using a modified processing chain.
This is not carried out here for clarity of the presentation.

\ref{lem:evalrelaxslide-a}
By requirement on the input signal length $D$ it follows that
\begin{displaymath}
  D - \ROI
  = k_L^*t - k_\ell^*
  = (k_{\ell + 1}\cdots k_L\cdot t - 1)\cdot k_\ell^*\text{,}
\end{displaymath}
therefore $k_\ell^*$ divides the difference $D - \ROI$.
Since $\EvalRelaxSlide_j^\ell(\xi) = \EvalRelax_j(\xi)$ for $j\inint{0}{\ell}$, Lemma~\ref{lem:evalrelax} now guarantees that the $\EvalRelaxSlide^\ell$ operator can be applied up to the $\ell$-th layer.
The claimed dimensionality follows directly from Lemma~\ref{lem:evalrelax}\ref{lem:evalrelax-b}.

\ref{lem:evalrelaxslide-b}
This is shown with induction.
Consider the case $j = \ell + 1$ and let $\pi_\ell := \EvalRelaxSlide_\ell^\ell(\xi)\in M_\ell^{W_\ell}$ be an abbreviation for the input to the considered layer.
Since the layers $1,\dotsc,\ell$ are completely agnostic of any fragmentation, it indeed holds that $\rdim_{M_\ell}(\pi_\ell) = W_\ell$ and $\cdim_{M_\ell}(\pi_\ell) = 1$.
Now
\begin{align*}
  \chi_{\ell + 1}\ &:=\ \rdim_{M_{\ell + 1}}(\Slide_{g_{\ell + 1}}(\Slide_{f_{\ell + 1}}(\pi_\ell)))\\
  &\equsing{D.~\ref{def:sliding-function}}\ W_\ell - c_{\ell + 1} + 1 - k_{\ell + 1} + 1\\
  &\equsing{\ref{lem:evalrelaxslide-a}}\ \tfrac{1}{k_\ell^*}(D - \ROI) + u_\ell - c_{\ell + 1} + 1 - k_{\ell + 1} + 1\\
  &\equsing{($\lozenge$)}\ \tfrac{1}{k_\ell^*}(k_L^*t - k_\ell^*) + k_{\ell + 1}u_{\ell + 1} + c_{\ell + 1} - 1 - c_{\ell + 1} + 1 - k_{\ell + 1} + 1\\
  &=\ k_{\ell + 1}\cdots k_L\cdot t + k_{\ell + 1}u_{\ell + 1} - k_{\ell + 1}\text{,}
\end{align*}
where in the ($\lozenge$) step $D - \ROI = k_L^*t - k_\ell^*$ from~\ref{lem:evalrelaxslide-a} and $u_{\ell} = k_{\ell + 1}u_{\ell + 1} + c_{\ell + 1} - 1$ from Lemma~\ref{lem:processing-chain}\ref{lem:processing-chain-a} have been substituted.
As $k_{\ell + 1}$ divides $\chi_{\ell + 1}$, $\Fragmentation_{k_{\ell + 1}}$ can be applied to $\Slide_{g_{\ell + 1}}(\Slide_{f_{\ell + 1}}(\pi_\ell))$.
This yields a non-empty signal because $t$ is positive, and hence $\EvalRelaxSlide_{\ell + 1}^\ell(\xi)$ is well-defined.
The number of output fragments is
\begin{align*}
  \tilde{U}_{\ell + 1}^{\col}\ \ &\equsing{D.\ref{def:evalrelaxslide}}\ \ \cdim_{M_{\ell + 1}}(\Fragmentation_{k_{\ell + 1}}(\Slide_{g_{\ell + 1}}(\Slide_{f_{\ell + 1}}(\pi_\ell))))\\
  &\equsing{L.~\ref{lem:frag-ops}}\ \ k_{\ell + 1}\cdot \cdim_{M_{\ell + 1}}(\Slide_{g_{\ell + 1}}(\Slide_{f_{\ell + 1}}(\pi_\ell)))\\
  &\equsing{D.~\ref{def:fragmentwise-evaluation}}\ \ k_{\ell + 1}\cdot \cdim_{M_\ell}(\pi_\ell)\\
  &=\ \ k_{\ell + 1}\\
  &=\ \ \tfrac{k_{\ell + 1}^*}{k_\ell^*}\text{,}
\end{align*}
which matches the claimed identity for $j = \ell + 1$.
Furthermore, one obtains
\begin{displaymath}
  \tilde{U}_{\ell + 1}^{\row}
  \ \ \equsing{D.\ref{def:evalrelaxslide}}\ \ \rdim_{M_{\ell + 1}}(\Fragmentation_{k_{\ell + 1}}(\Slide_{g_{\ell + 1}}(\Slide_{f_{\ell + 1}}(\pi_\ell))))
  \ \ \equsing{L.~\ref{lem:frag-ops}}\ \ \tfrac{1}{k_{\ell + 1}}\cdot\chi_{\ell + 1}
  \ \ =\ \ \tfrac{k_L^*}{k_{\ell + 1}^*}\cdot t + u_{\ell + 1} - 1
\end{displaymath}
and hence $\tilde{U}_{\ell + 1}^{\row} - u_{\ell + 1} + 1 = \frac{k_L^*}{k_{\ell + 1}^*}t$ as claimed.

Turning now to $j - 1 \to j$, let $\pi_{j - 1} := \EvalRelaxSlide_{j - 1}^\ell(\xi)\in M_{j - 1}^{\tilde{U}_{j - 1}^{\row} \times \tilde{U}_{j - 1}^{\col}}$ be the input to the $j$-th layer.
It is then
\begin{align*}
  \chi_j\ &:=\ \rdim_{M_j}(\Slide_{g_j}(\Slide_{f_j}(\pi_{j - 1})))\\
  &\equsing{D.~\ref{def:sliding-function}}\ \tilde{U}_{j - 1}^{\row} - c_j + 1 - k_j + 1\\
  &\equsing{IH}\ \tfrac{k_L^*}{k_{j - 1}^*}\cdot t + u_{j - 1} - 1 - c_j + 1 - k_j + 1\\
  &\equsing{($\lozenge$)}\ k_j\cdots k_L\cdot t + k_ju_j - k_j\text{.}
\end{align*}
The identity $u_{j - 1} = k_ju_j + c_j - 1$ used in the ($\lozenge$) step is due to Lemma~\ref{lem:processing-chain}\ref{lem:processing-chain-a}.
Therefore, $k_j$ is a divisor of $\chi_j$ and hence the fragmentation operator in layer $j$ can be applied.
This yields an output signal with
\begin{align*}
  \tilde{U}_j^{\col}\ \ &\equsing{D.\ref{def:evalrelaxslide}}\ \ \cdim_{M_j}(\Fragmentation_{k_j}(\Slide_{g_j}(\Slide_{f_j}(\pi_{j - 1}))))\\
  &\equsing{L.~\ref{lem:frag-ops}}\ \ k_j\cdot \cdim_{M_j}(\Slide_{g_j}(\Slide_{f_j}(\pi_{j - 1})))\\
  &\equsing{D.~\ref{def:fragmentwise-evaluation}}\ \ k_j\cdot \cdim_{M_{j - 1}}(\pi_{j - 1})\\
  &=\ \ k_j\cdot \tilde{U}_{j - 1}^{\col}\\
  &\equsing{IH}\ k_j\cdot\tfrac{k_{j - 1}^*}{k_\ell^*}\\
  &=\ \ \tfrac{k_j^*}{k_\ell^*}
\end{align*}
fragments, where each fragment has
\begin{displaymath}
  \tilde{U}_j^{\row}
  \ \ \equsing{D.\ref{def:evalrelaxslide}}\ \ \rdim_{M_j}(\Fragmentation_{k_j}(\Slide_{g_j}(\Slide_{f_j}(\pi_{j - 1}))))
  \ \ \equsing{L.~\ref{lem:frag-ops}}\ \ \tfrac{1}{k_j}\cdot\chi_j
  \ \ =\ \ \tfrac{k_L^*}{k_j^*}\cdot t + u_j - 1
\end{displaymath}
samples.
The claimed identity for $\tilde{U}_j^{\row}$ immediately follows from this result.
Eventually, this finished the induction, proved that the analyzed operator is well-defined for all $j\inint{\ell + 1}{L}$, and that the claimed identities hold.

\ref{lem:evalrelaxslide-c}
Follows directly from Lemma~\ref{lem:evalrelax}\ref{lem:evalrelax-d}.

\ref{lem:evalrelaxslide-d}
The induction for $j$ is started with $j = \ell + 1$.
Let $i\inint{1}{W_\ell - u_\ell + 1}$ and $\mu\inint{1}{u_{\ell + 1}}$ be arbitrary, and let $\tau^\ell := \EvalStride_\ell(\Subsignal_\ROI(\xi,\; k_\ell^*(i - 1) + 1))\in M_\ell^{u_\ell}$ be an abbreviation.
Then the $\mu$-th sample of the left-hand side of the claimed identity is
\begin{align*}
  & \EvalStride_{\ell + 1}(\Subsignal_\ROI(\xi,\; k_\ell^*(i - 1) + 1))_\mu\\
  \equsing{D.~\ref{def:processing-chain}}\ \ & \Stride_{g_{\ell + 1}}(\Slide_{f_{\ell + 1}}(\tau^\ell))_\mu\\
  \equsing{D.~\ref{def:strided-function}}\ \ & g_{\ell + 1}\!\left( \sum\nolimits_{\nu = 1}^{k_{\ell + 1}} \Slide_{f_{\ell + 1}}(\tau^\ell)_{k_{\ell + 1}(\mu - 1) + \nu}\cdot e_\nu^{k_{\ell + 1}} \right)\\
    \equsing{D.~\ref{def:sliding-function}}\ \ & g_{\ell + 1}\!\left( \sum\nolimits_{\nu = 1}^{k_{\ell + 1}} f_{\ell + 1}\!\left( \sum\nolimits_{\lambda = 1}^{c_{\ell + 1}} \tau^\ell_{k_{\ell + 1}(\mu - 1) + \nu + \lambda - 1}\cdot e_\lambda^{c_{\ell + 1}} \right) \cdot e_\nu^{k_{\ell + 1}} \right)\text{.}
\end{align*}
Considering the right-hand side, one obtains
\begin{align*}
  & \EvalRelaxSlide_{\ell + 1}^\ell(\xi)_{\div{i - 1}{k_{\ell + 1}^* / k_\ell^*} + \mu,\;\rem{i - 1}{k_{\ell + 1}^* / k_\ell^*} + 1}\\
  \equsing{D.~\ref{def:evalrelaxslide}}\ \ & \Fragmentation_{k_{\ell + 1}}(\Slide_{g_{\ell + 1}}(\Slide_{f_{\ell + 1}}(\EvalRelaxSlide_\ell^\ell(\xi))))_{\div{i - 1}{k_{\ell + 1}} + \mu,\;\rem{i - 1}{k_{\ell + 1}} + 1}\\
  \equsing{L.~\ref{lem:frag-ops}}\ \ & \Slide_{g_{\ell + 1}}(\Slide_{f_{\ell + 1}}(\EvalRelaxSlide_\ell^\ell(\xi)))_{\div{\phi_{\ell + 1}}{1} + 1,\;\rem{\phi_{\ell + 1}}{1} + 1}\text{,}
\end{align*}
where $\phi_{\ell + 1} := (\div{i - 1}{k_{\ell + 1}} + \mu - 1)k_{\ell + 1} + \rem{i - 1}{k_{\ell + 1}}$.
Note that in the application of Lemma~\ref{lem:frag-ops} it was used that there is only one input fragment to this layer as already noted in~\ref{lem:evalrelaxslide-b}.
The definition of the operators from Euclidean division implies that $\phi_{\ell + 1} = i - 1 + k_{\ell + 1}(\mu - 1)$.
Therefore,
\begin{align*}
  & \EvalRelaxSlide_{\ell + 1}^\ell(\xi)_{\div{i - 1}{k_{\ell + 1}^* / k_\ell^*} + \mu,\;\rem{i - 1}{k_{\ell + 1}^* / k_\ell^*} + 1}\\
  \equsing{P.~\ref{prop:number-theory}}\ \ & \Slide_{g_{\ell + 1}}(\Slide_{f_{\ell + 1}}(\EvalRelaxSlide_\ell^\ell(\xi)))_{i + k_{\ell + 1}(\mu - 1)}\\
    \equsing{D.~\ref{def:sliding-function}}\ \ & g_{\ell + 1}\!\left( \sum\nolimits_{\nu = 1}^{k_{\ell + 1}} f_{\ell + 1}\!\left( \sum\nolimits_{\lambda = 1}^{c_{\ell + 1}} \EvalRelaxSlide_\ell^\ell(\xi)_{i + k_{\ell + 1}(\mu - 1) + \nu - 1 + \lambda - 1}\cdot e_\lambda^{c_{\ell + 1}} \right) \cdot e_\nu^{k_{\ell + 1}} \right)\text{.}
\end{align*}
Now consider $\tilde{\mu} := k_{\ell + 1}(\mu - 1) + \nu + \lambda - 1$.
Clearly $\tilde{\mu}\in\N_1$ and further
\begin{displaymath}
  \tilde{\mu}
  \ \ \ \leq\ \ \ k_{\ell + 1}(u_{\ell + 1} - 1) + k_{\ell + 1} + c_{\ell + 1} - 1
  \ \ \ =\ \ \ k_{\ell + 1}u_{\ell + 1} + c_{\ell + 1} - 1
  \ \ \ \equsing{L.~\ref{lem:processing-chain}\ref{lem:processing-chain-a}}\ \ \ u_\ell\text{,}
\end{displaymath}
hence $\tilde{\mu}\inint{1}{u_\ell}$.
Therefore the inner part of the right-hand side of the claim leads to
\begin{align*}
  & \EvalRelaxSlide_\ell^\ell(\xi)_{i + k_{\ell + 1}(\mu - 1) + \nu - 1 + \lambda - 1}\\
  =\ & \EvalRelaxSlide_\ell^\ell(\xi)_{i + \tilde{\mu} - 1}\\
  \equsing{D.~\ref{def:subsignal}}\ & \Subsignal_{u_\ell}(\EvalRelaxSlide_\ell^\ell(\xi),\; i)_{\tilde{\mu}}\\
  \equsing{\ref{lem:evalrelaxslide-c}}\ & \EvalStride_\ell(\Subsignal_\ROI(\xi,\; k_\ell^*(i - 1) + 1))_{\tilde{\mu}}\\
  =\ & \tau^\ell_{\tilde{\mu}}\text{.}
\end{align*}
Substitution into the right-hand side of the claim yields the left-hand side of the claim, proving it for $j = \ell + 1$.

Turning to the induction step $j - 1 \to j$, let $i\inint{1}{W_\ell - u_\ell + 1}$ and $\mu\inint{1}{u_j}$ be indices and let $\tau^{j - 1} := \EvalStride_{j - 1}(\Subsignal_\ROI(\xi,\; k_\ell^*(i - 1) + 1))\in M_{j - 1}^{u_{j - 1}}$.
Completely analogous to the case when $j = \ell + 1$ follows for the left-hand side of the claim that
\begin{displaymath}
   \EvalStride_j(\Subsignal_\ROI(\xi,\; k_\ell^*(i - 1) + 1))_\mu
  \ =\ g_j\!\left( \sum\nolimits_{\nu = 1}^{k_j} f_j\!\left( \sum\nolimits_{\lambda = 1}^{c_j} \tau^{j - 1}_{k_j(\mu - 1) + \nu + \lambda - 1}\cdot e_\lambda^{c_j} \right) \cdot e_\nu^{k_j} \right)\text{.}
\end{displaymath}
On the other hand, for the right-hand side one obtains
\begin{align*}
  & \EvalRelaxSlide_j^\ell(\xi)_{\div{i - 1}{k_j^* / k_\ell^*} + \mu,\;\rem{i - 1}{k_j^* / k_\ell^*} + 1}\\
  \equsing{D.~\ref{def:evalrelaxslide}}\ \ & \Fragmentation_{k_j}(\Slide_{g_j}(\Slide_{f_j}(\EvalRelaxSlide_{j - 1}^\ell(\xi))))_{\div{i - 1}{k_j^* / k_\ell^*} + \mu,\;\rem{i - 1}{k_j^* / k_\ell^*} + 1}\\
  \equsing{L.~\ref{lem:frag-ops}}\ \ & \Slide_{g_j}(\Slide_{f_j}(\EvalRelaxSlide_{j - 1}^\ell(\xi)))_{\div{\phi_j}{k_{j - 1}^* / k_\ell^*} + 1,\;\rem{\phi_j}{k_{j - 1}^* / k_\ell^*} + 1}\text{,}
\end{align*}
where $\phi_j := (\div{i - 1}{k_j^* / k_\ell^*} + \mu - 1)k_j\frac{k_{j - 1}^*}{k_\ell^*} + \rem{i - 1}{k_j^* / k_\ell^*}$.
As shown in~\ref{lem:evalrelaxslide-b}, the number of input fragments is here $\tilde{U}_{j - 1}^{\col} = \frac{k_{j - 1}^*}{k_\ell^*}$.
Clearly $\phi_j = i - 1 + k_j(\mu - 1)\frac{k_{j - 1}^*}{k_\ell^*}$ by the definition of the operators from Euclidean division.
Therefore,
\begin{displaymath}
  \div{\phi_j}{k_{j - 1}^* / k_\ell^*} + 1
  \ \ \equsing{P.~\ref{prop:number-theory}}\ \ \div{i - 1}{k_{j - 1}^* / k_\ell^*} + k_j(\mu - 1) + 1
\end{displaymath}
and
\begin{displaymath}
  \rem{\phi_j}{k_{j - 1}^* / k_\ell^*} + 1
  \ \ \equsing{P.~\ref{prop:number-theory}}\ \ \rem{i - 1}{k_{j - 1}^* / k_\ell^*} + 1\text{.}
\end{displaymath}
This leads to
\begin{align*}
  & \EvalRelaxSlide_j^\ell(\xi)_{\div{i - 1}{k_j^* / k_\ell^*} + \mu,\;\rem{i - 1}{k_j^* / k_\ell^*} + 1}\\
  \equsing{D.~\ref{def:sliding-function}}\ \ & g_j\!\left( \sum\nolimits_{\nu = 1}^{k_j} f_j\!\left( \sum\nolimits_{\lambda = 1}^{c_j}\EvalRelaxSlide_{j - 1}^\ell(\xi)_{\div{i - 1}{k_{j - 1}^* / k_\ell^*} + k_j(\mu - 1) + \nu + \lambda - 1,\;\rem{i - 1}{k_{j - 1}^* / k_\ell^*} + 1} \!\cdot\! e_\lambda^{c_j} \right) \!\cdot\! e_\nu^{k_j} \right)\text{.}
\end{align*}
Define $\tilde{\mu} := k_j(\mu - 1) + \nu + \lambda - 1$, then Lemma~\ref{lem:processing-chain}\ref{lem:processing-chain-a} implies that $\tilde{\mu}\inint{1}{u_{j - 1}}$.
Eventually, it follows that
\begin{align*}
  & \EvalRelaxSlide_j^\ell(\xi)_{\div{i - 1}{k_j^* / k_\ell^*} + \mu,\;\rem{i - 1}{k_j^* / k_\ell^*} + 1}\\
  =\ \ & g_j\!\left( \sum\nolimits_{\nu = 1}^{k_j} f_j\!\left( \sum\nolimits_{\lambda = 1}^{c_j}\EvalRelaxSlide_{j - 1}^\ell(\xi)_{\div{i - 1}{k_{j - 1}^* / k_\ell^*} + \tilde{\mu},\;\rem{i - 1}{k_{j - 1}^* / k_\ell^*} + 1} \cdot e_\lambda^{c_j} \right) \cdot e_\nu^{k_j} \right)\\
\equsing{IH}\ \ &g_j\!\left( \sum\nolimits_{\nu = 1}^{k_j} f_j\!\left( \sum\nolimits_{\lambda = 1}^{c_j}\EvalStride_{j - 1}(\Subsignal_\ROI(\xi,\; k_\ell^*(i - 1) + 1))_{\tilde{\mu}} \cdot e_\lambda^{c_j} \right) \cdot e_\nu^{k_j} \right)\\
  =\ \ & g_j\!\left( \sum\nolimits_{\nu = 1}^{k_j} f_j\!\left( \sum\nolimits_{\lambda = 1}^{c_j} \tau^{j - 1}_{k_j(\mu - 1) + \nu + \lambda - 1}\cdot e_\lambda^{c_j} \right) \cdot e_\nu^{k_j} \right)\text{,}
\end{align*}
proving the claim.
\end{proof}

\begin{landscape}
\begin{figure}[p]
\centering
\scalebox{0.85}{\includegraphics[page=19]{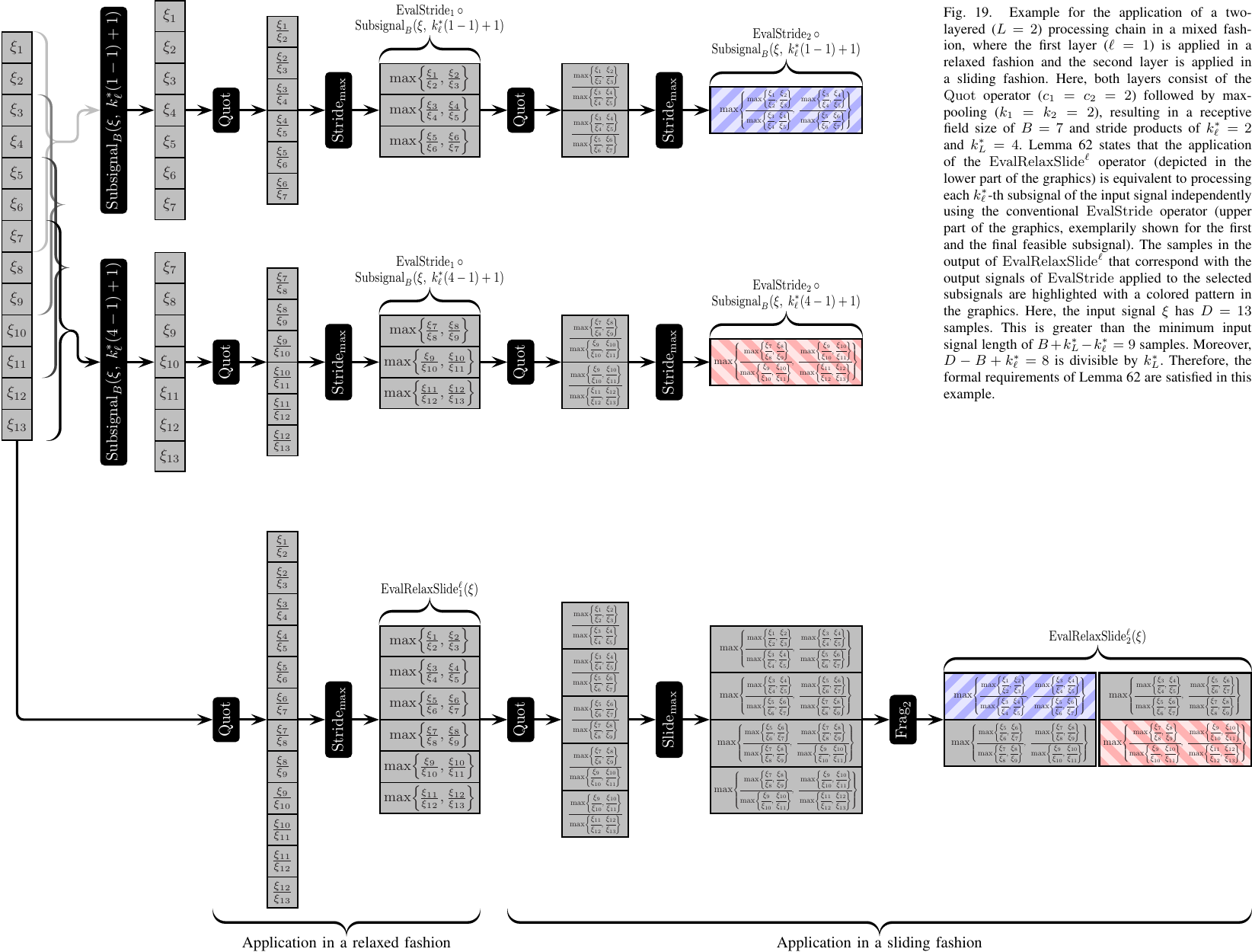}}
\newsavebox{\figrelaxslidecaption}
\sbox{\figrelaxslidecaption}{\parbox{20.78mm}{\caption{}}}
\label{fig:relaxslide}
\end{figure}
\end{landscape}

Lemma~\ref{lem:evalrelaxslide} has shown that the application of the $\EvalRelaxSlide^\ell$ operator is well-defined provided the final stride product $k_L^*$ divides $D - \ROI + k_\ell^*$ and the minimum input signal length of $\ROI + k_L^* - k_\ell^*$ is satisfied.
Further, the $\EvalRelaxSlide_L^\ell$ operator computes exactly the same as the $\EvalStride_L$ operator applied to each $k_\ell^*$-th subsignal with $\ROI$ samples of the input signal.
Therefore, by tuning the number of layers $\ell$ that are applied in a relaxed fashion, the degree of the loss of spatial output resolution becomes controllable between the factors $k_1^*,\dotsc,k_{L - 1}^*$.
If no resolution loss is desired, equivalent to a factor of $k_0^* = 1$, a conventional fragmentation-based approach can be used.
The other extreme of full resolution loss by a factor of $k_L^*$ can be achieved if all layers are applied in a relaxed fashion as is done during evaluation of the $\EvalRelax$ operator.

\subsection{Processing of Signals with Arbitrary Input Length}
Lemma~\ref{lem:evalrelaxslide} requires that the length of the input signal satisfies certain divisibility requirements.
Here, an approach is developed that allows processing of input signals with an arbitrary number of samples.
First, a relationship between the remainders with respect to two numbers, where one divides the other, is established:
\begin{lemma}
\label{lem:remdiv}
Let $a\in\N$ and let $m,n\in\N_1$ so that $m$ divides $n$.
Then there is an integer $z\in\Z$ with $\rem{a}{n} = \rem{a}{m} + m\cdot z$.
\end{lemma}
\begin{proof}
Since $m$ divides $n$ there is a positive natural number $t\in\N_1$ with $n = t\cdot m$.
As
\begin{displaymath}
  a = \div{a}{m}\cdot m + \rem{a}{m}
  \quad\text{and}\quad
  a = \div{a}{n}\cdot t\cdot m + \rem{a}{n}\text{,}
\end{displaymath}
one obtains
\begin{align*}
  &\rem{a}{n} - \rem{a}{m}\\
  =\ & a - \div{a}{n}\cdot t\cdot m - a + \div{a}{m}\cdot m\\
  =\ & (\div{a}{m} - \div{a}{n}\cdot t)\cdot m\text{,}
\end{align*}
and the claim follows by setting $z := \div{a}{m} - \div{a}{n}\cdot t$, which is indeed an integer.
\end{proof}

Next, an identity concerning the ceiling function's result for the quotient of two numbers, where the numerator does not divide the denominator, is developed:
\begin{lemma}
\label{lem:ceilrem}
Let $a\in\N$ and $b\in\N_1$ where it is required that $a$ does not divide $b$.
Then $\nceil{\frac{a}{b}} = \frac{a + b - \rem{a}{b}}{b}$.
\end{lemma}
\begin{proof}
It holds that $a = \div{a}{b}\cdot b + \rem{a}{b}$.
Therefore,
\begin{displaymath}
  \left\lceil \frac{a}{b} \right\rceil
  \ =\ \left\lceil \frac{\div{a}{b}\cdot b + \rem{a}{b}}{b} \right\rceil
  \ =\ \left\lceil \div{a}{b} + \frac{\rem{a}{b}}{b} \right\rceil
  \ \equsing{($\lozenge$)}\ \div{a}{b} + 1\text{,}
\end{displaymath}
where $\div{a}{b}\in\N$ and $0 < \frac{\rem{a}{b}}{b} < 1$ have been used in the ($\lozenge$) step.

From $a = \div{a}{b}\cdot b + \rem{a}{b}$ follows also that $a + b - \rem{a}{b} = \div{a}{b}\cdot b + b$, which is indeed divisible by $b$.
This implies that $\frac{a + b - \rem{a}{b}}{b} = \div{a}{b} + 1 = \nceil{\frac{a}{b}}$.
\end{proof}

\begin{figure}[t]
  \centering
  \scalebox{0.9}{\includegraphics[page=20]{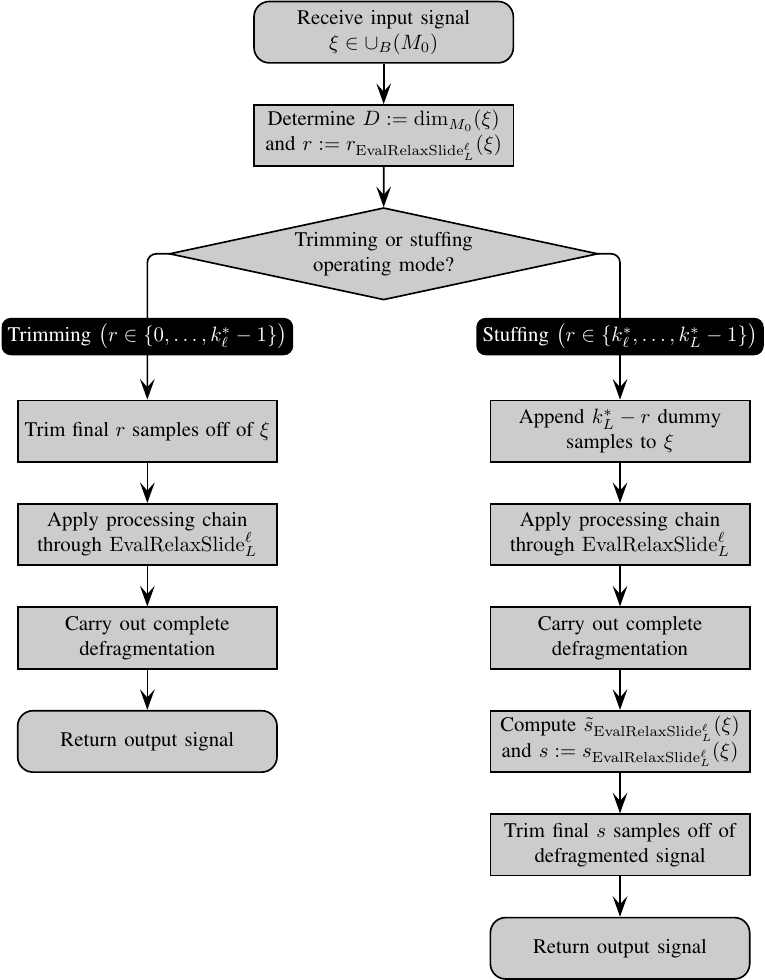}}
  \caption{Flowchart of the algorithm for the application of a processing chain in a mixed fashion to signals of arbitrary dimensionality.
    Depending on the concrete input signal length, either the trimming operating mode (left-hand side of the graphics) or the stuffing operating mode (right-hand side of the graphics) is carried out.
    Both are guaranteed to compute the same result as conventional dense signal scanning followed by downsampling using a factor equal to the intermediate stride product $k_\ell^*$ due to Theorem~\ref{thm:evalrelaxslide}.}
  \label{fig:relaxslide-alg}
\end{figure}

After these preparations, an algorithm for the application of a processing chain in a mixed fashion to a signal of arbitrary dimensionality can be proposed and proved correct.
The algorithm is summarized in Fig.~\ref{fig:relaxslide-alg}.
It consists of two operating modes, one in which the input signal is first trimmed and one in which dummy samples are first stuffed at the end of the input signal.
Which of these two modes is used eventually depends on the concrete input signal dimensionality.
\begin{theorem}
\label{thm:evalrelaxslide}
Suppose a processing chain as in Definition~\ref{def:evalrelaxslide} is given.
Further suppose the application of the processing chain in a strided fashion is well-defined and outputs signals with trivial spatial extent, that is $u_L := \dim_{M_L}(\EvalStride_L(\rho)) = 1$ for all $\rho\in M_0^\ROI$.

Let $r_{\EvalRelaxSlide_L^\ell}\colon\cup_\ROI(M_0)\to\discint{0}{k_L^* - 1}$, $\xi\mapsto\rem{\dim_{M_0}(\xi) - \ROI + k_\ell^*}{k_L^*}$, denote the number of samples that have to be trimmed away in the trimming operating mode, which is also the additive inverse modulo $k_L^*$ of the number of samples that have to be appended in the stuffing operating mode.
Moreover, let $\tilde{s}_{\EvalRelaxSlide_L^\ell}\colon\cup_\ROI(M_0)\to\discint{0}{k_\ell^*}$,
\begin{displaymath}
  \xi\mapsto
  \begin{cases}
    k_\ell^*\text{,} & \text{if }k_\ell^*\text{ divides }\dim_{M_0}(\xi) - \ROI + 1\text{,}\\
    \rem{\dim_{M_0}(\xi) - \ROI + 1}{k_\ell^*}\text{,} & \text{otherwise,}
  \end{cases}
\end{displaymath}
be an auxiliary function and let $s_{\EvalRelaxSlide_L^\ell}\colon\cup_\ROI(M_0)\to\N_1$,
\begin{displaymath}
  \xi\mapsto\tfrac{1}{k_\ell^*}(k_L^* - r_{\EvalRelaxSlide_L^\ell}(\xi) + \tilde{s}_{\EvalRelaxSlide_L^\ell}(\xi) - 1)\text{,}
\end{displaymath}
denote the number of superfluous samples that have to be trimmed away at the end of the stuffing operating mode due to the initial stuffing.

Let $T_{\Trimming}\colon\cup_{\ROI + k_L^* - k_\ell^*}(M_0)\to\cup_1(M_L)$,
\begin{displaymath}
  \xi\mapsto \Defragmentation_{k_L^* / k_\ell^*}( \EvalRelaxSlide_L^\ell( \Trimming_{r_{\EvalRelaxSlide_L^\ell}(\xi)}(\xi) ) )\text{,}
\end{displaymath}
denote the operator that implements the trimming operating mode.
Here, the input signal is first trimmed to satisfy divisibility constraints.
Then, the processing chain is applied in a mixed fashion.
Finally, the output signal is defragmented to obtain the overall result.

The function $T_{\Stuffing}\colon\cup_\ROI(M_0)\to\cup_1(M_L)$,
\begin{displaymath}
  \xi\mapsto \Trimming_{s_{\EvalRelaxSlide_L^\ell}(\xi)}( \Defragmentation_{k_L^* / k_\ell^*}( \EvalRelaxSlide_L^\ell( \Stuffing_{k_L^* - r_{\EvalRelaxSlide_L^\ell}(\xi)}(\xi) ) ) )\text{,}
\end{displaymath}
implements the stuffing operating mode.
First, divisibility requirements are fulfilled by appending a certain number of dummy samples to the input signal.
After the processing chain has been applied in a mixed fashion, its output is defragmented.
Superfluous samples that emerged from the stuffing are finally removed.

Eventually, the function $T\colon\cup_\ROI(M_0)\to\cup_1(M_L)$,
\begin{displaymath}
  \xi\mapsto
  \begin{cases}
    T_{\Trimming}(\xi)\text{,} & \text{if } r_{\EvalRelaxSlide_L^\ell}(\xi)\inint{0}{k_\ell^* - 1}\text{,}\\
    T_{\Stuffing}(\xi)\text{,} & \text{if } r_{\EvalRelaxSlide_L^\ell}(\xi)\inint{k_\ell^*}{k_L^* - 1}\text{,}
  \end{cases}
\end{displaymath}
decides which of the two operating modes should be used based on the input signal's number of samples and returns its result.

Then $T = \Downsampling_{k_\ell^*}\circ \Slide_{\EvalStride_L}$.
In other words, $T$ yields the same result as conventional dense signal scanning followed by downsampling by the intermediate stride product $k_\ell^*$.
\end{theorem}
\begin{proof}
Let $\xi\in\cup_\ROI(M_0)$ be an input signal and write $D := \dim_{M_0}(\xi)$.
Based on the concrete value of $D$, the proof is here divided into the cases in which $T = T_{\Trimming}$ and $T = T_{\Stuffing}$, respectively.
First, a few statements required for analyzing both cases are shown:
\begin{enumerate}\setlength{\baselineskip}{1.1\baselineskip}
  \item \label{thm:evalrelaxslide-preq-a} $\dim_{M_L}( \Downsampling_{k_\ell^*}( \Slide_{\EvalStride_L}(\xi) ) ) = \nceil{\frac{1}{k_\ell^*}( D - \ROI + 1 )}$.
  \item \label{thm:evalrelaxslide-preq-b} $\Downsampling_{k_\ell^*}( \Slide_{\EvalStride_L}(\xi) )_i = \EvalStride_L(\Subsignal_\ROI(\xi,\; k_\ell^*(i - 1) + 1))$ for all feasible sample indices $i\inint{1}{\dim_{M_L}( \Downsampling_{k_\ell^*}( \Slide_{\EvalStride_L}(\xi) ) )}$.
\end{enumerate}

The proof is accomplished first for $T_{\Trimming}$, that is the case in which $r_{\EvalRelaxSlide_L^\ell}(\xi)\inint{0}{k_\ell^* - 1}$ is fulfilled, in these steps:
\begin{enumerate}\setlength{\baselineskip}{1.1\baselineskip}
  \setcounter{enumi}{2}
  \item \label{thm:evalrelaxslide-trim-a} $D_{\Trimming} := \dim_{M_0}( \Trimming_{r_{\EvalRelaxSlide_L^\ell}(\xi)}(\xi) ) = D - r_{\EvalRelaxSlide_L^\ell}(\xi) \geq \ROI + k_L^* - k_\ell^*$. In other words, the minimum input signal length required in Lemma~\ref{lem:evalrelaxslide} is fulfilled.
  \item \label{thm:evalrelaxslide-trim-b} There exists $t_{\Trimming}\in\N_1$ so that $D_{\Trimming} = \ROI + k_L^*t_{\Trimming} - k_\ell^*$. Hence, the divisibility requirements of Lemma~\ref{lem:evalrelaxslide} are met as well, and $\EvalRelaxSlide_L^\ell( \Trimming_{r_{\EvalRelaxSlide_L^\ell}(\xi)}(\xi) )$ is well-defined.
  \item \label{thm:evalrelaxslide-trim-c} $\Defragmentation_{k_L^* / k_\ell^*}$ can be applied to $\EvalRelaxSlide_L^\ell( \Trimming_{r_{\EvalRelaxSlide_L^\ell}(\xi)}(\xi) )$. This renders $T_{\Trimming}(\xi)$ well-defined with exactly one output fragment with $\dim_{M_L}( T_{\Trimming}(\xi) ) = \tfrac{1}{k_\ell^*}(D_{\Trimming} - \ROI) + 1$ samples.
  \item \label{thm:evalrelaxslide-trim-d} $\dim_{M_L}( T_{\Trimming}(\xi) ) = \dim_{M_L}( \Downsampling_{k_\ell^*}( \Slide_{\EvalStride_L}(\xi) ) )$, therefore both output signals are of the same dimensionality.
  \item \label{thm:evalrelaxslide-trim-e} $T_{\Trimming}(\xi) = \Downsampling_{k_\ell^*}( \Slide_{\EvalStride_L}(\xi) )$, proving correctness of the trimming operating mode.
\end{enumerate}

Afterwards, the situation in which $r_{\EvalRelaxSlide_L^\ell}(\xi)\inint{k_\ell^*}{k_L^* - 1}$ and hence $T_{\Stuffing}$ is used, is analyzed:
\begin{enumerate}\setlength{\baselineskip}{1.1\baselineskip}
  \setcounter{enumi}{7}
  \item \label{thm:evalrelaxslide-stuff-a} $D_{\Stuffing} := \dim_{M_0}( \Stuffing_{k_L^* - r_{\EvalRelaxSlide_L^\ell}(\xi)}(\xi) ) = D + k_L^* - r_{\EvalRelaxSlide_L^\ell}(\xi) \geq \ROI + k_L^* - k_\ell^*$, hence the minimum input signal length of Lemma~\ref{lem:evalrelaxslide} is satisfied.
  \item \label{thm:evalrelaxslide-stuff-b} There is a number $t_{\Stuffing}\in\N_1$ with $D_{\Stuffing} = \ROI + k_L^*t_{\Stuffing} - k_\ell^*$, therefore also the divisibility requirements of Lemma~\ref{lem:evalrelaxslide} are fulfilled and $\EvalRelaxSlide_L^\ell( \Stuffing_{k_L^* - r_{\EvalRelaxSlide_L^\ell}(\xi)}(\xi) )$ is well-defined.
  \item \label{thm:evalrelaxslide-stuff-c} The requirements for the application of $\Defragmentation_{k_L^* / k_\ell^*}$ to $\EvalRelaxSlide_L^\ell( \Stuffing_{k_L^* - r_{\EvalRelaxSlide_L^\ell}(\xi)}(\xi) )$ are satisfied.
    Complete defragmentation produces an output signal with exactly one fragment and with $\dim_{M_L}( \Defragmentation_{k_L^* / k_\ell^*}( \EvalRelaxSlide_L^\ell( \Stuffing_{k_L^* - r_{\EvalRelaxSlide_L^\ell}(\xi)}(\xi) ) ) ) = \tfrac{1}{k_\ell^*}(D_{\Stuffing} - \ROI) + 1$ output samples.
  \item \label{thm:evalrelaxslide-stuff-d} $s_{\EvalRelaxSlide_L^\ell}(\xi)$ and $T_{\Stuffing}(\xi)$ are well-defined.
    The number of samples in both output signals matches, that is $\dim_{M_L}( T_{\Stuffing}(\xi) ) = \dim_{M_L}( \Downsampling_{k_\ell^*}( \Slide_{\EvalStride_L}(\xi) ) )$.
  \item \label{thm:evalrelaxslide-stuff-e} $T_{\Stuffing}(\xi) = \Downsampling_{k_\ell^*}( \Slide_{\EvalStride_L}(\xi) )$, or in other words, the stuffing operating mode is also correct.
\end{enumerate}
The combination of~\ref{thm:evalrelaxslide-trim-e} and~\ref{thm:evalrelaxslide-stuff-e} finally implies that $T = \Downsampling_{k_\ell^*}\circ \Slide_{\EvalStride_L}$ regardless of the concrete input signal dimensionality $D$.

{\parindent0mm \emph{Proof of the common statements.}}
First, the statements on downsampling commonly required by the two cases are shown:

\ref{thm:evalrelaxslide-preq-a}
One obtains
\begin{align*}
  & \dim_{M_L}( \Downsampling_{k_\ell^*}( \Slide_{\EvalStride_L}(\xi) ) )\\
  \equsing{D.~\ref{def:downsampling}}\ \ & \nceil{\tfrac{1}{k_\ell^*}( \dim_{M_L}( \Slide_{\EvalStride_L}(\xi) ) )}\\
  \equsing{D.~\ref{def:sliding-function}}\ \ & \nceil{\tfrac{1}{k_\ell^*}( D - \ROI + 1 )}\text{.}
\end{align*}

\ref{thm:evalrelaxslide-preq-b}
Let $i\inint{1}{\dim_{M_L}( \Downsampling_{k_\ell^*}( \Slide_{\EvalStride_L}(\xi) ) )}$, then
\begin{align*}
  & \Downsampling_{k_\ell^*}( \Slide_{\EvalStride_L}(\xi) )_i\\
  \equsing{D.~\ref{def:downsampling}}\ \ & \Slide_{\EvalStride_L}(\xi)_{k_\ell^*(i - 1) + 1}\\
  \equsing{D.~\ref{def:sliding-function}}\ \ & \EvalStride_L(\Subsignal_\ROI(\xi,\; k_\ell^*(i - 1) + 1))\text{.}
\end{align*}

{\parindent0mm \emph{Proof for $T = T_{\Trimming}$.}}
Here it holds that $r_{\EvalRelaxSlide_L^\ell}(\xi)\inint{0}{k_\ell^* - 1}$ by requirement.

\ref{thm:evalrelaxslide-trim-a}
Clearly $D \geq \ROI \geq k_L^* > r_{\EvalRelaxSlide_L^\ell}(\xi)$, therefore $\Trimming_{r_{\EvalRelaxSlide_L^\ell}(\xi)}$ can be applied to $\xi$.
With Definition~\ref{def:stuff-trim} follows that $D_{\Trimming} := \dim_{M_0}( \Trimming_{r_{\EvalRelaxSlide_L^\ell}(\xi)}(\xi) ) = D - r_{\EvalRelaxSlide_L^\ell}(\xi)$.
By definition it holds that $D - \ROI + k_\ell^* = \div{D - \ROI + k_\ell^*}{k_L^*}\cdot k_L^* + r_{\EvalRelaxSlide_L^\ell}(\xi)$.
First assume $\div{D - \ROI + k_\ell^*}{k_L^*}$ would vanish, then for the left-hand side it would hold that $D - \ROI + k_\ell^* \geq k_\ell^*$, whereas the right-hand side reads $r_{\EvalRelaxSlide_L^\ell}(\xi) < k_\ell^*$.
As this is impossible, $\div{D - \ROI + k_\ell^*}{k_L^*}$ must be positive.
From the identity above one eventually obtains $D - r_{\EvalRelaxSlide_L^\ell}(\xi) = \ROI + \div{D - \ROI + k_\ell^*}{k_L^*}\cdot k_L^* - k_\ell^* \geq \ROI + k_L^* - k_\ell^*$, as claimed.

\ref{thm:evalrelaxslide-trim-b}
With the identities from~\ref{thm:evalrelaxslide-trim-a} follows that
\begin{align*}
  & D_{\Trimming} - \ROI + k_\ell^*\\
  \equsing{\ref{thm:evalrelaxslide-trim-a}}\ \ & D - r_{\EvalRelaxSlide_L^\ell}(\xi) - \ROI + k_\ell^*\\
  \equsing{\ref{thm:evalrelaxslide-trim-a}}\ \ & \div{D - \ROI + k_\ell^*}{k_L^*}\cdot k_L^*\text{,}
\end{align*}
hence $t_{\Trimming} := \div{D - \ROI + k_\ell^*}{k_L^*}\in\N$ implies $D_{\Trimming} = \ROI + k_L^*t_{\Trimming} - k_\ell^*$.
It has already been shown in~\ref{thm:evalrelaxslide-trim-a} that $t_{\Trimming}\neq 0$, hence $t_{\Trimming}\in\N_1$.
Since all requirements from Lemma~\ref{lem:evalrelaxslide} are fulfilled, $\EvalRelaxSlide_L^\ell$ can be applied to $\Trimming_{r_{\EvalRelaxSlide_L^\ell}(\xi)}(\xi)$.

\ref{thm:evalrelaxslide-trim-c}
The application of Lemma~\ref{lem:evalrelaxslide}\ref{lem:evalrelaxslide-b} yields $\cdim_{M_L}( \EvalRelaxSlide_L^\ell( \Trimming_{r_{\EvalRelaxSlide_L^\ell}(\xi)}(\xi) ) ) = \frac{k_L^*}{k_\ell^*}$ and $\rdim_{M_L}( \EvalRelaxSlide_L^\ell( \Trimming_{r_{\EvalRelaxSlide_L^\ell}(\xi)}(\xi) ) ) = \frac{k_L^*}{k_L^*}t_{\Trimming} = t_{\Trimming}$ as $u_L = 1$ by requirement.
Since the number of fragments equals $\frac{k_L^*}{k_\ell^*}$, $\Defragmentation_{k_L^* / k_\ell^*}$ can be applied to $\EvalRelaxSlide_L^\ell( \Trimming_{r_{\EvalRelaxSlide_L^\ell}(\xi)}(\xi) )$ which eventually yields $T_{\Trimming}(\xi)$.
As this is a complete defragmentation, the number of output fragments is unity due to Lemma~\ref{lem:defrag-ops}, that is $\cdim_{M_L}( T_{\Trimming}(\xi) ) = 1$.
From~\ref{thm:evalrelaxslide-trim-b} follows $k_L^*t_{\Trimming} = D_{\Trimming} - \ROI + k_\ell^*$, therefore the output signal length equals
\begin{align*}
  & \dim_{M_L}( T_{\Trimming}(\xi) )\\
  =\ \ & \rdim_{M_L}( T_{\Trimming}(\xi) )\\
  \equsing{L.~\ref{lem:defrag-ops}}\ \ & \tfrac{k_L^*}{k_\ell^*}\cdot\rdim_{M_L}( \EvalRelaxSlide_L^\ell( \Trimming_{r_{\EvalRelaxSlide_L^\ell}(\xi)}(\xi) ) )\\
  =\ \ & \tfrac{k_L^*t_{\Trimming}}{k_\ell^*}\\
  =\ \ & \tfrac{1}{k_\ell^*}(D_{\Trimming} - \ROI) + 1\text{.}
\end{align*}

\ref{thm:evalrelaxslide-trim-d}
The case in which $D - \ROI + 1$ is divisible by $k_\ell^*$ is considered first.
Using~\ref{thm:evalrelaxslide-preq-a} one obtains $\dim_{M_L}( \Downsampling_{k_\ell^*}( \Slide_{\EvalStride_L}(\xi) ) ) = \frac{1}{k_\ell^*}( D - \ROI + 1 )$.
Moreover, it here holds that $\rem{D - \ROI + 1}{k_\ell^*} = 0$, which implies that $\rem{D - \ROI + k_\ell^*}{k_\ell^*} = \rem{D - \ROI + 1 + (k_\ell^* - 1)}{k_\ell^*} = k_\ell^* - 1$.
Because $k_\ell^*$ divides $k_L^*$, Lemma~\ref{lem:remdiv} guarantees there is an integer $z\in\Z$ with
\begin{displaymath}
  r_{\EvalRelaxSlide_L^\ell}(\xi) = \rem{D - \ROI + k_\ell^*}{k_L^*} = \rem{D - \ROI + k_\ell^*}{k_\ell^*} + k_\ell^*\cdot z = k_\ell^* - 1 + k_\ell^*\cdot z\text{.}
\end{displaymath}
By requirement $r_{\EvalRelaxSlide_L^\ell}(\xi)\inint{0}{k_\ell^* - 1}$, hence it can only hold that $r_{\EvalRelaxSlide_L^\ell}(\xi) = k_\ell^* - 1$.
Substitution now yields
\begin{displaymath}
  \dim_{M_L}( T_{\Trimming}(\xi) )
  \ \equsing{\ref{thm:evalrelaxslide-trim-c}}\ \tfrac{1}{k_\ell^*}(D_{\Trimming} - \ROI) + 1
  \ \equsing{\ref{thm:evalrelaxslide-trim-a}}\ \tfrac{1}{k_\ell^*}(D - r_{\EvalRelaxSlide_L^\ell}(\xi) - \ROI) + 1
  \ =\ \tfrac{1}{k_\ell^*}( D - \ROI + 1 )\text{,}
\end{displaymath}
therefore the output dimensionalities match as claimed in this case.

Now to the case in which $k_\ell^*$ does not divide $D - \ROI + 1$.
As the requirements of Lemma~\ref{lem:ceilrem} are fulfilled, it follows that
\begin{align*}
  & \dim_{M_L}( \Downsampling_{k_\ell^*}( \Slide_{\EvalStride_L}(\xi) ) )\\
  \equsing{\ref{thm:evalrelaxslide-preq-a}}\ \ & \nceil{\tfrac{1}{k_\ell^*}( D - \ROI + 1 )}\\
  \equsing{L.~\ref{lem:ceilrem}} \ \ & \tfrac{1}{k_\ell^*}(D - \ROI + 1 + k_\ell^* - \rem{D - \ROI + 1}{k_\ell^*})\text{.}
\end{align*}
Because $\rem{D - \ROI + 1}{k_\ell^*} \neq 0$ by requirement, $\rem{D - \ROI + 1}{k_\ell^*} - 1 = \rem{D - \ROI}{k_\ell^*}$ holds.
Further, $\rem{D - \ROI}{k_\ell^*} \ \equsing{P.~\ref{prop:number-theory}} \ \rem{D - \ROI + k_\ell^*}{k_\ell^*}$.
As $k_\ell^*$ divides $k_L^*$, there is an integer $z\in\Z$ due to Lemma~\ref{lem:remdiv} satisfying
\begin{displaymath}
  r_{\EvalRelaxSlide_L^\ell}(\xi) = \rem{D - \ROI + k_\ell^*}{k_L^*} = \rem{D - \ROI + k_\ell^*}{k_\ell^*} + k_\ell^*\cdot z\text{.}
\end{displaymath}
Here, $z$ must vanish as both $r_{\EvalRelaxSlide_L^\ell}(\xi)$ and $\rem{D - \ROI + k_\ell^*}{k_\ell^*}$ lie within $\discint{0}{k_\ell^* - 1}$.
This yields $r_{\EvalRelaxSlide_L^\ell}(\xi) = \rem{D - \ROI + 1}{k_\ell^*} - 1$.
Therefore,
\begin{align*}
  & \dim_{M_L}( \Downsampling_{k_\ell^*}( \Slide_{\EvalStride_L}(\xi) ) )\\
  =\ \ & \tfrac{1}{k_\ell^*}(D - \ROI + k_\ell^* - r_{\EvalRelaxSlide_L^\ell}(\xi))\\
  \equsing{\ref{thm:evalrelaxslide-trim-a}}\ \ & \tfrac{1}{k_\ell^*}(D_{\Trimming} - \ROI) + 1\\
  \equsing{\ref{thm:evalrelaxslide-trim-c}}\ \ & \dim_{M_L}( T_{\Trimming}(\xi) )\text{.}
\end{align*}
The dimensionalities match in this case as well.

\ref{thm:evalrelaxslide-trim-e}
The output signals are of the same length due to~\ref{thm:evalrelaxslide-trim-d} and are here compared sample-wise.
Let $i\inint{1}{\dim_{M_L}( T_{\Trimming}(\xi) )}$ be an arbitrary sample index.
Then
\begin{align*}
  & T_{\Trimming}(\xi)_i\\
  =\ \ \ & \Defragmentation_{k_L^* / k_\ell^*}( \EvalRelaxSlide_L^\ell( \Trimming_{r_{\EvalRelaxSlide_L^\ell}(\xi)}(\xi) ) )_{i,\;1}\\
  \equsing{L.~\ref{lem:defrag-ops}}\ \ \ & \EvalRelaxSlide_L^\ell( \Trimming_{r_{\EvalRelaxSlide_L^\ell}(\xi)}(\xi) )_{\div{i - 1}{k_L^* / k_\ell^*} + 1,\;\rem{i - 1}{k_L^* / k_\ell^*} + 1}\\
  \equsing{L.~\ref{lem:evalrelaxslide}\ref{lem:evalrelaxslide-d}}\ \ \ & \EvalStride_L(\Subsignal_\ROI(\Trimming_{r_{\EvalRelaxSlide_L^\ell}(\xi)}(\xi),\; k_\ell^*(i - 1) + 1))_1\\
  \equsing{$u_L = 1$}\ \ \ & \EvalStride_L(\Subsignal_\ROI(\Trimming_{r_{\EvalRelaxSlide_L^\ell}(\xi)}(\xi),\; k_\ell^*(i - 1) + 1))\\
  \equsing{D.~\ref{def:subsignal}}\ \ \ & \EvalStride_L\!\left(\sum\nolimits_{\nu = 1}^\ROI \Trimming_{r_{\EvalRelaxSlide_L^\ell}(\xi)}(\xi)_{k_\ell^*(i - 1) + 1 + \nu - 1}\cdot e_\nu^\ROI\right)\\
  \equsing{($\lozenge$)}\ \ \ & \EvalStride_L\!\left(\sum\nolimits_{\nu = 1}^\ROI \xi_{k_\ell^*(i - 1) + 1 + \nu - 1}\cdot e_\nu^\ROI\right)\\
  \equsing{D.~\ref{def:subsignal}}\ \ \ & \EvalStride_L(\Subsignal_\ROI(\xi,\; k_\ell^*(i - 1) + 1))\\
  \equsing{\ref{thm:evalrelaxslide-preq-b}}\ \ \ & \Downsampling_{k_\ell^*}( \Slide_{\EvalStride_L}(\xi) )_i\text{,}
\end{align*}
where in the ($\lozenge$) step it has to be shown that $k_\ell^*(i - 1) + 1 + \nu - 1\inint{1}{\dim_{M_0}( \Trimming_{r_{\EvalRelaxSlide_L^\ell}(\xi)}(\xi) )}$ so that the trimming operator can be omitted using Definition~\ref{def:stuff-trim}.
This sample index is clearly an integer and no less than unity.
Moreover,
\begin{align*}
  & k_\ell^*(i - 1) + 1 + \nu - 1\\
  \leq\ & k_\ell^*(\dim_{M_L}( T_{\Trimming}(\xi) ) - 1) + \ROI\\
  \equsing{\ref{thm:evalrelaxslide-trim-c}}\ & k_\ell^*(\tfrac{1}{k_\ell^*}(D_{\Trimming} - \ROI) + 1 - 1) + \ROI\\
  =\ & D_{\Trimming}\\
  \equsing{\ref{thm:evalrelaxslide-trim-a}}\ & \dim_{M_0}( \Trimming_{r_{\EvalRelaxSlide_L^\ell}(\xi)}(\xi) )\text{.}
\end{align*}
In conclusion, correctness for the $T_{\Trimming}$ case has been shown.

{\parindent0mm \emph{Proof for $T = T_{\Stuffing}$.}}
Here, the case in which $r_{\EvalRelaxSlide_L^\ell}(\xi)\inint{k_\ell^*}{k_L^* - 1}$ is considered.

\ref{thm:evalrelaxslide-stuff-a}
Because of $r_{\EvalRelaxSlide_L^\ell}(\xi) \leq k_L^* - 1$ it follows that $k_L^* - r_{\EvalRelaxSlide_L^\ell}(\xi) \geq 1$, therefore $\Stuffing_{k_L^* - r_{\EvalRelaxSlide_L^\ell}(\xi)}(\xi)$ is well-defined with Definition~\ref{def:stuff-trim}.
From Definition~\ref{def:stuff-trim} follows directly that $D_{\Stuffing} := \dim_{M_0}( \Stuffing_{k_L^* - r_{\EvalRelaxSlide_L^\ell}(\xi)}(\xi) ) = D + k_L^* - r_{\EvalRelaxSlide_L^\ell}(\xi)$.
As by definition it holds that $D - \ROI + k_\ell^* = \div{D - \ROI + k_\ell^*}{k_L^*}\cdot k_L^* + r_{\EvalRelaxSlide_L^\ell}(\xi)$, one obtains $D - \ROI + k_\ell^* - r_{\EvalRelaxSlide_L^\ell}(\xi) = \div{D - \ROI + k_\ell^*}{k_L^*}\cdot k_L^* \geq 0$, hence $D_{\Stuffing} - (\ROI + k_L^* - k_\ell^*) = D + k_L^* - r_{\EvalRelaxSlide_L^\ell}(\xi) - \ROI - k_L^* + k_\ell^* \geq 0$, therefore the required minimum input signal length is fulfilled after stuffing.

\ref{thm:evalrelaxslide-stuff-b}
Using the equations derived in~\ref{thm:evalrelaxslide-stuff-a} one obtains
\begin{align*}
  & D_{\Stuffing} - \ROI + k_\ell^*\\
  \equsing{\ref{thm:evalrelaxslide-stuff-a}}\ \ & D + k_L^* - r_{\EvalRelaxSlide_L^\ell}(\xi) - \ROI + k_\ell^*\\
  \equsing{\ref{thm:evalrelaxslide-stuff-a}}\ \ & (\div{D - \ROI + k_\ell^*}{k_L^*} + 1)\cdot k_L^*\text{,}
\end{align*}
and therefore with $t_{\Stuffing} := \div{D - \ROI + k_\ell^*}{k_L^*} + 1\in\N_1$ follows that $D_{\Stuffing} = \ROI + k_L^* t_{\Stuffing} - k_\ell^*$.
In conclusion, $\EvalRelaxSlide_L^\ell( \Stuffing_{k_L^* - r_{\EvalRelaxSlide_L^\ell}(\xi)}(\xi) )$ is well-defined using Lemma~\ref{lem:evalrelaxslide}.

\ref{thm:evalrelaxslide-stuff-c}
From Lemma~\ref{lem:evalrelaxslide}\ref{lem:evalrelaxslide-b} follows that $\cdim_{M_L}( \EvalRelaxSlide_L^\ell( \Stuffing_{k_L^* - r_{\EvalRelaxSlide_L^\ell}(\xi)}(\xi) ) ) = \frac{k_L^*}{k_\ell^*}$ and $\rdim_{M_L}( \EvalRelaxSlide_L^\ell( \Stuffing_{k_L^* - r_{\EvalRelaxSlide_L^\ell}(\xi)}(\xi) ) ) = \frac{k_L^*}{k_L^*}t_{\Stuffing} = t_{\Stuffing}$ as $u_L = 1$ by requirement.
Hence, a complete defragmentation using $\Defragmentation_{k_L^* / k_\ell^*}$ can be carried out.
This results in a single fragment with Lemma~\ref{lem:defrag-ops}, that is $\cdim_{M_L}( \Defragmentation_{k_L^* / k_\ell^*}( \EvalRelaxSlide_L^\ell( \Stuffing_{k_L^* - r_{\EvalRelaxSlide_L^\ell}(\xi)}(\xi) ) ) ) = 1$.
The number of samples equals
\begin{align*}
  & \dim_{M_L}( \Defragmentation_{k_L^* / k_\ell^*}( \EvalRelaxSlide_L^\ell( \Stuffing_{k_L^* - r_{\EvalRelaxSlide_L^\ell}(\xi)}(\xi) ) ) )\\
  =\ \ & \rdim_{M_L}( \Defragmentation_{k_L^* / k_\ell^*}( \EvalRelaxSlide_L^\ell( \Stuffing_{k_L^* - r_{\EvalRelaxSlide_L^\ell}(\xi)}(\xi) ) ) )\\
  \equsing{L.~\ref{lem:defrag-ops}}\ \ & \tfrac{k_L^*}{k_\ell^*}\cdot\rdim_{M_L}( \EvalRelaxSlide_L^\ell( \Stuffing_{k_L^* - r_{\EvalRelaxSlide_L^\ell}(\xi)}(\xi) ) )\\
  =\ \ & \tfrac{k_L^*t_{\Stuffing}}{k_\ell^*}\\
  \equsing{\ref{thm:evalrelaxslide-stuff-b}}\ \ & \tfrac{1}{k_\ell^*}(D_{\Stuffing} - \ROI) + 1\text{.}
\end{align*}

\ref{thm:evalrelaxslide-stuff-d}
First, the situation in which $k_\ell^*$ divides $D - \ROI + 1$ is considered.
Here, $\tilde{s}_{\EvalRelaxSlide_L^\ell}(\xi) = k_\ell^*$.
For showing that $s_{\EvalRelaxSlide_L^\ell}(\xi)$ is well-defined it has to be verified that its output is a positive natural number.
Because of $r_{\EvalRelaxSlide_L^\ell}(\xi) \leq k_L^* - 1$ and $\tilde{s}_{\EvalRelaxSlide_L^\ell}(\xi) \geq 1$ it holds that $k_L^* - r_{\EvalRelaxSlide_L^\ell}(\xi) + \tilde{s}_{\EvalRelaxSlide_L^\ell}(\xi) - 1 \geq 1$, thus $s_{\EvalRelaxSlide_L^\ell}(\xi)$ is positive.
With Lemma~\ref{lem:remdiv} there exists an integer $z\in\Z$ with
\begin{displaymath}
  r_{\EvalRelaxSlide_L^\ell}(\xi) = \rem{D - \ROI + k_\ell^*}{k_L^*} = \rem{D - \ROI + k_\ell^*}{k_\ell^*} + k_\ell^*\cdot z
\end{displaymath}
as $k_\ell^*$ divides $k_L^*$.
From $\rem{D - \ROI + 1}{k_\ell^*} = 0$ follows $\rem{D - \ROI + k_\ell^*}{k_\ell^*} = k_\ell^* - 1$, so that $r_{\EvalRelaxSlide_L^\ell}(\xi) = k_\ell^* - 1 + k_\ell^*\cdot z$.
Hence
\begin{align*}
  & k_L^* - r_{\EvalRelaxSlide_L^\ell}(\xi) + \tilde{s}_{\EvalRelaxSlide_L^\ell}(\xi) - 1\\
  =\ & k_L^* - k_\ell^* + 1 - k_\ell^*\cdot z + k_\ell^* - 1\\
  =\ & k_L^* - k_\ell^*\cdot z\\
  =\ & (k_{\ell + 1}\cdots k_L - z)\cdot k_\ell^*
\end{align*}
is divisible by $k_\ell^*$, which finally implies that $s_{\EvalRelaxSlide_L^\ell}(\xi)$ is a positive natural number.
A comparison of the output dimensionality of defragmentation with the number of samples that should be trimmed away leads to the difference
\begin{align*}
  & \dim_{M_L}( \Defragmentation_{k_L^* / k_\ell^*}( \EvalRelaxSlide_L^\ell( \Stuffing_{k_L^* - r_{\EvalRelaxSlide_L^\ell}(\xi)}(\xi) ) ) ) - s_{\EvalRelaxSlide_L^\ell}(\xi)\\
  \equsing{\ref{thm:evalrelaxslide-stuff-c}}\ & \tfrac{1}{k_\ell^*}(D_{\Stuffing} - \ROI + k_\ell^*) - \tfrac{1}{k_\ell^*} (k_L^* - r_{\EvalRelaxSlide_L^\ell}(\xi) + \tilde{s}_{\EvalRelaxSlide_L^\ell}(\xi) - 1)\\
  \equsing{\ref{thm:evalrelaxslide-stuff-a}}\ & \tfrac{1}{k_\ell^*}\big( (D + k_L^* - r_{\EvalRelaxSlide_L^\ell}(\xi) - \ROI + k_\ell^*) - (k_L^* - r_{\EvalRelaxSlide_L^\ell}(\xi) + \tilde{s}_{\EvalRelaxSlide_L^\ell}(\xi) - 1) \big)\\
  =\ & \tfrac{1}{k_\ell^*}(D - \ROI + 1)\text{.}
\end{align*}
By requirement this is a positive natural number.
This proves that $\Trimming_{s_{\EvalRelaxSlide_L^\ell}(\xi)}$ can be applied to $\Defragmentation_{k_L^* / k_\ell^*}( \EvalRelaxSlide_L^\ell( \Stuffing_{k_L^* - r_{\EvalRelaxSlide_L^\ell}(\xi)}(\xi) ) )$ due to Definition~\ref{def:stuff-trim}, hence $T_{\Stuffing}(\xi)$ is well-defined with an output signal length of $\tfrac{1}{k_\ell^*}(D - \ROI + 1)$, as has just been shown.
Finally, \ref{thm:evalrelaxslide-preq-a}~implies that $\dim_{M_L}( T_{\Stuffing}(\xi) ) = \dim_{M_L}( \Downsampling_{k_\ell^*}( \Slide_{\EvalStride_L}(\xi) ) )$.

Now suppose $k_\ell^*$ does not divide $D - \ROI + 1$.
Here $\tilde{s}_{\EvalRelaxSlide_L^\ell}(\xi) = \rem{D - \ROI + 1}{k_\ell^*}$ by definition.
This is a positive number by requirement.
Now $k_L^* - r_{\EvalRelaxSlide_L^\ell}(\xi) + \tilde{s}_{\EvalRelaxSlide_L^\ell}(\xi) - 1 \geq 1$ as $r_{\EvalRelaxSlide_L^\ell}(\xi) \leq k_L^* - 1$ and $\tilde{s}_{\EvalRelaxSlide_L^\ell}(\xi) \geq 1$ as just stated.
This shows that $s_{\EvalRelaxSlide_L^\ell}(\xi)$ is a positive number.
For showing it is also a natural number, note that
\begin{align*}
  & k_L^* - r_{\EvalRelaxSlide_L^\ell}(\xi) + \tilde{s}_{\EvalRelaxSlide_L^\ell}(\xi) - 1\\
  =\ & k_L^* - \rem{D - \ROI + k_\ell^*}{k_L^*} + \rem{D - \ROI + 1}{k_\ell^*} - 1\\
  =\ & k_L^* - (D - \ROI + k_\ell^* - \div{D - \ROI + k_\ell^*}{k_L^*}\cdot k_L^*) + (D - \ROI + 1 - \div{D - \ROI + 1}{k_\ell^*}\cdot k_\ell^*) - 1\\
  =\ & \big(k_{\ell + 1}\cdots k_L - 1 + \div{D - \ROI + k_\ell^*}{k_L^*}\cdot k_{\ell + 1}\cdots k_L - \div{D - \ROI + 1}{k_\ell^*} \big)\cdot k_\ell^*\text{.}
\end{align*}
As $k_\ell^*$ divides this number, $s_{\EvalRelaxSlide_L^\ell}(\xi)$ is guaranteed eventually to be a positive natural number.
Moreover,
\begin{align*}
  & \dim_{M_L}( \Defragmentation_{k_L^* / k_\ell^*}( \EvalRelaxSlide_L^\ell( \Stuffing_{k_L^* - r_{\EvalRelaxSlide_L^\ell}(\xi)}(\xi) ) ) ) - s_{\EvalRelaxSlide_L^\ell}(\xi)\\
  \equsing{\ref{thm:evalrelaxslide-stuff-c}}\ & \tfrac{1}{k_\ell^*}(D_{\Stuffing} - \ROI + k_\ell^*) - \tfrac{1}{k_\ell^*} (k_L^* - r_{\EvalRelaxSlide_L^\ell}(\xi) + \tilde{s}_{\EvalRelaxSlide_L^\ell}(\xi) - 1)\\
  \equsing{\ref{thm:evalrelaxslide-stuff-a}}\ & \tfrac{1}{k_\ell^*}\big( (D + k_L^* - r_{\EvalRelaxSlide_L^\ell}(\xi) - \ROI + k_\ell^*) - (k_L^* - r_{\EvalRelaxSlide_L^\ell}(\xi) + \rem{D - \ROI + 1}{k_\ell^*} - 1) \big)\\
  =\ & \tfrac{1}{k_\ell^*}(D - \ROI + 1 + k_\ell^* - \rem{D - \ROI + 1}{k_\ell^*})\text{.}
\end{align*}
From $D \geq \ROI$ and $\rem{D - \ROI + 1}{k_\ell^*} \leq k_\ell^* - 1$ follows that this difference is positive.
Therefore, $\Trimming_{s_{\EvalRelaxSlide_L^\ell}(\xi)}$ is applicable to $\Defragmentation_{k_L^* / k_\ell^*}( \EvalRelaxSlide_L^\ell( \Stuffing_{k_L^* - r_{\EvalRelaxSlide_L^\ell}(\xi)}(\xi) ) )$.
In conclusion, $T_{\Stuffing}(\xi)$ is well-defined.
Its dimensionality equals the difference just analyzed due to Definition~\ref{def:stuff-trim}.
By requirement, this can be simplified using Lemma~\ref{lem:ceilrem}:
\begin{align*}
  & \dim_{M_L}( T_{\Stuffing}(\xi) )\\
  =\ \ & \tfrac{1}{k_\ell^*}(D - \ROI + 1 + k_\ell^* - \rem{D - \ROI + 1}{k_\ell^*})\\
  \equsing{L.~\ref{lem:ceilrem}}\ \ & \nceil{\tfrac{1}{k_\ell^*}( D - \ROI + 1 )}\\
  \equsing{\ref{thm:evalrelaxslide-preq-a}}\ \ & \dim_{M_L}( \Downsampling_{k_\ell^*}( \Slide_{\EvalStride_L}(\xi) ) )\text{.}
\end{align*}
Summing up, the claim holds regardless whether $k_\ell^*$ divides $D - \ROI + 1$ or not.

\ref{thm:evalrelaxslide-stuff-e}
Since it has been shown in~\ref{thm:evalrelaxslide-stuff-d} that the dimensionality of both signals matches, it remains to be shown that the individual samples are equal.
For this, let $i\inint{1}{\dim_{M_L}( T_{\Stuffing}(\xi) )}$ be a sample index.
It is then
\begin{align*}
  & T_{\Stuffing}(\xi)_i\\
  =\ \ \ & \Trimming_{s_{\EvalRelaxSlide_L^\ell}(\xi)}( \Defragmentation_{k_L^* / k_\ell^*}( \EvalRelaxSlide_L^\ell( \Stuffing_{k_L^* - r_{\EvalRelaxSlide_L^\ell}(\xi)}(\xi) ) ) )_i\\
  \equsing{D.~\ref{def:stuff-trim}}\ \ \ & \Defragmentation_{k_L^* / k_\ell^*}( \EvalRelaxSlide_L^\ell( \Stuffing_{k_L^* - r_{\EvalRelaxSlide_L^\ell}(\xi)}(\xi) ) )_{i,\;1}\\
  \equsing{L.~\ref{lem:defrag-ops}}\ \ \ & \EvalRelaxSlide_L^\ell( \Stuffing_{k_L^* - r_{\EvalRelaxSlide_L^\ell}(\xi)}(\xi) )_{\div{i - 1}{k_L^* / k_\ell^*} + 1,\;\rem{i - 1}{k_L^* / k_\ell^*} + 1}\\
  \equsing{L.~\ref{lem:evalrelaxslide}\ref{lem:evalrelaxslide-d}}\ \ \ & \EvalStride_L(\Subsignal_\ROI(\Stuffing_{k_L^* - r_{\EvalRelaxSlide_L^\ell}(\xi)}(\xi),\; k_\ell^*(i - 1) + 1))_1\\
  \equsing{$u_L = 1$}\ \ \ & \EvalStride_L(\Subsignal_\ROI(\Stuffing_{k_L^* - r_{\EvalRelaxSlide_L^\ell}(\xi)}(\xi),\; k_\ell^*(i - 1) + 1))\\
  \equsing{D.~\ref{def:subsignal}}\ \ \ & \EvalStride_L\!\left(\sum\nolimits_{\nu = 1}^\ROI \Stuffing_{k_L^* - r_{\EvalRelaxSlide_L^\ell}(\xi)}(\xi)_{k_\ell^*(i - 1) + 1 + \nu - 1}\cdot e_\nu^\ROI\right)\\
  \equsing{($\lozenge$)}\ \ \ & \EvalStride_L\!\left(\sum\nolimits_{\nu = 1}^\ROI \xi_{k_\ell^*(i - 1) + 1 + \nu - 1}\cdot e_\nu^\ROI\right)\\
  \equsing{D.~\ref{def:subsignal}}\ \ \ & \EvalStride_L(\Subsignal_\ROI(\xi,\; k_\ell^*(i - 1) + 1))\\
  \equsing{\ref{thm:evalrelaxslide-preq-b}}\ \ \ & \Downsampling_{k_\ell^*}( \Slide_{\EvalStride_L}(\xi) )_i\text{,}
\end{align*}
where in the ($\lozenge$) step the stuffing operator has been omitted due to Definition~\ref{def:stuff-trim} because the sample index $k_\ell^*(i - 1) + 1 + \nu - 1 = k_\ell^*(i - 1) + \nu$ lies in the set $\discint{1}{D}$.
Clearly, this index is an integer.
Substitution of the lower bounds $i \geq 1$ and $\nu \geq 1$ shows that the sample index is always greater than or equal to one.
If $k_\ell^*$ divides $D - \ROI + 1$, then
\begin{align*}
  & k_\ell^*(i - 1) + \nu\\
  \leq\ & k_\ell^*(\dim_{M_L}( T_{\Stuffing}(\xi) ) - 1) + \ROI\\
  \equsing{\ref{thm:evalrelaxslide-stuff-d}}\ & k_\ell^*\left(\tfrac{1}{k_\ell^*}(D - \ROI + 1) - 1\right) + \ROI\\
  =\ & D - \ROI + 1 - k_\ell^* + \ROI\\
  \leq\ & D
\end{align*}
because $k_\ell^* \geq 1$.
If, on the other hand, $k_\ell^*$ does not divide $D - \ROI + 1$, one obtains
\begin{align*}
  & k_\ell^*(i - 1) + \nu\\
  \leq\ \ & k_\ell^*(\dim_{M_L}( T_{\Stuffing}(\xi) ) - 1) + \ROI\\
  \equsing{\ref{thm:evalrelaxslide-stuff-d}}\ \ & k_\ell^*\left(\nceil{\tfrac{1}{k_\ell^*}( D - \ROI + 1 )} - 1\right) + \ROI\\
  \equsing{L.~\ref{lem:ceilrem}}\ \ & k_\ell^*\left(\tfrac{1}{k_\ell^*}(D - \ROI + 1 + k_\ell^* - \rem{D - \ROI + 1}{k_\ell^*}) - 1\right) + \ROI\\
  =\ \ & D - \ROI + 1 + k_\ell^* - \rem{D - \ROI + 1}{k_\ell^*} - k_\ell^* + \ROI\\
  \leq\ \ & D
\end{align*}
as $\rem{D - \ROI + 1}{k_\ell^*} \geq 1$.

This finally proved correctness in the $T_{\Stuffing}$ case.
\end{proof}

There is one notable special case in the statements of Theorem~\ref{thm:evalrelaxslide}:
If $\dim_{M_0}(\xi) \geq \ROI + k_L^* - k_\ell^*$ and $r_{\EvalRelaxSlide_L^\ell}(\xi) = 0$, where the latter is equivalent to $k_L^*$ dividing $\dim_{M_0}(\xi) - \ROI + k_\ell^*$, then $\Defragmentation_{k_L^* / k_\ell^*}( \EvalRelaxSlide_L^\ell( \xi ) ) = \Downsampling_{k_\ell^*}( \Slide_{\EvalStride_L} ( \xi ) )$.
Only in this situation, no non-trivial trimming or stuffing of the input signal is necessary for processing chain application in a mixed fashion.
Note that these conditions are exactly the requirements of Lemma~\ref{lem:evalrelaxslide}, underlining that Theorem~\ref{thm:evalrelaxslide} is a non-trivial generalization of that result.

In conclusion, a method for efficiently scanning signals with a controllable trade-off between spatial output signal resolution and required computational complexity has been proved to always produce the correct outcome.
Since this approach relies on both relaxed and fragmentation-based processing, all benefits such as efficient homogeneous data structures and the absence of necessary modifications to the computationally demanding functions can be conserved.

\bibliographystylerlxsld{IEEEtran}
\bibliographyrlxsld{IEEEabrv,the}

\end{document}